\documentclass{article}
\usepackage{amsmath}
\usepackage{times}
\usepackage{graphicx}
\usepackage{multirow}
\usepackage[authoryear]{natbib}
\usepackage{rotating}
\usepackage{bbm}
\usepackage{latexsym}

\usepackage{mathtools}
\usepackage{bm}
\usepackage{amsfonts}
\usepackage{amssymb}
\usepackage{amsthm}
\usepackage[nobreak]{cite}
\usepackage{afterpage}
\usepackage[subrefformat=parens]{subcaption}
\usepackage[ruled,linesnumbered,lined]{algorithm2e}
\usepackage{comment}
\usepackage[hypertexnames=false]{hyperref}
\usepackage{cleveref}
\usepackage{autonum}
\usepackage{xcolor}
\usepackage{authblk}

\DeclareMathOperator*{\argmin}{arg\,min}
\DeclareMathOperator*{\argmax}{arg\,max}
\newcommand{\bs}[1]{\bm{#1}}

\newcommand{\bx}{\bm{x}}
\newcommand{\by}{\bm{y}}

\newcommand{\mr}[1]{\mathrm{#1}}
\newcommand{\mcal}[1]{\mathcal{#1}}

\newcommand{\N}{ \mathcal{N}}
\newcommand{\R}{\mathbb{R}}
\newcommand{\D}{\mathcal{D}}

\newcommand{\dX}{\mathcal{X}}
\newcommand{\dY}{\mathcal{Y}}

\newcommand{\GP}{\mathcal{GP}}

\newcommand{\K}{\bm{K}}
\newcommand{\I}{\bm{I}}

\newcommand{\E}{\mathbb{E}}
\newcommand{\veps}{\varepsilon}

\newcommand{\stg}[1]{^{(#1)}}

\newtheorem{theorem}{Theorem}[section]
\crefname{theorem}{Theorem}{Theorems}

\crefname{definition}{Definition}{Definitions}

\newtheorem{lemma}[theorem]{Lemma}
\crefname{lemma}{Lemma}{Lemmas}

\newtheorem{corollary}[theorem]{Corollary}
\crefname{corollary}{Corollary}{Corollaries}

\crefname{proposition}{Proposition}{Propositions}

\newtheorem{assumption}[theorem]{Assumption}
\crefname{assumption}{Assumption}{Assumptions}

\crefname{equation}{}{}
\Crefname{equation}{Equation}{Equations}

\crefname{figure}{Figure}{Figures}

\makeatletter
\autonum@generatePatchedReferenceCSL{Cref}
\makeatother

\renewcommand{\cite}{\citep}
\usepackage[margin=1in]{geometry}

\title{Bayesian Optimization for Cascade-type Multistage Processes}

\author[1]{Shunya Kusakawa}
\author[1]{Shion Takeno}
\author[1]{Yu Inatsu}
\author[2, 3]{Kentaro Kutsukake}
\author[1]{Shogo Iwazaki}
\author[3]{Takashi Nakano}
\author[3]{Toru Ujihara}
\author[1]{Masayuki Karasuyama}
\author[2, 3]{Ichiro Takeuchi}

\affil[1]{Nagoya Institute of Technology}
\affil[2]{RIKEN Center for Advanced Intelligent Project}
\affil[3]{Nagoya University}

\affil[ ]{\texttt{ \{kusakawa.s, takeno.s, iwazaki.s\}.mllab.nit@gmail.com, \{inatsu.yu, karasuyama\}@nitech.ac.jp, kentaro.kutsukake@riken.jp,}}
\affil[ ]{\texttt{nakano.t@unno.material.nagoya-u.ac.jp, ujihara@nagoya-u.jp, }}
\affil[ ]{\texttt{ichiro.takeuchi@mae.nagoya-u.ac.jp}}
\date{}

\begin{document}
\maketitle

\begin{abstract}
  Complex processes in science and engineering are often formulated as multistage decision-making problems.
  In this paper, we consider a type of multistage decision-making process called a cascade process.
  A cascade process is a multistage process in which the output of one stage is used as an input for the subsequent stage. When the cost of each stage is expensive, it is difficult to search for the optimal controllable parameters for each stage exhaustively.
  To address this problem, we formulate the optimization of the cascade process as an extension of the Bayesian optimization framework and propose two types of acquisition functions based on credible intervals and expected improvement.
  We investigate the theoretical properties of the proposed acquisition functions and demonstrate their effectiveness through numerical experiments.
  In addition, we consider an extension called suspension setting in which we are allowed to suspend the cascade process at the middle of the multistage decision-making process that often arises in practical problems.
  We apply the proposed method in a test problem involving a solar cell simulator, which was the motivation for this study.
\end{abstract}

%%%%%%%%%%%

%-------------------------------------------------------------------------------------------
\section{Introduction}
\label{sec:introduction}

\begin{figure}[t!]
    % \centering
    \begin{subfigure}{\linewidth}
        \centering
        \includegraphics[width=0.6\linewidth]{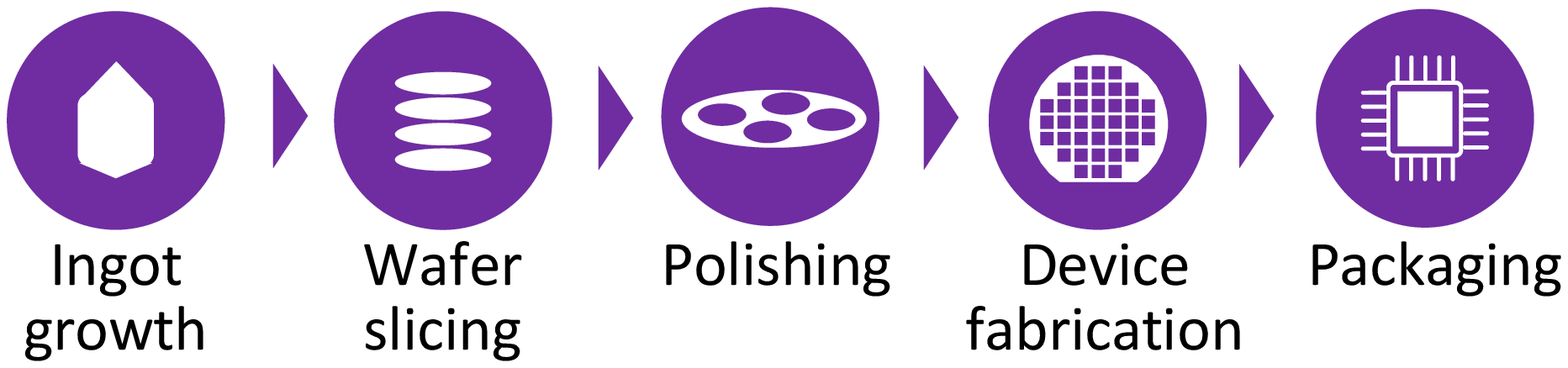}
        \subcaption{Production process of semiconductor chips.}
    \end{subfigure}
    \begin{subfigure}{\linewidth}
        \centering
        \includegraphics[width=\linewidth]{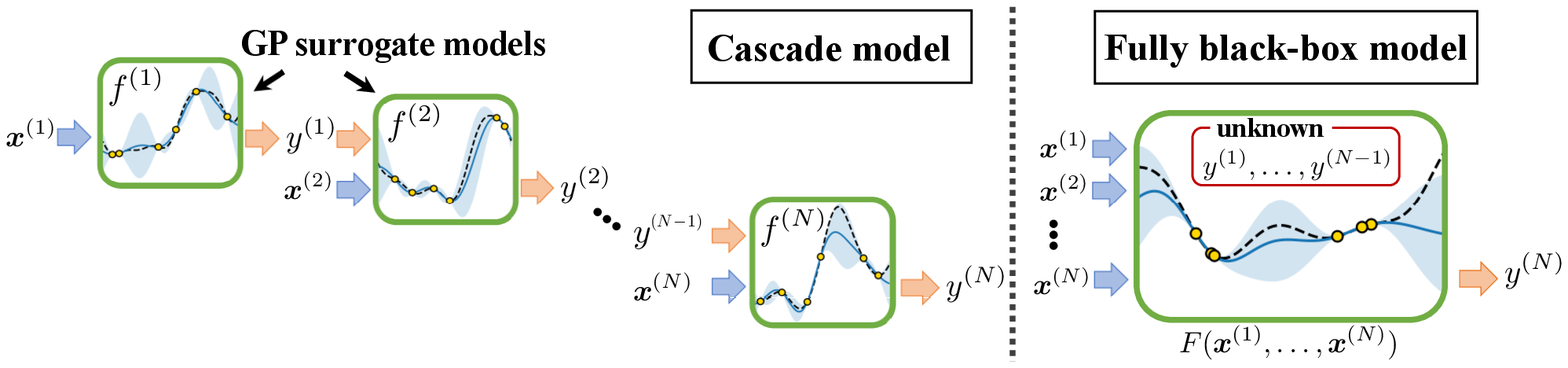}
        \subcaption{Schematic illustration of a cascade process.}
    \end{subfigure}
    \caption{
    (a) Example of the cascade manufacturing process for semiconductor chips.
    (b)
    The left part shows a cascade process with $N$ stages, where the function ${f}^{(n)}$ is the black-box function representing the $n^{\mr th}$ stage for $n \in [N]$.
    The function ${f}^{(n)}$ considers two types of inputs: the controllable parameters of that stage $\bm x^{(n)}$ and the output of the previous stage.
    The goal of the cascade process optimization is to identify the controllable parameters of all the stages $\{\bm x^{(n)}\}_{n \in [N]}$ that optimize the output of the final stage.
    The right part shows the fully black-box model view of the problem, where the function $F$ collectively considers all the controllable parameters $\{\bm x^{(n)}\}_{n \in [N]}$ as the inputs.
    By properly modeling each stage and incorporating the observable outputs in the middle of the cascade process $ y^{(1)}, \ldots, y^{(N-1)}$, more efficient optimization than that of the fully black-box model $F$ is possible.
    }
    \label{fig:cascade_and_target}
\end{figure}

A complex process in science and engineering problems is often formulated as a multistage \emph{cascade} process.
%
% The manufacturing process for today's industrial products is a complex multistage process.
%
For example, the production process of semiconductor chips consists of hundreds of process steps such as ingot growth, wafer slicing, and polishing, device fabrication, and packaging as shown in Figure~\ref{fig:cascade_and_target}~(a).
Similarly, most manufacturing processes, including garment manufacturing, automobile manufacturing, and building construction are multi-stage processes.
These multistage processes are often formulated as a cascade process in which the output of one stage is used as a part of the input for the subsequent stage.

\Cref{fig:cascade_and_target}~(b) shows a schematic illustration of a cascade process.
Each stage of a cascade process is formulated as a function with two types of inputs: the controllable parameters of that stage and the output of the previous stage. The former is controllable, whereas the latter is uncontrollable because of the uncertainty in the previous stage.
The optimization of the entire cascade process can be formulated as a joint optimization problem by collectively considering the controllable parameters of all the stages as the inputs.
Nevertheless, more efficient optimization is possible by properly modeling each stage and incorporating the observable outputs in the middle of the cascade process.

In this study, we consider the problem of optimizing a cascade process composed of black-box functions with expensive evaluation costs within the framework of Gaussian process-based (GP-based) Bayesian optimization (BO).
Each stage is modeled as a GP, whose inputs consist of controllable parameters and the outputs from the previous stage.
To optimize the output of the final stage, we consider the identification of the controllable parameters for each stage by considering the uncertainties of the GP models.
The difficulty with this problem is that when setting the controllable parameters for each stage, decisions are made by considering the influence of the output of that stage on the subsequent stages.

Considering the main contribution of this study, we propose a method that deals with the intractable predictive distribution and develop two acquisition functions (AFs) based on the expected improvement (EI) and credible interval (CI).
The proposed AFs can quantify the uncertainties of the subsequent stages in the cascade process using techniques developed in a multistep look-ahead strategy~\cite{ginsbourger2010towards,lam2016bayesian}.
The validity of the AFs was clarified through theoretical analysis, and their effectiveness was demonstrated using numerical experiments.
Furthermore, as generalizations, we consider extensions of a cascade process optimization problem, such as the case where suspensions and resumes are possible in the middle of the cascade process and where the cost of each stage is different.
Finally, we apply the proposed method to a test problem involving a solar cell simulator, which is the motivation for this study.

\paragraph{Related Studies}
%
% \erase{The}
GP-based BO has been intensively studied as an efficient way to optimize black-box functions with high evaluation costs~\cite{shahriari2015taking,frazier2018tutorial}.
Various types of AFs were proposed for BO, such as Gaussian process upper confidence bound (GP-UCB)~\cite{srinivas2010gaussian} and expected improvement~\cite{movckus1975bayesian,jones1998efficient}.
The GP-based BO framework was extended to various problem settings, such as constrained optimization~\cite{gardner2014bayesian,takeno2022sequential}, multiobjective optimization~\cite{couckuyt2014fast, suzuki2020multi}, and multifidelity optimization~\cite{swersky2013multi, takeno2020multi,takeno2022generalized}.

However, the only existing studies on cascade process optimization using a GP-based BO framework can be found in \cite{dai2016cascade} and \cite{astudillo2021bayesian}.
In CBO~\cite{dai2016cascade}, the controllable parameters for each stage are determined in a reverse order (i.e., starting from the controllable parameters for the last stage, the second last stage, etc).
That is, CBO selects the controllable parameters that are likely to produce the desired output, which is defined through the inverse function of the predictive mean function of the GP model in the subsequent stage.
Importantly, since this desired output does not depend on the outputs from the previous stages, incorporating the observed outputs of the previous stages is difficult in CBO.
Furthermore, if the earlier stages cannot achieve the desired output (which typically occurs when the range of each stage is unknown), the algorithm can become stuck.
In addition, the exploration-exploitation trade-off cannot be considered in their method because the uncertainty of each stage is ignored when the desired output is predetermined by the predictive mean functions.
Recently, a modified version of CBO was proposed in material science~\cite{nakanounpublished}.
However, their approach is to address the practical application issues of CBO with some heuristics and does not fundamentally solve the drawbacks of CBO.
EI-FN~\cite{astudillo2021bayesian} focuses on the optimization of a function network represented as a directed acyclic graph (DAG).
Whereas their problem settings include the cascade structure as one of the DAGs, decision-making at each middle stage is not incorporated.
In addition, noisy observations and suspension settings are not considered in their study.
Furthermore, their approach is based on EI with full sampling (even the final stage), whereas our EI-based approach uses partial sampling, and we also provide a CI-based AF.
Thus, our proposed method is clearly different from EI-FN.

One important related study is the study on multistep forward time-series prediction based on GP~\cite{quinonero2002prediction}.
In their study, the output of the GP at a time point becomes the input of the GP at the subsequent time point.
This can be interpreted as a cascade process without controllable parameters.
They introduced an iterative Gaussian approximation method to approximate the predictive distribution for the multistep forward time points.
However, their method cannot be directly extended to cases with controllable parameters at each stage.
Cascade process optimization is partially related to BO under input uncertainty because the output of the previous stage with uncertainty becomes the input of the subsequent stage.
Recently, BO under input uncertainty was intensively studied~\cite{beland2017bayesian,oliveira2019bayesian,iwazaki2021mean,inatsu2021active,inatsu2022bayesian}.
However, these existing methods cannot be easily extended to our problem because the uncertainties in multiple stages are accumulated in a complicated manner in a cascade process.
For example, an approach using the Bayesian quadrature framework \citep{o1991bayes,beland2017bayesian} cannot model the same cascade process correctly (see Appendix~\ref{app:BQ} for details).
In our proposed method, the expected improvement in the cascade process is computed based on a multistage look-ahead strategy.
Therefore, the BO methods for multistep look-ahead~\cite{ginsbourger2010towards,lam2016bayesian} are closely related to our method.
In general, the exact evaluation of a look-ahead AF is difficult owing to its computational complexity.
Our proposed method is based on several computational analyses developed in look-ahead type AFs, especially batch-type approximations~\cite{jiang2020binoculars}.
%
%

% \erase{
% Reinforcement Learning (RL) is also known as an optimization method for a black-box function}~\citep{bertsekas2019reinforcement}.
% \erase{
% %
% Typically, a black-box function represents an unknown reward function that depends on the current state, and RL aims to maximize a cumulative reward.
% %
% Moreover, RL generally does not take into account situations where the number of observations is limited.
% %
% In these respects, the cascade process optimization and RL differ in their problem settings.
% }

Reinforcement learning (RL) \citep{sutton2018reinforcement,bertsekas2019reinforcement} is also formulated as a multi-stage decision-making problem, which often involves several uncertainties similar to the output of each stage in the cascade process.
Thus, RL can be casted into the optimization of the cascade process by setting the state and action as the output from the previous stage and the input of the current stage, respectively.
On the other hand, it is difficult to directly apply the RL algorithm to our cascade optimization problem because the problem setup differs in many aspects.
For example, while the goal of RL is to maximize cumulative rewards, the goal of cascade process optimization is to find optimal input conditions for multiple stages.
Furthermore, cascade process optimization has the limitation that function evaluation is costly and cannot be performed many times, making it difficult to apply the RL algorithm under such a limitation.

%-------------------------------------------------------------------------------------------
\section{Preliminaries}
\label{sec:preliminaries}

\subsection{Cascade Process Optimization}
We consider a cascade process with $N$ stages.
Let $\bm x\stg{n} \in \dX\stg{n}\subset \R^{D^{(n)}}$ be a $D^{(n)}$-dimensional controllable input and $y\stg{n} \in \dY\stg{n} \subset \R$ be a scalar output of the stage $n \in [N] := \{1, \ldots, N\}$.
% \blue{(takeno:$n$の添字追加)}
%
Each stage is formulated as a function $f\stg{n}: \dY\stg{n-1} \times \dX\stg{n} \to \dY\stg{n}$ and is written as
\begin{equation}
  \label{eq:stage_func}
  y\stg{n} = f\stg{n}(y\stg{n-1}, \bm x\stg{n}), n \in [N],
\end{equation}
where we define $y\stg{0}=0$ and $\dY\stg{0}=\{0\}$ for notational simplicity.

Combining all the inputs $\bm x\stg{1}, \ldots, \bm x\stg{N}$, the entire cascade process can be represented as $y\stg{N} = F(\bm x\stg{1}, \ldots, \bm x\stg{N})$,
where $F: \dX\stg{1}\times \cdots \times \dX\stg{N} \to \dY\stg{N}$ is recursively defined using~\cref{eq:stage_func}.
The goal of a cascade process optimization is to solve the following optimization problem:
\begin{equation}
  \label{eq:opt-problem}
  \bx\stg{1}_*,\dots,\bx\stg{N}_* =  \argmax_{ ( \bx\stg{1},\dots,\bx\stg{N} ) \in \dX} F(\bx\stg{1},\dots,\bx\stg{N})
\end{equation}
% \blue{(takeno:括弧を追加)}
with a number of function evaluations as small as possible, where
$\dX := \dX\stg{1}\times \cdots \times \dX\stg{N}$.

For simplicity, we consider the case in which the output of each stage is scalar.
Furthermore, we assume that the output $y\stg{n}$ is observed without noise.
Extensions to the case of multidimensional output and noisy observation settings are described in the Appendix.

\subsection{GP Models}
\label{sec:modeling}
In this study, we employed GP models as surrogate models for black-box functions.
One simple way to model the cascade process is the \textit{fully black-box model} view, where we regard $F$ as a single black-box function that outputs $y\stg{N}$ for a collected input $(\bm x\stg{1}, \ldots, \bm x\stg{N})$.
However, regarding the fully black-box model view, the outputs observed in the intermediate stages of the cascade process cannot be effectively used.
Therefore, we employ a \textit{cascade model}
% \blue{in which each ${f}\stg{n}$ is modeled using a GP surrogate model.}
, in which all stages are modeled by independent GP surrogate models.
% \blue{(takeno:独立性を明記しろとの指示)}
%
We assume that the prior distribution for $f\stg{n}$ is $\GP(0,k\stg{n})$, where $\GP(\mu,k)$ denotes a GP with mean and kernel functions $\mu$ and $k$, respectively.
From the properties of a GP, given the observed data, the posterior distribution of $f\stg{n}, n \in [N]$ is also represented as a GP, and its mean and variance functions can be obtained in a closed form~\cite{rasmussen2005gaussian}.

%-------------------------------------------------------------------------------------------
\section{Proposed Method}
\label{sec:method}
In this section, we consider the sequential observations of a cascade process from stage $1$ to $N$.
For each iteration $t \in \{0,N,2N,\ldots \}$, users determine $\bx\stg{1}_{t+1} $, a controllable parameter of stage $1$,  and observe an output $y_{t+1}\stg{1}=f\stg{1} (0, \bx\stg{1}_{t+1} )$.
Subsequently, users choose $\bx\stg{2}_{t+2} $, a controllable parameter of stage $2$,  and observe $y_{t+2}\stg{2}=f\stg{2}( y\stg{1}_{t+1}  , \bx\stg{2}_{t+2} )$.
By repeating this operation, users obtain $y\stg{N}_{t+N}= f\stg{N} (y\stg{N-1}_{t+N-1}, \bx\stg{N}_{t+N} )$.

Regarding the cascade process optimization problem in~\cref{eq:opt-problem}, the following two points should be considered:
First, because the optimization target is the output of the final stage, a multistep look-ahead is indispensable when a decision is made in the earlier stages.
Second, the input at each stage can be determined after observing the output of the previous stage.
Therefore, when designing the AF for stage $n$, we need to consider $F(\bm x^{(n:N)} \mid y\stg{n-1})$, where the output of the final stage is represented as a function of the remaining controllable parameters $\bm x^{(n:N)} := (\bm x\stg{n}, \ldots, \bm x\stg{N})$ given the output of the previous stage $y\stg{n-1}$.
If the predictive distribution of $F(\bm x^{(n:N)} \mid y\stg{n-1})$ is available, appropriate AFs can be easily derived for stage $n$.
However, in the cascade model, the predictive distributions of $F(\bx\stg{n:N}| y\stg{n-1})$ cannot be explicitly written because of the nested structure of the cascade process.
To address this problem, we consider two approaches.
First, by utilizing the property that is easy to sample from nested predictive distributions, we propose an EI-based AF in~\cref{sec:cascadeEI}.
Second, by constructing the credible interval of $F(\bx\stg{n:N}| y\stg{n-1})$, we propose a CI-based AF in~\cref{sec:ci-based-af}.

\subsection{EI-based Acquisition Function}
\label{sec:cascadeEI}
In this subsection, we assume that the true black-box function $f^{(n)}$ is sampled from the GP prior $\mathcal{G} \mathcal{P} (0,k^{(n)} )$ for each $n \in [N]$.
Let $F_{\text{best}}=\max_{1 \le t^\prime \le t} y^{(N)}_{t^\prime}$ be the maximum value of the objective function $F$ observed up to iteration $t$.
Thereafter, we define the improvement $U_n({\bm x}^{(n)}|{ y}^{(n-1)})$ for the observation of stage $n$ with input $({ y}^{(n-1)},{\bm x}^{(n)} )$ as the expected improvement of $F_{\text{best}}$.
First, in the case of $n=N$, $F_{\text{best}}$ is improved when $f^{(N)}({y}\stg{N-1},\bx\stg{N})>F_{\text{best}}$.
Therefore, the expected improvement of $F_{\text{best}}$, $U_N(\bx\stg{N}|y\stg{N-1})$, is given by:
\begin{equation}
    U_N(\bx\stg{N}|y\stg{N-1})  = \E_{f\stg{N}}   \left[\left( F(\bx\stg{N}|y\stg{N-1}) - F_{\text{best}} \right)^+\right],
    \label{eq:EI_final_def}
\end{equation}
where $(\cdot )^+ \coloneqq \max (0,\cdot)$.
\Cref{eq:EI_final_def} is the same formulation as in the ordinary EI, and its expectation can be calculated analytically.

With regard to the case of $n\neq N$, we define $U_n(\bx\stg{n}|y\stg{n-1})$ as the maximum expected improvement of $F(\bx\stg{n:N}|y\stg{n-1})$:
\begin{equation}
        U_n(\bx\stg{n}|y\stg{n-1})  =  \E_{f\stg{n}} \left[ \max_{ \bx\stg{n+1}} U_{n+1}(\bx\stg{n+1}|y\stg{n}) \right].
    \label{eq:EI_ideal}
\end{equation}
\Cref{eq:EI_ideal} is a recursive expression that contains the max operator and expectation.
Thus, it is difficult to calculate it analytically.
In the context of multistep look-ahead approaches, methods to avoid this problem through approximation and sampling have been investigated~\cite{lam2016bayesian,gonzalez2016glasses,wu2019practical,jiang2020binoculars}.
We use the similar approach as in~\cite{jiang2020binoculars} to approximate the lower bound of~\cref{eq:EI_ideal}.
Using the Monte Carlo integration with $S$ samples and the exchange of expectation and max operators (note that~\cref{eq:EI_ideal_expand} contains the nested max operators and expectation), \cref{eq:EI_ideal} can be approximated as follows:
\begin{subequations}
    \begin{align}
        {U}_n(\bx\stg{n}|y\stg{n-1}) & = \E_{f\stg{n}}  \left[ \max_{ \bx\stg{n+1}}\cdots \E_{f\stg{N-1}} \left[ \max_{ \bx\stg{N}}U_N(\bx\stg{N}|y\stg{N-1}) \right] \right] \label{eq:EI_ideal_expand} \\
                                     & \ge \max_{ \bx\stg{n+1},\dots,\bx\stg{N} }\E_{f\stg{n},\ldots,f\stg{N-1}}\left[   U_N(\bx\stg{N}|y\stg{N-1})\right]  \label{eq:EI_lower}                          \\
                                     & \approx \max_{ \bx\stg{n+1},\dots,\bx\stg{N} } \frac{1}{S}  \sum_{s=1}^{S}    U_N(\bx\stg{N}|y_s\stg{N-1}),  \label{eq:EI_lower_mc}
    \end{align}
\end{subequations}
where the inequality~\cref{eq:EI_lower} can be derived by~\cite{jiang2020binoculars}, and the sampling of $y\stg{N-1}_s$ is based on the GP model.
First, we generate each $y\stg{n}_s$ from the predicted distribution of ${f}\stg{n}(y\stg{n-1},\bx\stg{n})$ independently.
Then, we calculate the predicted  distribution of ${f}\stg{n+1}(y\stg{n}_s,\bx\stg{n+1})$ using the generated $y\stg{n}_s$, and we generate $y\stg{n+1}_s$ based on that.
By repeating this process, $y\stg{N-1}_s$ can be generated.
We propose the approximated utility function $\widetilde{U}_{n}(\bx\stg{n}|y\stg{n-1})$, defined as~\cref{eq:EI_lower_mc}, as the EI-based AF.
Therefore, given the observation $y_{t+n-1}\stg{n-1} $ of the previous stage, the observation point of the subsequent stage $n$ is given by:
\begin{equation}
        {\bm x}^{(n)}_{t+n}  = \argmax_{ {\bm x}^{(n) } \in \mathcal{X}\stg{n} }  \widetilde{U}_{n}(\bx\stg{n}|y_{t+n-1}\stg{n-1}).
  \label{eq:EI_seq_rule}
\end{equation}
Although~\cref{eq:EI_lower_mc} is the optimization problem for a stochastically determined function, deterministic gradient-based methods
% \blue{gradient methods (takeno:reviewer suggestion)}
can be applied by applying the reparameterization trick~\cite{kingma2014auto}.
Compared to EI-FN, which approximates all expectations by the Monte Carlo estimation, we analytically calculate the expectation with respect to $f\stg{N}$ in \cref{eq:EI_lower_mc}.
Furthermore, we select the controllable input in each stage depending on the output from the previous stage by using \cref{eq:EI_seq_rule} in contrast to EI-FN which does not incorporate intermediate observations.

\subsection{CI-based Acquisition Function}
\label{sec:ci-based-af}
Thus far, we assume that each $f\stg{n}$ is sampled from the GP prior.
Hereafter, we assume that each $f\stg{n}$ is an element of a reproducing kernel Hilbert space (RKHS).
Under this RKHS setting, we propose a CI-based AF that can be interpreted as an optimistic improvement.
First, we provide a credible interval of $F(\bx\stg{n:N}|y\stg{n-1})$ and then, design the AF.
To construct a valid CI, we assume the following regularity assumptions.
\paragraph{Regularity Assumptions}
We assume that $\mathcal{Y}^{(n-1)} \times \mathcal{X}^{(n)}$ is a compact set, and
let $k^{(n)}$ be a positive definite kernel with $k^{(n)} (  ({w},{\bm x}),   ({w},{\bm x}) ) \leq 1$ for any
$n \in [N]$ and $ ({w },{\bm x}) \in \mathcal{Y}^{(n-1)} \times \mathcal{X}^{(n)}$.
Furthermore, let $\mathcal{H}_{ k^{(n)} } $ be an RKHS corresponding to the kernel $k^{(n)} $.
Additionally, for each $n \in [N]$, we assume that $f^{(n)} \in \mathcal{H}_{ k^{(n)} } $ and $\|f\stg{n}\|_{{k\stg{n}}}\le B$,
where $B>0$ is a constant, and $\|\cdot \|_{{k\stg{n}}}$ denotes the RKHS norm on $\mcal{H}_{k\stg{n}}$.
There are several studies on BO using a GP model for the black-box function assumed as an element of an RKHS~\cite{srinivas2010gaussian,oliveira2019bayesian,iwazaki2021mean}.

To construct the credible interval of $F(\bx\stg{n:N}|y\stg{n-1})$,  we first formally define a posterior mean and variance of a GP with independent Gaussian noise $\mathcal{N} (0, \sigma^2)$.
Note that what we have just introduced is the noise model $\mathcal{N} (0, \sigma^2)$ of a GP, and the actual observations are still noiseless.
For each $n \in [N]$, $(w,{\bm x} ) \in \mathcal{Y}^{(n-1)} \times \mathcal{X}^{(n)}$ and $t \geq 1$, let
$\mu ^{(n)}_t (w,{\bm x} ) $ and $\sigma^{(n)2}_t (w,{\bm x})$ be the posterior mean and variance of $f^{(n)} (w,{\bm x} )$, respectively.
The interval $[\mu ^{(n)}_t (w,{\bm x} ) \pm \beta^{1/2} \sigma^{(n)}_t (w,{\bm x}) ]$ with an appropriate trade-off parameter $\beta $ is the credible interval for $f^{(n)} (w,{\bm x} ) $~\cite{srinivas2010gaussian}.
We apply this interval to construct a valid credible interval for $F(\bx\stg{n:N}|y\stg{n-1})$.
However, it cannot be used directly because it has an uncontrollable variable $w$.
%
% \blue{To avoid this issue, we assume additional assumptions for Lipschitz continuity.}
% \blue{(takeno:assumeとassumptionが重複.)}
%
To avoid this issue, we additionally consider the following assumptions.
\paragraph{Lipschitz Continuity Assumptions}
We assume that $f^{(n)}$ and $\sigma^{(n)}_{t} $ satisfy the following assumptions:
\begin{description}
    \item[(L1)] Assume that $f\stg{n}$ is $L_f$-Lipschitz continuous with respect to $L_1$-distance for any $n \in \{2,\ldots, N\}$, where  $L_f>0$ is a Lipschitz constant.
    \item[(L2)] Assume that $\sigma^{(n)}_{t} $ is $L_\sigma$-Lipschitz continuous with respect to $L_1$-distance for any $n \in \{2,\ldots, N\}$ and $t \geq 1$, where  $L_\sigma>0$ is a Lipschitz constant.
\end{description}
This assumption enables us to give the CI of the output using the CI of the input.
Since the output becomes the input of the next stage in the cascade process, CIs of the subsequent stages can be constructed in a chain reaction.

Under these assumptions, we introduce a credible interval of $F(\bx\stg{n:N}|y\stg{n-1})$ using a cascade model.
\begin{theorem}\label{theorem:CI}
    Let
    \begin{equation}
        \begin{aligned}
              \tilde{{\mu}}\stg{m}_t(\bx\stg{n:m}|y\stg{n-1}) & ={\mu}\stg{m}_{t} \left(\tilde{{\mu}}\stg{m-1}_t(\bx\stg{n:m-1} |y\stg{n-1}),\; \bx\stg{m}\right),   \\
              \tilde{\sigma}\stg{m}_{t}(\bx\stg{n:m}| y\stg{n-1})  &   ={\sigma}\stg{m}_{t} \left(\tilde{{\mu}}\stg{m-1}_t(\bx\stg{n:m-1} |y\stg{n-1}),\; \bx\stg{m}\right)  +L_f \tilde{\sigma}^{(m-1)}_{t} (\bx\stg{n:m-1}|y\stg{n-1} )  ,
        \end{aligned}\label{eq:tilde_pred}
    \end{equation}
    where $\tilde{\mu}^{(n)}_t  ({\bm x}^{(n)}|y\stg{n-1} ) =  {\mu}^{(n)}_t  (y\stg{n-1},{\bm x}^{(n)} )$ and $\tilde{\sigma}^{(n)}_{t}  ({\bm x} ^{(n)} |y\stg{n-1})=  {\sigma}^{(n)}_{t}  (y\stg{n-1},{\bm x} ^{(n)} )$.
    Assume that regularity assumptions and the Lipschitz continuity assumption (L1) hold.
    Also assume that $\tilde{\mu}^{(m)}_t ( {\bm x}^{(n:m)} |y\stg{n-1}) \in \mathcal{Y}^{(n)} $ for all $m \in [N]$, $t \geq 1$ and $ {\bm x}^{(n:m)}$.
    Define $\beta = B^2$.
    Then, the following holds:
    \begin{equation}
          |F(\bx\stg{n:N}|y\stg{n-1} ) - \tilde{\mu}^{(N)}_{t} (\bx\stg{n:N}|y\stg{n-1})|  \leq   \beta^{1/2}  \tilde{\sigma}^{(N)}_{t} (\bx\stg{n:N} |y\stg{n-1} )  .
    \end{equation}
\end{theorem}
From~\cref{theorem:CI}, a lower confidence bound $\mr{LCB}^{(F)}_t (\bx\stg{n:N}|y\stg{n-1}  )$ and an upper confidence bound $\mr{UCB}^{(F)}_t (\bx\stg{n:N} |y\stg{n-1} )$ of $F(\bx\stg{n:N}|y\stg{n-1})$ are given by:
\begin{equation}
    \begin{aligned}
          \mr{LCB}^{(F)}_t  (\bx\stg{n:N} |y\stg{n-1} ) &  =\tilde{\mu}^{(N)}_{t-1} (\bx\stg{n:N}|y\stg{n-1} )   - \beta^{1/2}  \tilde{\sigma}^{(N)}_{t-1} (\bx\stg{n:N}|y\stg{n-1}), \\
          \mr{UCB}^{(F)}_t  (\bx\stg{n:N} |y\stg{n-1} ) &  =\tilde{\mu}^{(N)}_{t-1} (\bx\stg{n:N} |y\stg{n-1} )  + \beta^{1/2}  \tilde{\sigma}^{(N)}_{t-1} (\bx\stg{n:N} |y\stg{n-1}).
    \end{aligned} \label{eq:cadcade_ci}
\end{equation}

Based on the above credible intervals, we define the pessimistic maximum estimator of $F(\bx\stg{1:n})$ as $ Q_{t} \coloneqq \max_{\bx\stg{1:N} } \mr{LCB}\stg{F}_{t} (\bx\stg{1:N}) $.
In addition, given the observation $y_{t+n-1}\stg{n-1}$ in stage $n-1$, we define the pessimistic maximum estimator of $F(\bx\stg{n:N}|y\stg{n-1})$ as follows:
\begin{equation}
    \mr{ LCB}^{(F)}_{t+n} ( y_{t+n-1}\stg{n-1} ) = \max_{  \bx\stg{n:N} } \mr{ LCB}^{(F)}_{t+n} (\bx\stg{n:N}| y_{t+n-1}\stg{n-1} ), \label{eq:lower_y}
\end{equation}
where the $\max$ operator is not necessary when $n=N$.
Similarly, the optimistic maximum estimator of $F(\bx\stg{n:N}|y_{t+n-1}\stg{n-1})$  is defined as follows:
\begin{align}
    \mr{UCB}^{(F)}_{t+n} ({\bm x}^{(n) } |  y_{t+n-1}\stg{n-1} ) \coloneqq  \max_{  \bx\stg{n+1:N} }
    \mr{UCB}^{(F)}_{t+n} (\bx\stg{n:N}|y_{t+n-1}\stg{n-1} ).
    \label{eq:upper_y}
\end{align}
Then, we define the optimistic improvement with respect to $(y\stg{n-1}, {\bm x}^{(n)})$ as follows:
\begin{equation}
    a^{(n)}_{t+n} ({\bm x}^{(n)} | y_{t+n-1}\stg{n-1} )  =  \mr{ UCB}^{(F)}_{t+n} ({\bm x}^{(n) } | y_{t+n-1}\stg{n-1} )  - \max \left\{ \mr{ LCB}^{(F)}_{t+n} (  y_{t+n-1}\stg{n-1} ) ,\: Q_{t+n} \right\}. \label{eq:optimistic_improve}
\end{equation}
Furthermore, we define the maximum uncertainty
\begin{equation}
    b^{(n)}_{t+n} ({\bm x}^{(n)} | { y}_{t+n-1}^{(n-1) } )=  \max_{  { {\bm x}^{(n+1:N)}  } }    \tilde{\sigma}^{(N)}_{t+n-1} (\bx\stg{n:N}   |{ y}_{t+n-1}^{(n-1)} ). \label{eq:US_F}
\end{equation}

Using~\cref{eq:optimistic_improve,eq:US_F},
we propose a CI-based AF $c^{(n)}_{t+n} ({\bm x}^{(n)} | { y}_{t+n-1}^{(n-1) } )$:
\begin{equation}
    c^{(n)}_{t+n} ({\bm x}^{(n)} | {y}_{t+n-1}^{(n-1) } )  =  \max  \left\{ a^{(n)}_{t+n} ({\bm x}^{(n)} | {y}_{t+n-1}^{(n-1) } ),   \eta_t b^{(n)}_{t+n} ({\bm x}^{(n)} | { y}_{t+n-1}^{(n-1) } ) \right\},  \label{eq:af_CI}
\end{equation}
where $\eta_t $ is some learning rate and tends to zero.
Therefore, the subsequent observation point is given by $ \bx\stg{n}_{t+n} \coloneqq \argmax_{\bx\stg{n} \in \dX\stg{n}}  c\stg{n}_{t+n} (\bx\stg{n} | y_{t+n-1}\stg{n-1}  )$.

\Cref{eq:lower_y} denotes the pessimistic maximum when we observe in the subsequent stages with the previous output,
and $Q_t$ represents the pessimistic maximum when the observation is performed from the first stage.
Thus, the second term of~\cref{eq:optimistic_improve} indicates a pessimistic maximum estimator in the current iteration,
and $a\stg{n}_t$  optimistically evaluates how much the observed value exceeds the pessimistically estimated maximum value.
Intuitively, CI-based AF selects the point that has high optimistic improvement,
and if no optimistic improvement is expected (i.e., $a\stg{n}_t$  is small), it selects the point with the highest uncertainty of the cascade process.

The Lipschitz constant $L_{f}$ is a new parameter derived from our proposed method.
Since each $f\stg{n}$ is a black-box function, it is difficult to obtain the exact value of $L_{f}$.
In practice, we have to estimate $L_f$, and one simple way is to determine it from prior knowledge.
Another way is to estimate it from a GP surrogate model.
For any Lipschitz continuous function $f$ on a compact set $\dX \subset \R^d$, $\bar{L}=\max_{\bx \in \dX} \|\nabla f(\bx)\|_1$ satisfies the Lipschitz condition~\cite{gonzalez2016batch}.
Additionally, it is known that if a GP is differentiable, its derivative is also a GP.
Based on these facts, we can estimate $L_f$ by constructing a GP surrogate model of $\nabla f$ and using its sample paths and predictive mean~\cite{sui2015safe,gonzalez2016batch}.
On the other hand, for the Lipschitz continuity assumption (L2), it depends on how the kernel function is chosen.
If we use a kernel that does not consider any similarity between different points, i.e., a pathological kernel such as $k({\bm x},{\bm x}^\prime)=1$ if ${\bm x}={\bm x}^\prime$, and otherwise zero, the posterior standard deviation is discontinuous at the observed points, and (L2) does not hold.
On the other hand, (L2) is shown to hold for commonly used kernels such as linear kernels, Gaussian kernels, and Mat\'{e}rn kernels with more than one degree of freedom (see Appendix \ref{app:modify} for details).

We discuss the multidimensional output setting and the noisy observation setting in the Appendix.
Particularly in noisy situations, two different target functions can be considered.
One is to maximize $F$ through noisy observations, and the other is to maximize the expected final output with respect to the noise at each stage.
We also propose the modified version of CI-based AFs for both target functions and show the theoretical analyses of them in Appendix~\ref{app:CI-based-AF-noisy}.
\section{Theoretical Results}
\label{sec:guarantees}
In this section, we provide the theoretical guarantee for the CI-based AF.
First, we define the estimated solution $\hat{\bx}_t\stg{1},\dots,\hat{\bx}_t\stg{N}$ and regret $r_t$ at iteration $t$ as follows:
\begin{align}
     & \hat{\bm x}^{(1)}_t, \ldots, \hat{\bm x}^{(N)}_t  = \argmax _{ \bx\stg{1:N} \in \mathcal{X}              , 1\leq \tilde{t} \leq t  } \mr{ LCB} ^{(F)}_{\tilde{t}}  (\bx\stg{1:N}), \label{eq:est_solution} \\
     & r_t= F ({\bm x}^{(1)}_\ast , \ldots , {\bm x}^{(N)}_\ast ) - F( \hat{\bm x}^{(1)}_t, \ldots, \hat{\bm x}^{(N)}_t ).
\end{align}
Then, the following theorem holds.
\begin{theorem}
    \label{theorem:error_t}
    Under the same assumptions as in~\cref{theorem:CI}, define the estimated solution $( \hat{\bm x}^{(1)}_t, \ldots, \hat{\bm x}^{(N)}_t )$ by~\cref{eq:est_solution}.
    Then, for any positive number $\xi $, the following holds:
    \begin{align}
        % \begin{align}
             & \max_{\bx\stg{1:N}} \mr{ UCB} ^{(F)}_t  (\bx\stg{1:N} )   -  \max_{\bx\stg{1:N}}  \mr{ LCB} ^{(F)}_t  (\bx\stg{1:N}) < \xi    \label{eq:stopping_condition}    \\
             & \Rightarrow  F ({\bm x}^{(1)}_\ast , \ldots , {\bm x}^{(N)}_\ast ) - F( \hat{\bm x}^{(1)}_t, \ldots, \hat{\bm x}^{(N)}_t )  <\xi. \nonumber
        % \end{align}
    \end{align}
\end{theorem}
\Cref{theorem:error_t} states that if the credible interval width for $F$ is small, then regret $r_t$ is also small.
On the contrary, it does not guarantee whether the credible interval width becomes small or not.
\Cref{theorem:convergence_rt} shows that the interval width can be made arbitrarily small when~\cref{eq:af_CI} is used as the AF.
Let $\gamma\stg{n}_t$ be a \textit{maximum information gain} for $f\stg{n}$ at iteration $t$, and let $\gamma_t=\max_{n\in [N]} \gamma_t\stg{n}$.
Here, the maximum information gain is a commonly used sample complexity measure in the context of the GP-based BO~\cite{srinivas2010gaussian}.
The exact formulation is provided in Appendix~\ref{app:preliminaries}.
The following theorem also holds.
\begin{theorem}
    \label{theorem:convergence_rt}
    Assume that the same conditions as in \cref{theorem:CI} hold.
    Also assume that the Lipschitz continuity assumption (L2) holds.
    Let $\xi$ be a positive number, and let $\eta_t = (1+\log t )^{-1} $.
    Then, the following inequality holds after at most $T$ iterations:
    \begin{equation}
        F({\bm x}^{(1)}_\ast, \ldots ,{\bm x}^{(N)}_\ast ) - F(  \hat{\bm x}^{(1)}_T, \ldots , \hat{\bm x}^{(N)}_T )  < \xi,
    \end{equation}
    where $T $ is the smallest positive integer satisfying $T \in  N \mathbb{Z}_{\geq 0} =\{0,N,2N,\ldots \}$ and
    \begin{equation}
        \frac{8 \beta C^2_4   N^3 }{\log(1+\sigma^{-2})}        \gamma_T \eta^{-2N-2}_T T^{-1}  < \xi^2. \label{eq:xi_inequality}
    \end{equation}
    Here, each constant is given by $C_0= L_\sigma \beta^{1/2} + L_f +1 ,  C_1= \max \{ 1, L_f, L^{-1}_f \}, C_2= 4N^2   C^{2N-3} _0 C^N_1,  C_3= N C^N _2,  C_4  = (2 \beta^{1/2} +2 )^N C^N _3 $.
\end{theorem}
%From Theorems \ref{theorem:error_t} and \ref{theorem:convergence_rt}, it is guaranteed that we can obtain a solution that achieves an arbitrary accuracy $\xi$ in a finite number of observations.
%
The inequality~\cref{eq:xi_inequality} still has the variable $\gamma_T$.
Nevertheless, the order of $\gamma_T$ for commonly used kernels such as the linear and Gaussian kernels is sub-linear under mild conditions~\cite{srinivas2010gaussian}.
Hence, the integer $T$ satisfying~\cref{eq:xi_inequality} exists in these cases.
This indicates that a solution $ \hat{\bm x}^{(1)}_T, \ldots , \hat{\bm x}^{(N)}_T$ that achieves an arbitrary accuracy $\xi$ can be obtained in a finite number of observations.

In terms of the stopping criterion, if the accuracy parameter $\xi$ is provided, we can use the condition \cref{eq:stopping_condition} as the stopping criterion for EI- and CI-based AFs.
Although EI-based AF is not necessarily terminated by this stopping criterion, Theorem~\ref{theorem:convergence_rt} shows that CI-based AF terminates after at most $T$ iteration that satisfies \cref{eq:xi_inequality} when all assumptions hold.

% In the view of the stopping criterion, if the accuracy parameter $\xi$ is provided, we can use the condition \cref{eq:stopping_condition} as the stopping criterion for EI- and CI-based AFs.
% %
% Although EI-based AF is not necessarily terminated by this stopping criterion, Theorem~\ref{theorem:convergence_rt} shows that CI-based AF terminates after at most $T$ iteration that satisfies \cref{eq:xi_inequality} when all assumptions hold.
%

%-------------------------------------------------------------------------------------------
\section{Extensions}
\label{sec:extension}
In this section, we consider an extension called \emph{suspension} setting in which we are allowed to suspend the cascade process in the middle of the multistage decision-making process.
Suspension is beneficial, especially when the output of a middle stage is significantly different from the prediction, and the output is not expected to be beneficial for the subsequent stages.
For example, if a suspension occurs at stage $n$, the output $y^{(n-1)}$ of the previous stage remains unused, and this can be stored as a \emph{stock}.
If a stored stock turns out to be useful later, we can reuse the stock and resume the cascade process from the middle stage.

\paragraph{Formulation}
Let $\mcal{S}_t\stg{n}$ be the set of stocks at stage $n \in \{0, ..., N-1\}$ in iteration $t$~\footnote{We set $\mcal{S}_t\stg{0}=\{{0}\}$ for all $t$.}.
Because the process can be resumed from the middle stage in the suspension setting, the user's task in each iteration $t$ is to select the best pair $(y\stg{n-1}, \bm x\stg{n})$ from the set of candidates $\{\mcal{S}_t\stg{n} \times \dX\stg{n}\}_{n = 0}^{N-1}$.
Because of a user's choice, the used stock $y\stg{n-1}$ is removed from the set of stocks, and the newly obtained output $y\stg{n}$ is added to the set of stocks.
The difference in the cost of each stage is important in the suspension setting because, for example, if the costs of the later stages are greater than those of former stages, then the suspension strategy can be more beneficial.
Therefore, we introduce the cost of each stage $\lambda\stg{n}>0$ for $n \in [N]$.
\Cref{fig:suspension_resumption} shows a conceptual diagram of the suspension setting.

\begin{figure}[t]
    \centering
    \begin{subfigure}[b]{.5\linewidth}
        \centering
        \includegraphics[width=0.8\linewidth]{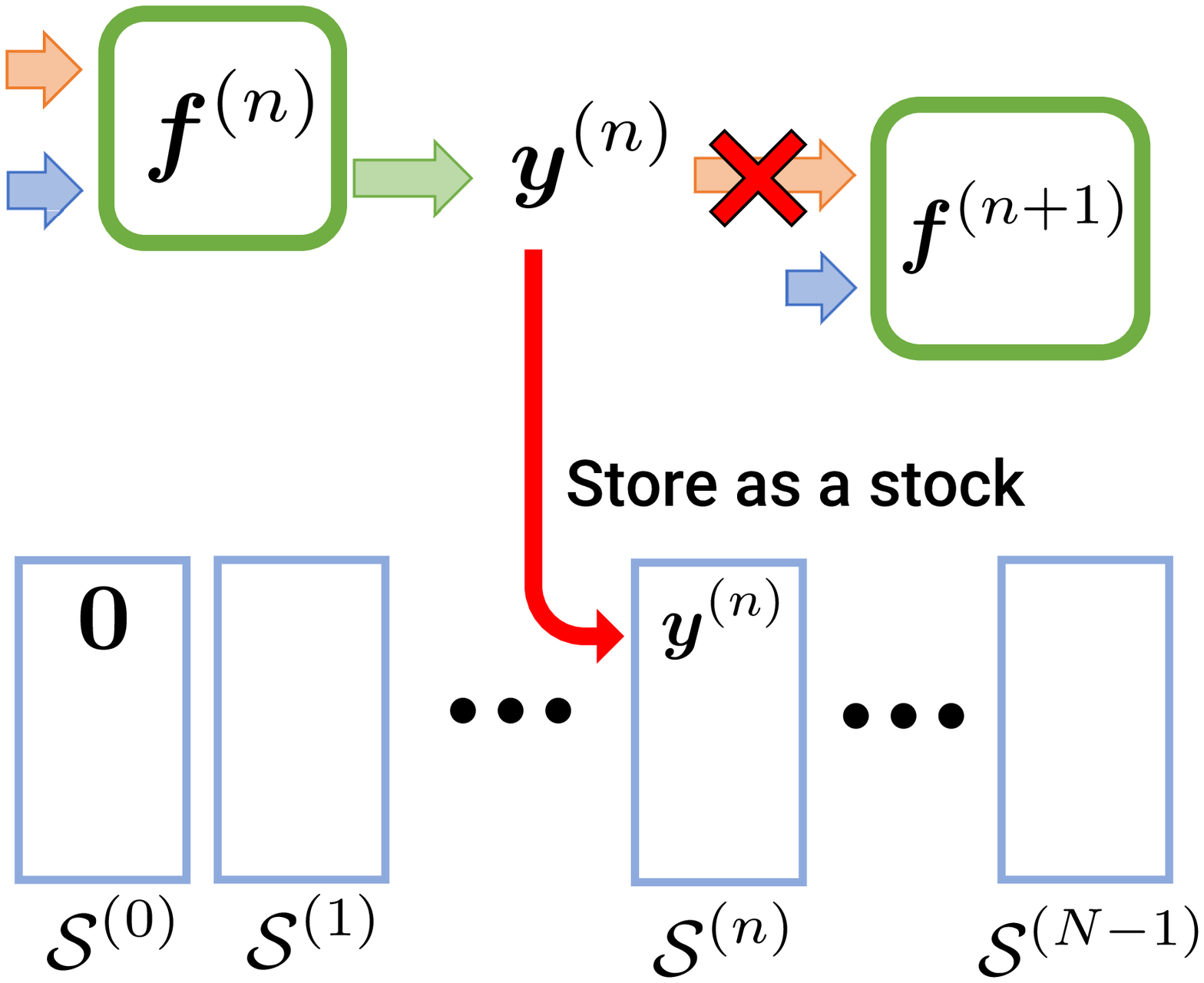}
        \subcaption{Suspension}
        \label{fig:suspension}
    \end{subfigure}%
    \centering
    \begin{subfigure}[b]{.5\linewidth}
        \centering
        \includegraphics[width=0.8\linewidth]{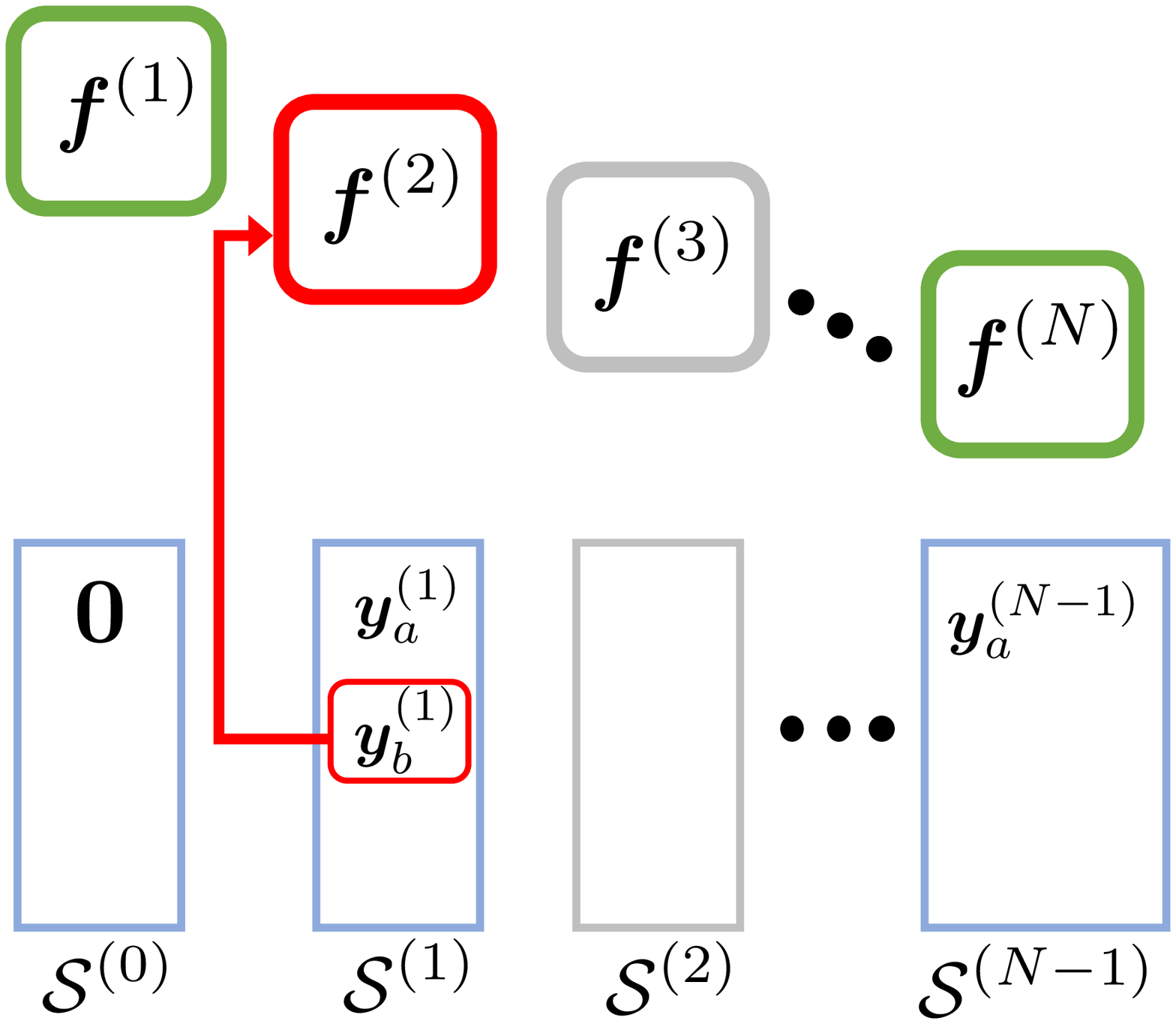}
        \subcaption{Resumption}
        \label{fig:resumption}
    \end{subfigure}
    \caption{
        Conceptual diagram of the suspension setting.
 \subref{fig:suspension} shows the case where the output $y\stg{n}$ is stored as a stock in stage $n$, and the observation from the subsequent stage is suspended.
 \subref{fig:resumption} shows the case where the observation is reused from stage two using the stock $y\stg{1}_b$.
    }
    \label{fig:suspension_resumption}
\end{figure}

\paragraph{Acquisition function for suspension setting}
We propose the following AF for the suspension setting:
\begin{equation}
    n_t, \by_{t}\stg{n-1},\bx_t\stg{n} = \argmax_{\mathclap{\substack{i\in [N],\\ y\stg{i-1}\in\mcal{S}_t\stg{i-1},\\ \bx\stg{i}\in \dX\stg{i}}}} \widetilde{U}_{i} (y\stg{i-1},\bx\stg{i}) / \sum_{j=i}^{N}\lambda\stg{j}.
    \label{eq:next_point}
\end{equation}
There are two differences between the AF in~\cref{eq:next_point} and the EI-based AF in~\cref{eq:EI_lower_mc}.
First, in~\cref{eq:next_point}, based on the set of stocks $\{\mcal{S}_t\stg{i-1}\}_{i \in [N]}$, we determine which stage to resume from, which stock to use, and what input to use.
Thus, \cref{eq:next_point} implicitly determines whether the sequential evaluation in cascade is suspended or not.
Second, the utility is divided by the total cost from stages $n$ to $N$, which suggests that a cost-effective choice is performed.
Resuming from a later stage has advantages (considering cost) because the goal is to optimize the output of the final stage.
The AF in~\cref{eq:next_point} can be interpreted as an extension of the EI-based AF in~\cref{eq:EI_lower_mc} because it handles the two cases of starting from the first stage and resuming from the middle stage using a stock.
It is necessary to compute the utility function for many candidates when solving the optimization problem in~\cref{eq:next_point}.
Nonetheless, this can be done efficiently by exploiting the fact that the evaluation of $\widetilde{U}_n$ in stage $n$ does not depend on the observations in the earlier stages.

\paragraph{Stock Reduction}
In the suspension setting, having a larger number of stocks provides us a wider choice.
However, practically, it can be costly to store several stocks.
In such a situation, it is necessary to be able to decide which stocks to retain and which ones to discard.
A reasonable way is to discard the stocks that are not expected to contribute to the optimal solution.
We implement this based on the credible interval.

For any stock $y\stg{n}\in \mathcal{S}\stg{n}$ in stage $0\le n \le N-1 $, let
\begin{equation}
    F(y\stg{n})=\max_{\bx\stg{n+1:N}}F(\bx\stg{n+1:N}|y\stg{n})
\end{equation}
be the maximum function value when the observation is performed until the final stage using $y\stg{n}$.
Therefore, the LCB and UCB of $F(y\stg{n})$ are given as
\begin{align}
 \mr{ LCB}^{(F)}_t ({ y}\stg{n} ) & =  \max_{ \bx\stg{n+1:N}   } \mr{LCB}\stg{F}_t(\bx\stg{n+1:N}| y\stg{n}), \\
 \mr{ UCB}^{(F)}_t ({ y}\stg{n} ) & =  \max_{ \bx\stg{n+1:N}  } \mr{UCB}\stg{F}_t(\bx\stg{n+1:N}| y\stg{n}).
\end{align}
Then, the following theorem holds.
\begin{theorem}\label{theorem:reduction}
 For any $n \in [N-1]$ and ${ y}\stg{n} \in \mathcal{S}^{(n)} $, under the same assumptions as in~\cref{theorem:CI}, assume that the following holds:
 \begin{equation}
  \mr{ UCB}^{(F)}_t ({ y}\stg{n} )  <  \max_{  \tilde{  y} \in \bigcup _{s=0}^{N-1}  \mathcal{S}_t^{(s)}   }   \mr{ LCB}^{(F)}_t (\tilde{ y}\stg{s} ) .
  \label{eq:delete_cond}
 \end{equation}
 Then, $F({ y}\stg{n} )  < F({\bm x}^{(1)}_\ast ,\ldots, {\bm x}^{(N)}_\ast ) $ holds.% with probability at least $1-\delta$.
\end{theorem}
The proof of the theorem is presented in Appendix~\ref{app:CI-based-AF}.
From~\cref{theorem:reduction}, the condition~\cref{eq:delete_cond} is used to decide which stocks to discard.
\Cref{theorem:reduction} only guarantees that the stock will not become the optimal value.
Suboptimal stocks may also be effectively used in the optimization process.

%-------------------------------------------------------------------------------------------
\section{Experiments}
\label{sec:experiments}

% \begin{figure}[t]
%     \begin{subfigure}{.5\linewidth}
%         \centering
%         \includegraphics[width=\linewidth]{figure/sample_path_s3d2_cycle100_result_same_2.pdf}
%         \subcaption{Sample path ($N=3$)}
%     \end{subfigure}%
%     \begin{subfigure}{.5\linewidth}
%         \centering
%         \includegraphics[width=\linewidth]{figure/sample_path_s5d2_cycle100_result_same_2.pdf}
%         \subcaption{Sample path ($N=5$)}
%     \end{subfigure}
%     \par \medskip
%     \begin{subfigure}{.5\linewidth}
%         \centering
%         \includegraphics[width=\linewidth]{figure/rosenbrock_s3d2_cycle100_result_same_2.pdf}
%         \subcaption{Rosenbrock ($N=3$)}
%         \label{fig:rosen3}
%     \end{subfigure}%
%     \begin{subfigure}{.5\linewidth}
%         \centering
%         \includegraphics[width=\linewidth]{figure/rosenbrock_s5d2_cycle100_result_same_2.pdf}
%         \subcaption{Rosenbrock ($N=5$)}
%     \end{subfigure}
%     \caption{Experimental results of the synthetic functions.
%         The solid line represents the average performance, and the error bar represents the standard error.
%     }
%     \label{fig:synth_result}
% \end{figure}

\begin{figure}[t]
    \centering
    \includegraphics[width=0.8\linewidth]{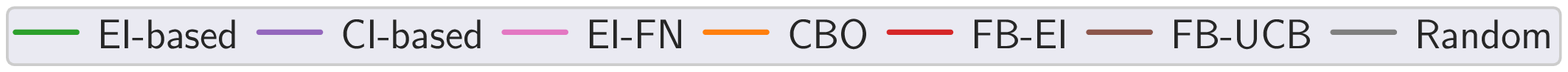}
    \par \medskip
    \includegraphics[width=0.45\linewidth]{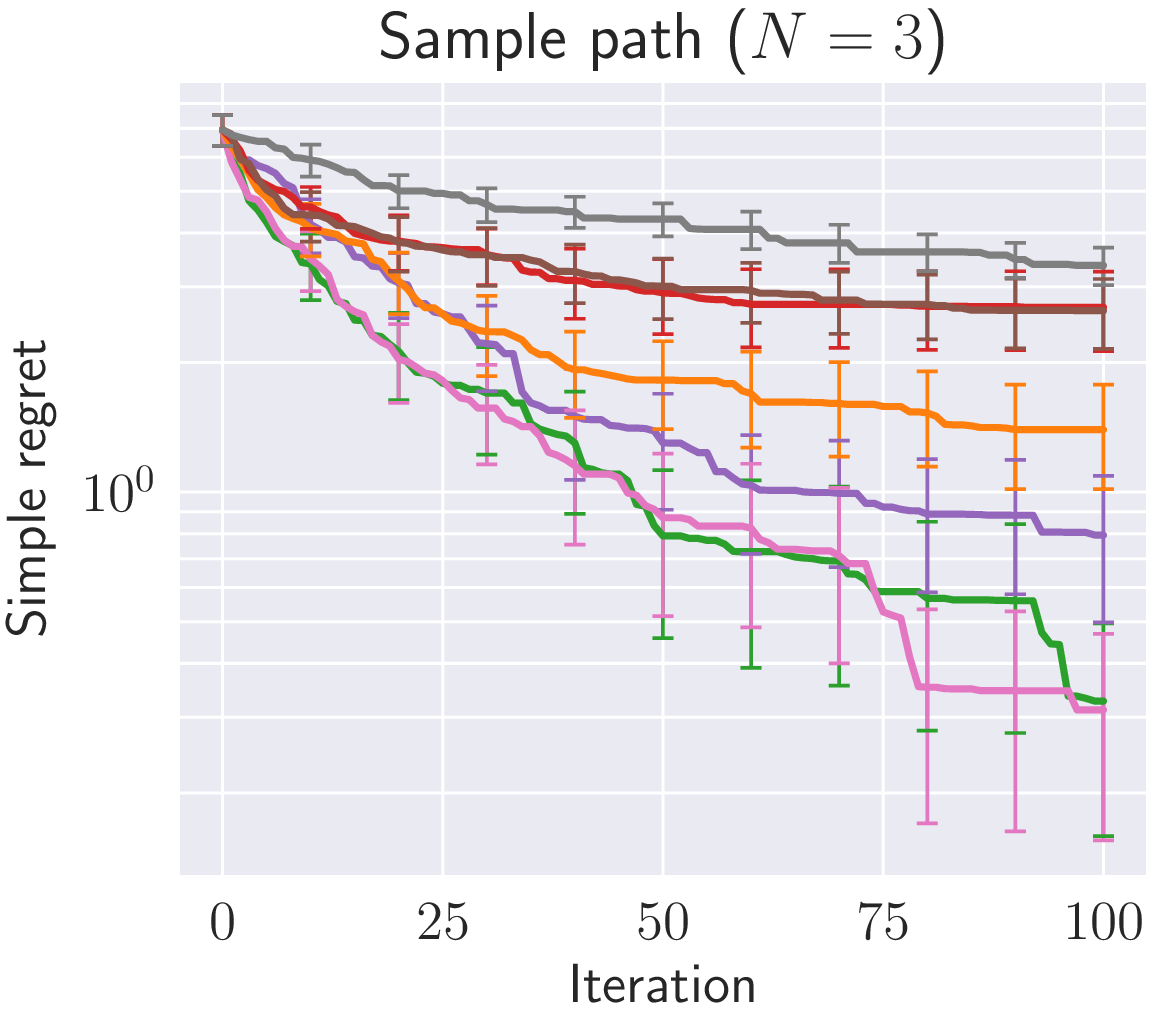}
    \includegraphics[width=0.45\linewidth]{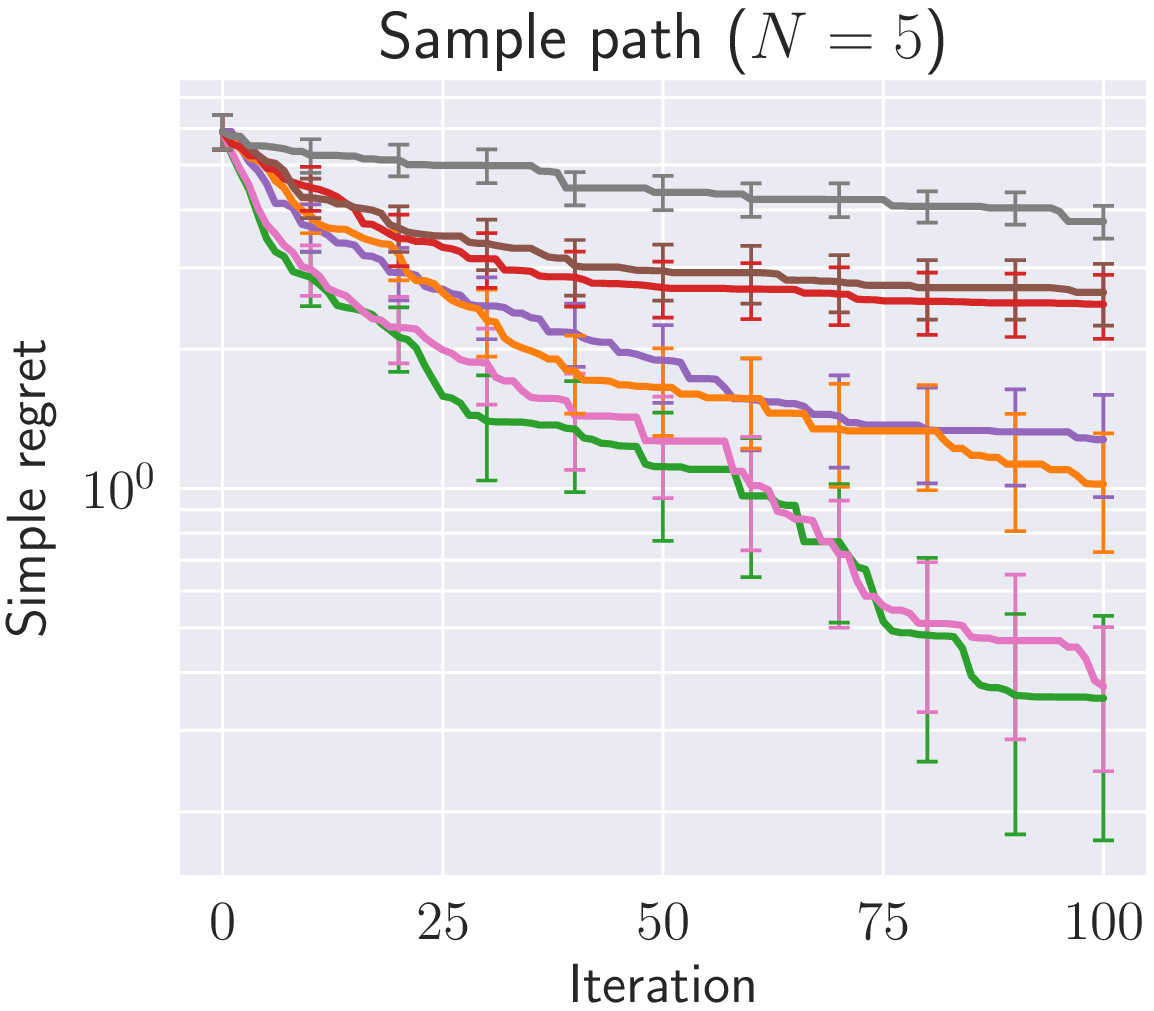}
    \par \medskip
    \includegraphics[width=0.45\linewidth]{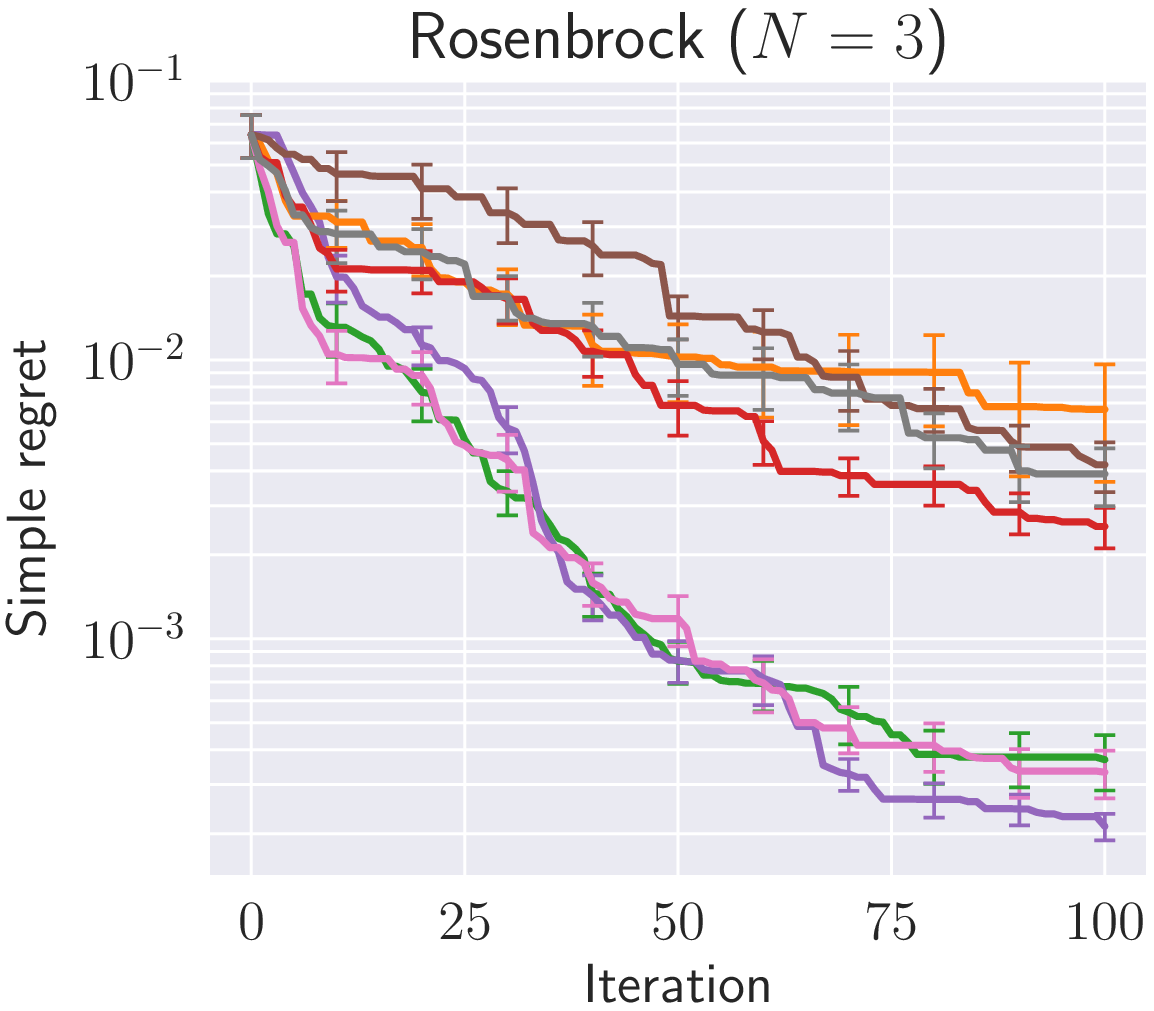}
    \includegraphics[width=0.45\linewidth]{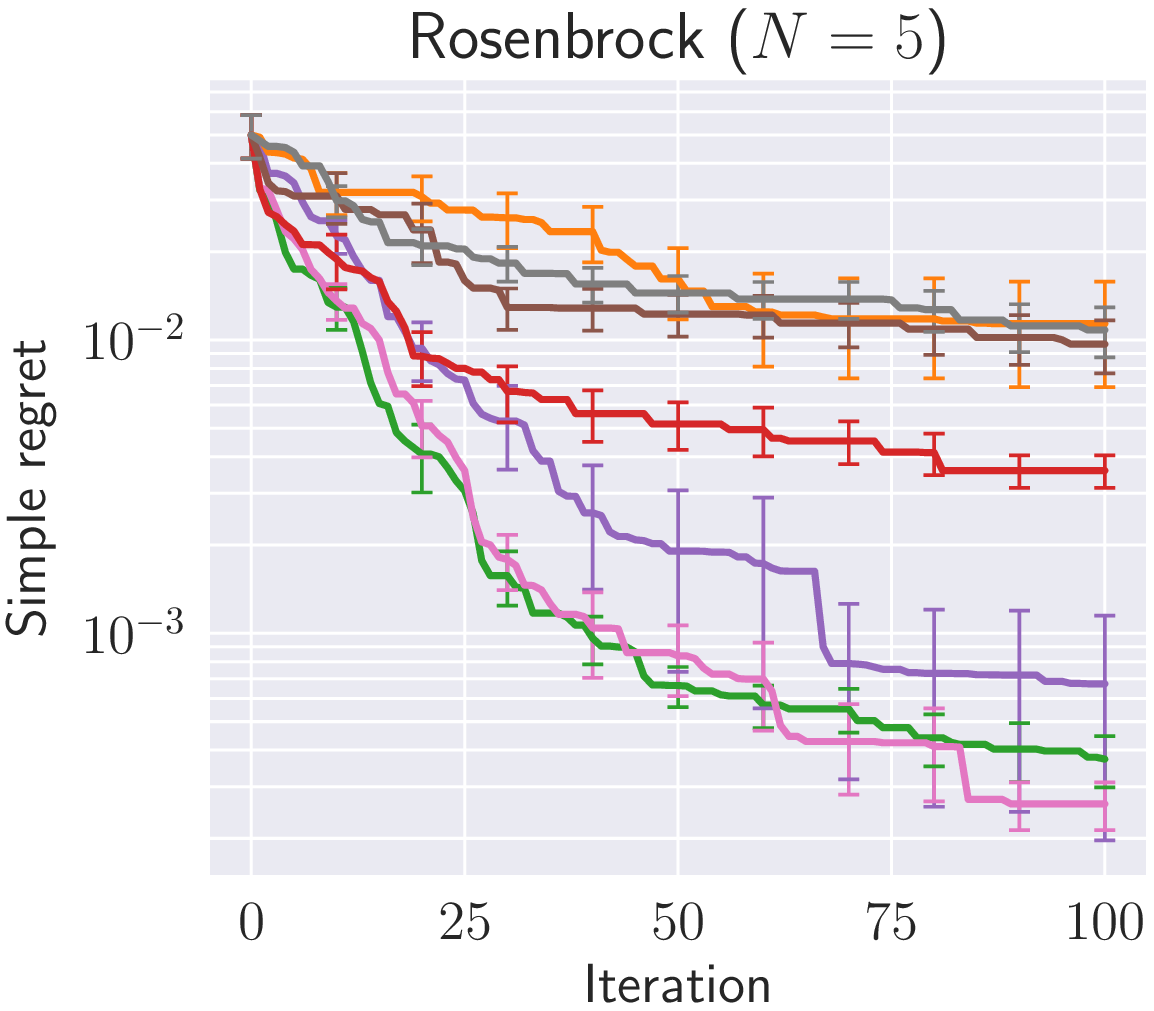}
    \par \medskip
    \includegraphics[width=0.45\linewidth]{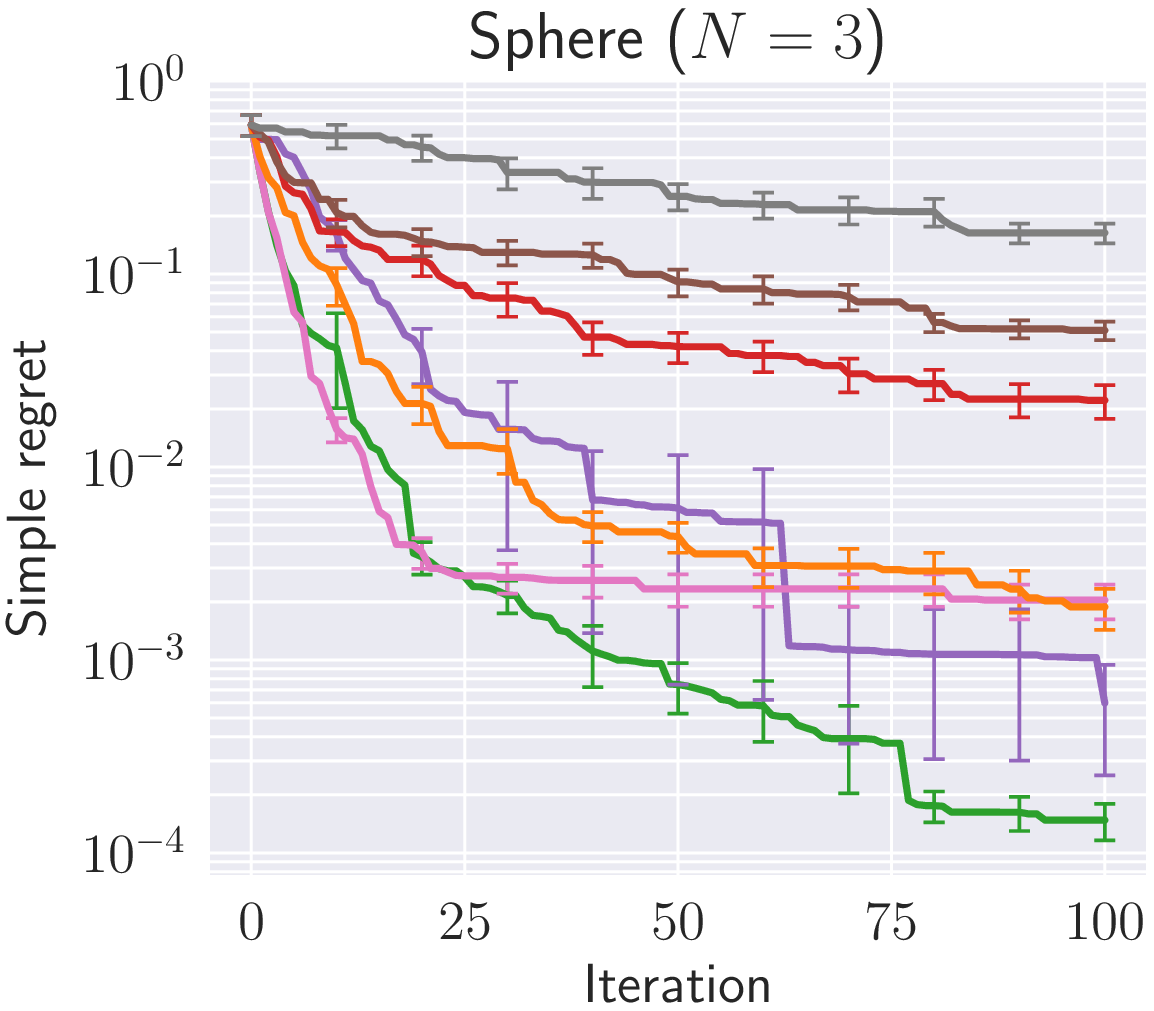}
    \includegraphics[width=0.45\linewidth]{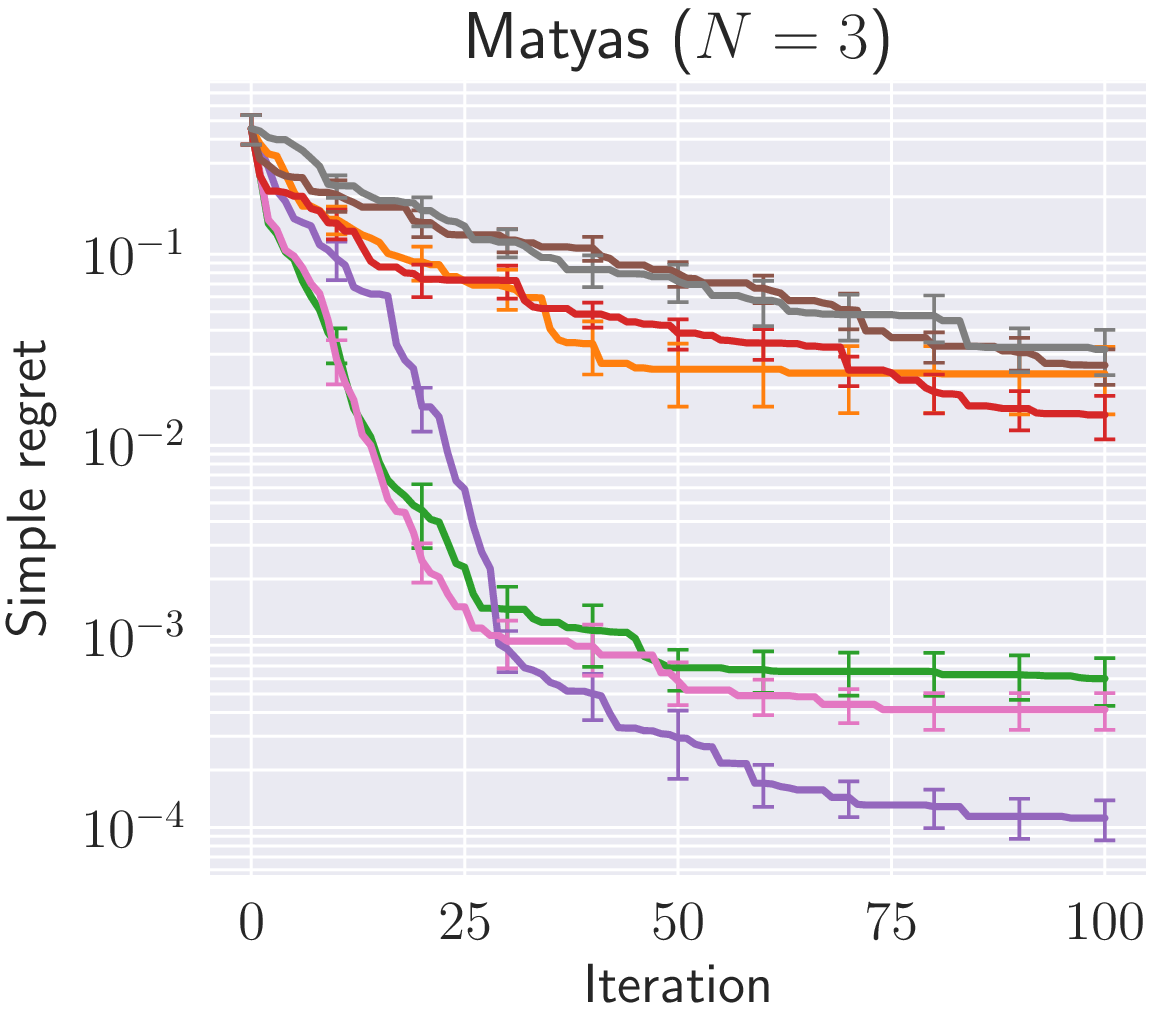}
    \caption{
        Experimental results of the synthetic functions.
        The solid line represents the average performance, and the error bar represents the standard error.
    }
    \label{fig:synth_result}
\end{figure}

% \begin{figure}[t]
%     \begin{subfigure}{.5\linewidth}
%         \centering
%         \includegraphics[width=\linewidth]{figure/solarcell3s_cycle50_result_same_1.pdf}
%     \end{subfigure}%
%     \begin{subfigure}{.5\linewidth}
%         \centering
%         \includegraphics[width=\linewidth]{figure/solarcell3s_cycle50_result_same_2.pdf}
%     \end{subfigure}
%     \caption{
%         Results of the solar cell simulator.
%         %
%         The right plot is an enlarged version of the left plot.
%     }
%     \label{fig:solarcell_result}
% \end{figure}

\begin{figure}[t]
    \centering
    \includegraphics[width=0.8\linewidth]{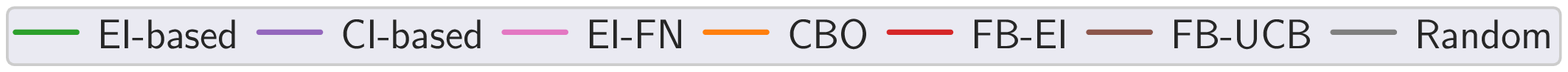}
    \par \medskip
    \includegraphics[width=0.45\linewidth]{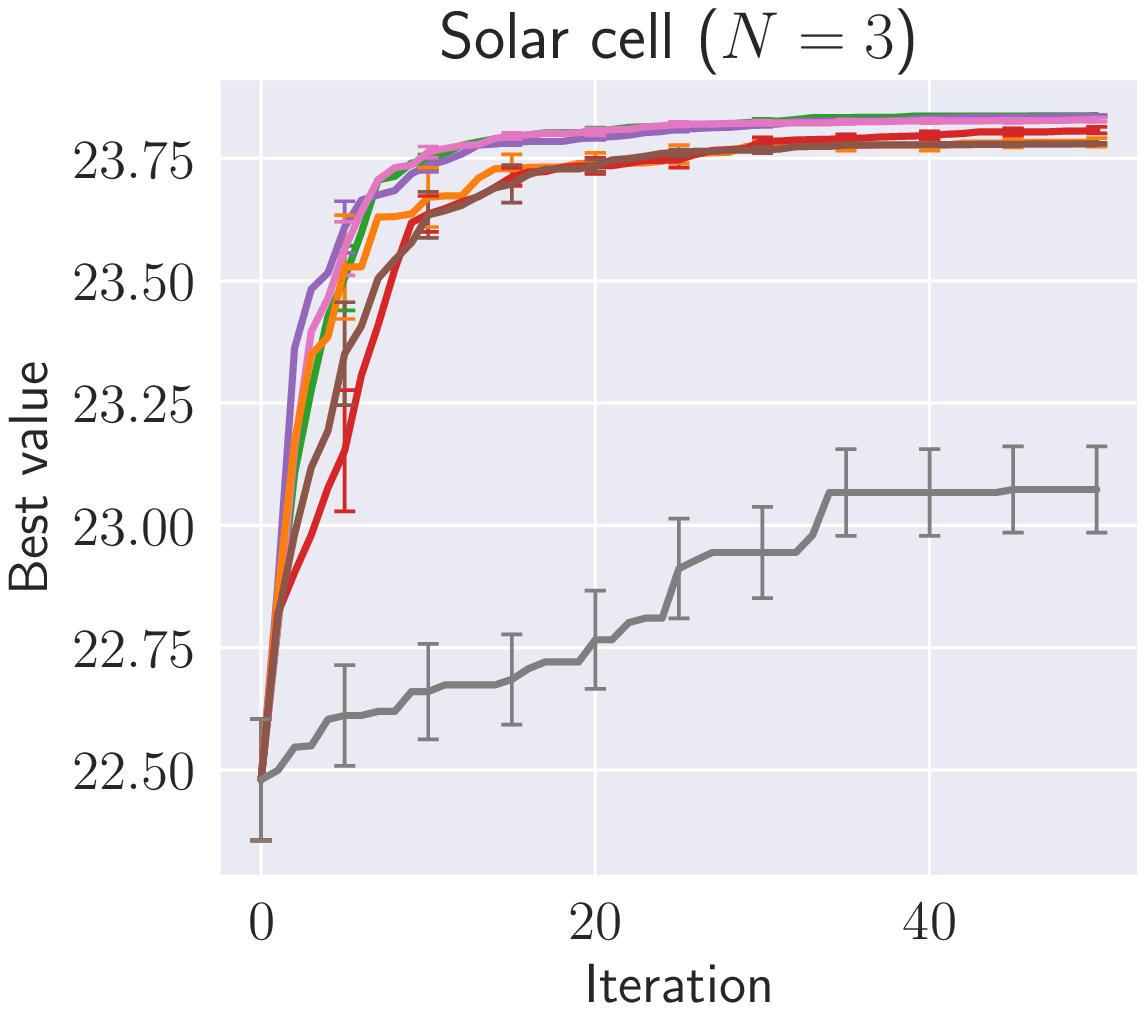}
    \includegraphics[width=0.45\linewidth]{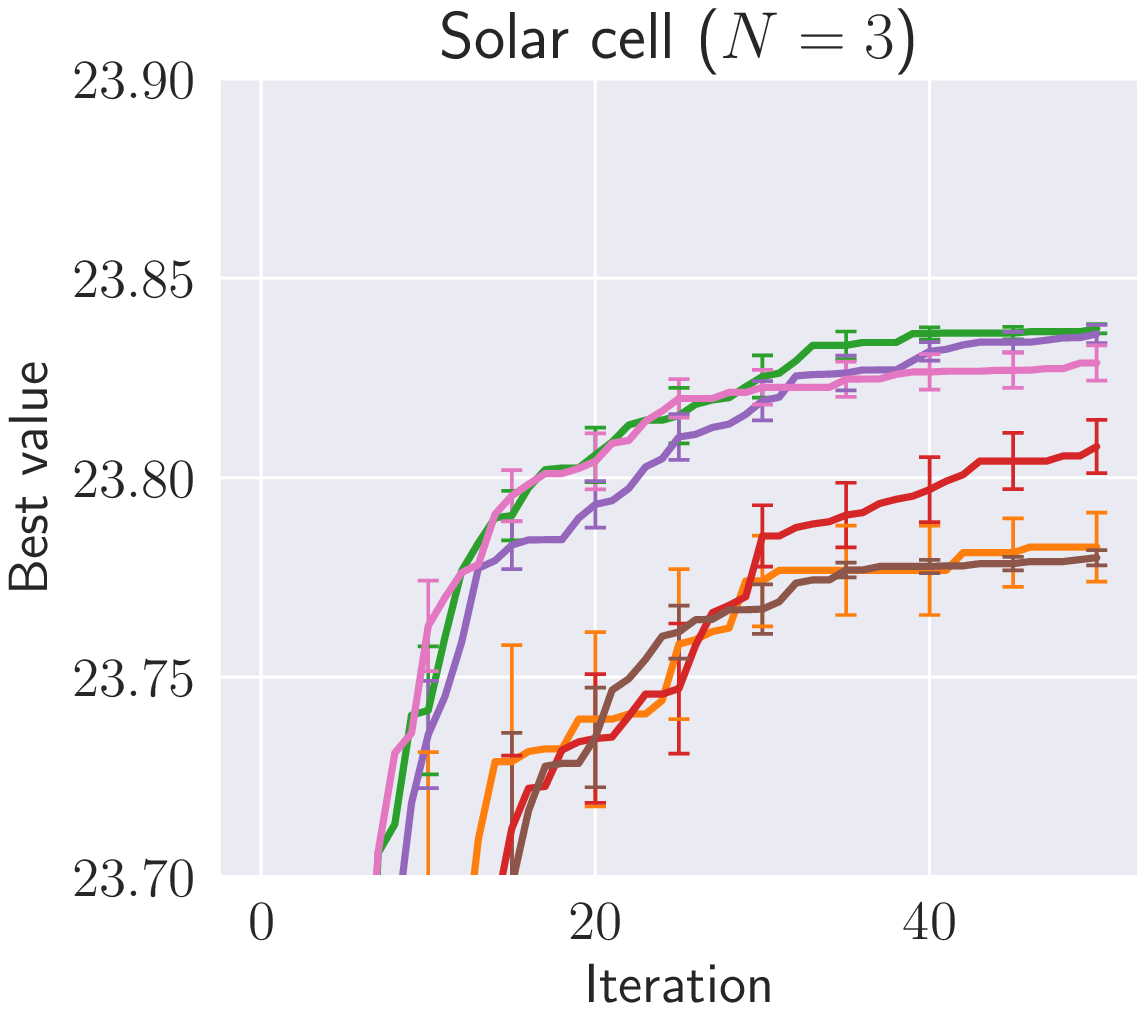}
    \caption{
        Results of the solar cell simulator.
        The right plot is an enlarged version of the left plot.
    }
    \label{fig:solarcell_result}
\end{figure}

\begin{figure}[t]
    \centering
    \includegraphics[width=0.8\linewidth]{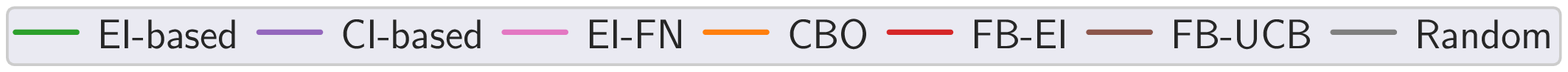}
    \par \medskip
    \includegraphics[width=0.45\linewidth]{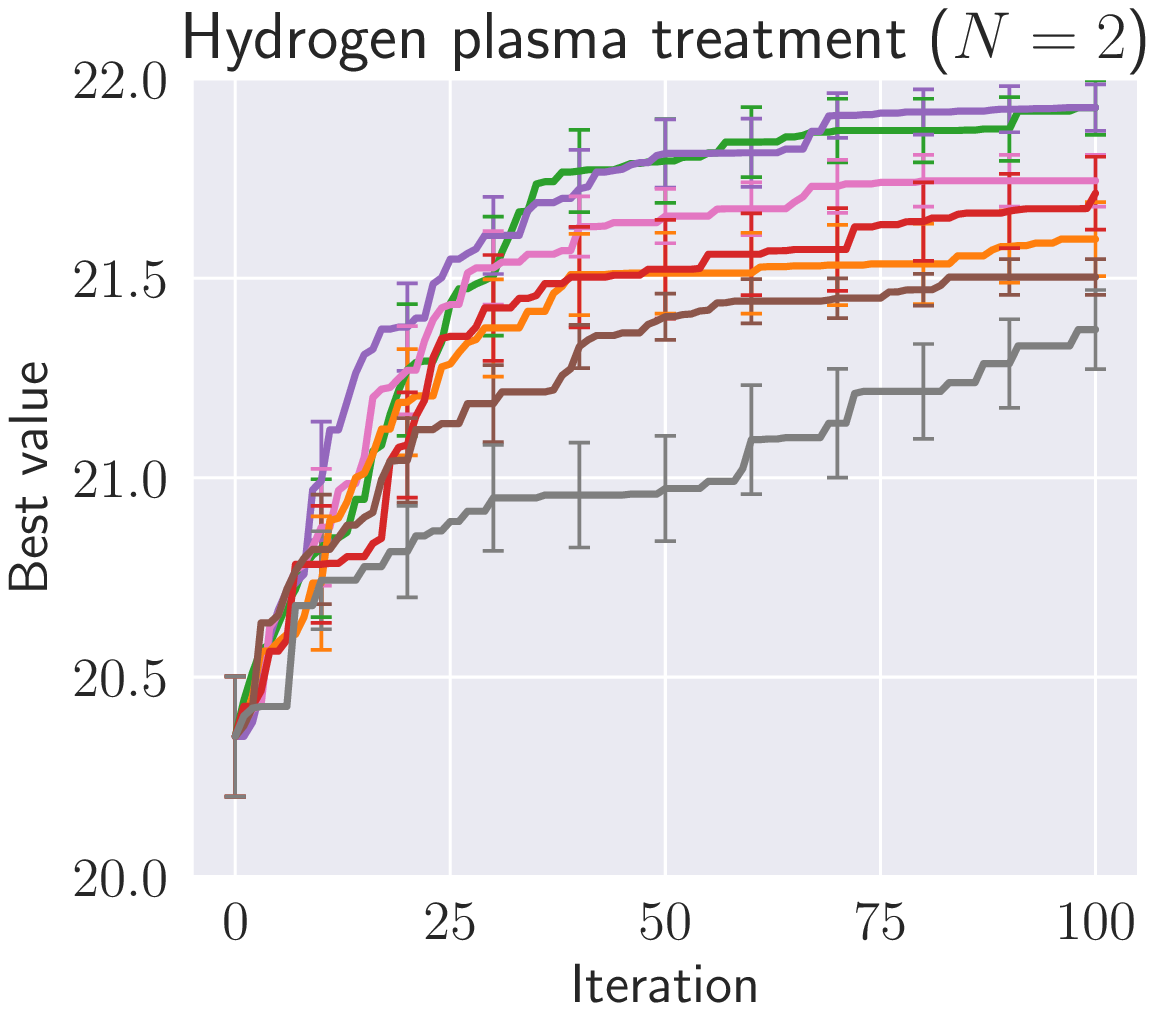}
    \caption{
        Results of the hydrogen plasma treatment.
        % %
        % The right plot is an enlarged version of the left plot.
    }
    \label{fig:hpt_result}
\end{figure}

% \begin{figure}[t]
%     \begin{subfigure}{.5\linewidth}
%         \centering
%         \includegraphics[width=\linewidth]{figure/sample_path_s3d2_cycle100_result_sus_same_2.pdf}
%         \subcaption{$\bs{\lambda}=(1,1,1)$}
%         \label{fig:ext_same}
%     \end{subfigure}%
%     \begin{subfigure}{.5\linewidth}
%         \centering
%         \includegraphics[width=\linewidth]{figure/sample_path_s3d2_cycle100_result_sus_cons_2.pdf}
%         \subcaption{$\bs{\lambda}=(1,1,10)$}
%         \label{fig:ext_tail}
%     \end{subfigure}
%     \begin{subfigure}{\linewidth}
%         \centering
%         \includegraphics[width=0.65\linewidth]{figure/solarcell3s_cycle50_result_reuse_same_2_wide.pdf}
%         \subcaption{Stock reuse}
%         \label{fig:ext_reuse}
%     \end{subfigure}
%     \caption{
%         Results in extension setting.
%     }
%     \label{fig:extension_result}
% \end{figure}

\begin{figure}[t]
    \centering
    \includegraphics[width=1.0\linewidth]{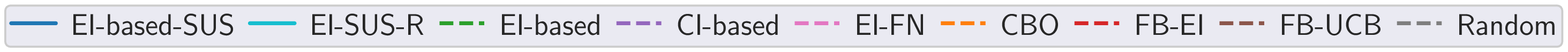}
    \par \medskip
    \begin{subfigure}{.5\linewidth}
        \centering
        \includegraphics[width=0.85\linewidth]{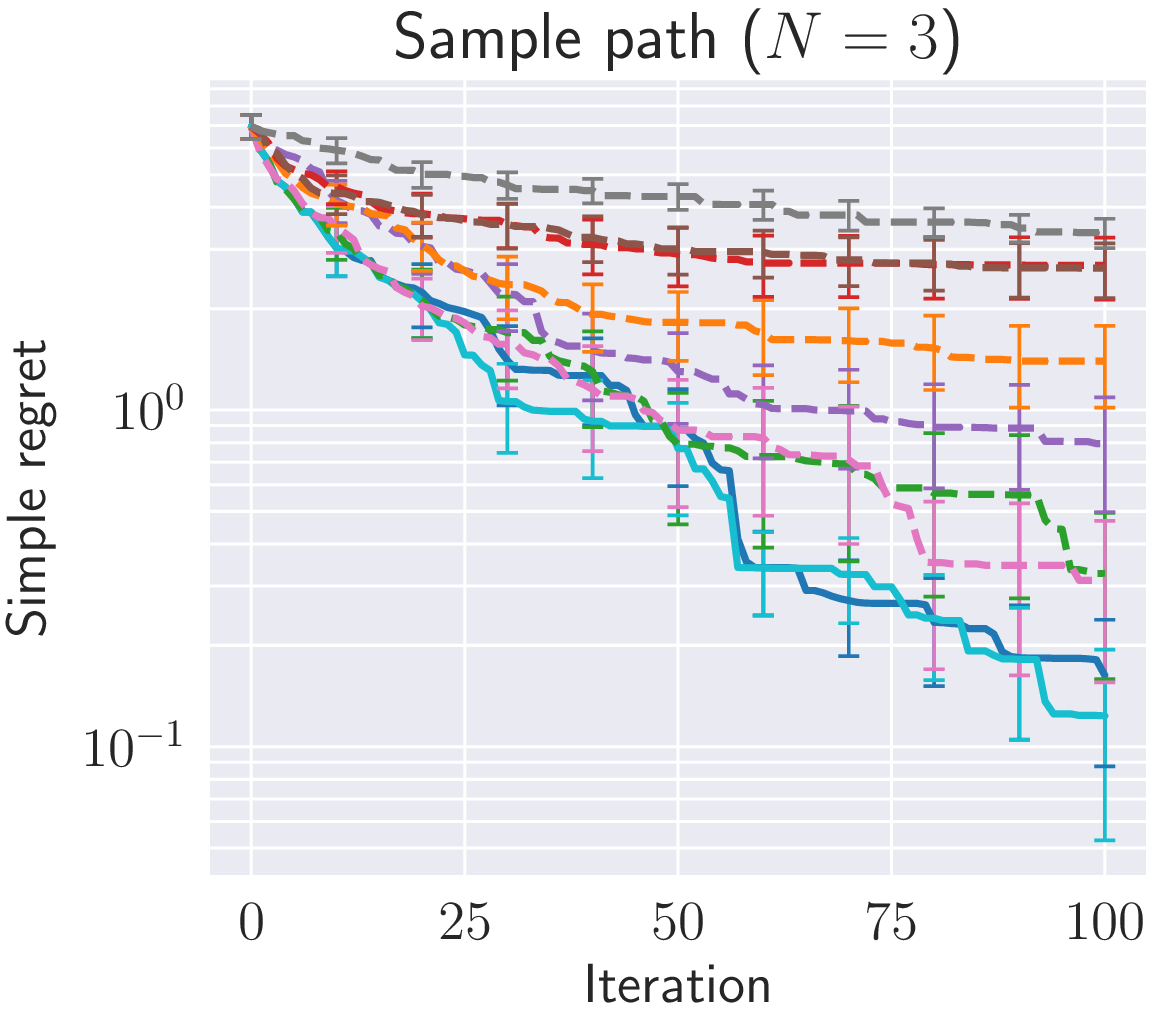}
        \includegraphics[width=0.85\linewidth]{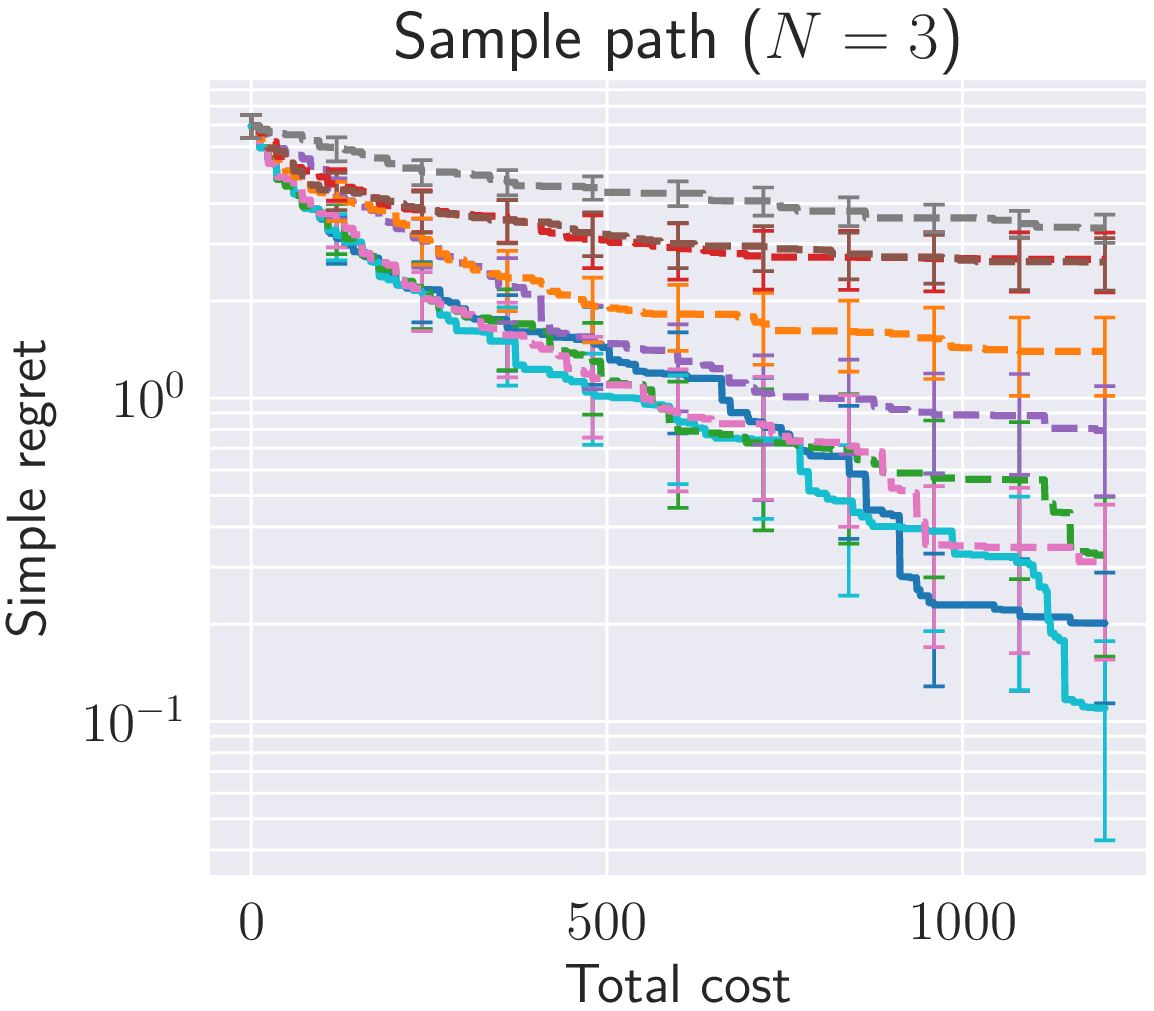}
        \subcaption{$\bs{\lambda}=(1,1,1)$}
        \label{fig:ext_same}
    \end{subfigure}%
    \begin{subfigure}{.5\linewidth}
        \centering
        \includegraphics[width=0.85\linewidth]{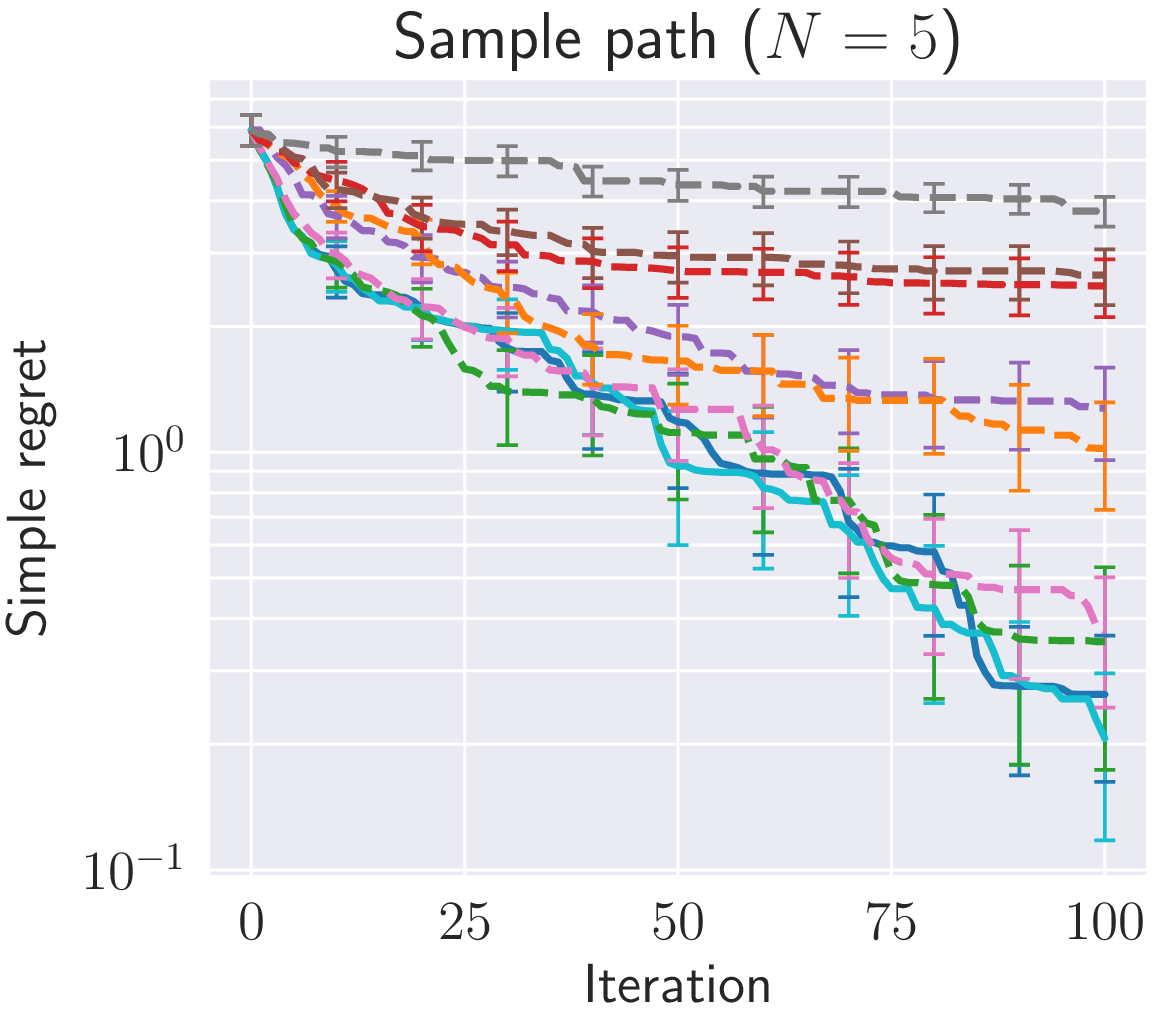}
        \includegraphics[width=0.85\linewidth]{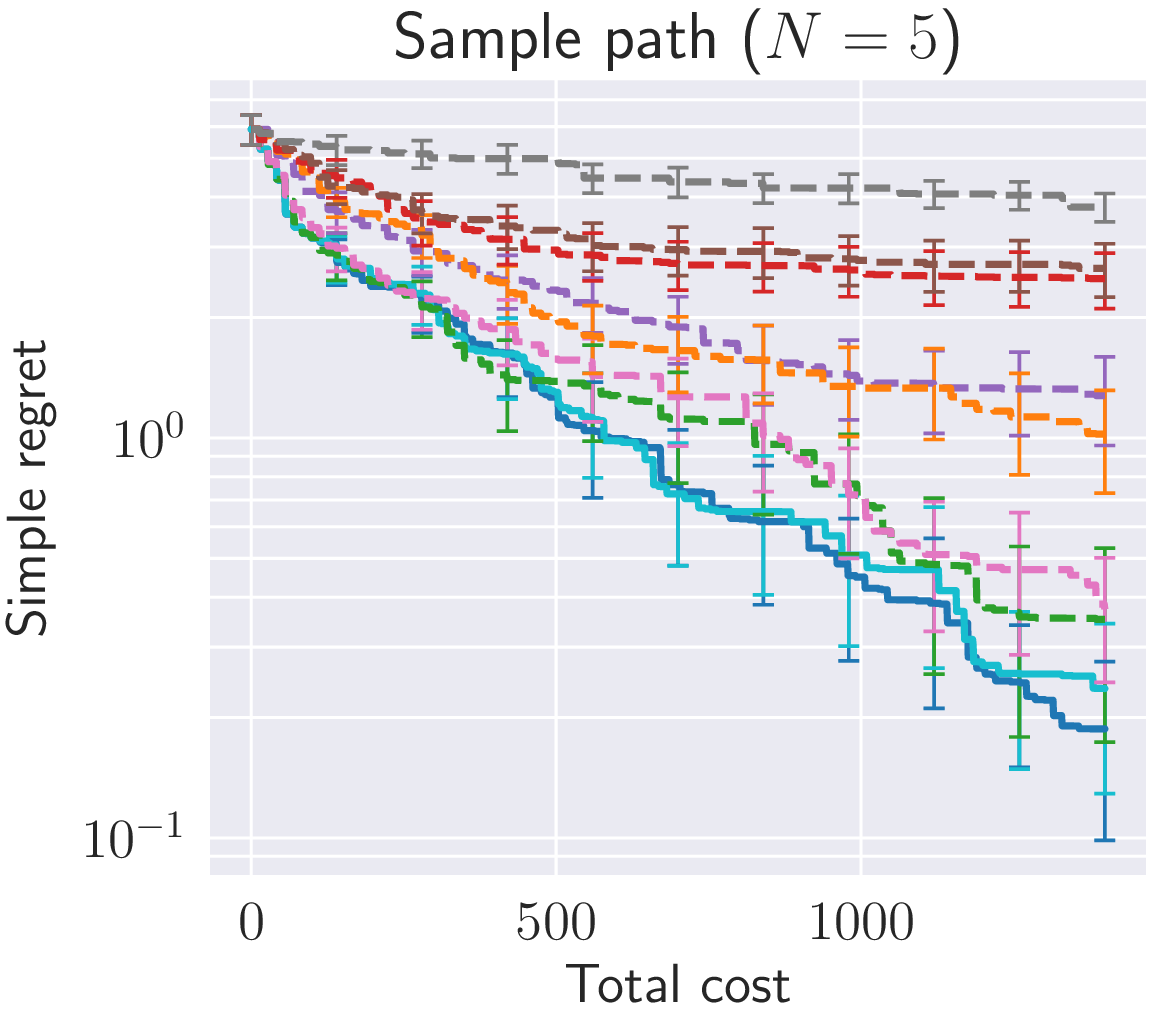}
        \subcaption{$\bs{\lambda}=(1,1,10)$}
        \label{fig:ext_tail}
    \end{subfigure}
    \begin{subfigure}{\linewidth}
        \centering
        \includegraphics[width=0.45\linewidth]{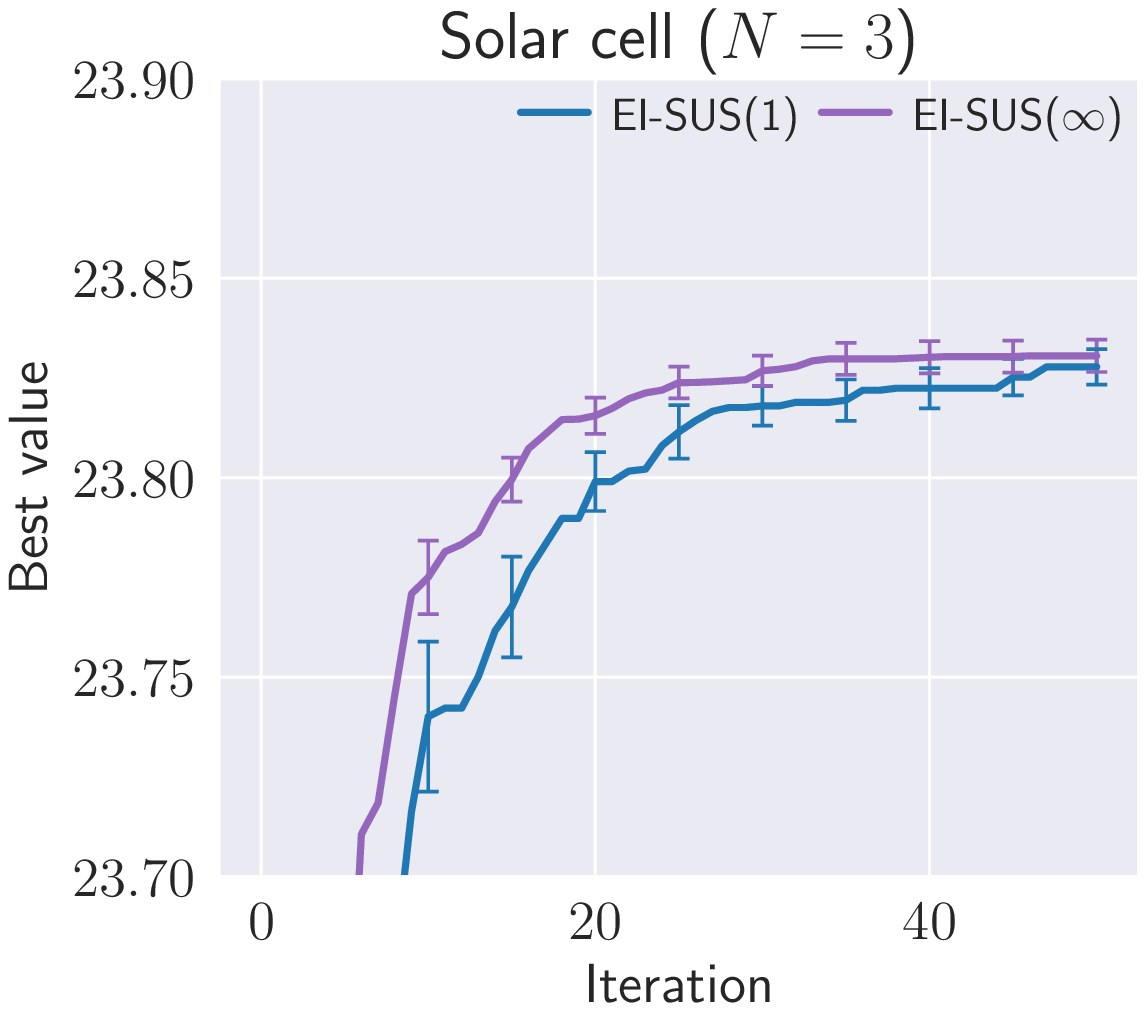}
        \subcaption{Stock reuse}
        \label{fig:ext_reuse}
    \end{subfigure}
    \caption{
        Results in extension setting.
        In the above experiments for sample paths, a solid line implies the methods with suspension and a dashed line represents a sequential method.
    }
    \label{fig:extension_result}
\end{figure}

We demonstrated the optimization performance of the proposed methods in both synthetic functions and a solar cell simulator.
Details of the experimental settings are provided in Appendix~\ref{app:exp_setting}.
First, we compared the methods in the sequential setting.
We used CBO, EI-FN and random sampling (Random) as the comparison methods.
In Random, each $\bx\stg{n}\in\dX\stg{n}$ is randomly and uniformly selected.
Regarding CBO, its AF is optimized by considering the output of the previous stage as the controllable variable.
Because the range of the previous output is unknown, we used a widely estimated range that was twice the actual range.
Additionally, we set its hyperparameters $\kappa_1,\kappa_2$ to one.
We also compared the proposed methods to a fully black-box BO that used EI and GP-UCB under a fully black-box model (FB-EI, FB-UCB).
The proposed methods with EI- and CI-based AFs are labeled as \textsc{EI-based} and \textsc{CI-based}, respectively.
We set the number of Monte Carlo sampling to $S=1000$, and we used $\eta_t=10^{-4}(1+\log t)^{-1}$ to calculate \textsc{CI-based}.
In all the experiments, we employed a Gaussian kernel
$k\stg{n}\left( (w,\bx), (w',\bx')\right)=\sigma\stg{n}_{f}\exp\left(- \frac{(w-w')^2}{2\ell_w^{2\: (n)}}-\sum_{d=1}^{D^{(n)}} \frac{(\bx_d-\bx'_d)^2}{2\ell_d^{2\: (n)}}  \right)$
and we set the noise variance of the GP model as $\sigma^2=10^{-4}$.
The performance was evaluated by the simple regret $F(\bx\stg{1}_*,\dots,\bx\stg{N}_*)-F(\bx\stg{1}_{\bar{t}},\dots,\bx\stg{N}_{\bar{t}})$,
where $\bar{t}= \argmax_{1 \le t' \le t} y\stg{N}_{t'}$.
Additional results comparing \textsc{EI-based} and EI-FN are shown in Appendix~\ref{app:additional_exp}.
%
% \blue{
% The results in the case where the estimated solution is reported as the maximum of $\mr{LCB}\stg{F}_t$ are given in Appendix~\ref{app:additional_exp}.
% }
%

\subsection{Synthetic Functions}
We used sample paths from the GP priors, Rosenbrock function, Sphere function, and Matyas function as the synthetic functions.
Regarding both functions, we constructed three- and five-stage cascade processes, and set $D^{(n)}=2$ for all $n$.
We used $L_{f}=1, \beta^{1/2}=2$ for the calculation~\cref{eq:cadcade_ci}.
In addition, 10 and 20 points for $N=3$ and $N=5$ were randomly selected and provided as the initial data.

\paragraph{Sample Paths from GP Priors:}
We employed the random Fourier feature~\cite{rahimi2007random} to sample ${f}\stg{n}$ from the GP prior and constructed $F$ using them.
Each ${f}\stg{n}$ was sampled ten times, and the experiments were conducted with two different random seeds for each.
The hyperparameters were set to $\sigma\stg{n}_f=15.02, \ell_d^{(n)}=3, \ell_w^{(n)}=3$.
We also set the domain of the control parameter to $\dX\stg{n}=[-10,10]^{D^{(n)}}$.

For the following synthetic functions, we ran experiments with 20 different random seeds.
Furthermore, we scaled ${f}\stg{n}$ such that the range of the function value is equal to the input domain for numerical stability.
The GP hyperparameters were selected by maximizing the marginal likelihood at every iteration.

\paragraph{Rosenbrock Function:}
Each ${f}\stg{n}$ is Rosenbrock function, whose domain of the control parameters were set to $\dX\stg{n}=[-2,2]^{D^{(n)}}$.
We perform the experiments with the number of stages $N = 3$ and $5$.
We set $\bx\stg{1} \in \R^3$ and $\bx\stg{n} \in \R^2$ for each $ n = 2, \dots, N$, and output $y\stg{n} \in \R$ for $n \in [N]$.

% \blue{
% We also used the Rosenbrock function for each ${f}\stg{n}$ and ran experiments with 20 different random seeds.
% %
% The domain of the control parameters were set to $\dX\stg{n}=[-2,2]^{\red{D^{(n)}}}$.
% %
% We scaled ${f}\stg{n}$ such that the range of the function is equal to $[-2,2]$ in the input domain for numerical stability.
% %
% The GP hyperparameters were selected by maximizing the marginal likelihood at every iteration.
% }

\paragraph{Sphere function:}
Each ${f}\stg{n}$ is Sphere function, whose domain of the control parameters were set to $\dX\stg{n}=[-5.12, 5.12]^{D^{(n)}}$.
Each output $y\stg{n} \in \R$ for $n \in [N]$ and the number of stages is $N=3$.
We set $\bx\stg{1} \in \R^3$ and $\bx\stg{2}, \bx\stg{3} \in \R^2$.

\paragraph{Matyas function:}
Each ${f}\stg{n}$ is Matyas function, whose domain of the control parameters were set to $\dX\stg{n}=[-10, 10]^{D^{(n)}}$.
Each output $y\stg{n} \in \R$ for $n \in [N]$ and the number of stages is $N=3$.
We set $\bx\stg{1} \in \R^2$ and $x\stg{2}, x\stg{3} \in \R^1$.

\Cref{fig:synth_result} shows the average value of the simple regret.
%
% \blue{
% It is confirmed that the proposed method performs as well as or better than the comparison methods.
% %
% Particularly, \textsc{EI-based} outperforms the comparison methods in all the results.
% %
% Additionally, excluding~\subref{fig:rosen3}, it also performs better than \textsc{CI-based}, which has superior theoretical properties.
% }
We see that our proposed methods and EI-FN clearly outperform other baselines including CBO.
Although \textsc{EI-based}, which can be roughly seen as the adaptive version of EI-FN, is comparable to EI-FN in most experiments, \textsc{EI-based} shows better performance than EI-FN in the Sphere function.
This can be seen as a benefit of adaptive decision-making.
Although \textsc{CI-based} has superior theoretical properties, \textsc{CI-based} is inferior to \textsc{EI-based} except for Rosenbrock ($N=3$) and Matyas functions.
One of the reasons for these results is the setting of the hyperparameters, such as $\beta$ and $L_f$.

\subsection{Solar Cell Simulator}
We applied the proposed methods to the solar cell simulator.
This simulator consists of three-stage processes.
Stages one and two are two-step annealing processes to diffuse phosphorus into the silicon substrate from the surface, forming a p-n junction near the surface.
The controllable parameters of stage one are the phosphorus concentration at the surface, temperature, and time of the first-step annealing.
In addition, the controllable parameters of stage two are the temperature and time of the second-step annealing.
The outputs of stages one and two are the four parameters that indicate the distribution of phosphorus concentration in the depth direction.
In stage three, the solar cell is constructed using controllable parameters composed of wafer thickness and boron concentration of the substrate, and the performance is evaluated under standard measurement conditions.
The final output is the power generation efficiency of the solar cell, and our goal is to maximize this output.
Regarding the real-world simulators, the simulators of stages one and two are based on the physical model~\cite{bentzen2006phosphorus}.
Moreover, the simulator of stage three was constructed using the data collected from PC1Dmod6.2~\cite{haug2016pc1dmod}.
In stages one and two, the simulators produce vector outputs.
However, CBO does not support vector outputs, so we calculated its AF by replacing the predictive mean and variance with the mean vector and covariance matrix, respectively.
The domain of the controllable parameters are $\dX\stg{1}=[700,1050]\times[100,5000]\times[19,21.18]$,
$\dX\stg{2}=[700,1050]\times[100,5000]$, and $\dX\stg{3}=[50,250]\times[14,17]$.
We randomly chose 20 points as the initial data.
In addition, we set $L_f=0.1, \beta^{1/2}=2$ in this setting.
Furthermore, we tuned the hyperparameters by maximizing the marginal likelihood and ran the experiment for 50 iterations using 20 different random seeds.

\Cref{fig:solarcell_result} shows the average of the best observed value $\max_{1\le t' \le t} y\stg{N}_{t'}$.
This result shows that the proposed method outperforms the existing methods in the simulator experiments.
It is also confirmed that the best value found in 50 iterations in the existing methods is achieved in less than half of the iterations in the proposed method.
In a comparison between \textsc{EI-based} and EI-FN, the error bars are not overlapped after the $40$ iteration.
Thus, \textsc{EI-based} shows a slightly small but substantial improvement by adaptive decision-making.

\subsection{Hydrogen Plasma Treatment Process}

We applied the proposed method to the hydrogen plasma treatment (HPT) process, which is a part of the production process of solar cells.
In the previous practical study, one of the authors (KK) optimized one-stage HPT process parameters through real experiments using simple BO \citep{miyagawa2021application-b,miyagawa2021application-a}.
In this study, we extended this HPT process to the virtual two-stage cascade process.
The first stage is the HPT process with 7 inputs, temperature, pressure, flow rate, process time, electrode distance, radio frequency power, and cycle time, and 2 outputs, saturation current density, and contact resistance.
The second process is the solar cell production process in which surface electrode width is the controllable parameter.
The final output is the power generation efficiency of the solar cell, and our goal is to maximize this output as in the case of the solar cell simulation.
%
% (以降竹野記述)
%
The domain of the controllable parameters are $\dX\stg{1} = [50, 300] \times [0.25, 4] \times [100, 1000] \times [10, 100] \times [270, 420] \times [10, 40] \times [15, 60]$ and $\dX\stg{2} = [0.01, 0.1]$.
Since the real dataset is small with respect to the input domain, we used surrogate objectives, which are sample paths of GPs fitting to the real dataset for each stage.
The details of these sample paths are shown in Appendix~\ref{app:exp_setting}.
Other experimental settings are set as with the solar cell simulator experiment.

\Cref{fig:hpt_result} shows the average of the best observed value $\max_{1\le t' \le t} y\stg{N}_{t'}$.
Our proposed methods \textsc{EI-based} and \textsc{CI-based} are superior to other baselines including EI-FN and CBO.
In particular, the difference between \textsc{EI-based} and EI-FN implies the improvement by adaptive decision-making.

\subsection{Suspension Setting}
We also conducted experiments in a suspension setting using the proposed method~\cref{eq:next_point} (\textsc{EI-based-SUS}).
In this setting, we used the sample path function with $N=3$ and $5$.
For the cost of each stage $\bs{\lambda}\coloneqq (\lambda\stg{1},\lambda\stg{2},\lambda\stg{3})$, we consider two settings:
$ \bs{\lambda}=(1,1,1),\,  \bs{\lambda}=(1,1,10)$.
Furthermore, we apply the stock reduction rule~\cref{eq:delete_cond} to \textsc{EI-based-SUS} and executed it in both settings.
We refer to this as \textsc{EI-SUS-R}.
The results are shown in~\cref{fig:extension_result}\subref{fig:ext_same} and~\ref{fig:extension_result}\subref{fig:ext_tail}.
Comparing \textsc{EI-based} and \textsc{EI-based-SUS}, we can observe that the performance is improved by incorporating the suspension.
Moreover, the performance did not deteriorate even when the stock reduction rule was applied.
In addition, the stocks are not consumed in the simulator, and once a stock is acquired, it can be used a number of times.
In this case, we can reduce the number of observations in the earlier stages by reusing the stock.
We compared the situation in which stocks are available only once (\textsc{EI-SUS (1)}) and the situation in which stocks can be used a number of times (\textsc{EI-SUS ($\infty$)}).
From~\cref{fig:extension_result}\subref{fig:ext_reuse}, we confirm that \textsc{EI-SUS ($\infty$)} performs a more efficient optimization.

\afterpage{\clearpage}

%-------------------------------------------------------------------------------------------
\section*{Conclusion}
We proposed a new BO framework for cascade-type multistage processes that often appear in science and engineering.
Moreover, we have designed two AFs based on CIs and EI by handling intractable predictive distributions using different approaches.
From both the theoretical analysis and numerical experiments, it is confirmed that the proposed methods have a superior performance.

\subsection*{Acknowledgments}

This study was partially supported by MEXT/JSPS KAKENHI (16H06538, 17H04694, 20H00601, 21H03498, 22H00300, JP21J14673), JST CREST (JPMJCR21D3),
JST Moonshot R\&D (JPMJMS2033-05), NEDO (JPNP18002, JPNP20006), and RIKEN Center for Advanced Intelligence Project.
The authors acknowledge Takuto Kojima and Kazuhiro Gotoh of Nagoya University for their support of the solar cell simulation, and Shinsuke Miyagawa, Kazuhiro Gotoh, Yasuyoshi Kurokawa, and Noritaka Usami of Nagoya University for providing HPT data.

% \bibliographystyle{myapalike}
% \bibliography{ref}

%-------------------------------------------------------------------------------------------
\allowdisplaybreaks[1]
% appendix

\section*{Appendix}

\setcounter{section}{0}
\renewcommand{\thesection}{\Alph{section}}
\numberwithin{equation}{section}

%-------------------------------------------------------------------------------------------
\section{Generalization of Problem Setting}
\label{app:preliminaries}
Hereafter, we consider the generalized settings, including vector output and noisy observations.
First, we generalize the problem setting in this section.
In Appendix~\ref{app:CI-based-AF}, we consider the noiseless setting.
We also consider the noisy observation setting in Appendix~\ref{app:CI-based-AF-noisy} and provide the optimization algorithm.
Furthermore, we discuss the conditions of our theorems in Appendix~\ref{app:modify}.
Details of our experiments and additional experiments are described in Appendix~\ref{app:exp_setting} and~\ref{app:additional_exp}, respectively.

Let $\dY\stg{n} \subset \R^{M\stg{n}}$ be the $M\stg{n}$-dimensional output space\footnote{
    Since we focus on single-objective optimization, the output of the final stage is assumed to be scalar (i.e., $M\stg{N} = 1$).}, and
vector-output black-box function of stage $n$ is denoted by $\bs{f}\stg{n}$, and $f\stg{n}_m$ denotes the $m$-th function of $\bs{f}\stg{n}$.
Output $\by\stg{n}$ corresponding to an input $(\by\stg{n-1},\bx\stg{n})$ is observed with noise $\bs{\epsilon}\stg{n}$:
$\by\stg{n}=\bs{f}\stg{n}(\by\stg{n-1},\bx\stg{n})+\bs{\epsilon}\stg{n}$.
The noiseless settings are the case of $\bs{\epsilon}\stg{n}=\bs{0}$.
Furthermore, we consider that $\bs{\epsilon}\stg{n}$ is uniformly bounded and zero mean noise in Appendix~\ref{app:CI-based-AF-noisy}.

In order to construct a surrogate model of $\bs{f}\stg{n}$, we set $\GP(0,k\stg{n})$ to the prior for each $f\stg{n}_m$,
where $\GP(\mu,k\stg{n})$ represents the GP with mean function $\mu$ and kernel function $k\stg{n}$.
Additionally, we assume that $k\stg{n}$ is a positive-definite kernel and $\forall (\by\stg{n-1}, \bx\stg{n})\in \dY\stg{n-1}\times \dX\stg{n},\: k\stg{n}\left((\by\stg{n-1}, \bx\stg{n}),\, (\by\stg{n-1}, \bx\stg{n})\right)\le 1$.
Let $\D\stg{n}_t=\left\{\left( (\by_i\stg{n-1} ,\bx_i\stg{n}), \by_i\stg{n}  \right)\right\}_{i=1}^{L_t\stg{n}}$ be observed data of stage $n$ at iteration $t$.
As the noise model of GP, we use $\bs{\epsilon}\stg{n} \sim \N(\bs{0}, \sigma^{2} \bs{I}_{M\stg{n}})$, where $\bs{I}_{M\stg{n}}$ denotes $M\stg{n}\times M\stg{n}$ identity matrix.
Note that this noise model is different from the actual noise assumption.
Given the observation $\D\stg{n}_t$, the posterior of $f\stg{n}_m$ is also GP, and the predictive distribution of $f\stg{n}_m(\by\stg{n-1},\bx\stg{n})$ is given by:
\begin{align}
    & f\stg{n}_m(\by\stg{n-1},\bx\stg{n}) \sim \N\bigl(\mu\stg{n}_{m,t}(\by\stg{n-1},\bx\stg{n}),\:  \sigma^{(n)\, 2}_{m,t}(\by\stg{n-1},\bx\stg{n}) \bigr),                                \\
    & \mu\stg{n}_{m,t}(\by\stg{n-1},\bx\stg{n})   =\bs{k}\left(\by\stg{n-1},\bx\stg{n}\right)^\top (\K\stg{n}_t + \sigma^2\I_{L_t\stg{n}})^{-1}\by\stg{n}_m, \\
     &\begin{aligned}
        {\sigma^{ (n)\, 2}_{m,t}}(\by\stg{n-1},\bx\stg{n}) =&k\stg{n}\left((\by\stg{n-1},\bx\stg{n}),(\by\stg{n-1},\bx\stg{n})\right)                     \\
        & \quad -      \bs{k}(\by\stg{n-1},\bx\stg{n})^\top(\K\stg{n}_t + \sigma^2\I_{L_t\stg{n}})^{-1}\bs{k}(\by\stg{n-1},\bx\stg{n}).
     \end{aligned} \label{eq:predict}
\end{align}
Here, \sloppy$\bs{k}(\by\stg{n-1},\bx\stg{n})= \left[k\stg{n}\left((\by\stg{n-1},\bx\stg{n}),\, (\by\stg{n-1}_i,\bx\stg{n}_i)\right)\right]_{i=1}^{L_t\stg{n}}$,
\sloppy$\by\stg{n}_m= \left[y\stg{n}_{1m},\dots,y\stg{n}_{Lm}\right]^\top$,
and $\K\stg{n}_t$ is a kernel matrix which has \sloppy$k\stg{n}\bigl((\by\stg{n-1}_i,\bx\stg{n}_i),\,  (\by\stg{n-1}_j,\bx\stg{n}_j)\bigr)$ in $(i,j)$-th element.
In addition, we define $\bs{\mu}_t\stg{n}(\by\stg{n-1},\bx\stg{n})= \left[\, \mu\stg{n}_{m,t}(\by\stg{n-1},\bx\stg{n})\, \right]_{m=1}^{M\stg{n}}$.

For a GP model of $f\stg{n}_m$, we give a definition of the maximum information gain.
Let $A\stg{n}=\{ \bs{a}\stg{n}_1,\dots,\bs{a}\stg{n}_T\} \subset \dY\stg{n-1}\times \dX\stg{n} $ be a finite set of sampling points.
We define $\by\stg{n}_{A,m} \in \R^{T}$ as observation vector w.r.t. $A\stg{n}$, whose $i$-th element is given by $y\stg{n}_{\bs{a}_i,m}=f\stg{n}_m(\bs{a}_i)+{\veps}_{\bs{a}_i,m}\stg{n}$.
Then, the maximum information gain $\gamma\stg{n}_{m,T}$ is defined as:
\begin{equation}
    \gamma\stg{n}_{m,T}=\max_{\substack{A\stg{n}\subset  \dY\stg{n-1}\times \dX\stg{n} ,\:  |A\stg{n}|=T }} \mathrm{I} (\by\stg{n}_{A,m}; f\stg{n}_m),
\end{equation}
where $\mathrm{I} (\by\stg{n}_{A,m}; f\stg{n}_m)$ is the mutual information between $\by\stg{n}_{A,m}$ and $f\stg{n}_m$.
Furthermore, it is known that this mutual information can be written in closed form as follows~\cite{srinivas2010gaussian}:
\begin{equation}
    \mathrm{I} (\by\stg{n}_{A,m}; f\stg{n}_m) = \frac{1}{2} \log \det \left( \bs{I}_{|A\stg{n}|} +  \sigma^{-2}\bs{K}\stg{n}_{A\stg{n}} \right),
\end{equation}
where $\bs{K}\stg{n}_{A\stg{n}} = \left[ k\stg{n}(\bs{a}\stg{n}_i,\bs{a}\stg{n}_j)\right]_{\bs{a}\stg{n}_i\in A\stg{n},\: \bs{a}\stg{n}_j \in A\stg{n}}$.

Additionally, we define $\mcal{S}_t\stg{n}$ as the set of stocks in stage $n$ at iteration $t$.

\subsection{Proofs of Theorems}
\Cref{theorem:CI} is a special case of~\cref{theorem:Nstage_given_obs} with $M\stg{n}=1$ for all $n$.
Likewise,~\cref{theorem:error_t,theorem:convergence_rt} are corresponding to~\cref{eq:est_opt_bound,theorem:i_regret} with $M\stg{n}=1$, respectively.
The proofs of these theorems are given in the generalized problem setting.
Moreover, we also provide the proof of~\cref{theorem:reduction} in~\cref{cor:delete}.

%-------------------------------------------------------------------------------------------
\section{Prediction of Cascade Processes using Bayesian Quadrature}
\label{app:BQ}
In this section, we consider the cascade process as a Bayesian quadrature \cite{o1991bayes} framework and introduce one of its problems.
For black-box functions at each stage of the cascade process, we consider a predictive model using GP.
The problem is that it is difficult to predict each stage from the first stage because each stage contains controllable variables and outputs from the previous stage that are not controllable.
Nevertheless, the output from the previous stage can be predicted using the posterior distribution.
Therefore, integrating the black-box function of each stage with respect to this posterior distribution, i.e., taking the expectation, allows prediction of each stage with respect to the average case of uncontrollable inputs.
This approach is known as Bayesian quadrature, and furthermore, since each stage follows a GP, it is known that the integration of the black-box function is again a GP (see, e.g., \cite{papoulis2002probability}).
Therefore, the advantage of this approach is that it is easy to construct credible intervals based on the properties of GP.
However, this modeling has the problem that it cannot always correctly predict the target it originally wants to predict.
\begin{lemma}\label{prop:bq}
Suppose that $f_1:\mathbb{R} \to \mathbb{R}$ follows $\mathcal{G} \mathcal{P} (0, k_1(x,x^\prime))$.
Also suppose that
 $f_2 :\mathbb{R}^2 \to \mathbb{R}$ follows
$\mathcal{G} \mathcal{P} (0, k_2((x_1,x_2),(x^\prime_1    ,x^\prime_2))$.
Assume that the first variable of $f_2$ is the output of $f_1$.
Then, the stochastic process $f_2 (f_1 (x_1),x_2)$ is not necessarily the same as
\begin{align}
\mathbb{E}_{f_1 (x_1) \sim  \mathcal{G} \mathcal{P} (0, k_1(x,x^\prime)) }  [ f_2 (f_1 (x_1),x_2)   ] . \label{eq:GP_integral}
\end{align}.
\end{lemma}
\begin{proof}
Let $k_1 (x,x^\prime ) = \exp (-(x-x^\prime)^2) $ and  $k_2 ({\bm y},{\bm y}^\prime ) = {\bm y}^\top {\bm y}^\prime$.
Since the expectation of  GP with respect to inputs is again a GP,
\eqref{eq:GP_integral} follows GP.
Therefore, the probability distribution given by \eqref{eq:GP_integral} at point $x_1=x_2=0$ follows some normal distribution.
On the other hand, since $f_1 (x_1) \sim  \mathcal{G} \mathcal{P} (0, k_1(x,x^\prime))$, from the definition of $k_1 (x,x^\prime)$ we have
$f_1 (0) \sim N(0,1)$.
Similarly, we get $f_2 (f_1 (0),0) \sim N (0,f^2_1 (0) ) \stackrel{\mathrm{d}}{=}	 N (0, \chi^2_1 ) \stackrel{\mathrm{d}}{=}	\sqrt{\chi^2_1 }  N (0, 1 )$,
where  $\chi^2_1$ is the chi-squared distribution with one degree of freedom.
The mean and variance of  $\sqrt{\chi^2_1 }  N (0, 1 )$ are zero and one, respectively.
Furthermore, the fourth moment of  $\sqrt{\chi^2_1 }  N (0, 1 )$ is given by
$$
\mathbb{E} \left [ \left ( \sqrt{\chi^2_1 }  N (0, 1 ) \right )^4   \right ] = \mathbb{E}  [ (\chi^2_1 )^2 ]  \mathbb{E} [N (0, 1 ) )^4] =(\mathbb{V} [\chi^2_1] + \mathbb{E}[\chi^2_1]^2)   3 = 9.
$$
Hence,
 $\sqrt{\chi^2_1 }  N (0, 1 )$ does not follow a normal distribution because the fourth moment of the normal distribution with mean zero and variance one, i.e., the standard normal distribution, is three.
Thus, the stochastic process $f_2 (f_1 (x_1),x_2)$ is not the same as \eqref{eq:GP_integral}.
\end{proof}
Although it is possible to construct a GP prediction model as an integral of GP, the final stage does not necessarily follow GP.
Hence, it is not always easy to judge whether the composition of the credible interval or the design of AF based on the constructed GP prediction model is appropriate or not.
Therefore, modeling the final stage of the cascade process based on the integration of GP is not the most natural approach.

%-------------------------------------------------------------------------------------------
\section{Cascade Process Optimization Using CI-based AFs under Noiseless Setting}
\label{app:CI-based-AF}
In this section, we consider CI-based cascade process optimization methods without observation noise.
\subsection{Credible Interval}
We construct a valid CI for the objective function $F ({\bm x}^{(1)},\ldots, {\bm x}^{(N)})$.
First, we assume the following regularity assumption which is commonly assumed in many BO studies.
\begin{assumption}[Regularity assumption under noiseless setting]\label{assumption:regu1}
  For each $n \in [N]$,
  let $\mathcal{Y}^{(n-1)} \times \mathcal{X}^{(n)}$ be a compact set, and let $\mathcal{H}_{ k^{(n)} } $ be an RKHS corresponding to the kernel $k^{(n)} $.
  In addition, for each $n \in [N]$ and $m \in [M^{(n)}]$, assume that $f^{(n)} _m \in \mathcal{H}_{ k^{(n)} } $
  with $\|f\stg{n} _m  \|_{{k\stg{n}}}\le B$, where $B>0$ is some constant, and $\|\cdot \|_{{k\stg{n}}}$ denotes the RKHS norm on $\mcal{H}_{k\stg{n}}$.
  Furthermore, assume that the observation noise $\epsilon^{(n)}_m $ is zero.
\end{assumption}
Under this assumption, it is known that the following lemma holds.
\begin{lemma}[{\citealt[Theorem~3.11]{szepesvari2012online}}]
  Assume that~\cref{assumption:regu1} holds.
  Define $\beta=B^2$.
  Then, for any $n \in [N]$ and $m \in [M^{(n)}]$, the following inequality holds:
  \begin{equation}
    \left|f\stg{n}_m(\boldsymbol{w}, \boldsymbol{x})-\mu\stg{n}_{m,t}(\boldsymbol{w}, \boldsymbol{x})\right| \leq \beta^{1/2} \sigma\stg{n}_{m,t}(\boldsymbol{w}, \boldsymbol{x}),\:  \forall \boldsymbol{w} \in \mathcal{Y}^{(n-1)},\; \forall \boldsymbol{x} \in \mathcal{X}\stg{n},\; \forall t \geq 1.
  \end{equation}
  \label{lem:ci}
\end{lemma}
Based on~\cref{lem:ci}, we construct the valid CI.
However, we cannot use~\cref{lem:ci} to construct CIs directly because the input ${\bm w} \in \mathcal{Y}^{(n-1)}$ is the output of the previous stage.
In order to avoid this issue, we introduce additional assumptions for Lipschitz continuity.
\begin{assumption}[Lipschitz continuity for $f^{(n)}_m$]\label{assumption:L1}
  Assume that $f\stg{n}_m$ is $L_f$-Lipschitz continuous with respect to $L_1$-distance for any $n \in \{2,\ldots, N\}$ and $m \in [M^{(n)}]$,  where  $L_f>0$ is a Lipschitz constant.
\end{assumption}
\begin{assumption}[Lipschitz continuity for $\sigma^{(n)}_m$]\label{assumption:L2}
  Assume that $\sigma^{(n)}_{m,t} $ is $L_\sigma$-Lipschitz continuous with respect to $L_1$-distance for any $n \in \{2,\ldots, N\}$, $m \in [M^{(n)}]$ and $t \geq 1$, where  $L_\sigma>0$ is a Lipschitz constant.
\end{assumption}
Then, the following theorem gives CIs for the $N$-stage cascade process.
\begin{theorem}[CIs for $N$-stage cascade process]\label{theorem:Nstage}
  Assume that~\cref{assumption:regu1,assumption:L1} hold.
  Define $\beta = B^2$ and
  \begin{align}
    {\bm z}^{(n)} ({\bm x}^{(1)},\ldots,{\bm x}^{(n)} )                & = \begin{cases}
                                                                             {\bm f}^{(1)} ({\bm 0},{\bm x}^{(1)} )                              & (n=1)  ,       \\
                                                                             {\bm f}^{(n)} ({\bm z}^{(n-1)} ({\bm x}^{(1:n-1)} ),{\bm x}^{(n)} ) & (2\le n\le N),
                                                                           \end{cases}                \\
    \tilde{\bm \mu}^{(n)} _t ({\bm x}^{(1)},\ldots,{\bm x}^{(n)} )     & =\begin{cases}
                                                                            {\bm \mu}^{(1)}_t ({\bm 0},{\bm x}^{(1)} )                                       & (n=1) ,        \\
                                                                            {\bm \mu}_t^{(n)} (\tilde{\bm \mu}^{(n-1)} ({\bm x}^{(1:n-1)} ) ,{\bm x}^{(n)} ) & (2\le n\le N),
                                                                          \end{cases}    \\
    \tilde{\sigma }^{(n)}_{m,t}  ({\bm x}^{(1)},\ldots,{\bm x}^{(n)} ) & = \begin{dcases}
                                                                             \sigma^{(1)}_{m,t}  ({\bm 0},{\bm x}^{(1)})                                                & (n=1), \\
                                                                             \begin{aligned}
         & \sigma^{(n)}_{m,t}  (   \tilde{\bm\mu}^{(n-1)}_t ({\bm x}^{(1:n-1)}),{\bm x}^{(n)} ) \\
         & \quad +L_f \sum_{s=1}^{M^{(n-1)}} \tilde{\sigma}^{(n-1)}_{s,t}   ({\bm x}^{(1:n-1)})
      \end{aligned} & (2\le n\le N).
                                                                           \end{dcases}
  \end{align}
  Moreover, assume that $\tilde{\bm \mu}^{(n)} _t ({\bm x}^{(1)},\ldots,{\bm x}^{(n)} ) \in \mathcal{Y}^{(n)}$
  for any $n \in [N]$,  $t \geq 1$ and $({\bm x}^{(1) } , \ldots, {\bm x}^{(n) }  ) \in \mathcal{X}^{(1)} \times \cdots \times \mathcal{X}^{(n)}$.
  Then,
  it follows that
  \begin{equation}
    |{ z}_m^{(n)} ({\bm x}^{(1)},\ldots,{\bm x}^{(n)} )  - \tilde{\mu}^{(n)}_{m,t} ({\bm x}^{(1) } , \ldots, {\bm x}^{(n) }  )|  \leq \beta^{1/2}  \tilde{\sigma}^{(n)}_{m,t} ({\bm x}^{(1)}, \ldots, {\bm x}^{(n)} )  ,
  \end{equation}
  where $m \in [M^{(n)}]$, and $z^{(n)}_m (\cdot)$ and  $\tilde{\mu}^{(n)}_{m,t} (\cdot)$  are the $m$-th element of
  ${\bm z}^{(n)}_m (\cdot)$ and
  $\tilde{\bm\mu}^{(n)}_{m,t} (\cdot )$, respectively.
  In particular, when $n=N$, it follows that
  \begin{equation}
    |F({\bm x}^{(1)},\ldots,{\bm x}^{(N)} ) - \tilde{\mu}^{(N)}_{1,t} ({\bm x}^{(1) } , \ldots, {\bm x}^{(N) }  )|  \leq \quad \beta^{1/2}  \tilde{\sigma}^{(N)}_{1,t} ({\bm x}^{(1)}, \ldots, {\bm x}^{(N)} )  .
  \end{equation}
\end{theorem}
\begin{proof}
  Fix
  ${\bm x}^{(1)},\ldots, {\bm x}^{(n)}$, $t \geq 1$ and $m \in [M^{(n)}]$.
  For simplicity, hereafter, we sometimes omit  the notation $({\bm x}^{(1)},\ldots,{\bm x}^{(n)} )$ such as
  $z^{(n)}_m$ and
  $\tilde{\mu}^{(n)}_{m,t}$.
  Then, for $i \in [M^{(2)}]$, it follows that
  \begin{align}
    |  z^{(2)}_i -\tilde{\mu}^{(2)}_{i,t} | & =  | z^{(2)}_i - f^{(2)}_i (\tilde{\bm \mu}^{(1)}_t,{\bm x}^{(2)} ) +  f^{(2)}_i (\tilde{\bm \mu}_t^{(1)},{\bm x}^{(2)} )   -\tilde{\mu}^{(2)}_{i,t} |                                                                                                    \\
                                            & \leq  | z^{(2)}_i - f^{(2)}_i (\tilde{\bm \mu}_t^{(1)},{\bm x}^{(2)} ) |+|  f^{(2)}_i (\tilde{\bm \mu}_t^{(1)},{\bm x}^{(2)} )   -\tilde{\mu}^{(2)}_{i,t}     |                                                                                           \\
                                            & =  | z^{(2)}_i ({\bm z}^{(1)},{\bm x}^{(2)})- f^{(2)}_i (\tilde{\bm \mu}_t^{(1)},{\bm x}^{(2)} ) |   + |  f^{(2)}_i (\tilde{\bm \mu}_t^{(1)},{\bm x}^{(2)} )    -                   {\mu}^{(2)}_{i,t} (\tilde{\bm\mu}^{(1)}_t ,{\bm x}^{(2)})           | \\
                                            & \leq L_f \| {\bm z}^{(1)} -\tilde{\bm \mu}^{(1)}_t \| _1 + \beta^{1/2} \sigma ^{(2)}_{i,t} (\tilde{\bm \mu}_t^{(1)},{\bm x}^{(2)} )                                                                                                                       \\
                                            & = \beta^{1/2} \sigma ^{(2)}_{i,t} (\tilde{\bm \mu}_t^{(1)},{\bm x}^{(2)} )+ L_f \sum_{m=1}^{M^{(1)}} | z^{(1)}_m -\tilde{\mu}^{(1)}_{m,t}  |                                                                                                              \\
                                            & \leq \beta^{1/2} \sigma ^{(2)}_{i,t} (\tilde{\bm \mu}_t^{(1)},{\bm x}^{(2)} )+ L_f \sum_{m=1}^{M^{(1)}} \beta^{1/2} \sigma^{(1)}_{m,t} ({\bm 0},{\bm x}^{(1)} )                                                                                           \\
                                            & = \beta^{1/2} \tilde{\sigma}^{(2)} _{i,t} ({\bm x}^{(1)},{\bm x}^{(2) } ). \label{eq:2bound}
  \end{align}
  Similarly,  $z^{(3)}_j $ and $\tilde{\mu}^{(3)}_{j,t}$ satisfy that
  \begin{align}
    |  z^{(3)}_j -\tilde{\mu}^{(3)}_{j,t} | & =  | z^{(3)}_j - f^{(3)}_j (\tilde{\bm \mu}_t^{(2)},{\bm x}^{(3)} ) +  f^{(3)}_j (\tilde{\bm \mu}_t^{(2)},{\bm x}^{(3)} )   -\tilde{\mu}^{(3)}_{j,t} |            \\
                                            & \leq  | z^{(3)}_j - f^{(3)}_j (\tilde{\bm \mu}_t^{(2)},{\bm x}^{(3)} ) |+|  f^{(3)}_j (\tilde{\bm \mu}_t^{(2)},{\bm x}^{(3)} )   -\tilde{\mu}^{(3)}_{j,t} |       \\
                                            & \leq L_f \| {\bm z}^{(2)} -\tilde{\bm \mu}_t^{(2)} \| _1 + \beta^{1/2} \sigma ^{(3)}_{j,t} (\tilde{\bm \mu}_t^{(2)},{\bm x}^{(3)} )                               \\
                                            & = \beta^{1/2} \sigma ^{(3)}_{j,t} (\tilde{\bm \mu}_t^{(2)},{\bm x}^{(3)} )+ L_f \sum_{i=1}^{M^{(2)}} | z^{(2)}_i -\tilde{\mu}^{(2)}_{i,t}  | . \label{eq:3bound1}
  \end{align}
  Hence, by substituting~\cref{eq:2bound} into~\cref{eq:3bound1}, we get
  \begin{equation}
    \begin{aligned}
      |  z^{(3)}_j -\tilde{\mu}^{(3)}_{j,t} | & \leq \beta ^{1/2} \left ( \sigma ^{(3)}_{j,t} (\tilde{\bm \mu}_t^{(2)},{\bm x}^{(3)} )+ L_f \sum_{u=1}^{M^{(2)}} \tilde{\sigma}^{(2)}_{u,t}  ({\bm x}^{(1)} ,{\bm x}^{(2)} ) \right ) \\
                                              & = \beta ^{1/2} \tilde {\sigma}^{(3)}_{j,t} ({\bm x}^{(1)},{\bm x}^{(2)},{\bm x} ^{(3) } ).
    \end{aligned}
  \end{equation}
  By repeating this process up to $n$, we have~\cref{theorem:Nstage}.
\end{proof}
From~\cref{theorem:Nstage}, we can construct the valid CI $Q^{(F)}_t ({\bm x}^{(1)},\ldots, {\bm x}^{(N)} ) $ of
$F({\bm x}^{(1)},\ldots, {\bm x}^{(N)} ) $ as follows:
\begin{align}
  Q^{(F)}_t  (  {\bm x}^{(1)},\ldots, {\bm x}^{(N)} ) & =          [ \tilde{\mu}^{(N)}_{1,t} ({\bm x}^{(1) } , \ldots, {\bm x}^{(N) }  ) \pm \beta^{1/2} \tilde{\sigma}^{(N)}_{1,t} ({\bm x}^{(1)}, \ldots, {\bm x}^{(N)} )]                     \\
                                                      & =          [  \mr{ LCB}^{(F)}_t ({\bm x}^{(1) } , \ldots, {\bm x}^{(N) }  ),   \mr{ UCB}^{(F)}_t ({\bm x}^{(1) } , \ldots, {\bm x}^{(N) }  )                ] .   \label{eq:proposed_ci}
\end{align}

Next, we consider the property of estimated solutions based on the proposed CI~\cref{eq:proposed_ci}.
For any $t \geq 1$, we define the estimated solution $( \hat{\bm x}^{(1)}_t, \ldots, \hat{\bm x}^{(N)}_t )$ as
\begin{equation}
  (  \hat{\bm x}^{(1)}_t, \ldots, \hat{\bm x}^{(N)}_t )   = \argmax _{  ({\bm x}^{(1)},\ldots,{\bm x}^{(N)} ) \in \mathcal{X}       ,1\leq \tilde{t} \leq t  } \mr{ LCB} ^{(F)}_{\tilde{t}}  ({\bm x}^{(1)},\ldots,{\bm x}^{(N)} ).
  \label{eq:estimated_solution}
\end{equation}
Then, the following theorem holds.
\begin{theorem}\label{eq:est_opt_bound}
  Let $( \hat{\bm x}^{(1)}_t, \ldots, \hat{\bm x}^{(N)}_t )$ be the estimated solution given by~\cref{eq:estimated_solution}.
  Assume that the same assumption as in~\cref{theorem:Nstage} holds.
  Then, for any $t \geq 1$ and $\xi >0$, it follows that
  \begin{align}
     & \max _{  {\bm x}^{(1:N)} \in \mathcal{X}   } \mr{ UCB} ^{(F)}_t  ({\bm x}^{(1:N)} )   -  \max _{  {\bm x}^{(1:N)} \in \mathcal{X}   } \mr{ LCB} ^{(F)}_t  ({\bm x}^{(1:N)} ) < \xi \\
     & \Rightarrow F ({\bm x}^{(1)}_\ast , \ldots , {\bm x}^{(N)}_\ast ) - F( \hat{\bm x}^{(1)}_t, \ldots, \hat{\bm x}^{(N)}_t )  <\xi.
  \end{align}
\end{theorem}
\begin{proof}
  From the definition of CIs, using~\cref{theorem:Nstage} we have
  \begin{align}
    F ({\bm x}^{(1)}_\ast , \ldots , {\bm x}^{(N)}_\ast )                           & \leq \mr{ UCB} ^{(F)}_t  ({\bm x}^{(1)}_\ast , \ldots , {\bm x}^{(N)}_\ast ), \\
    \mr{ LCB} ^{(F)}_{\hat{t}} ( \hat{\bm x}^{(1)}_t, \ldots, \hat{\bm x}^{(N)}_t ) & \leq F ( \hat{\bm x}^{(1)}_t, \ldots, \hat{\bm x}^{(N)}_t ),
  \end{align}
  where  \sloppy$\hat{t}= \argmax _{  ({\bm x}^{(1)},\ldots,{\bm x}^{(N)} ) \in \mathcal{X}       ,1\leq \tilde{t} \leq t  } \mr{ LCB} ^{(F)}_{\tilde{t}}  ({\bm x}^{(1:N)} ).
  $
  Similarly, from the definition of $ ( \hat{\bm x}^{(1)}_t, \ldots, \hat{\bm x}^{(N)}_t )$,
  noting that
  \begin{equation}
    \max _{  {\bm x}^{(1:N)}\in \mathcal{X}   } \mr{ LCB} ^{(F)}_t  ({\bm x}^{(1:N)} ) \leq
    \mr{ LCB} ^{(F)}_{\hat{t}} ( \hat{\bm x}^{(1:N)}_t, \ldots, \hat{\bm x}^{(N)}_t ) ,
  \end{equation}
  we get
  \begin{alignat}{2}
    F ({\bm x}^{(1)}_\ast , \ldots , {\bm x}^{(N)}_\ast )                                 & \leq \mr{ UCB} ^{(F)}_t  ({\bm x}^{(1)}_\ast , \ldots , {\bm x}^{(N)}_\ast )         &  & \leq \max  _{  {\bm x}^{(1:N)}\in \mathcal{X}   }  \mr{ UCB} ^{(F)}_t  ({\bm x}^{(1:N)}) , \\
    \max  _{  {\bm x}^{(1:N)}  \in \mathcal{X}   }  \mr{ LCB} ^{(F)}_t  ({\bm x}^{(1:N)}) & \leq \mr{ LCB} ^{(F)}_{\hat{t}} ( \hat{\bm x}^{(1)}_t, \ldots, \hat{\bm x}^{(N)}_t ) &  & \leq F ( \hat{\bm x}^{(1)}_t, \ldots, \hat{\bm x}^{(N)}_t ).
  \end{alignat}
  This implies that
  \begin{align}
     & F ({\bm x}^{(1)}_\ast , \ldots , {\bm x}^{(N)}_\ast )  - F ( \hat{\bm x}^{(1)}_t, \ldots, \hat{\bm x}^{(N)}_t )                                                                         \\
     & \leq    \max _{  {\bm x}^{(1:N)} \in \mathcal{X}   }   \mr{ UCB} ^{(F)}_t  ({\bm x}^{(1:N)} ) -  \max  _{  {\bm x}^{(1:N)} \in \mathcal{X}   }  \mr{ LCB} ^{(F)}_t  ({\bm x}^{(1:N)} ).
    \label{eq:UCB_LCB_diff}
  \end{align}
  Therefore, by combining~\cref{eq:UCB_LCB_diff} and
  \begin{equation}
    \max _{  {\bm x}^{(1:N)} \in \mathcal{X}   } \mr{ UCB} ^{(F)}_t  ({\bm x}^{(1:N)} ) -  \max _{  ({\bm x}^{(1:N)} ) \in \mathcal{X}   } \mr{ LCB} ^{(F)}_t  ({\bm x}^{(1:N)} ) < \xi ,
  \end{equation}
  we get~\cref{eq:est_opt_bound}.
\end{proof}

Finally, we consider the construction of CIs when the observations up to the $s$-th stage are given.
Let $s$ be an integer with $0 \leq s \leq N-1$, and let ${\bm y}$ be an element of $\mathcal{Y}^{(s)}$.
Then, for each $n \in \{ s+1,\ldots, N \}$, $m \in [M^{(n)}]$, $t \geq 1$ and ${\bm x}^{(s+1)},\ldots,{\bm x}^{(n)}$,
we define ${\bm z}^{(n)} ({\bm x}^{(s+1)},\ldots,{\bm x}^{(n)} | {\bm y})$, $\tilde{\bm \mu}^{(n)} _t ({\bm x}^{(s+1)},\ldots,{\bm x}^{(n)}|{\bm y} )$ and
$\tilde{\sigma }^{(n)}_{m,t}  ({\bm x}^{(s+1)},\ldots,{\bm x}^{(n)} |{\bm y})$ as
\begin{align}
  {\bm z}^{(n)} ({\bm x}^{(s+1)},\ldots,{\bm x}^{(n)} |{\bm y})                & = \begin{cases}
                                                                                     {\bm f}^{(s+1)} ({\bm y},{\bm x}^{(s+1)} )                                    & (n=s+1),    \\
                                                                                     {\bm f}^{(n)} ({\bm z}^{(n-1)} ({\bm x}^{(s+1:n-1)} |{\bm y}),{\bm x}^{(n)} ) & (n\ge s+2),
                                                                                   \end{cases}             \\
  \tilde{\bm \mu}^{(n)} _t ({\bm x}^{(s+1)},\ldots,{\bm x}^{(n)} |{\bm y})     & =\begin{cases}
                                                                                    {\bm \mu}^{(s+1)}_t ({\bm y},{\bm x}^{(s+1)} )                                             & (n=s+1),    \\
                                                                                    {\bm \mu}_t^{(n)} (\tilde{\bm \mu}^{(n-1)} ({\bm x}^{(s+1:n-1)} |{\bm y}) ,{\bm x}^{(n)} ) & (n\ge s+2),
                                                                                  \end{cases} \\
  \tilde{\sigma }^{(n)}_{m,t}  ({\bm x}^{(s+1)},\ldots,{\bm x}^{(n)} |{\bm y}) & = \begin{cases}
                                                                                     \sigma^{(s+1)}_{m,t}  ({\bm y},{\bm x}^{(s+1)}) & (n=s+1) ,   \\
                                                                                     \begin{aligned}
       & \sigma^{(n)}_{m,t}  (   \tilde{\bm\mu}^{(n-1)}_t ({\bm x}^{(s+1:n-1)} |{\bm y}),{\bm x}^{(n)} ) \\
       & \:\: +L_f \sum_{u=1}^{M^{(n-1)}} \tilde{\sigma}^{(n-1)}_{u,t}   ({\bm x}^{(s+1:n-1)} |{\bm y})
    \end{aligned}
                                                                                                                                     & (n\ge s+2).
                                                                                   \end{cases}
\end{align}
Moreover, we formally define
${\bm z}^{(s)} ({\bm x}^{(s+1)},{\bm x}^{(s)} | {\bm y}) = \tilde{\bm \mu}^{(s)} _t ({\bm x}^{(s+1)},{\bm x}^{(s)}|{\bm y} ) ={\bm y}$ and
$\tilde{\sigma }^{(s)}_{m,t}  ({\bm x}^{(s+1)},{\bm x}^{(s)} |{\bm y})=0$.
Then, the following theorem holds.
\begin{theorem}[CIs for $N$-stage cascade process under given observation]\label{theorem:Nstage_given_obs}
  Assume that~\cref{assumption:regu1,assumption:L1} hold.
  Define $\beta = B^2$, and
  assume that $\tilde{\bm \mu}^{(n)} _t ({\bm x}^{(s+1)},\ldots,{\bm x}^{(n)} |{\bm y}) \in \mathcal{Y}^{(n)}$
  for any $s \in \{0,\ldots, N-1\}$,
  $n \in \{s+1,\ldots, N \}$,  ${\bm y} \in \mathcal{Y}^{(s)}$, $t \geq 1$ and $({\bm x}^{(s+1) } , \ldots, {\bm x}^{(n) }  ) \in \mathcal{X}^{(s+1)} \times \cdots \times \mathcal{X}^{(n)}$.
  Then,
  it follows that
  \begin{equation}
    |{ z}_m^{(n)} ({\bm x}^{(s+1)},\ldots,{\bm x}^{(n)} |{\bm y})  - \tilde{\mu}^{(n)}_{m,t} ({\bm x}^{(s+1) } , \ldots, {\bm x}^{(n) }  |{\bm y})| \leq \beta^{1/2}  \tilde{\sigma}^{(n)}_{m,t} ({\bm x}^{(s+1)}, \ldots, {\bm x}^{(n)} |{\bm y})  ,
  \end{equation}
  where $m \in [M^{(n)}]$, and $z^{(n)}_m (\cdot |{\bm y})$ and  $\tilde{\mu}^{(n)}_{m,t} (\cdot |{\bm y})$  are the $m$-th element of
  ${\bm z}^{(n)}_m (\cdot |{\bm y})$ and
  $\tilde{\bm\mu}^{(n)}_{m,t} (\cdot |{\bm y})$, respectively.
\end{theorem}
\begin{proof}
  By using the same argument as in the proof of~\cref{theorem:Nstage}, we get~\cref{theorem:Nstage_given_obs}.
\end{proof}
Based on~\cref{theorem:Nstage_given_obs}, we give a stock reduction rule.
For each $t \geq 1$ and ${\bm y} \in \mathcal{Y}^{(s)}$ with $0 \leq s \leq N-1$, we define $F({\bm y} )$, $\mr{ LCB}^{(F)}_t ({\bm y} )$ and $ \mr{ UCB}^{(F)}_t ({\bm y} )$
as
\begin{align}
  F ({\bm y} )                 & = \max_{  {\bm x}^{(s+1)}  \cdots  {\bm x}^{(N) }     } {\bm z}^{(N)} ({\bm x}^{(s+1)}, \ldots,{\bm x}^{(N)} |{\bm y}),                                                                                                           \\
  \mr{ LCB}^{(F)}_t ({\bm y} ) & =  \max_{  {\bm x}^{(s+1)} \cdots {\bm x}^{(N) }     } (\tilde{\mu}^{(N)}_{1,t} ({\bm x}^{(s+1)},\ldots, {\bm x}^{(N)}   |{\bm y}) - \beta^{1/2} \tilde{\sigma}^{(N)}_{1,t} ({\bm x}^{(s+1)},\ldots,{\bm x}^{(N)} |{\bm y}  ) ),  \\
  \mr{ UCB}^{(F)}_t ({\bm y} ) & =  \max_{  {\bm x}^{(s+1)} \cdots  {\bm x}^{(N) }      } (\tilde{\mu}^{(N)}_{1,t} ({\bm x}^{(s+1)},\ldots, {\bm x}^{(N)}  |{\bm y}) + \beta^{1/2} \tilde{\sigma}^{(N)}_{1,t} ({\bm x}^{(s+1)},\ldots,{\bm x}^{(N)}  |{\bm y} ) ),
\end{align}
where $\tilde{\mu}^{(N)}_{1,t} ({\bm x}^{(s+1)},\ldots, {\bm x}^{(N)}  |{\bm y} ) $ is the first element of $\tilde{\bm\mu}^{(N)}_{t} ({\bm x}^{(s+1)},\ldots, {\bm x}^{(N)}  |{\bm y} ) $.
Then, the following corollary holds.
\begin{corollary}[Stock reduction]\label{cor:delete}
  Assume that the same assumption as in~\cref{theorem:Nstage_given_obs} holds.
  Let $t \geq 1$, and let $\mathcal{S}^{(u)} _t$ be a set of stocks at stage $u \in \{0,\ldots, N-1 \}$ in iteration $t$.
  Assume that an element ${\bm y}$ in $\mathcal{S} ^{(s)}_t $ satisfies
  \begin{equation}
    \mr{ UCB}^{(F)}_t ({\bm y} )  <  \max_{  \tilde{\bm  y} \in \bigcup _{u=0}^{N-1}  \mathcal{S}^{(u)} _t  }   \mr{ LCB}^{(F)}_t (\tilde{\bm y} ) .
    \label{eq:delete}
  \end{equation}
  Then, it follows that $F({\bm y} )  < F({\bm x}^{(1)}_\ast ,\ldots, {\bm x}^{(N)}_\ast ) $.
\end{corollary}
\begin{proof}
  From~\cref{theorem:Nstage,theorem:Nstage_given_obs},
  noting that  ${\bm y} \in \mathcal{S}^{(u)}_t$ is the observed value corresponding to some input, it follows that
  \begin{align}
    F({\bm y} ) & \leq \mr{ UCB}^{(F)}_t ({\bm y} )                                                                                                   \\
                & <  \max_{  \tilde{\bm  y} \in \bigcup _{u=0}^{N-1}  \mathcal{S}^{(u)} _t  }   \mr{ LCB}^{(F)}_t (\tilde{\bm y} )                    \\
                & \leq \max_ { ({\bm x}^{(1)},\ldots,{\bm x}^{(N)} ) \in \mathcal{X} } {\bm z}^{(N)}    ({\bm x}^{(1)},\ldots,{\bm x}^{(N)} |{\bm 0}) \\
                & =    F({\bm x}^{(1)}_\ast ,\ldots,{\bm x}^{(N)}_\ast ).
  \end{align}
\end{proof}

%-------------------------------------------------------------------------------------------
\subsection{Cascade Process Upper Confidence Bound}
\label{subsec:cUCB}
Here, we consider a UCB-based optimization strategy, and give a
cascade process upper confidence bound (cUCB) AF.
For each iteration $t \geq 1$ and input $({\bm x}^{(1)},\ldots,{\bm x}^{(N)})$,
we define cUCB as
\begin{equation}
  \mr{ cUCB}_t ({\bm x}^{(1)},\ldots, {\bm x}^{(N) } )  =   \tilde{\mu}^{(N)}_{1,t} ({\bm x}^{(1)},\ldots, {\bm x}^{(N) } )    + \beta^{1/2} \tilde{\sigma}^{(N)}_{1,t} ({\bm x}^{(1)},\ldots, {\bm x}^{(N) } ) .
\end{equation}

Next, we consider the theoretical property of cUCB.
Suppose that the next evaluation point is selected by
\begin{equation}
  ({\bm x}^{(1)}_{t+1}, \ldots , {\bm x}^{(N)}_{t+1} ) = \argmax _{ ({\bm x}^{(1)}, \ldots , {\bm x}^{(N)} )   \in \mathcal{X} } \mr{ cUCB}_t ({\bm x}^{(1)},\ldots, {\bm x}^{(N) } ) .
  \label{eq:decision_cUCB}
\end{equation}
Moreover, in order to evaluate the goodness of the optimization strategy, we introduce the regret
$r_t$,  cumulative regret $R_T$ and simple regret $r^{(\text{S})}_T$ as
\begin{gather}
  r_t               =  F({\bm x}^{(1)}_\ast, \ldots , {\bm x}^{(N)}_\ast )  -  F({\bm x}^{(1)}_{t}, \ldots , {\bm x}^{(N)}_{t} ), \\
  R_T               = \sum_{t=1}^T r_t,     \quad
  r^{(\text{S})}_T  =  \min _{1 \leq t \leq T}  r_t.
\end{gather}
Then, the following theorem gives regret bounds for  $R_T$ and   $r^{(\text{S})}_T$.
\begin{theorem}
  Assume that~\cref{assumption:regu1,assumption:L1,assumption:L2} hold.
  Define $\beta =B^2$, and
  assume that $\tilde{\bm \mu}^{(n)} _t ({\bm x}^{(s+1)},\ldots,{\bm x}^{(n)} |{\bm y}) \in \mathcal{Y}^{(n)}$
  for any $s \in \{0,\ldots, N-1\}$,
  $n \in \{s+1,\ldots, N \}$,  ${\bm y} \in \mathcal{Y}^{(s)}$, $t \geq 1$ and $({\bm x}^{(s+1) } , \ldots, {\bm x}^{(n) }  ) \in \mathcal{X}^{(s+1)} \times \cdots \times \mathcal{X}^{(n)}$.
  Then, when the optimization is performed using cUCB, the following inequality holds for any $T \geq 1$:
  \begin{align}
    R_T              & \leq \sqrt{\frac{   8  \beta   C _0 ^ {2(N-1)}  M^2_{\text{prod}} M^2_{\text{sum}}          }{ \log (1+\sigma^{-2} ) }  T \gamma_T }  ,         \\
    r^{(\text{S})}_T & \leq T^{-1/2}  \sqrt{\frac{   8  \beta   C _0 ^ {2(N-1)}  M^2_{\text{prod}} M^2_{\text{sum}}          }{ \log (1+\sigma^{-2} ) }   \gamma_T } ,
  \end{align}
  where $C_0 =L_\sigma \beta^{1/2} +L_f +1$, $M_{\text{prod}} = \prod_{n=1}^N M^{(n)}$ and
  $M_{\text{sum}} = \sum_{n=1}^N M^{(n)}$.
\end{theorem}
\begin{proof}
  From the definition of $\tilde{\sigma}^{(n)}_{m,t} (\cdot) $, using  Lipschitz continuity of ${\sigma}^{(n)}_{m,t} (\cdot)$  we have
  \begin{align}
    \tilde{\sigma}^{(n)}_{m,t} & ({\bm x}^{(1)},\ldots, {\bm x}^{(n)} )                                                                                                                                                                                             \\
                               & = \sigma^{(n)} _{m,t} ({\bm z}^{(n-1)} ({\bm x}^{(1)},\ldots , {\bm x}^{(n-1)} ),{\bm x}^{(n)} )  +L_f \sum_{s=1}^{M^{(n-1)} }  \tilde{\sigma}^{(n-1)}_{s,t}  ({\bm x}^{(1)},\ldots,{\bm x}^{(n-1) } )                             \\
                               & \quad  + \sigma^{(n)} _{m,t} (\tilde{\bm \mu}_t^{(n-1)} ({\bm x}^{(1)},\ldots , {\bm x}^{(n-1)} ),{\bm x}^{(n)} )      -  \sigma^{(n)} _{m,t} ({\bm z}^{(n-1)} ({\bm x}^{(1)},\ldots , {\bm x}^{(n-1)} ),{\bm x}^{(n)} )           \\
                               & \leq  \sigma^{(n)} _{m,t} ({\bm z}^{(n-1)} ({\bm x}^{(1)},\ldots , {\bm x}^{(n-1)} ),{\bm x}^{(n)} ) +L_f \sum_{s=1}^{M^{(n-1)} }  \tilde{\sigma}^{(n-1)}_{s,t}  ({\bm x}^{(1)},\ldots,{\bm x}^{(n-1) } )                          \\
                               & \quad    +  |\sigma^{(n)} _{m,t} ({\bm z}^{(n-1)} ({\bm x}^{(1)},\ldots , {\bm x}^{(n-1)} ),{\bm x}^{(n)} )     - \sigma^{(n)} _{m,t} (\tilde{\bm \mu}_t^{(n-1)} ({\bm x}^{(1)},\ldots , {\bm x}^{(n-1)} ),{\bm x}^{(n)} )|        \\
                               & \leq  \sigma^{(n)} _{m,t} ({\bm z}^{(n-1)} ({\bm x}^{(1)},\ldots , {\bm x}^{(n-1)} ),{\bm x}^{(n)} )  +L_f \sum_{s=1}^{M^{(n-1)} }  \tilde{\sigma}^{(n-1)}_{s,t}  ({\bm x}^{(1)},\ldots,{\bm x}^{(n-1) } )                         \\
                               & \quad   + L_\sigma \sum_{s=1} ^{M^{(n-1)}}   |   { z}_s^{(n-1)} ({\bm x}^{(1)},\ldots , {\bm x}^{(n-1)} )    -     \tilde{ \mu}_{s,t}^{(n-1)} ({\bm x}^{(1)},\ldots , {\bm x}^{(n-1)} )      |                                     \\
                               & \leq    \sigma^{(n)} _{m,t} ({\bm z}^{(n-1)} ({\bm x}^{(1)},\ldots , {\bm x}^{(n-1)} ),{\bm x}^{(n)} )  +L_f \sum_{s=1}^{M^{(n-1)} }  \tilde{\sigma}^{(n-1)}_{s,t}  ({\bm x}^{(1)},\ldots,{\bm x}^{(n-1) } )                       \\
                               & \quad     + L_\sigma \beta^{1/2} \sum_{s=1} ^{M^{(n-1)}}    \tilde{\sigma}^{(n-1)}_{s,t}  ({\bm x}^{(1)},\ldots,{\bm x}^{(n-1) } )                                                                                                 \\
                               & =\sigma^{(n)} _{m,t} ({\bm z}^{(n-1)} ({\bm x}^{(1)},\ldots , {\bm x}^{(n-1)} ),{\bm x}^{(n)} )   + (L_\sigma \beta^{1/2} +L_f ) \sum_{s=1} ^{M^{(n-1)}}    \tilde{\sigma}^{(n-1)}_{s,t}  ({\bm x}^{(1)},\ldots,{\bm x}^{(n-1) } ) \\
                               & \leq   \sigma^{(n)} _{m,t} ({\bm z}^{(n-1)} ({\bm x}^{(1)},\ldots , {\bm x}^{(n-1)} ),{\bm x}^{(n)} )   + C_0 \sum_{s=1} ^{M^{(n-1)}}    \tilde{\sigma}^{(n-1)}_{s,t}  ({\bm x}^{(1)},\ldots,{\bm x}^{(n-1) } ). \label{eq:Ctilde}
  \end{align}
  Thus, by repeating the same argument as~\cref{eq:Ctilde} up to $N$, we get
  \begin{align}
     & \tilde{\sigma}^{(N)}_{1,t} ({\bm x}^{(1)},\ldots, {\bm x}^{(N)} )                                                                                                                                                   \\
     & \leq\sigma^{(N)} _{1,t} ({\bm z}^{(N-1)} ({\bm x}^{(1)},\ldots , {\bm x}^{(N-1)} ),{\bm x}^{(N)} )  +C_0 \sum_{s=1} ^{M^{(N-1)}}    \tilde{\sigma}^{(N-1)}_{s,t}  ({\bm x}^{(1)},\ldots,{\bm x}^{(N-1) } )          \\
     & \leq\sigma^{(N)} _{1,t} ({\bm z}^{(N-1)} ({\bm x}^{(1)},\ldots , {\bm x}^{(N-1)} ),{\bm x}^{(N)} )                                                                                                                  \\
     & \quad+C_0 \sum_{s=1} ^{M^{(N-1)}}   \sigma^{(N-1)} _{s,t} ({\bm z}^{(N-2)} ({\bm x}^{(1)},\ldots , {\bm x}^{(N-2)} ),{\bm x}^{(N-1)} )                                                                              \\
     & \quad +C^2_0  M^{(N-1)}  \sum_{u=1} ^{M^{(N-2)}}    \tilde{\sigma}^{(N-2)}_{u,t}  ({\bm x}^{(1)},\ldots,{\bm x}^{(N-2) } )                                                                                          \\
     & \leq                                                                                                                                                                                                                \\
     & \vdots                                                                                                                                                                                                              \\
     & \leq \sigma^{(N)} _{1,t} ({\bm z}^{(N-1)} ({\bm x}^{(1)},\ldots , {\bm x}^{(N-1)} ),{\bm x}^{(N)} )                                                                                                                 \\
     & \quad+ \sum _{n=1}^{N-1}  C _0 ^ {N-n}   \prod _{s= n+1} ^N   M^{(s)}  \sum_{m=1} ^ {M^{(n)} }\Bigl[    \sigma^{(n)} _{m,t} ({\bm z}^{(n-1)} ({\bm x}^{(1)},\ldots , {\bm x}^{(n-1)} ),{\bm x}^{(n)} )       \Bigr] \\
     & \leq C _0 ^ {N-1}  M_{\text{prod}}  \sum _{n=1}^{N}   \sum_{m=1} ^ {M^{(n)} }   \sigma^{(n)} _{m,t} ({\bm z}^{(n-1)} ({\bm x}^{(1)},\ldots , {\bm x}^{(n-1)} ),{\bm x}^{(n)} ).
  \end{align}
  In addition, using the Cauchy--Schwarz inequality, it follows that
  \begin{align}
     & \tilde{\sigma}^{(N)2}_{1,t} ({\bm x}^{(1)},\ldots, {\bm x}^{(N)} )                                                                                                        \\
     & \leq C _0 ^ {2(N-1)}  M^2_{\text{prod}}   \left (  \sum _{n=1}^N  \sum_{m=1} ^ {M^{(n)} }  1  \right )    \left (  \sum _{n=1}^N  \sum_{m=1} ^ {M^{(n)} }  \sigma^{(n) 2} _{m,t} ({\bm z}^{(n-1)} ({\bm x}^{(1)},\ldots , {\bm x}^{(n-1)} ),{\bm x}^{(n)} ) \right ) \\
     & = C _0 ^ {2(N-1)}  M^2_{\text{prod}} M_{\text{sum}}  \cdot \sum _{n=1}^N  \sum_{m=1} ^ {M^{(n)} }  \sigma^{(n) 2} _{m,t} ({\bm z}^{(n-1)} ({\bm x}^{(1)},\ldots , {\bm x}^{(n-1)} ),{\bm x}^{(n)} ) . \label{eq:C_S}
  \end{align}
  Moreover, from~\cref{theorem:Nstage} and the selection rule~\cref{eq:decision_cUCB},
  $F({\bm x}^{(1)}_\ast,\ldots,{\bm x}^{(N)}_\ast ) $ can be bounded as follows:
  \begin{align}
    F({\bm x}^{(1)}_\ast,\ldots,{\bm x}^{(N)}_\ast ) & \leq \mr{ cUCB}_t ({\bm x}^{(1)}_\ast,\ldots,{\bm x}^{(N)}_\ast )                                    \\
                                                     & \leq \mr{ cUCB}_t ({\bm x}^{(1)}_{t+1},\ldots,{\bm x}^{(N)}_{t+1} )                                  \\
                                                     & = \tilde{\mu}^{(N)}_{1,t} ({\bm x}^{(1)}_{t+1},\ldots,{\bm x}^{(N)}_{t+1} )    + \beta^{1/2} \tilde{\sigma}^{(N)}_{1,t} ({\bm x}^{(1)}_{t+1},\ldots,{\bm x}^{(N)}_{t+1} )  .
  \end{align}
  Similarly, since $F ({\bm x}^{(1)}_{t+1},\ldots,{\bm x}^{(N)}_{t+1} ) $ can be bounded as
  \begin{equation}
     F ({\bm x}^{(1)}_{t+1},\ldots,{\bm x}^{(N)}_{t+1} )   \geq \tilde{\mu}^{(N)}_{1,t} ({\bm x}^{(1)}_{t+1},\ldots,{\bm x}^{(N)}_{t+1} )   - \beta^{1/2} \tilde{\sigma}^{(N)}_{1,t} ({\bm x}^{(1)}_{t+1},\ldots,{\bm x}^{(N)}_{t+1} ) ,
  \end{equation}
  we get
  \begin{align}
    r_t & = F({\bm x}^{(1)}_\ast,\ldots,{\bm x}^{(N)}_\ast )  -F ({\bm x}^{(1)}_{t+1},\ldots,{\bm x}^{(N)}_{t+1} ) \\
        & \leq 2 \beta^{1/2} \tilde{\sigma}^{(N)}_{1,t} ({\bm x}^{(1)}_{t+1},\ldots,{\bm x}^{(N)}_{t+1} ) .
    \label{eq:regret_bound}
  \end{align}
  Here, from the Cauchy--Schwarz inequality,
  $R^2_T$ can  be evaluated as
  \begin{equation}
    R^2_T = \left ( \sum _{t=1}^T r_t \right )^2 \leq T \sum _{t=1} ^T r^2_t. \label{eq:CR}
  \end{equation}
  Hence, by combining~\cref{eq:C_S,eq:regret_bound} we have
  \begin{align}
    \sum_{t=1}^T r^2_t & \leq 4 \sum_{t=1}^T \left ( \beta   C _0 ^ {2(N-1)}  M^2_{\text{prod}} M_{\text{sum}} \sum _{n=1}^N  \sum_{m=1} ^ {M^{(n)} }  \sigma^{(n) 2} _{m,t} ({\bm z}^{(n-1)} ({\bm x}^{(1)}_{t+1},\ldots , {\bm x}^{(n-1)}_{t+1} ),{\bm x}^{(n)}_{t+1} )  \right )          \\
                       & \leq 4  \beta   C _0 ^ {2(N-1)}  M^2_{\text{prod}} M_{\text{sum}} sum _{n=1}^N  \sum_{m=1} ^ {M^{(n)} }  \sum_{t=1}^T \Bigl[    \sigma^{(n) 2} _{m,t} ({\bm z}^{(n-1)} ({\bm x}^{(1)}_{t+1},\ldots , {\bm x}^{(n-1)}_{t+1} ),{\bm x}^{(n)}_{t+1} )\Bigr]  .  \label{eq:r2}
  \end{align}
  Furthermore, by using the same argument as in Lemma~5.3 and~5.4 of~\cite{srinivas2010gaussian}, under  the assumption $k^{(n)} (\cdot, \cdot ) \leq 1$ we get
  \begin{equation}
    \sum_{t=1}^T \sigma^{(n) 2} _{m,t} ({\bm z}^{(n-1)} ({\bm x}^{(1)}_{t+1},\ldots , {\bm x}^{(n-1)}_{t+1} ),{\bm x}^{(n)}_{t+1} )    \leq \frac{2}{ \log (1+\sigma^{-2} ) }\gamma^{(n)}_{m,T}  \leq \frac{2}{ \log (1+\sigma^{-2} ) }\gamma _T . \label{eq:gamma_bound}
  \end{equation}
  Thus, from~\cref{eq:r2,eq:gamma_bound} we obtain
  \begin{align}
    \sum_{t=1}^T r^2_t & \leq   \frac{   8  \beta   C _0 ^ {2(N-1)}  M^2_{\text{prod}} M^2_{\text{sum}}          }{ \log (1+\sigma^{-2} ) }   \gamma_T .
    \label{eq:J_t}
  \end{align}
  Hence, by using~\cref{eq:CR,eq:J_t}, it follows that
  \begin{equation}
    R_t \leq \sqrt{\frac{   8  \beta   C _0 ^ {2(N-1)}  M^2_{\text{prod}} M^2_{\text{sum}}          }{ \log (1+\sigma^{-2} ) }  T \gamma_T }.
  \end{equation}
  Finally, since $r^{(\text{S})}_T$ satisfies
  \begin{equation}
    T r^{(\text{S})}_T \leq \sum_{t=1}^T r_t =R_T \leq \sqrt{\frac{   8  \beta   C _0 ^ {2(N-1)}  M^2_{\text{prod}} M^2_{\text{sum}}          }{ \log (1+\sigma^{-2} ) }  T \gamma_T },
  \end{equation}
  the following inequality holds:
  \begin{equation}
    r^{(\text{S})}_T \leq T^{-1/2}  \sqrt{\frac{   8  \beta   C _0 ^ {2(N-1)}  M^2_{\text{prod}} M^2_{\text{sum}}          }{ \log (1+\sigma^{-2} ) }   \gamma_T }.
  \end{equation}
\end{proof}

%-------------------------------------------------------------------------------------------
\subsection{Optimistic Improvement-based AF}
In this subsection, we consider sequential observations of a cascade process from stage $1$ to $N$.
For each iteration $t \in \{0,N,2N,\ldots \} \equiv N \mathbb{Z}_{\geq 0}$, users determine ${\bm x}^{(1)}_{t+1} $ and observe $\by_{t+1}\stg{1}={\bm f}^{(1)} ({\bm 0}, {\bm x}^{(1)}_{t+1} )$.
After that, users choose ${\bm x}^{(2)}_{t+2} $ and observe $\by_{t+2}\stg{2}={\bm f}^{(2)} ( \by\stg{1}_{t+1}  , {\bm x}^{(2)}_{t+2} )$.
By repeating this operation, users obtain $y\stg{N}_{t+N}={\bm f}^{(N)} (\by\stg{N-1}_{t+N-1}, {\bm x}^{(N)}_{t+N} )$ finally.
We design the CI-based AF according to the following strategy:
(1) given an observation ${\bm y}^{(n)}$, we seek the maximum of $F$ if it is expected to be found;
(2) if the maximum is not expected to be found, we collect the information by using another policy.
We use the optimistic improvement for (1), and we adopt uncertainty sampling (US) policy for (2).
First, we define the pessimistic maximum of $F(\bx\stg{1},\dots,\bx\stg{N})$ as
\begin{equation}
  Q_T = \max_{  ( {\bm x}^{(1)} ,\ldots , {\bm x}^{(N)} ) }  \left(\tilde{\mu}^{(N)}_{1,T} ({\bm x}^{(1)},\ldots, {\bm x}^{(N)} ) -  \beta^{1/2} \tilde{\sigma}^{(N)}_{1,T} ({\bm x}^{(1)},\ldots,{\bm x}^{(N)} ) \right).
\end{equation}
In addition, given the observation ${\bm y}^{(n-1)}$ in stage $n-1$, we define the pessimistic maximum of $F$ obtained through ${\bm y}^{(n-1)}$ as follows:
\begin{equation}
  \mr{ LCB}^{(F)}_t ({\bm y}^{(n-1)} ) = \max_{   {\bm x}^{(n:N)} }  \left(\tilde{\mu}^{(N)}_{1,t} ({\bm x}^{(n:N)}  |{\bm y}^{(n-1)}) - \beta^{1/2} \tilde{\sigma}^{(N)}_{1,t} ({\bm x}^{(n:N)} |{\bm y}^{(n-1)} ) \right), \label{eq:LCB_y}
\end{equation}
where the $\max$ operator is not necessary when $n=N$.
Similarly, the optimistic maximum for given the input $(\by\stg{n-1} ,\bx\stg{n})$ is defined as follows:
\begin{equation}
  \mr{ UCB}^{(F)}_t ({\bm x}^{(n) } | {\bm y}^{(n-1)} ) =  \max_{  {\bm x}^{(n+1:N)} } \left(\tilde{\mu}^{(N)}_{1,t} ({\bm x}^{(n:N)} |{\bm y}^{(n-1)})+ \beta^{1/2} \tilde{\sigma}^{(N)}_{1,t} ({\bm x}^{(n:N)}  |{\bm y}^{(n-1)} ) \right) . \label{eq:UCB_y}
\end{equation}
Then, we define the optimistic improvement w.r.t. $(\by\stg{n-1}, {\bm x}^{(n)})$ as follows:
\begin{equation}
  a^{(n)}_t ({\bm x}^{(n)} | {\bm y}^{(n-1) } )  =   \mr{ UCB}^{(F)}_t ({\bm x}^{(n) } | {\bm y}^{(n-1)} )   - \max \{ \mr{ LCB}^{(F)}_t ( {\bm y}^{(n-1)} ) , Q_{t+n-1} \}. \label{eq:improve_y}
\end{equation}
Furthermore, we define the maximum uncertainty
\begin{equation}
  b^{(n)}_t ({\bm x}^{(n)} | {\bm y}^{(n-1) } )   =  \max_{  ( {\bm x}^{(n+1)} ,\ldots , {\bm x}^{(N)} ) } \tilde{\sigma}^{(N)}_{1,t} ({\bm x}^{(n)},\ldots,{\bm x}^{(N)}  |{\bm y}^{(n-1)} ). \label{eq:US_y}
\end{equation}
Using~\cref{eq:improve_y,eq:US_y}, optimistic improvement-based AF (presented as CI-based AF in~\cref{sec:ci-based-af}) $c^{(n)}_t ({\bm x}^{(n)} | {\bm y}^{(n-1) } )$ is defined as
\begin{equation}
  c^{(n)}_t ({\bm x}^{(n)} | {\bm y}^{(n-1) } ) = \max \left\{ a^{(n)}_t ({\bm x}^{(n)} | {\bm y}^{(n-1) } ), \eta_t b^{(n)}_t ({\bm x}^{(n)} | {\bm y}^{(n-1) } ) \right\},  \label{eq:af_seq}
\end{equation}
where $\eta_t $ is some learning rate tends to zero.
Therefore, given the observation $\by\stg{n-1}$ at iteration $t$, the next observation point is given by
\begin{equation}
  {\bm x}^{(n)}_{t+n} = \argmax _{  {\bm x}^{(n) } \in \mathcal{X}\stg{n} }  c^{(n)}_t ({\bm x}^{(n)} | {\bm y}_{t+n-1}^{(n-1) } ),\label{eq:seq_rule}
\end{equation}
where ${\bm y}^{(0)}  = {\bm 0}$.

\begin{theorem}\label{theorem:i_regret}
  Assume that~\cref{assumption:regu1,assumption:L1,assumption:L2} hold.
  Also
  assume that $\tilde{\bm \mu}^{(n)} _t ({\bm x}^{(s+1)},\ldots,{\bm x}^{(n)} |{\bm y}) \in \mathcal{Y}^{(n)}$
  for any $s \in \{0,\ldots, N-1\}$,
  $n \in \{s+1,\ldots, N \}$,  ${\bm y} \in \mathcal{Y}^{(s)}$, $t \geq 1$ and $({\bm x}^{(s+1) } , \ldots, {\bm x}^{(n) }  ) \in \mathcal{X}^{(s+1)} \times \cdots \times \mathcal{X}^{(n)}$.
  Let $\xi $ be a positive number, and define $\beta =B^2$ and $\eta_t = (1+ \log t)^{-1} $.
  Then, when the optimization is performed using~\cref{eq:seq_rule}, the estimated solution $( \hat{\bm x}^{(1)}_T, \ldots , \hat{\bm x}^{(N)}_T)$ satisfies that
  \begin{equation}
    F({\bm x}^{(1)}_\ast, \ldots ,{\bm x}^{(N)}_\ast ) -    F(  \hat{\bm x}^{(1)}_T, \ldots , \hat{\bm x}^{(N)}_T )  < \xi,
  \end{equation}
  where $T$ is the smallest positive integer satisfying $ T \in  N \mathbb{Z}_{\geq 0}$ and
  \begin{equation}
    \frac{8 \beta C^2_4  M^2_{\text{sum}}  N }{\log(1+\sigma^{-2})}
    \gamma_T \eta^{-2N-2}_T T^{-1} < \xi^2.
  \end{equation}
  Here, $C_4$ is the positive constant given by
  \begin{equation}
    C_1  = \max \{ 1, L_f, L^{-1}_f \},  C_2=4 N M^2_{\text{prod}} M_{\text{sum}} C^{2N-3} _0 C^N_1, C_3  = NC_2^N,  C_4 = (2 \beta^{1/2} +2 ) ^N C^N_3.
  \end{equation}
\end{theorem}
In order to prove~\cref{theorem:i_regret}, we give four lemmas.
\begin{lemma}\label{lem:fms}
  Assume that the same condition as in~\cref{theorem:i_regret} holds.
  Let $s \in \{1,\ldots, N-1 \}$ and $n \in \{ s+1, \ldots , N \}$.
  Then, for any iteration $t \geq 1$, element $ m \in [M^{(n)}]$ and input ${\bm x}^{(1)},\ldots ,{\bm x}^{(N)}$, the following inequality holds:
  \begin{align}
     & |\sigma^{(n)} _{m,t} ( \tilde{\bm  \mu}^{(n-1)}_t ( {\bm x}^{(s) },\ldots, {\bm x}^{(n-1)} | {\bm z}^{(s-1)}),{\bm x}^{(n)} )  -  \sigma^{(n)} _{m,t} ( \tilde{\bm  \mu}^{(n-1)}_t ( {\bm x}^{(s+1) },\ldots, {\bm x}^{(n-1)} | {\bm z}^{(s)}),{\bm x}^{(n)} ) | \\
     & \leq 2    M_{\text{prod}} C_0^{N-1}  \sum_{p=0} ^{n-s-1}  \sum_{i=1}^{   M^{(n-1-p)} } \Bigl[   \sigma^{(n-1-p)} _{i,t} (   \tilde{ \bm \mu}^{(n-2-p)}_{t} ( {\bm x}^{(s:n-2-p) } | {\bm z}^{(s-1)}),{\bm x}^{(n-1-p)}) \Bigr].
  \end{align}
\end{lemma}
\begin{proof}
  From Lipschitz continuity of $\sigma^{(n)}_{m,t} (\cdot)$, the following holds:
  \begin{align}
     & |\sigma^{(n)} _{m,t} ( \tilde{\bm  \mu}^{(n-1)}_t ( {\bm x}^{(s) },\ldots, {\bm x}^{(n-1)} | {\bm z}^{(s-1)}),{\bm x}^{(n)} )     -  \sigma^{(n)} _{m,t} ( \tilde{\bm  \mu}^{(n-1)}_t ( {\bm x}^{(s+1) },\ldots, {\bm x}^{(n-1)} | {\bm z}^{(s)}),{\bm x}^{(n)} )|                \\
     & \leq L_\sigma \| \tilde{\bm  \mu}^{(n-1)}_t ( {\bm x}^{(s) },\ldots, {\bm x}^{(n-1)} | {\bm z}^{(s-1)})   -  \tilde{\bm  \mu}^{(n-1)}_t ( {\bm x}^{(s+1) },\ldots, {\bm x}^{(n-1)} | {\bm z}^{(s)}) \|_1                                                                          \\
     & = L_\sigma \sum_{j=1}^{M^{(n-1)} }  | \tilde{  \mu}^{(n-1)}_{j,t} ( {\bm x}^{(s) },\ldots, {\bm x}^{(n-1)} | {\bm z}^{(s-1)})  -  \tilde{  \mu}^{(n-1)}_{j,t} ( {\bm x}^{(s+1) },\ldots, {\bm x}^{(n-1)} | {\bm z}^{(s)})|                                                        \\
     & = L_\sigma \sum_{j=1}^{M^{(n-1)} }   | {\mu}^{(n-1)}_{j,t}(\tilde{ \bm \mu}^{(n-2)}_{t} ( {\bm x}^{(s:n-2) }| {\bm z}^{(s-1)}),{\bm x}^{(n-1)}) -  {\mu}^{(n-1)}_{j,t}(\tilde{ \bm \mu}^{(n-2)}_{t} ( {\bm x}^{(s+1:n-2) } | {\bm z}^{(s)}),{\bm x}^{(n-1)}) | . \label{eq:ineq1}
  \end{align}
  Here, noting that
  \begin{align}
    f^{(k)}_m  ({\bm y},{\bm x} ) -\beta^{1/2} \sigma^{(k)}_{m,t}   ({\bm y},{\bm x} ) & \leq \mu^{(k)}_{m,t}   ({\bm y},{\bm x} )                                                  \\
                                                                                       & \leq  f^{(k)}_m ({\bm y},{\bm x} ) +\beta^{1/2} \sigma^{(k)}_{m,t}   ({\bm y},{\bm x} )  ,
  \end{align}
  we have
  \begin{align}
    | & \mu^{(k)}_{m,t}   ({\bm y},{\bm x} )  -\mu^{(k)}_{m,t}   ({\bm y}^\prime,{\bm x} )  |      \\
      & \leq  |f^{(k)}_{m}   ({\bm y},{\bm x} )  - f^{(k)}_{m}   ({\bm y}^\prime,{\bm x} )  |    +
    \beta^{1/2} \sigma^{(k)}_{m,t}   ({\bm y},{\bm x} ) +
    \beta^{1/2} \sigma^{(k)}_{m,t}   ({\bm y}^\prime,{\bm x} )                                     \\
      & \leq |f^{(k)}_{m}   ({\bm y},{\bm x} )  -   f^{(k)}_{m}   ({\bm y}^\prime,{\bm x} )  |   +
    \beta^{1/2} |\sigma^{(k)}_{m,t}   ({\bm y},{\bm x} ) -
    \sigma^{(k)}_{m,t}   ({\bm y}^\prime,{\bm x} )|+
    2\beta^{1/2} \sigma^{(k)}_{m,t}   ({\bm y}^\prime,{\bm x} )                                    \\
      & \leq
    (L_f + \beta^{1/2}_t L_\sigma ) \| {\bm y} - {\bm y}^\prime \| _1 +
    2\beta^{1/2} \sigma^{(k)}_{m,t}   ({\bm y}^\prime,{\bm x} ).  \label{eq:ineq2}
  \end{align}
  Therefore, by substituting~\cref{eq:ineq2} into~\cref{eq:ineq1}, it follows that
  \begin{align}
     & |\sigma^{(n)} _{m,t} ( \tilde{\bm  \mu}^{(n-1)}_t ( {\bm x}^{(s) },\ldots, {\bm x}^{(n-1)} | {\bm z}^{(s-1)}),{\bm x}^{(n)} ) -    \sigma^{(n)} _{m,t} ( \tilde{\bm  \mu}^{(n-1)}_t ( {\bm x}^{(s+1) },\ldots, {\bm x}^{(n-1)} | {\bm z}^{(s)}),{\bm x}^{(n)} )| \\
     & \leq 2 \beta^{1/2} L_\sigma \sum_{j=1}^{M^{(n-1)} } \sigma^{(n-1)} _{j,t} \bigl(   \tilde{ \bm \mu}^{(n-2)}_{t} ( {\bm x}^{(s+1) },\ldots, {\bm x}^{(n-2)} | {\bm z}^{(s)})  ,      {\bm x}^{(n-1)})+ L_\sigma M^{(n-1)} (L_f + \beta^{1/2} L_\sigma )           \\
     & \quad  \cdot  \| \tilde{ \bm \mu}^{(n-2)}_{t} ( {\bm x}^{(s) },\ldots, {\bm x}^{(n-2)} | {\bm z}^{(s-1)})     - \tilde{ \bm \mu}^{(n-2)}_{t} ( {\bm x}^{(s+1) },\ldots, {\bm x}^{(n-2)} | {\bm z}^{(s)}) \| _1                                                   \\
     & \leq 2 \beta^{1/2} L_\sigma \sum_{j=1}^{M^{(n-1)} } \sigma^{(n-1)} _{j,t} \bigl(   \tilde{ \bm \mu}^{(n-2)}_{t} ( {\bm x}^{(s+1) },\ldots, {\bm x}^{(n-2)} | {\bm z}^{(s)}) ,   {\bm x}^{(n-1)} )                                                                \\
     & \quad + L_\sigma M^{(n-1)} (L_f + \beta^{1/2} L_\sigma ) \sum_{i=1} ^{M^{(n-2)}}  | \tilde{  \mu}^{(n-2)}_{i,t} ( {\bm x}^{(s:n-2) } | {\bm z}^{(s-1)} )      - \tilde{  \mu}^{(n-2)}_{i,t} ( {\bm x}^{(s+1:n-2) } | {\bm z}^{(s)})   |                          \\
     & \leq     2 \beta^{1/2} L_\sigma \sum_{j=1}^{M^{(n-1)} } \sigma^{(n-1)} _{j,t} (   \tilde{ \bm \mu}^{(n-2)}_{t} ( {\bm x}^{(s:n-2) } | {\bm z}^{(s-1)}),{\bm x}^{(n-1)})                                                                                          \\
     & \quad +  2 \beta^{1/2} L_\sigma M^{(n-1)} (L_f + \beta^{1/2} L_\sigma )   \cdot \sum_{i=1} ^{M^{(n-2)}} \sigma^{(n-2)} _{i,t} (   \tilde{ \bm \mu}^{(n-3)}_{t} ( {\bm x}^{(s) },\ldots, {\bm x}^{(n-3)} | {\bm z}^{(s-1)}),{\bm x}^{(n-2)})                      \\
     & + L_\sigma M^{(n-1)} M^{(n-2)} (L_f +\beta^{1/2} L_\sigma )^2\sum_{q=1} ^{M^{(n-3)}} \Bigl[   | \tilde{  \mu}^{(n-3)}_{q,t} ( {\bm x}^{(s:n-3) } | {\bm z}^{(s-1)}) - \tilde{  \mu}^{(n-3)}_{q,t} ( {\bm x}^{(s+1:n-3) } | {\bm z}^{(s)}) | \Bigr]               \\
     & \leq                                                                                                                                                                                                                                                             \\
     & \vdots                                                                                                                                                                                                                                                           \\
     & \leq  2 \beta^{1/2} L_\sigma   M_{\text{prod}}  (L_f + \beta^{1/2} L_\sigma +1)^{N-2}                                                                                                                                                                            \\
     & \quad  \cdot \sum_{p=0} ^{n-s-2}\sum_{i=1}^{   M^{(n-1-p)} }\Bigl[  \sigma^{(n-1-p)} _{i,t} \Bigl(   \tilde{ \bm \mu}^{(n-2-p)}_{t} ( {\bm x}^{(s:n-2-p) } | {\bm z}^{(s-1)}),{\bm x}^{(n-1-p)})      \Bigr]                                                     \\
     & \quad + 2 M_{\text{prod}}  \beta^{1/2} L_\sigma (L_f +\beta^{1/2} L_\sigma +1)^{N-2}  \sum_{q=1} ^{M^{(s)}}  \sigma^{(s)}_{q,t}  ( {\bm z}^{(s-1)}, {\bm x}^{(s)}   )                                                                                            \\
     & \leq     2  M_{\text{prod}}  (L_f + \beta^{1/2} L_\sigma +1)^{N-1}  \sum_{p=0} ^{n-s-1}  \sum_{i=1}^{   M^{(n-1-p)} }\Bigl[  \sigma^{(n-1-p)} _{i,t} (   \tilde{ \bm \mu}^{(n-2-p)}_{t} ( {\bm x}^{(s:n-2-p) } | {\bm z}^{(s-1)}),{\bm x}^{(n-1-p)})   \Bigr]    \\
     & \leq     2  M_{\text{prod}}  C^{N-1} _0 \sum_{p=0} ^{n-s-1}   \sum_{i=1}^{   M^{(n-1-p)} } \Bigl[ \sigma^{(n-1-p)} _{i,t} \bigl(   \tilde{ \bm \mu}^{(n-2-p)}_{t} ( {\bm x}^{(s) },\ldots, {\bm x}^{(n-2-p)} | {\bm z}^{(s-1)}),{\bm x}^{(n-1-p)} \bigr) \Bigr].
  \end{align}
\end{proof}
\begin{lemma}\label{lem:nextAAA}
  Assume that the same condition as in~\cref{theorem:i_regret} holds.
  Let  $s \in \{1,\ldots, N-1 \}$, and let $ j \geq 0$ be an integer with $ s+j \leq N$.
  Then, for any iteration $t \geq 1$, element $ m \in [M^{(n)}]$ and input ${\bm x}^{(1)},\ldots ,{\bm x}^{(N)}$, the following inequality holds:
  \begin{align}
    \tilde{\sigma}^{(N-j)}_{t} ({\bm x} ^{(s)}, \ldots ,{\bm x}^{(N-j)} | {\bm z}^{(s-1)} ) & \leq  \tilde{C}_2 \tilde{\sigma}^{(N-j)}_{t} ({\bm x} ^{(s+1)}, \ldots ,{\bm x}^{(N-j)} | {\bm z}^{(s)} )  +     \tilde{C}_2 \sum _{i=1}^{M^{(s)} } \sigma^{(s)}_{i,t} (  {\bm z}^{(s-1)}  ,{\bm x}^{(s)} ) \\
                                                                                            & \quad +  \tilde{C}_2 \tilde{\sigma}^{(N-j-1)}_{t} ({\bm x} ^{(s)}, \ldots ,{\bm x}^{(N-j-1)} | {\bm z}^{(s-1)} )  ,
  \end{align}
  where
  \begin{align}
     & \tilde{\sigma}^{(N-j)}_{t} ({\bm x} ^{(s)}, \ldots ,{\bm x}^{(N-j)} | {\bm z}^{(s-1)} ) \\
     & =
    \sum_{p=j}^{N-s}  \prod_{q=1}^p  M^{(N-q+1)}  L_f ^p
    \sum _{i=1}^{M^{(N-p)} } \Bigl[     \sigma^{(N-p)}_{i,t} (   \tilde{\bm \mu}^{(N-p-1)}_t ({\bm x}^{(s)} , \ldots , {\bm x}^{(N-p-1)} |{\bm z}^{(s-1)} ) ,{\bm x}^{(N-p)} )\Bigr]
  \end{align}
  and   $\tilde{C}_2=4 N M^2_{\text{prod}} M_{\text{sum}} C^{2N-2} _0 C^N_1 $.
\end{lemma}
\begin{proof}
  From the definition of \sloppy$\tilde{\sigma}^{(N-j)}_{t}$  \sloppy$({\bm x} ^{(s)}, \ldots ,{\bm x}^{(N-j)} | {\bm z}^{(s-1)} )$, the following inequality holds:
  \begin{align}
     & \tilde{\sigma}^{(N-j)}_{t} ({\bm x} ^{(s)}, \ldots ,{\bm x}^{(N-j)} | {\bm z}^{(s-1)} )                                                                                                                                                 \\
     & =  \sum_{p=j}^{N-s}  \prod_{q=1}^p  M^{(N-q+1)}  L_f ^p \sum _{i=1}^{M^{(N-p)} }\Bigl[   \sigma^{(N-p)}_{i,t} (   \tilde{\bm \mu}^{(N-p-1)}_t ({\bm x}^{(s)} , \ldots , {\bm x}^{(N-p-1)} |{\bm z}^{(s-1)} ) ,{\bm x}^{(N-p)} )\Bigr]   \\
     & \leq    M_{\text{prod}} C^{N-1}_0 \sum_{p=j}^{N-s} \sum _{i=1}^{M^{(N-p)} }\Bigl[  \sigma^{(N-p)}_{i,t} (   \tilde{\bm \mu}^{(N-p-1)}_t ({\bm x}^{(s)} , \ldots , {\bm x}^{(N-p-1)} |{\bm z}^{(s-1)} ) ,{\bm x}^{(N-p)} )    \Bigr]     \\
     & =  M_{\text{prod}} C^{N-1}_0 \sum_{p=j}^{N-s-1} \sum _{i=1}^{M^{(N-p)} } \Bigl[      \sigma^{(N-p)}_{i,t} (   \tilde{\bm \mu}^{(N-p-1)}_t ({\bm x}^{(s+1)} , \ldots , {\bm x}^{(N-p-1)} |{\bm z}^{(s)} ) ,{\bm x}^{(N-p)} )      \Bigr] \\
     & + M_{\text{prod}} C^{N-1}_0 \sum_{p=j}^{N-s-1}
    \sum _{i=1}^{M^{(N-p)} }
    \left ( \sigma^{(N-p)}_{i,t} (   \tilde{\bm \mu}^{(N-p-1)}_t ({\bm x}^{(s)} , \ldots , {\bm x}^{(N-p-1)} |{\bm z}^{(s-1)} ) ,{\bm x}^{(N-p)} )
    \right .                                                                                                                                                                                                                                   \\
     & \left. \quad -
    \sigma^{(N-p)}_{i,t} (   \tilde{\bm \mu}^{(N-p-1)}_t ({\bm x}^{(s+1)} , \ldots , {\bm x}^{(N-p-1)} |{\bm z}^{(s)} ) ,{\bm x}^{(N-p)} )
    \right )    +
    M_{\text{prod}} C^{N-1}_0
    \sum _{i=1}^{M^{(s)} }
    \sigma^{(s)}_{i,t} (  {\bm z}^{(s-1)}  ,{\bm x}^{(s)} ) .
  \end{align}
  Thus, from~\cref{lem:fms} we get
  \begin{align}
     & |\sigma^{(N-p)} _{i,t} ( \tilde{\bm  \mu}^{(N-p-1)}_t ( {\bm x}^{(s) },\ldots, {\bm x}^{(N-p-1)} | {\bm z}^{(s-1)}),{\bm x}^{(N-p)} )                                         \\
     & \qquad -
    \sigma^{(N-p)} _{i,t} ( \tilde{\bm  \mu}^{(N-p-1)}_t ( {\bm x}^{(s+1) },\ldots, {\bm x}^{(N-p-1)} | {\bm z}^{(s)}),{\bm x}^{(N-p)} )
    |                                                                                                                                                                                \\
     &
    \leq 2 M_{\text{prod}} C^{N-1}_0  \sum_{r=0} ^{N-p-s-1}
                                                                                                                                          \\
     & \quad  \cdot  \sum_{j=1}^{   M^{(N-p-1-r)} } \Bigl[    \sigma^{(N-p-1-r)} _{j,t} \Bigl(    \tilde{ \bm \mu}^{(N-p-2-r)}_{t} ( {\bm x}^{(s) },\ldots, {\bm x}^{(N-p-2-r)} | {\bm z}^{(s-1)}),{\bm x}^{(N-p-1-r)}\Bigr) \Bigr].
  \end{align}
  By using this, $\tilde{\sigma}^{(N-j)}_{t} ({\bm x} ^{(s)}, \ldots ,{\bm x}^{(N-j)} | {\bm z}^{(s-1)} )$ can be written as
  \begin{align}
     & \tilde{\sigma}^{(N-j)}_{t} ({\bm x} ^{(s)}, \ldots ,{\bm x}^{(N-j)} | {\bm z}^{(s-1)} )                                                       \\
     & \leq M_{\text{prod}} C^{N-2}_0 \sum_{p=j}^{N-s-1}
    \sum _{i=1}^{M^{(N-p)} }
    \sigma^{(N-p)}_{i,t} \Bigl(   \tilde{\bm \mu}^{(N-p-1)}_t ({\bm x}^{(s+1)} , \ldots , {\bm x}^{(N-p-1)} |{\bm z}^{(s)} ) ,{\bm x}^{(N-p)} \Bigr) \\
     & \quad  +  M_{\text{prod}} C^{N-2}_0
    \sum _{i=1}^{M^{(s)} }
    \sigma^{(s)}_{i,t} (  {\bm z}^{(s-1)}  ,{\bm x}^{(s)} ) + 2M^2_{\text{prod}} C^{2N-2}_0 M_{\text{sum}}                                           \\
     & \quad  \cdot \sum_{p=j}^{N-s-1}
    \sum_{r=0} ^{N-p-s-1}
    \sum_{j=1}^{   M^{(N-p-1-r)} } \Bigl[    \sigma^{(N-p-1-r)} _{j,t} \Bigl( \tilde{ \bm \mu}^{(N-p-2-r)}_{t} ( {\bm x}^{(s:N-p-2-r) } | {\bm z}^{(s-1)}),    {\bm x}^{(N-p-1-r)}\Bigr) \Bigr]   .
  \end{align}
  Here, we set
  $v=p+r$. Then, noting that   $| \{ (p,r) \mid p+r = a \}| \leq 2a$, we obtain
  \begin{align}
     & \tilde{\sigma}^{(N-j)}_{t} ({\bm x} ^{(s)}, \ldots ,{\bm x}^{(N-j)} | {\bm z}^{(s-1)} )                                                                                                                                    \\
     & \leq M_{\text{prod}} C^{N-2}_0 \sum_{p=j}^{N-s-1}
    \sum _{i=1}^{M^{(N-p)} }
    \sigma^{(N-p)}_{i,t} \Bigl(    \tilde{\bm \mu}^{(N-p-1)}_t ({\bm x}^{(s+1)} , \ldots , {\bm x}^{(N-p-1)} |{\bm z}^{(s)} ) ,{\bm x}^{(N-p)} \Bigr)                                                                             \\
     & + 2M^2_{\text{prod}} C^{2N-2}_0 M_{\text{sum}} \sum_{v=j}^{N-s-1}
    2N
    \sum_{j=1}^{   M^{(N-v-1)} } \sigma^{(N-v-1)} _{j,t} \Bigl(    \tilde{ \bm \mu}^{(N-v-2)}_{t} ( {\bm x}^{(s) },\ldots, {\bm x}^{(N-v-2)} | {\bm z}^{(s-1)}),{\bm x}^{(N-v-1)}\Bigr)                                           \\
     & +
    M_{\text{prod}} C^{N-2}_0
    \sum _{i=1}^{M^{(s)} }
    \sigma^{(s)}_{i,t} (  {\bm z}^{(s-1)}  ,{\bm x}^{(s)} )                                                                                                                                                                       \\
     & \leq  4 N M^2_{\text{prod}} C^{2N-2}_0 M_{\text{sum}} \sum_{p=j}^{N-s-1}
    \sum _{i=1}^{M^{(N-p)} }
    \sigma^{(N-p)}_{i,t} \Bigl(   \tilde{\bm \mu}^{(N-p-1)}_t ({\bm x}^{(s+1)} , \ldots , {\bm x}^{(N-p-1)} |{\bm z}^{(s)} ) ,{\bm x}^{(N-p)} \Bigr)                                                                              \\
     & + 4 N M^2_{\text{prod}} C^{2N-2}_0 M_{\text{sum}} \sum_{v=j}^{N-s-1}
    \sum_{j=1}^{   M^{(N-v-1)} } \sigma^{(N-v-1)} _{j,t} \Bigl(  \tilde{ \bm \mu}^{(N-v-2)}_{t} ( {\bm x}^{(s) },\ldots, {\bm x}^{(N-v-2)} | {\bm z}^{(s-1)}),{\bm x}^{(N-v-1)}\Bigr)                                             \\
     & +
    4 N M^2_{\text{prod}} C^{2N-2}_0 M_{\text{sum}} C^N_1
    \sum _{i=1}^{M^{(s)} }
    \sigma^{(s)}_{i,t} (  {\bm z}^{(s-1)}  ,{\bm x}^{(s)} )                                                                                                                                                                       \\
     & \leq  \tilde{C}_2 \tilde{\sigma}^{(N-j)}_{t} ({\bm x} ^{(s+1)}, \ldots ,{\bm x}^{(N-j)} | {\bm z}^{(s)} )        + \tilde{C}_2 \tilde{\sigma}^{(N-j-1)}_{t} ({\bm x} ^{(s)}, \ldots ,{\bm x}^{(N-j-1)} | {\bm z}^{(s-1)} ) \\
     & \quad +\tilde{C}_2 \sum _{i=1}^{M^{(s)} }
    \sigma^{(s)}_{i,t} (  {\bm z}^{(s-1)}  ,{\bm x}^{(s)} )  .
  \end{align}
\end{proof}
\begin{lemma}\label{lem:next}
  Assume that the same condition as in~\cref{theorem:i_regret} holds.
  Let $s \in \{1,\ldots, N-1 \}$ and $n \in \{ s+1, \ldots , N \}$.
  Then, for any iteration $t \geq 1$, element $ m \in [M^{(n)}]$ and input ${\bm x}^{(1)},\ldots ,{\bm x}^{(N)}$, the following inequality holds:
  \begin{equation}
    \tilde{\sigma}^{(N)}_{1,t} ({\bm x} ^{(s:N)} | {\bm z}^{(s-1)} )  \leq C_3 \tilde{\sigma}^{(N)}_{1,t} ({\bm x} ^{(s+1:N)} | {\bm z}^{(s)} )
    +C_3 \sum_{i=1} ^{M^{(s)} } \sigma^{(s)} _{i,t}  ({\bm z}^{(s-1)} ,{\bm x}^{(s)} ).
  \end{equation}
\end{lemma}
\begin{proof}
  By repeatedly using~\cref{lem:nextAAA}, we obtain
  \begin{align}
     & \tilde{\sigma}^{(N)}_{1,t} ({\bm x} ^{(s)}, \ldots ,{\bm x}^{(N)} | {\bm z}^{(s-1)} )                                                                                                                                                                                                                            \\
     & = \tilde{\sigma}^{(N-0)}_{t} ({\bm x} ^{(s)}, \ldots ,{\bm x}^{(N-0)} | {\bm z}^{(s-1)} )                                                                                                                                                                                                                        \\
     & \leq   \tilde{C}_2 \tilde{\sigma}^{(N-0)}_{t} ({\bm x} ^{(s+1)}, \ldots ,{\bm x}^{(N-0)} | {\bm z}^{(s)} )  + \tilde{C}_2 \sum _{i=1}^{M^{(s)} }  \sigma^{(s)}_{i,t} (  {\bm z}^{(s-1)}  ,{\bm x}^{(s)} )   +\tilde{C}_2 \tilde{\sigma}^{(N-1)}_{t} ({\bm x} ^{(s)}, \ldots ,{\bm x}^{(N-1)} | {\bm z}^{(s-1)} ) \\
     & \leq  \tilde{C}_2\tilde{\sigma}^{(N-0)}_{t} ({\bm x} ^{(s+1)}, \ldots ,{\bm x}^{(N-0)} | {\bm z}^{(s)} ) +  \tilde{C}_2\sum _{i=1}^{M^{(s)} }\sigma^{(s)}_{i,t} (  {\bm z}^{(s-1)}  ,{\bm x}^{(s)} )     + \tilde{C}^2_2 \sum _{i=1}^{M^{(s)} } \sigma^{(s)}_{i,t} (  {\bm z}^{(s-1)}  ,{\bm x}^{(s)} )          \\
     & \quad+  \tilde{C}^2_2 \tilde{\sigma}^{(N-1)}_{t} ({\bm x} ^{(s+1)}, \ldots ,{\bm x}^{(N-1)} | {\bm z}^{(s)} )  +\tilde{C}^2_2 \tilde{\sigma}^{(N-2)}_{t} ({\bm x} ^{(s)}, \ldots ,{\bm x}^{(N-2)} | {\bm z}^{(s-1)} )                                                                                            \\
     & \leq                                                                                                                                                                                                                                                                                                             \\
     & \vdots                                                                                                                                                                                                                                                                                                           \\
     & \leq    (\tilde{C}_2 + \tilde{C}^2_2 + \cdots + \tilde{C}^{N-1}_2 ) \tilde{\sigma}^{(N-0)}_{t} ({\bm x} ^{(s+1:N-0)} | {\bm z}^{(s)} ) + (\tilde{C}_2 + \tilde{C}^2_2 + \cdots + \tilde{C}^{N-1}_2 )    \sum _{i=1}^{M^{(s)} }   \sigma^{(s)}_{i,t} (  {\bm z}^{(s-1)}  ,{\bm x}^{(s)} )                         \\
     & \leq  (N-1) \tilde{C}^{N-1}_2  \tilde{\sigma}^{(N)}_{1,t} ({\bm x} ^{(s+1)}, \ldots ,{\bm x}^{(N)} | {\bm z}^{(s)} ) + (N-1) \tilde{C}^{N-1}_2    \sum _{i=1}^{M^{(s)} }
    \sigma^{(s)}_{i,t} (  {\bm z}^{(s-1)}  ,{\bm x}^{(s)} ).
  \end{align}
  In addition, $(N-1) \tilde{C}^{N-1}_2$ can be bounded by
  \begin{align}
    (N-1) \tilde{C}^{N-1}_2 & \leq N \tilde{C}^{N-1}_2                                                 \\
                            & = N (4 N M^2_{\text{prod}} M_{\text{sum}} C^{2N-2} _0 C^N_1)^{N-1}       \\
                            & \leq N (4 N M^2_{\text{prod}} M_{\text{sum}}  C^N_1)^N  C^{2N^2-4N+2} _0 \\
                            & \leq N (4 N M^2_{\text{prod}} M_{\text{sum}}  C^N_1)^N  C^{2N^2-4N+N} _0 \\
                            & = N (4 N M^2_{\text{prod}} M_{\text{sum}} C_0 ^{2N-3} C^N_1)^N           \\
                            & = N C^N_2 =C_3,
  \end{align}
  we get~\cref{lem:next}.
\end{proof}
\begin{lemma}\label{lem:c_bound}
  Assume that the same condition as in~\cref{theorem:i_regret} holds.
  Let $ n \in [N]$ and ${\bm y}^{(n-1)} \in \mathcal{Y}^{(n-1)}$.
  Then, for any iteration $t \geq 1$ and input ${\bm x}^{(n)} \in \mathcal{X}^{(n)}$, the following inequality holds:
  \begin{align}
    \eta_t b^{(n)}_t ({\bm x}^{(n)} |{\bm y}^{(n-1)} ) & \leq c^{(n)}_t ({\bm x}^{(n)} |{\bm y}^{(n-1)} )                             \\
                                                       & \leq (2 \beta^{1/2}  + \eta_t ) b^{(n)}_t ({\bm x}^{(n)} |{\bm y}^{(n-1)} ).
  \end{align}
\end{lemma}
\begin{proof}
  From the definition of
  $c^{(n)}_t ({\bm x}^{(n)} |{\bm y}^{(n-1)} )$,
  it is clear that $\eta_t b^{(n)}_t ({\bm x}^{(n)} |{\bm y}^{(n-1)} ) \leq c^{(n)}_t ({\bm x}^{(n)} |{\bm y}^{(n-1)} ) $.
  On the other hand, from the definition of $\mr{ UCB}^{(F)}_t ({\bm x}^{(n)} |{\bm y}^{(n-1) } ) $, letting
  \begin{equation}
    (\tilde{\bm x}^{(n+1)}, \ldots, \tilde{\bm x}^{(N)})  =  \argmax_{  ( {\bm x}^{(n+1)} ,\ldots , {\bm x}^{(N)} ) } \Bigl(\tilde{\mu}^{(N)}_{1,t} ({\bm x}^{(n:N)}  |{\bm y}^{(n-1)})  + \beta^{1/2} \tilde{\sigma}^{(N)}_{1,t} ({\bm x}^{(n:N)}  |{\bm y}^{(n-1)} \Bigr)
  \end{equation}
  we obtain
  \begin{equation}
    \mr{ UCB}^{(F)}_t ({\bm x}^{(n)} |{\bm y}^{(n-1) } )  =    \tilde{\mu}^{(N)}_{1,t} ({\bm x}^{(n)},\tilde{\bm x}^{(n+1)} \ldots, \tilde{\bm x}^{(N)}  |{\bm y}^{(n-1)})  + \beta^{1/2} \tilde{\sigma}^{(N)}_{1,t} ({\bm x}^{(n)}, \tilde{\bm x}^{(n+1)},  \ldots,\tilde{\bm x}^{(N)}  |{\bm y}^{(n-1)} ).
  \end{equation}
  Similarly,  $\mr{ LCB}^{(F)}_t ({\bm y}^{(n-1) } ) $ can be bounded as follows:
  \begin{equation}
    \mr{ LCB}^{(F)}_t ({\bm y}^{(n-1) } ) \geq  \tilde{\mu}^{(N)}_{1,t} ({\bm x}^{(n)},\tilde{\bm x}^{(n+1)} \ldots, \tilde{\bm x}^{(N)}  |{\bm y}^{(n-1)})    - \beta^{1/2} \tilde{\sigma}^{(N)}_{1,t} ({\bm x}^{(n)}, \tilde{\bm x}^{(n+1)},  \ldots,\tilde{\bm x}^{(N)}  |{\bm y}^{(n-1)} ).
  \end{equation}
  Hence, from the definition of
  $a^{(n)}_t ({\bm x}^{(n)} |{\bm y}^{(n-1)} )$ and $b^{(n)}_t ({\bm x}^{(n)} |{\bm y}^{(n-1)} )$, we get
  \begin{align}
    a^{(n)}_t ({\bm x}^{(n)} |{\bm y}^{(n-1)} ) & \leq \mr{ UCB}^{(F)}_t ({\bm x}^{(n)} |{\bm y}^{(n-1) } ) -  \mr{ LCB}^{(F)}_t ({\bm y}^{(n-1) } )                                    \\
                                                & \leq 2  \beta^{1/2} \tilde{\sigma}^{(N)}_{1,t} ({\bm x}^{(n)}, \tilde{\bm x}^{(n+1)},  \ldots,\tilde{\bm x}^{(N)}  |{\bm y}^{(n-1)} ) \\
                                                & \leq 2  \beta^{1/2} b^{(n)}_t ({\bm x}^{(n)} |{\bm y}^{(n-1)} ).
  \end{align}
  Therefore, $c^{(n)}_t ({\bm x}^{(n)} |{\bm y}^{(n-1)} ) $ can be written as
  \begin{align}
    c^{(n)}_t ({\bm x}^{(n)} |{\bm y}^{(n-1)} ) & = \max \{  a^{(n)}_t ({\bm x}^{(n)} |{\bm y}^{(n-1)} ) , \eta_t b^{(n)}_t ({\bm x}^{(n)} |{\bm y}^{(n-1)} ) \} \\
                                                & \leq
    \max \{   2  \beta^{1/2} b^{(n)}_t ({\bm x}^{(n)} |{\bm y}^{(n-1)} ) , \eta_t b^{(n)}_t ({\bm x}^{(n)} |{\bm y}^{(n-1)} ) \}                                 \\
                                                & \leq (  2  \beta^{1/2}  + \eta_t ) b^{(n)}_t ({\bm x}^{(n)} |{\bm y}^{(n-1)} ) .
  \end{align}
\end{proof}

By using these lemmas, we prove~\cref{theorem:i_regret}.
\begin{proof}
  Let $t \in N \mathbb{Z}_{\geq 0} $.
  Then, from~\cref{lem:c_bound}, ${\bm x}^{(1)}_{t+1} $ satisfies that
  \begin{align}
    c^{(1)}_t ({\bm x}^{(1)}_{t+1} | {\bm 0} ) & \leq (2 \beta ^{1/2} + \eta _t ) b^{(1)}_t ({\bm x}^{(1)}_{t+1} | {\bm 0} )                                                                    \\
                                               & =  (2 \beta ^{1/2} + \eta _t ) \tilde{\sigma}^{(N)}_{1,t} ({\bm x}^{(1)}_{t+1}, \tilde{\bm x}^{(2)} ,\ldots , \tilde{\bm x}^{(N)} | {\bm 0} ).
    \label{eq:ineq001}
  \end{align}
  Thus, by combining~\cref{eq:ineq001,lem:next}, $c^{(1)}_t ({\bm x}^{(1)}_{t+1} | {\bm 0} )$ can be bounded as follows:
  \begin{align}
    c^{(1)}_t & ({\bm x}^{(1)}_{t+1} | {\bm 0} )                                                                                   \\
              & \leq
    (2 \beta ^{1/2} + \eta _t ) C_3 \sum_{i=1} ^{M^{(1)} } \sigma^{(1)}_{i,t} ({\bm 0},{\bm x}^{(1)}_{t+1} )   +
    (2 \beta ^{1/2} + \eta _t )C_3  \tilde{\sigma}^{(N)}_{1,t} (\tilde{\bm x}^{(2)} ,\ldots , \tilde{\bm x}^{(N)} |{\bm y}^{(1)} ) \\
              & \leq
    (2 \beta ^{1/2} + \eta _t ) C_3 \sum_{i=1} ^{M^{(1)} } \sigma^{(1)}_{i,t} ({\bm 0},{\bm x}^{(1)}_{t+1} )  +
    (2 \beta ^{1/2} + \eta _t )C_3  b^{(2)}_t (\tilde{\bm x}^{(2)} |{\bm y}^{(1)} )                                                \\
              & \leq
    (2 \beta ^{1/2} + \eta _t ) C_3 \sum_{i=1} ^{M^{(1)} } \sigma^{(1)}_{i,t} ({\bm 0},{\bm x}^{(1)}_{t+1} )   +
    (2 \beta ^{1/2} + \eta _t )C_3 \eta^{-1}_t  c^{(2)}_t (\tilde{\bm x}^{(2)} |{\bm y}^{(1)} )                                    \\
              & \leq
    (2 \beta ^{1/2} + \eta _t ) C_3 \sum_{i=1} ^{M^{(1)} } \sigma^{(1)}_{i,t} ({\bm 0},{\bm x}^{(1)}_{t+1} )     +
    (2 \beta ^{1/2} + \eta _t )C_3 \eta^{-1}_t  c^{(2)}_t ( {\bm x}^{(2)}_{t+2} |{\bm y}^{(1)} )                                   \\
              & \leq
    (2 \beta ^{1/2} + \eta _t ) C_3 \sum_{i=1} ^{M^{(1)} } \sigma^{(1)}_{i,t} ({\bm 0},{\bm x}^{(1)}_{t+1} ) +
    (2 \beta ^{1/2} + \eta _t )^2 C_3  \eta^{-1}_t b^{(2)}_t ({\bm x}^{(2)}_{t+2} |{\bm y}^{(1)} )                                 \\
              & \leq
    (2 \beta ^{1/2} + \eta _t ) C_3 \sum_{i=1} ^{M^{(1)} } \sigma^{(1)}_{i,t} ({\bm 0},{\bm x}^{(1)}_{t+1} )    +
    (2 \beta ^{1/2} + \eta _t )^2 C_3 \eta^{-1}_t     \cdot \tilde{\sigma}^{(N)}_{1,t} ( {\bm x}^{(2)}_{t+2}, \tilde{\bm x}^{(3)},\ldots , \tilde{\bm x}^{(N)} |{\bm y}^{(1)} ) .
  \end{align}
  Hence, by using~\cref{lem:next} again, we have
  \begin{align}
    c^{(1)}_t ({\bm x}^{(1)}_{t+1} | {\bm 0} ) & \leq
    (2 \beta^{1/2} + \eta_t +1 ) ^N C^N_3 \eta^{-N}_t  \sum_{n=1}^{N} \sum_{i=1} ^{M^{(n)} } \sigma^{(n)}_{i,t} ({\bm y}^{(n-1)},{\bm x}^{(n)}_{t+n} ) \\
                                               & \leq
    (2 \beta^{1/2} + 1 +1 ) ^N C^N_3 \eta^{-N}_t  \sum_{n=1}^{N} \sum_{i=1} ^{M^{(n)} } \sigma^{(n)}_{i,t} ({\bm y}^{(n-1)},{\bm x}^{(n)}_{t+n} )      \\
                                               & =
    C_4 \eta^{-N}_t  \sum_{n=1}^{N} \sum_{i=1} ^{M^{(n)} } \sigma^{(n)}_{i,t} ({\bm y}^{(n-1)},{\bm x}^{(n)}_{t+n} )
    .
  \end{align}
  This implies that
  \begin{equation}
    c^{(1)2}_t ({\bm x}^{(1)}_{t+1} | {\bm 0} ) \leq
    C^2_4  M_{\text{sum}}  \eta^{-2N}_t \sum_{n=1}^{N} \sum_{i=1} ^{M^{(n)} } \sigma^{(n)2}_{i,t} ({\bm y}^{(n-1)},{\bm x}^{(n)}_{t+n} ),
  \end{equation}
  where the inequality is given by the Cauchy--Schwarz inequality.
  Next, let $T \in N \mathbb{Z}_{\geq 0}$ and $K =T/N $. Then, the following inequality holds:
  \begin{align}
    \sum_{ t \in N \mathbb{Z}_{\geq 0} }^{K N}  c^{(1)2}_t ({\bm x}^{(1)}_{t+1} | {\bm 0} ) &
    \leq
    C^2_4  M_{\text{sum}}  \eta^{-2N}_T \sum_{n=1}^{N} \sum_{i=1} ^{M^{(n)} }
    \sum_{ t \in N \mathbb{Z}_{\geq 0} }^{T}
    \sigma^{(n)2}_{i,t} ({\bm y}^{(n-1)},{\bm x}^{(n)}_{t+n} )                                  \\
                                                                                            &
    \leq
    \frac{2}{\log(1+\sigma^{-2})}
    C^2_4  M_{\text{sum}}  \eta^{-2N}_T  \sum_{n=1}^{N} \sum_{i=1} ^{M^{(n)} }
    \gamma_T                                                                                    \\
                                                                                            & =
    \frac{2C^2_4  M^2_{\text{sum}}   }{\log(1+\sigma^{-2})}
    \gamma_T \eta^{-2N}_T.
  \end{align}
  Similarly, let
  $ T^\ast = \argmin _{ t \in N \mathbb{Z}_{\geq 0} , t \leq T }  c^{(1)2}_t ({\bm x}^{(1)}_{t+1} | {\bm 0} ) $.
  Then, we obtain
  \begin{align}
    K c^{(1)2}_{T^\ast} ({\bm x}^{(1)}_{T^\ast+1} | {\bm 0} )
     & \leq \sum_{ t \in N \mathbb{Z}_{\geq 0} }^{K N}  c^{(1)2}_t ({\bm x}^{(1)}_{t+1} | {\bm 0} ) \\
     & \leq    \frac{2C^2_4  M^2_{\text{sum}}   }{\log(1+\sigma^{-2})}
    \gamma_T \eta^{-2N}_T.
  \end{align}
  This implies that
  \begin{align}
    c^{(1)}_{T^\ast} ({\bm x}^{(1)}_{T^\ast+1} | {\bm 0} )
    \leq \sqrt{ \frac{2C^2_4  M^2_{\text{sum}}   }{\log(1+\sigma^{-2})}
    \gamma_T \eta^{-2N}_T K^{-1}} . \label{eq:K_bound}
  \end{align}
  Furthermore, from the property of CIs and the definition of the estimated solution, the following inequalities hold:
  \begin{alignat}{2}
    F({\bm x}^{(1)}_\ast, \ldots ,{\bm x}^{(N)}_\ast )     & \leq \min _{ t \in N \mathbb{Z}_{\geq 0} , t \leq T }  \mr{ UCB}^{(F)} _t ({\bm x}^{(1)}_\ast ,\ldots ,{\bm x}^{(N)}_\ast ) & &\leq  \mr{ UCB}^{(F)} _{T^\ast }  ({\bm x}^{(1)}_\ast ,\ldots ,{\bm x}^{(N)}_\ast ) ,                                       \\
    F(\hat{\bm x}^{(1)}_T, \ldots , \hat{\bm x}^{(N)}_T  ) & \geq
    \max _{ t \in N \mathbb{Z}_{\geq 0} , t \leq T }  \mr{ LCB}^{(F)} _t ({\bm x}^{(1)} ,\ldots ,{\bm x}^{(N)} )   & & \geq
    \mr{ LCB}^{(F)} _{T^\ast }  ({\bm x}^{(1)}_\ast ,\ldots ,{\bm x}^{(N)}_\ast ) .
  \end{alignat}
  This implies that
  \begin{align}
    F({\bm x}^{(1)}_\ast, \ldots ,{\bm x}^{(N)}_\ast )  -
    F(\hat{\bm x}^{(1)}_T, \ldots , \hat{\bm x}^{(N)}_T  ) & \leq 2 \beta^{1/2} \tilde{\sigma}^{(N)} _{1,T^\ast } (   {\bm x}^{(1)}_\ast, \ldots ,{\bm x}^{(N)}_\ast  ) \\
                                                           & \leq 2\beta^{1/2}   b^{(1)}_{ T^\ast }  ({\bm x}^{(1)}_\ast | {\bm 0} )                                    \\
                                                           & \leq 2 {\beta}^{1/2}   \eta ^{-1}  _{T^\ast}   c^{(1)}_{ T^\ast }    ({\bm x}^{(1)}_\ast | {\bm 0} )  \leq
    2 {\beta}^{1/2}   \eta ^{-1}  _{T}   c^{(1)}_{ T^\ast }    ({\bm x}^{(1)}_{T^\ast +1}  | {\bm 0} ) .
  \end{align}
  Finally, noting that $K= T/N$, from~\cref{eq:K_bound} we obtain
  \begin{align}
    F({\bm x}^{(1)}_\ast, \ldots ,{\bm x}^{(N)}_\ast )  -
    F(\hat{\bm x}^{(1)}_T, \ldots , \hat{\bm x}^{(N)}_T  ) & \leq
    2 \beta^{1/2} \eta^{-1}_T \sqrt{ \frac{2C^2_4  M^2_{\text{sum}}   }{\log(1+\sigma^{-2})}
    \gamma_T \eta^{-2N}_T K^{-1}}                                 \\
                                                           & =
    \sqrt{ \frac{8 \beta C^2_4  M^2_{\text{sum}}  N }{\log(1+\sigma^{-2})}
    \gamma_T \eta^{-2N-2}_T T^{-1}}                               \\
                                                           & <
    \sqrt{\xi^2} = \xi.
  \end{align}
\end{proof}

%-------------------------------------------------------------------------------------------
\section{Cascade Process Optimization Using CI-based AFs under Noisy Setting}\label{app:CI-based-AF-noisy}
In this section, we consider CI-based cascade process optimization methods with observation noise.
Hereafter, we assume that the observation noise $\epsilon^{(n)}_m $ is a random variable with
$\mathbb{E}[\epsilon^{(n)}_m ] =0$ and $-A \leq \epsilon^{(n)}_m \leq A$, where $A$ is some positive constant.
In addition, we assume that
$\epsilon^{(1)}_1,\ldots , \epsilon^{(N)}_{M^{(N)}} $ are mutually independent, and the distribution of the noise vector ${\bm \epsilon } =
    (\epsilon^{(1)}_1,\ldots , \epsilon^{(N)}_{M^{(N)}} )^\top $ is known.
Finally, we also assume that  noise vectors with respect to iteration $t$, ${\bm\epsilon}_1, \ldots ,{\bm \epsilon}_t$, are
independent and identically distributed random variables having the same distribution of ${\bm\epsilon}$.

Next, we define several notations.
For each $n \in [N]$, let $\tilde{\mathcal{Y}} ^{(n)}  (\supset \mathcal{Y}^{(n)} )$ be a set satisfying
\begin{equation}
    \forall  {\bm w} \in \tilde{\mathcal{Y}} ^{(n-1)}, \forall {\bm x} \in \mathcal{X}^{(n)},
    f^{(n)} ({\bm w} ,{\bm x} ) + {\bm \epsilon } ^{(n)} \in  \tilde{\mathcal{Y}} ^{(n)},
\end{equation}
where $\tilde{\mathcal{Y}} ^{(0)} =\{ {\bm 0} \}$ and ${\bm \epsilon }^{(n)} = (\epsilon^{(n)}_1,\ldots ,{\epsilon}^{(n)}_{M^{(n)}} )^\top$.
Note that
${\bm z}^{(n)} ({\bm x}^{(1)},\ldots,{\bm x}^{(n)} ) \in \mathcal{Y}^{(n)}  \subset \tilde{\mathcal{Y}} ^{(n)} $.
In addition, for any realization ${\bm \epsilon}$ and input ${\bm x}^{(1)},\ldots,{\bm x}^{(n)}$, we define ${\bm z}^{(n)}_{{\bm\epsilon}} ({\bm x}^{(1)},\ldots,{\bm x}^{(n)})$ as
\begin{equation}
    {\bm z}^{(n)}_{\bm \epsilon}  ({\bm x}^{(1)},\ldots,{\bm x}^{(n)} )    = \begin{cases}
        {\bm f}^{(1)} ({\bm 0},{\bm x}^{(1)} ) +{\bm \epsilon}^{(1)}                                                                 & (n=1),    \\
        {\bm f}^{(n)} ({\bm z}^{(n-1)} _{\bm\epsilon} ({\bm x}^{(1)},\ldots,{\bm x}^{(n-1)} ),{\bm x}^{(n)} )  +{\bm \epsilon}^{(n)} & (n\ge 2).
    \end{cases}
\end{equation}
Furthermore, we define the function $G ({\bm x}^{(1)},\ldots,{\bm x}^{(N)} )$ as
\begin{equation}
    G ({\bm x}^{(1)},\ldots,{\bm x}^{(N)} )=\mathbb{E}_{\bm\epsilon} [{\bm z}^{(N)}_{\bm \epsilon}  ({\bm x}^{(1)},\ldots,{\bm x}^{(n)} ) ].
\end{equation}
The function  $G ({\bm x}^{(1)},\ldots,{\bm x}^{(N)} )$ is the expected value of the final-stage output with respect to ${\bm\epsilon}$ when $ ({\bm x}^{(1)},\ldots,{\bm x}^{(N)} )$ is used.
We emphasize that
$F({\bm x}^{(1)},\ldots,{\bm x}^{(N)} ) \neq G ({\bm x}^{(1)},\ldots,{\bm x}^{(N)} )$ in general.
Similarly, we define the optimal solution of each function as
\begin{align}
    ({\bm x}^{(1)}_{F,\ast} ,\ldots ,{\bm x}^{(N)}_{F,\ast} ) & = \argmax_{  ({\bm x}^{(1)},\ldots,{\bm x}^{(N)} ) \in \mathcal{X} }  F ({\bm x}^{(1)},\ldots, {\bm x}^{(N)} ) , \\
    ({\bm x}^{(1)}_{G,\ast} ,\ldots ,{\bm x}^{(N)}_{G,\ast} ) & = \argmax_{  ({\bm x}^{(1)},\ldots,{\bm x}^{(N)} ) \in \mathcal{X} }  G ({\bm x}^{(1)},\ldots, {\bm x}^{(N)} ) .
    \label{eq:opt_FandG}
\end{align}
By using these, for the selected input  ${\bm x}^{(1)}_t,\ldots,{\bm x}^{(N)}_t$      at iteration $t$, we define the expected regret $ r_{G,t}$, cumulative expected regret $R_{G,T} $ and simple expected regret $r^{(\text{S})}_{G,T} $
as
\begin{gather}
    r_{G,t} = G ({\bm x}^{(1)}_{G,\ast} ,\ldots ,{\bm x}^{(N)}_{G,\ast} )  - G ({\bm x}^{(1)}_t,\ldots,{\bm x}^{(N)}_t),\\
    R_{G,t} = \sum _{t=1}^T r_{G_t} , \quad
    r^{(\text{S})}_{G,T} = \min_{1 \leq t \leq T} r_{G,t}.
\end{gather}
We also define the regret
$ r_{F,t}$, cumulative regret $R_{F,T} $ and simple regret $r^{(\text{S})}_{F,T} $ as
\begin{gather}
    r_{F,t} = F ({\bm x}^{(1)}_{F,\ast} ,\ldots ,{\bm x}^{(N)}_{F,\ast} )  - F ({\bm x}^{(1)}_t,\ldots,{\bm x}^{(N)}_t),\\
    R_{F,t} = \sum _{t=1}^T r_{F_t} , \quad
    r^{(\text{S})}_{F,T} = \min_{1 \leq t \leq T} r_{F,t}.
\end{gather}
Finally,
let $(\hat{\bm x}^{(1)}_{F,t}, \ldots , \hat{\bm x}^{(N)}_{F,t})$
and
$(\hat{\bm x}^{(1)}_{G,t}, \ldots , \hat{\bm x}^{(N)}_{G,t})$
be  respectively estimated solutions of $({\bm x}^{(1)}_{F,\ast} ,\ldots ,{\bm x}^{(N)}_{F,\ast} ) $
and $({\bm x}^{(1)}_{G,\ast} ,\ldots ,{\bm x}^{(N)}_{G,\ast} ) $
at iteration
$t$.
Then, we define the  regrets for estimated solutions, $\hat{r}_{F,t}$ and $\hat{r}_{G,t}$, as
\begin{align}
    \hat{r}_{F,t} & = F ({\bm x}^{(1)}_{F,\ast} ,\ldots ,{\bm x}^{(N)}_{F,\ast} )  - F (\hat{\bm x}^{(1)}_{F,t}, \ldots , \hat{\bm x}^{(N)}_{F,t})
    ,                                                                                                                                               \\
    \hat{r}_{G,t} & = G ({\bm x}^{(1)}_{G,\ast} ,\ldots ,{\bm x}^{(N)}_{G,\ast} )  - G(\hat{\bm x}^{(1)}_{G,t}, \ldots , \hat{\bm x}^{(N)}_{G,t}) .
\end{align}

%-------------------------------------------------------------------------------------------
\subsection{Credible Interval}
In this section, we construct a valid CI under the noisy setting.
First, we introduce the following regularity assumption instead of~\cref{assumption:regu1}.
\begin{assumption}[Regularity assumption under noisy setting]\label{assumption:regu2}
    For each $n \in [N]$,
    let $\tilde{\mathcal{Y}}^{(n-1)} \times \mathcal{X}^{(n)}$ be a compact set, and let $\mathcal{H}_{ k^{(n)} } $ be an RKHS corresponding to the kernel $k^{(n)} $.
    In addition, for each $n \in [N]$ and $m \in [M^{(n)}]$, assume that $f^{(n)} _m \in \mathcal{H}_{ k^{(n)} } $
    with $\|f\stg{n} _m  \|_{{k\stg{n}}}\le B$, where $B>0$ is some constant.
    Furthermore, assume that the observation noise $\epsilon^{(n)}_m $ is a random variable
    with
    $\mathbb{E}[\epsilon^{(n)}_m ] =0$ and $-A \leq \epsilon^{(n)}_m \leq A$, where $A$ is some positive constant.
    All elements of ${\bm \epsilon} =(\epsilon^{(1)}_1,\ldots , \epsilon^{(N)}_{M^{(N)}} )$ are mutually independent, and
    ${\bm \epsilon}_1, {\bm \epsilon}_2, \ldots $ are i.i.d. random variables having the same distribution of ${\bm \epsilon}$.
\end{assumption}
Then, the following lemma holds under the noisy setting.
\begin{lemma}[{\citealt[Theorem~3.11]{szepesvari2012online}}]
    Assume that~\cref{assumption:regu2} holds.
    Let
    $\delta\in (0,1)$, and define
    \begin{equation}
        \beta^{(n)}_{t}  = \Bigl( B+    \frac{A}{\sigma} \sqrt{\log \operatorname{det}\left(\boldsymbol{I}_{L_t^{(n)}}+\sigma^{-2} \boldsymbol{K}\stg{n}_{t}\right)+2 \log (1 / \delta)} \Bigr)^{2}.
    \end{equation}
    Then, for any $n \in [N]$ and $m \in [M^{(n)}]$, the following inequality holds with probability at least $1-\delta$:
    \begin{equation}
        \left|f\stg{n}_m(\boldsymbol{w}, \boldsymbol{x})-\mu\stg{n}_{m,t}(\boldsymbol{w}, \boldsymbol{x})\right| \leq ( \beta_{t}^{(n) })^{1/2} \sigma\stg{n}_{m,t}(\boldsymbol{w}, \boldsymbol{x}),  \;  \forall \boldsymbol{w} \in \tilde{\mathcal{Y}}^{(n-1)},\; \forall \boldsymbol{x} \in \mathcal{X}\stg{n},\; \forall t \geq 1.
    \end{equation}
    \label{lem:ci_2}
\end{lemma}
\begin{proof}
    From Theorem~3.11 of~\cite{szepesvari2012online}, it is sufficient to show that
    $\epsilon^{(n)}_m$ has $A$-sub-Gaussian property, i.e.,
    \begin{equation}
        \mathbb{E} [ \exp ( \lambda \epsilon^{(n)}_m ) ]  \leq \exp (   \lambda^2 A^2 /2 )    \quad \forall \lambda \in \mathbb{R}.
        \label{eq:A_sub_Gaussian}
    \end{equation}
    Noting that $\epsilon^{(n)}_m$ is a zero mean and bounded random variable, using
    Hoeffding's lemma~\cite{massart2007concentration} we have
    \begin{align}
        \mathbb{E}[ \exp (\lambda \epsilon^{(n)}_m) ] & \leq \exp ( \lambda^2 (A-(-A))^2 /8 )                               \\
                                                      & =\exp (   \lambda^2 A^2 /2 )  \quad \forall \lambda \in \mathbb{R}.
    \end{align}
    Thus, $\epsilon^{(n)}_m$ has $A$-sub-Gaussian property~\cref{eq:A_sub_Gaussian}.
\end{proof}
From~\cref{lem:ci_2}, we have the following uniform bound.
\begin{corollary}\label{cor:ci2}
    Assume that~\cref{assumption:regu2} holds.
    Let
    $\delta\in (0,1)$, and define
    \begin{align}
        \beta^{(n)}_{t} & = \Bigl( B+    \frac{A}{\sigma}  \sqrt{\log \operatorname{det}\left(\boldsymbol{I}_{L_t^{(n)}}+\sigma^{-2} \boldsymbol{K}\stg{n}_{t}\right)+2 \log (M_{\text{sum}} / \delta)} \Bigr)^{2}, \\
        \beta _t        & = \max _{1 \leq n \leq N, 1 \leq \tilde{t} \leq t} \beta ^{(n)}_{\tilde{t}} .
        \label{eq:definition_beta_t}
    \end{align}
    Then, for any $ n \in [N]$ and $ m \in [M^{(n)}],$ the following inequality holds with probability at least $1-\delta$:
    \begin{equation}
        \left|f\stg{n}_m(\boldsymbol{w}, \boldsymbol{x})-\mu\stg{n}_{m,t}(\boldsymbol{w}, \boldsymbol{x})\right| \leq \beta_{t}^{ 1 / 2} \sigma\stg{n}_{m,t}(\boldsymbol{w}, \boldsymbol{x}), \;   \forall \boldsymbol{w} \in \tilde{\mathcal{Y}}^{(n-1)},\; \forall \boldsymbol{x} \in \mathcal{X}\stg{n},\; \forall t \geq 1.
    \end{equation}
\end{corollary}
From~\cref{cor:ci2},
we can also construct a valid CI for the $N$-stage cascade process under the noisy setting.
First, we construct CIs for ${\bm z}^{(n)}_{\bm\epsilon} ( {\bm x}^{(1)},\ldots,{\bm x}^{(n)})$ and
$G ({\bm x}^{(1)},\ldots,{\bm x}^{(N)} )$.
For any iteration $t \geq 1$, realization ${\bm\epsilon}$ and input $    ({\bm x}^{(1)},\ldots,{\bm x}^{(n)} )      $,  we define $\tilde{\bm z}^{(n)}_{{\bm \epsilon},t} ( {\bm x}^{(1)},\ldots,{\bm x}^{(n)})$ as
\begin{align}
    \tilde{\bm z}^{(n)}_{{\bm \epsilon},t}  ({\bm x}^{(1)},\ldots,{\bm x}^{(n)} ) & =     \begin{cases}
                                                                                              {\bm \epsilon}^{(1)}     +
                                                                                              {\bm \mu}^{(1)}_t ({\bm 0},{\bm x}^{(1)} )
                                                                                               & (n=1),    \\
                                                                                              {\bm \epsilon}^{(n)} +
                                                                                              {\bm \mu}^{(n)}_t (\tilde{\bm z}^{(n-1)} _{{\bm \epsilon},t} ({\bm x}^{(1)},\ldots,{\bm x}^{(n-1)} ),{\bm x}^{(n)} )
                                                                                               & (n\ge 2).
                                                                                          \end{cases}
\end{align}
Similarly, we define $\tilde{\sigma}^{(n)}_{{\bm \epsilon},m,t } ( {\bm x}^{(1)},\ldots,{\bm x}^{(n)})$ as
\begin{align}
    \tilde{\sigma}^{(n)}_{{\bm \epsilon},m,t }  ( {\bm x}^{(1)},\ldots,{\bm x}^{(n)}) & =
    {\sigma}^{(n)}_{m,t } (\tilde{\bm z}^{(n-1)}_{{\bm \epsilon},t}  ({\bm x}^{(1)},\ldots,{\bm x}^{(n-1)} ) ,{\bm x}^{(n)}) \\
                                                                                      & \qquad +
    L_f \sum_{s=1}^ {M^{(n-1)}} \tilde{\sigma}^{(n-1)}_{{\bm \epsilon},s,t } ( {\bm x}^{(1)},\ldots,{\bm x}^{(n-1)}),
\end{align}
where $m \in [M^{(n)}]$ and
$\tilde{\sigma}^{(1)}_{{\bm \epsilon},s,t } ( {\bm x}^{(1)}) = \sigma^{(1)}_{s,t} ({\bm 0},{\bm x}^{(1)} )$.
Then, the following holds.
\begin{theorem}\label{theorem:Nstage3}
    Assume that~\cref{assumption:regu2,assumption:L1} hold.
    Also assume that  $\tilde{\bm {z}}^{(n)}_{{\bm\epsilon},t} ({\bm x}^{(1)},\ldots,{\bm x}^{(n)}) \in \tilde{\mathcal{Y}}^{(n)}$
    for any $n \in [N]$, iteration $t \geq 1$, realization ${\bm \epsilon}$ and input $ ({\bm x}^{(1)},\ldots,{\bm x}^{(n)}) $.
    Let $\delta \in (0,1)$, and define $\beta_t $ by~\ref{eq:definition_beta_t}.
    Then, the following inequality  holds with probability at least $1-\delta$:
    \begin{align}
         & |{z}^{(n)}_{{\bm \epsilon},m}({\bm x}^{(1)}, \ldots, {\bm x}^{(n)} )  -\tilde{z}^{(n)}_{{\bm \epsilon},m,t} ({\bm x}^{(1)}, \ldots, {\bm x}^{(n)} )    |  \leq
        {\beta}^{1/2}_t  \tilde{\sigma}^{(n)}_{{\bm \epsilon},m,t} ({\bm x}^{(1)}, \ldots, {\bm x}^{(n)} ),                                                               \\
         & \qquad \forall n \in [N], m \in [M^{(n)}], t \geq 1, {\bm\epsilon}, ({\bm x}^{(1)},\ldots,{\bm x}^{(n)})
        , \label{eq:n_s_tilde}
    \end{align}
    where ${z}^{(n)}_{{\bm \epsilon},m} $ and  $\tilde{z}^{(n)}_{{\bm \epsilon},m,t} $ are  the $m$-th element of
    ${\bm z}^{(n)}_{{\bm \epsilon}} $ and  $\tilde{\bm z}^{(n)}_{{\bm \epsilon},t} $, respectively.
\end{theorem}
\begin{proof}
    By using the same argument as in the proof of~\cref{theorem:Nstage}, we get~\cref{theorem:Nstage3}.
\end{proof}
From~\cref{theorem:Nstage3}, taking expectation with respect to ${\bm \epsilon}$, we get the following corollary.
\begin{corollary}\label{cor:Nstage3}
    Assume that the same condition as in~\cref{theorem:Nstage3} holds.
    Let $\delta \in (0,1)$, and define $\beta_t $ by~\cref{eq:definition_beta_t}.
    Then, with probability at least
    $1-\delta$, the following inequality holds for any $n \in [N]$, $m \in [M^{(n)}] $, iteration $t \geq 1$ and input ${\bm x}^{(1)},\ldots,{\bm x}^{(n)}$:
    \begin{align}
         & \mathbb{E}_{\bm \epsilon}  [\tilde{z}^{(n)}_{{\bm \epsilon},m,t} ({\bm x}^{(1)}, \ldots, {\bm x}^{(n)} )   -
        \beta^{1/2}_t  \tilde{\sigma}^{(n)}_{{\bm \epsilon},m,t} ({\bm x}^{(1)}, \ldots, {\bm x}^{(n)} )]                    \\
         & \leq
        \mathbb{E}_{\bm \epsilon} [{z}^{(n)}_{{\bm \epsilon},m}({\bm x}^{(1)}, \ldots, {\bm x}^{(n)} )]                      \\
         & \leq \mathbb{E}_{\bm \epsilon} [\tilde{z}^{(n)}_{{\bm \epsilon},m,t} ({\bm x}^{(1)}, \ldots, {\bm x}^{(n)} )    +
        \beta^{1/2}_t  \tilde{\sigma}^{(n)}_{{\bm \epsilon},m,t} ({\bm x}^{(1)}, \ldots, {\bm x}^{(n)} )  ].
    \end{align}
    In particular, when $n=N$ and  $m=1$, it follows that
    \begin{align}
         & \mathbb{E}_{\bm \epsilon} [\tilde{z}^{(N)}_{{\bm \epsilon},1,t} ({\bm x}^{(1)}, \ldots, {\bm x}^{(N)} )   -
        \beta^{1/2}_t  \tilde{\sigma}^{(N)}_{{\bm \epsilon},1,t} ({\bm x}^{(1)}, \ldots, {\bm x}^{(N)} )]                \\
         & \leq
        G ({\bm x}^{(1)}, \ldots, {\bm x}^{(N)} )                                                                        \\
         & \leq \mathbb{E}_{\bm \epsilon} [\tilde{z}^{(N)}_{{\bm \epsilon},1,t} ({\bm x}^{(1)}, \ldots, {\bm x}^{(N)} )+
        \beta^{1/2}_t  \tilde{\sigma}^{(N)}_{{\bm \epsilon},1,t} ({\bm x}^{(1)}, \ldots, {\bm x}^{(N)} )  ].
    \end{align}
\end{corollary}
%-------------------------------------------------------------------------------------------
\subsection{UCB-based Optimization Strategy for Expected Regrets}
Here, we give a UCB-based AF and regret bounds for
$R_{G,T}$ and $r^{(\text{S})}_{G,T}$.
We define an
expected cascade process upper confidence bound  (EcUCB) as
\begin{equation}
    \mr{ EcUCB}_t ({\bm x}^{(1)},\ldots,{\bm x}^{(N)} )
    =
    \mathbb{E}_{\bm \epsilon} [\tilde{z}^{(N)}_{{\bm \epsilon},1,t} ({\bm x}^{(1)}, \ldots, {\bm x}^{(N)} )     +
    \beta^{1/2}_t  \tilde{\sigma}^{(N)}_{{\bm \epsilon},1,t} ({\bm x}^{(1)}, \ldots, {\bm x}^{(N)} )  ].
\end{equation}

By using this AF, we select the next evaluation point
$({\bm x}^{(1)}_{t+1}, \ldots , {\bm x}^{(N)}_{t+1} ) $ by
\begin{align}
    ({\bm x}^{(1)}_{t+1}, \ldots , {\bm x}^{(N)}_{t+1} ) = \argmax _{ ({\bm x}^{(1)}, \ldots , {\bm x}^{(N)} )   \in \mathcal{X} } \mr{ EcUCB}_t ({\bm x}^{(1)},\ldots, {\bm x}^{(N) } ) .
    \label{eq:decision_EcUCB}
\end{align}
Moreover, let
$\tilde{A} = \{ {\bm \tilde{a}}_1,\ldots, {\bm \tilde{a}}_T \}$ be a subset of $\tilde{\mathcal{Y}}^{(n-1)} \times \mathcal{X}^{(n)} $,
and let
${\bm y}^{(n)}_{m,\tilde{A}}$ be a random vector, where the $i$-th element of ${\bm y}^{(n)}_{m,\tilde{A}}$ is given by  $y^{(n)}_{m,{\bm \tilde{a}}_i } =f^{(n)}_m ({\bm \tilde{a}}_i ) + \varepsilon^{(n)}_{{\bm \tilde{a}}_i } $. Then, the
maximum information gain $\tilde{\gamma}^{(n)}_{m,T} $ at iteration
$T$ is given by
\begin{align}
    \tilde{\gamma}^{(n)}_{m,T} = \max_{  \tilde{A} \subset \tilde{\mathcal{Y}}^{(n-1)} \times \mathcal{X}^{(n)}, |\tilde{A}| =T }  \mr{I} ({\bm y}^{(n)}_{m,\tilde{A}} ; f^{(n)}_m ).  \label{eq:tilde_MIG}
\end{align}
Furthermore, we define $\tilde{\gamma}_T = \max_{ 1 \leq n \leq N, 1 \leq m \leq M^{(n)}} \tilde{\gamma}^{(n)}_{m,T}$.
Then, the following theorem gives regret bounds for $R_{G,T}$ and $r^{(\text{S})}_{G,T}$.
\begin{theorem}\label{theorem:EX_CandS}
    Assume that~\cref{assumption:regu2,assumption:L1,assumption:L2} hold.
    Also assume that  $\tilde{\bm {z}}^{(n)}_{{\bm\epsilon},t} ({\bm x}^{(1)},\ldots,{\bm x}^{(n)}) \in \tilde{\mathcal{Y}}^{(n)}$
    for any $n \in [N]$, iteration $t \geq 1$, realization ${\bm \epsilon}$ and input $ ({\bm x}^{(1)},\ldots,{\bm x}^{(n)}) $.
    Let $\delta \in (0,1)$, and define $\beta_t $ by~\cref{eq:definition_beta_t}.
    Then, when the optimization is performed using~\ref{eq:decision_EcUCB},
    the following holds:
    \begin{align}
        \mathbb{P} & \left (  R_{G,T} \leq\sqrt{32  M^2_{\text{prod}} M^2_{\text{sum}}  T   {C}^{2N}_{0,T}  \left (  \log (5 M_{\text{sum}} /\delta) +
        \frac{  \tilde{\gamma}_T  }{   2 \log (1+\sigma^{-2} )   } \right)}   \quad \forall T \geq 1\right ) \geq 1- 2\delta,                                                  \\
        \mathbb{P} & \left (  r^{(\text{S})}_{G,T} \leq T^{-1/2} \sqrt{  32  M^2_{\text{prod}} M^2_{\text{sum}}     {C}^{2N}_{0,T}  \left (  \log (5 M_{\text{sum}} /\delta) +
            \frac{  \tilde{\gamma}_T  }{   2 \log (1+\sigma^{-2} )   } \right) }      \quad \forall T \geq 1 \right ) \geq 1- 2\delta,
    \end{align}
    where
    $C_{0,t} = (1+L_\sigma) \beta^{1/2}_t + L_f +1 $.
\end{theorem}
\begin{proof}
    From~\cref{theorem:Nstage3} and the definition of  $\tilde{\sigma}^{(n)}_{{\bm \epsilon},m,t} (\cdot) $, noting that ${\sigma}^{(n)}_{m,t} (\cdot)$ is  Lipschitz continuity, the following inequality holds with probability at least $1-\delta$:
    \begin{align}
         & \tilde{\sigma}^{(n)}_{{\bm \epsilon},m,t} ({\bm x}^{(1)},\ldots, {\bm x}^{(n)} )                                                                                                                                                                             \\
         & = \sigma^{(n)} _{m,t} ({\bm z}^{(n-1)}_{\bm \epsilon} ({\bm x}^{(1)},\ldots , {\bm x}^{(n-1)} ),{\bm x}^{(n)} )   +L_f \sum_{s=1}^{M^{(n-1)} }  \tilde{\sigma}^{(n-1)}_{{\bm \epsilon},s,t}  ({\bm x}^{(1)},\ldots,{\bm x}^{(n-1) } )                        \\
         & \quad  + \sigma^{(n)} _{m,t} (\tilde{\bm z}^{(n-1)}_{{\bm \epsilon},t} ({\bm x}^{(1)},\ldots , {\bm x}^{(n-1)} ),{\bm x}^{(n)} )    -  \sigma^{(n)} _{m,t} ({\bm z}^{(n-1)}_{\bm \epsilon} ({\bm x}^{(1)},\ldots , {\bm x}^{(n-1)} ),{\bm x}^{(n)} )         \\
         & \leq
        \sigma^{(n)} _{m,t} ({\bm z}^{(n-1)}_{\bm \epsilon} ({\bm x}^{(1)},\ldots , {\bm x}^{(n-1)} ),{\bm x}^{(n)} )        +L_f \sum_{s=1}^{M^{(n-1)} }  \tilde{\sigma}^{(n-1)}_{{\bm \epsilon},s,t}  ({\bm x}^{(1)},\ldots,{\bm x}^{(n-1) } )                        \\
         & \quad  +  |\sigma^{(n)} _{m,t} ({\bm z}^{(n-1)}_{\bm \epsilon} ({\bm x}^{(1)},\ldots , {\bm x}^{(n-1)} ),{\bm x}^{(n)} )     - \sigma^{(n)} _{m,t} (\tilde{\bm z}^{(n-1)}_{{\bm \epsilon},t} ({\bm x}^{(1)},\ldots , {\bm x}^{(n-1)} ),{\bm x}^{(n)} )|      \\
         & \leq
        \sigma^{(n)} _{m,t} ({\bm z}^{(n-1)}_{\bm \epsilon} ({\bm x}^{(1)},\ldots , {\bm x}^{(n-1)} ),{\bm x}^{(n)} )           +L_f \sum_{s=1}^{M^{(n-1)} }  \tilde{\sigma}^{(n-1)}_{{\bm \epsilon},s,t}  ({\bm x}^{(1)},\ldots,{\bm x}^{(n-1) } )                     \\
         & \quad + L_\sigma \sum_{s=1} ^{M^{(n-1)}}   |   { z}_{{\bm\epsilon},s}^{(n-1)} ({\bm x}^{(1)},\ldots , {\bm x}^{(n-1)} ) -     \tilde{ z}_{{\bm\epsilon},s,t}^{(n-1)} ({\bm x}^{(1)},\ldots , {\bm x}^{(n-1)} )      |                                        \\
         & \leq
        \sigma^{(n)}_{m,t} ({\bm z}^{(n-1)}_{\bm \epsilon} ({\bm x}^{(1)},\ldots , {\bm x}^{(n-1)} ),{\bm x}^{(n)} )                  +L_f \sum_{s=1}^{M^{(n-1)} }  \tilde{\sigma}^{(n-1)}_{ {\bm\epsilon}, s,t}  ({\bm x}^{(1)},\ldots,{\bm x}^{(n-1) } )              \\
         & \quad + L_\sigma \beta^{1/2}_t \sum_{s=1} ^{M^{(n-1)}}    \tilde{\sigma}^{(n-1)}_{ {\bm\epsilon}, s,t}  ({\bm x}^{(1)},\ldots,{\bm x}^{(n-1) } )                                                                                                             \\
         & =
        \sigma^{(n)} _{m,t}({\bm z}^{(n-1)}_{\bm \epsilon} ({\bm x}^{(1)},\ldots , {\bm x}^{(n-1)} ),{\bm x}^{(n)} )   + (L_\sigma \beta^{1/2}_t +L_f ) \sum_{s=1} ^{M^{(n-1)}}    \tilde{\sigma}^{(n-1)}_{{\bm\epsilon},s,t}  ({\bm x}^{(1)},\ldots,{\bm x}^{(n-1) } ) \\
         & \leq
        \sigma^{(n)} _{m,t} ({\bm z}^{(n-1)}_{\bm \epsilon} ({\bm x}^{(1)},\ldots , {\bm x}^{(n-1)} ),{\bm x}^{(n)} )  + {C}_{0,t} \sum_{s=1} ^{M^{(n-1)}}    \tilde{\sigma}^{(n-1)}_{{\bm\epsilon},s,t}  ({\bm x}^{(1)},\ldots,{\bm x}^{(n-1) } ). \label{eq:ECtilde}
    \end{align}
    Therefore, by repeating~\cref{eq:ECtilde} we get
    \begin{align}
         & \tilde{\sigma}^{(N)}_{  {\bm\epsilon},  1,t} ({\bm x}^{(1)},\ldots, {\bm x}^{(N)} )                                                                                                                                                                                         \\
         & \leq\sigma^{(N)} _{1,t} ({\bm z}^{(N-1)}_{\bm \epsilon} ({\bm x}^{(1)},\ldots , {\bm x}^{(N-1)} ),{\bm x}^{(N)} )   + {C}_{0,t} \sum_{s=1} ^{M^{(N-1)}}    \tilde{\sigma}^{(N-1)}_{{\bm\epsilon},s,t}  ({\bm x}^{(1)},\ldots,{\bm x}^{(N-1) } )                             \\
         & \leq\sigma^{(N)} _{1,t} ( {\bm z}^{(N-1)}_{\bm \epsilon} ({\bm x}^{(1)},\ldots , {\bm x}^{(N-1)} ),{\bm x}^{(N)} )  + {C}_{0,t} \sum_{s=1} ^{M^{(N-1)}}   \sigma^{(N-1)} _{s,t} ({\bm z}^{(N-2)}_{\bm \epsilon} ({\bm x}^{(1)},\ldots , {\bm x}^{(N-2)} ),{\bm x}^{(N-1)} ) \\
         & \qquad  + {C}^2_{0,t}  M^{(N-1)}  \sum_{u=1} ^{M^{(N-2)}}    \tilde{\sigma}^{(N-2)}_{{\bm\epsilon},u,t}  ({\bm x}^{(1)},\ldots,{\bm x}^{(N-2) } )                                                                                                                           \\
         & \vdots                                                                                                                                                                                                                                                                      \\
         & \leq  {C}^{N-1}_{0,t} M_{\text{prod}}\sum _{n=1}^N  \sum_{m=1} ^ {M^{(n)} }  \sigma^{(n)} _{m,t} ({\bm z}^{(n-1)}_{\bm \epsilon} ({\bm x}^{(1)},\ldots , {\bm x}^{(n-1)} ),{\bm x}^{(n)} ) ,
    \end{align}
    where ${\bm z}^{(0)}_{\bm \epsilon} ({\bm x}^{(1)}, {\bm x}^{(0)} )={\bm 0}$.
    Hence, from the Cauchy--Schwarz inequality, it follows that
    \begin{align}
         & \tilde{\sigma}^{(N)2}_{{\bm\epsilon},1,t} ({\bm x}^{(1)},\ldots, {\bm x}^{(N)} ) \\
         & \leq
        {C}^{2(N-1)}_{0,t} M^2_{\text{prod}}     \left ( \sum _{n=1}^N  \sum_{m=1} ^ {M^{(n)} }  1   \right )      \sum _{n=1}^N  \sum_{m=1} ^ {M^{(n)} } \Bigl[ \sigma^{(n)2} _{m,t} ({\bm z}^{(n-1)}_{\bm \epsilon} ({\bm x}^{(1)},\ldots , {\bm x}^{(n-1)} ),{\bm x}^{(n)} ) \Bigr]
        \\
         & =
        {C}^{2(N-1)}_{0,t} M^2_{\text{prod}}     M_{\text{sum}}      \sum _{n=1}^N  \sum_{m=1} ^ {M^{(n)} }  \Bigl[     \sigma^{(n)2} _{m,t} ({\bm z}^{(n-1)}_{\bm \epsilon} ({\bm x}^{(1)},\ldots , {\bm x}^{(n-1)} ),{\bm x}^{(n)} )\Bigr].
        \label{eq:EC_S}
    \end{align}

    Next, from~\cref{cor:Nstage3} and the selection rule~\cref{eq:decision_EcUCB}, the following holds:
    \begin{align}
        G({\bm x}^{(1)}_{G,\ast},\ldots,{\bm x}^{(N)}_{G,\ast} ) & \leq \mr{ EcUCB}_t ({\bm x}^{(1)}_{G,\ast},\ldots,{\bm x}^{(N)}_{G,\ast} )                                                                                                                                                         \\
                                                                 & \leq
        \mr{ EcUCB}_t ({\bm x}^{(1)}_{t+1},\ldots,{\bm x}^{(N)}_{t+1} )                                                                                                                                                                                                                               \\
                                                                 & = \mathbb{E}_{\bm\epsilon}[ \tilde{z}^{(N)}_{{\bm\epsilon},1,t} ({\bm x}^{(1)}_{t+1},\ldots,{\bm x}^{(N)}_{t+1} )   + \beta^{1/2}_t \tilde{\sigma}^{(N)}_{{\bm\epsilon},1,t} ({\bm x}^{(1)}_{t+1},\ldots,{\bm x}^{(N)}_{t+1} ) ] .
    \end{align}
    Similarly, since $G ({\bm x}^{(1)}_{t+1},\ldots,{\bm x}^{(N)}_{t+1} ) $ satisfies that
    \begin{equation}
        G ({\bm x}^{(1)}_{t+1},\ldots,{\bm x}^{(N)}_{t+1} ) \geq \mathbb{E}_{\bm\epsilon}  [ \tilde{z}^{(N)}_{{\bm\epsilon},1,t} ({\bm x}^{(1)}_{t+1},\ldots,{\bm x}^{(N)}_{t+1} )   \beta^{1/2}_t \tilde{\sigma}^{(N)}_{{\bm\epsilon},1,t} ({\bm x}^{(1)}_{t+1},\ldots,{\bm x}^{(N)}_{t+1} ) ]  ,
    \end{equation}
    the regret $r_{G,t} $ can be bounded as follows:
    \begin{align}
        r_{G,t} & = G({\bm x}^{(1)}_{G,\ast},\ldots,{\bm x}^{(N)}_{G,\ast} )  -G ({\bm x}^{(1)}_{t+1},\ldots,{\bm x}^{(N)}_{t+1} ) \\
                & \leq
        2\beta^{1/2}_t \mathbb{E}_{\bm\epsilon}[    \tilde{\sigma}^{(N)}_{{\bm\epsilon},1,t} ({\bm x}^{(1)}_{t+1},\ldots,{\bm x}^{(N)}_{t+1} ) ]
        .
        \label{eq:E_regret_bound}
    \end{align}
    Therefore, by using~\cref{eq:E_regret_bound},  $R^2_{G,T} $ can be written as
    \begin{align}
        R^2_{G,T} & = \left ( \sum _{t=1}^T r_{G,t} \right )^2                                                                                                                  \\
                  & \leq T \sum _{t=1} ^T r^2_{G,t}                                                                                                                             \\
                  & \leq T \sum_{t=1}^T 4 \beta_t (\mathbb{E}_{\bm\epsilon}[    \tilde{\sigma}^{(N)}_{{\bm\epsilon},1,t} ({\bm x}^{(1)}_{t+1},\ldots,{\bm x}^{(N)}_{t+1} ) ])^2 \\
                  & \leq
        T \sum_{t=1}^T 4 \beta_t \mathbb{E}_{\bm\epsilon}[    \tilde{\sigma}^{(N)2}_{{\bm\epsilon},1,t} ({\bm x}^{(1)}_{t+1},\ldots,{\bm x}^{(N)}_{t+1} ) ],
        \label{eq:E_CR}
    \end{align}
    where the first inequality is given by the Cauchy--Schwarz inequality, and the last inequality is given by Jensen's inequality.
    Thus, by substituting~\cref{eq:EC_S} into~\cref{eq:E_CR}, we obtain
    \begin{equation}
        R^2_{G,T} \leq 4T \beta_T  \tilde{C}^{2(N-1)}_T M^2_{\text{prod}} M_{\text{sum}}  \sum_{t=1}^T
        \mathbb{E}_{\bm\epsilon} [ S_{ {\bm\epsilon},t}], \label{eq:G_2_bound}
    \end{equation}
    where $S_{ {\bm\epsilon},t}$ is given by
    \begin{equation}
        S_{ {\bm\epsilon},t} =
        \sum _{n=1}^N  \sum_{m=1} ^ {M^{(n)} }  \sigma^{(n)2} _{m,t} ({\bm z}^{(n-1)}_{\bm \epsilon} ({\bm x}^{(1)}_{t+1},\ldots , {\bm x}^{(n-1)}_{t+1} ),{\bm x}^{(n)} _{t+1}).
    \end{equation}
    Here,
    since $k^{(n)} (\cdot, \cdot ) \leq 1$, the random variable
    $S_{ {\bm\epsilon},t}$ satisfies
    $0 \leq S_{ {\bm\epsilon},t} \leq M_{\text{sum}}$.
    Hence, from Lemma~3 of~\cite{kirschner2018information}, the following holds with probability at least $1-\delta$:
    \begin{align}
        \sum_{t=1}^T
        \mathbb{E}_{\bm\epsilon} [ S_{ {\bm\epsilon},t}] & \leq 2 \sum_{t=1}^T  S_{ {\bm\epsilon}_{t+1},t}   + 4 M_{\text{sum}} \log (1/\delta )  +
        8 M_{\text{sum}} \log (4 M_{\text{sum}} )+1                                                                                                 \\
                                                         & \leq
        2 \sum_{t=1}^T  S_{ {\bm\epsilon}_{t+1},t} + 8 M_{\text{sum}} \log (1/\delta )   +
        8 M_{\text{sum}} \log (4 M_{\text{sum}} )+ 8 M_{\text{sum}} \log 1.25                                                                       \\
                                                         & =2 \sum_{t=1}^T  S_{ {\bm\epsilon}_{t+1},t} +
        8 M_{\text{sum}} \log (5 M_{\text{sum}} /\delta).
        \label{eq:prob_bound}
    \end{align}
    Therefore, by combining~\cref{eq:G_2_bound,eq:prob_bound}, we have
    \begin{align}
        R^2_{G,T} & \leq 32T \beta_T  {C}^{2(N-1)}_{T,0} M^2_{\text{prod}} M^2_{\text{sum}} \log (5 M_{\text{sum}} /\delta) \\
                  & \qquad +
        8T \beta_T  {C}^{2(N-1)}_{0,T} M^2_{\text{prod}} M_{\text{sum}}  \sum_{t=1}^T
        \sum _{n=1}^N  \sum_{m=1} ^ {M^{(n)} } \Bigl[    \sigma^{(n)2} _{m,t} ({\bm z}^{(n-1)}_{{\bm \epsilon}_{t+1}} ({\bm x}^{(1)}_{t+1},\ldots , {\bm x}^{(n-1)}_{t+1} ),{\bm x}^{(n)} _{t+1})\Bigr]. \label{eq:R2G}
    \end{align}
    Furthermore, by using the same argument as in Lemma~5.3 and~5.4 of~\cite{srinivas2010gaussian}, we get
    \begin{align}
        \sum_{t=1}^T  \sigma^{(n) 2} _{m,t} ({\bm z}^{(n-1)}_{{\bm \epsilon}_{t+1}} ({\bm x}^{(1)}_{t+1},\ldots , {\bm x}^{(n-1)}_{t+1} ),{\bm x}^{(n)} _{t+1}) & \leq \frac{2}{ \log (1+\sigma^{-2} ) } \tilde{\gamma}^{(n)}_{m,T} \\
                                                                                                                                                                & \leq \frac{2}{ \log (1+\sigma^{-2} ) } \tilde{\gamma} _T  .
        \label{eq:E_gamma_bound}
    \end{align}
    Hence, from~\cref{eq:R2G} and~\cref{eq:E_gamma_bound}, noting that $\beta_T \leq {C}^2_{0,T}$ we obtain
    \begin{align}
        R^2_{G,T} & \leq 32T \beta_T  {C}^{2(N-1)}_{0,T} M^2_{\text{prod}} M^2_{\text{sum}} \log (5 M_{\text{sum}} /\delta)    + \frac{16}{   \log (1+\sigma^{-2} )   }  T \beta_T  {C}^{2(N-1)}_{0,T} M^2_{\text{prod}} M^2_{\text{sum}}  \tilde{\gamma}_T \\
                  & = 32T \beta_T  {C}^{2(N-1)}_{0,T} M^2_{\text{prod}} M^2_{\text{sum}} \left (  \log (5 M_{\text{sum}} /\delta)  +   \frac{  \tilde{\gamma}_T  }{   2 \log (1+\sigma^{-2} )   } \right)                                                   \\
                  & \leq    32  T   {C}^{2N}_{0,T} M^2_{\text{prod}} M^2_{\text{sum}} \left (  \log (5 M_{\text{sum}} /\delta) +   \frac{  \tilde{\gamma}_T  }{   2 \log (1+\sigma^{-2} )   } \right).
    \end{align}
    Therefore, with probability at least $1-2\delta$, $R_{G,T}$ can be bounded as follows:
    \begin{equation}
        R_{G,T} \leq   \sqrt{ 32  M^2_{\text{prod}} M^2_{\text{sum}}  T   {C}^{2N}_{0,T} \left (  \log (5 M_{\text{sum}} /\delta) +
        \frac{  \tilde{\gamma}_T  }{   2 \log (1+\sigma^{-2} )   } \right).
        }
    \end{equation}
    Similarly, from the definition of
    $r^{(\text{S})}_{G,T}$, it follows that
    \begin{align}
        T r^{(\text{S})}_{G,T} & \leq \sum_{t=1}^T r_{G,t}  =R_{G,T}                                                                               \\
                               & \leq  \sqrt{ 32  M^2_{\text{prod}} M^2_{\text{sum}}  T   {C}^{2N}_{0,T}\left (  \log (5 M_{\text{sum}} /\delta) +
        \frac{  \tilde{\gamma}_T  }{   2 \log (1+\sigma^{-2} )   } \right)
        }.
    \end{align}
\end{proof}

%-------------------------------------------------------------------------------------------
\subsection{Optimistic Improvement-based AF  for the Expectation of the Final Stage Output}
We give an optimistic improvement-based AF for $G$ under the noisy setting.
Let $s \in \{0,\ldots, N-1 \}$ and ${\bm { y} } \in \tilde{\mathcal{Y}}^{(s)}$.
Then, we define
${\bm z}^{(n)}_{{\bm \epsilon} } (\cdot|{\bm y}) $, $\tilde{\bm z}^{(n)}_{{\bm \epsilon} ,t}  (\cdot|{\bm y})$ and  $\tilde{\sigma}^{(n)}_{ {\bm\epsilon},m,t}  (\cdot|{\bm y})$ as
\begin{align}
    {\bm z}^{(n)}_{{\bm \epsilon}}  ({\bm x}^{(s+1)},\ldots,{\bm x}^{(n)} | {\bm y}) =           & \begin{cases}
                                                                                                       {\bm \epsilon}^{(s+1)}     +
                                                                                                       {\bm f}^{(s+1)} ({\bm {y}},{\bm x}^{(s+1)} )
                                                                                                        & (n=s+1),    \\
                                                                                                       {\bm \epsilon}^{(n)} +
                                                                                                       {\bm f}^{(n)} ({\bm z}^{(n-1)} _{{\bm \epsilon}} ({\bm x}^{(s+1:n-1)} |{\bm {y}}),{\bm x}^{(n)} )
                                                                                                        & (n\ge s+2),
                                                                                                   \end{cases}             \\
    \tilde{\bm z}^{(n)}_{{\bm \epsilon},t}  ({\bm x}^{(s+1)},\ldots,{\bm x}^{(n)} | {\bm y}) =   & \begin{cases}
                                                                                                       {\bm \epsilon}^{(s+1)}     +
                                                                                                       {\bm \mu}^{(s+1)}_t ({\bm {y}},{\bm x}^{(s+1)} )
                                                                                                        & (n=s+1),    \\
                                                                                                       {\bm \epsilon}^{(n)} +
                                                                                                       {\bm \mu}^{(n)}_t (\tilde{\bm z}^{(n-1)} _{{\bm \epsilon},t} ({\bm x}^{(s+1:n-1)} |{\bm {y}}),{\bm x}^{(n)} )
                                                                                                        & (n\ge s+2),
                                                                                                   \end{cases} \\
    \tilde{\sigma}^{(n)}_{{\bm \epsilon},m,t } ( {\bm x}^{(s+1)},\ldots,{\bm x}^{(n)} | {\bm y}) & =
    {\sigma}^{(n)}_{m,t } (\tilde{\bm z}^{(n-1)}_{{\bm \epsilon},t}  ({\bm x}^{(s+1)},\ldots,{\bm x}^{(n-1)} |{\bm y}) ,{\bm x}^{(n)})                                                                          \\
                                                                                                 & \quad   +
    L_f \sum_{s=1}^ {M^{(n-1)}} \tilde{\sigma}^{(n-1)}_{{\bm \epsilon},s,t } ( {\bm x}^{(s+1)},\ldots,{\bm x}^{(n-1)}|{\bm y}),
\end{align}
where
\sloppy$\tilde{\sigma}^{(s)}_{{\bm \epsilon},m,t } ( {\bm x}^{(s+1)},\ldots,{\bm x}^{(s)} | {\bm y})=0$ and
\sloppy${\bm z}^{(s)}_{{\bm \epsilon}}  ({\bm x}^{(s+1)},\ldots,{\bm x}^{(s)} | {\bm y})
    =
    \tilde{\bm z}^{(s)}_{{\bm \epsilon},t}  ({\bm x}^{(s+1)},\ldots,{\bm x}^{(s)} | {\bm y})  ={\bm y}$.
Then, the following theorem holds.
\begin{theorem}\label{theorem:Nstage_given}
    Assume that~\cref{assumption:regu2,assumption:L1} hold.
    Also assume that  $\tilde{\bm {z}}^{(n)}_{{\bm\epsilon},t} ({\bm x}^{(s+1)},\ldots,{\bm x}^{(n)} |{\bm y}) \in \tilde{\mathcal{Y}}^{(n)}$
    for any
    $s \in \{0,\ldots, N-1 \}$, $n \in \{s+1 ,\ldots, N \}$, iteration $t \geq 1$, realization ${\bm \epsilon}$, given $ {\bm y} \in \tilde{\mathcal{Y}}^{(s)} $ and input $ ({\bm x}^{(s+1)},\ldots,{\bm x}^{(n)}) $.
    Let $\delta \in (0,1)$, and define $\beta_t $ by~\cref{eq:definition_beta_t}.
    Then, the following inequality  holds with probability at least $1-\delta$:
    \begin{equation}
        |{z}^{(n)}_{{\bm \epsilon},m}({\bm x}^{(s+1)}, \ldots, {\bm x}^{(n)} |{\bm y})  -\tilde{z}^{(n)}_{{\bm \epsilon},m,t} ({\bm x}^{(s+1)}, \ldots, {\bm x}^{(n)} |{\bm y})    |  \leq
        {\beta}^{1/2}_t  \tilde{\sigma}^{(n)}_{{\bm \epsilon},m,t} ({\bm x}^{(s+1)}, \ldots, {\bm x}^{(n)} |{\bm y})
        , \label{eq:n_s_tilde_given}
    \end{equation}
    where $m \in [M^{(n)}]$, and ${z}^{(n)}_{{\bm \epsilon},m} (\cdot|{\bm y})$ and $\tilde{z}^{(n)}_{{\bm \epsilon},m,t} (\cdot|{\bm y})$ are the $m$-th element of
    ${\bm z}^{(n)}_{{\bm \epsilon}} (\cdot|{\bm y})$  and $\tilde{\bm z}^{(n)}_{{\bm \epsilon},t} (\cdot|{\bm y})$, respectively.
\end{theorem}
\begin{proof}
    By using the same argument as in the proof of~\cref{theorem:Nstage}, we have~\cref{theorem:Nstage_given}.
\end{proof}
From~\cref{theorem:Nstage_given}, taking expectation with respect to ${\bm\epsilon}$, we get the following corollary.
\begin{corollary}\label{co:Nstage_given_ex}
    Assume that the same condition as in~\cref{theorem:Nstage_given} holds.
    Let $\delta \in (0,1)$, and define $\beta_t $ by~\cref{eq:definition_beta_t}.
    Then, the following inequality  holds with probability at least $1-\delta$:
    \begin{align}
         &
        \mathbb{E}_{  {\bm\epsilon}|{\bm y}}  [
        \tilde{z}^{(n)}_{{\bm \epsilon},m,t} ({\bm x}^{(s+1)}, \ldots, {\bm x}^{(n)} |{\bm y})         -
        {\beta}^{1/2}_t  \tilde{\sigma}^{(n)}_{{\bm \epsilon},m,t} ({\bm x}^{(s+1)}, \ldots, {\bm x}^{(n)} |{\bm y}) ]                   \\
         & \leq \mathbb{E}_{  {\bm\epsilon} |{\bm y}}  [{z}^{(n)}_{{\bm \epsilon},m}({\bm x}^{(s+1)}, \ldots, {\bm x}^{(n)} |{\bm y})  ] \\
         & \leq
        \mathbb{E}_{  {\bm\epsilon} | {\bm y}}  [
        \tilde{z}^{(n)}_{{\bm \epsilon},m,t} ({\bm x}^{(s+1)}, \ldots, {\bm x}^{(n)} |{\bm y})       +
        {\beta}^{1/2}_t  \tilde{\sigma}^{(n)}_{{\bm \epsilon},m,t} ({\bm x}^{(s+1)}, \ldots, {\bm x}^{(n)} |{\bm y}) ] ,
    \end{align}
    where $\mathbb{E}_{  {\bm\epsilon}|{\bm y}}  [\cdot] $ is the conditional expectation of $(\cdot)$ given ${\bm y}$.
\end{corollary}
Based on this lemma, we give valid AFs.
Let $n \in [N]$ and ${\bm y}^{(n-1)} \in \tilde{\mathcal{Y}} ^{(n-1)}$.
Then, for any ${\bm x}^{(n)} \in \mathcal{X}^{(n)}$ and iteration $t \geq 1$, we define  the optimistic maximum value at  the final stage under given ${\bm y}^{(n-1)} $,
$\mr{ UCB}^{(G)}_t ({\bm x}^{(n) } | {\bm y}^{(n-1)} ) $, as
\begin{align}
     & \mr{ UCB}^{(G)}_t ({\bm x}^{(n) } | {\bm y}^{(n-1)} ) = \\
     & \max_{  ( {\bm x}^{(n+1)} ,\ldots , {\bm x}^{(N)} ) }
    \mathbb{E}_{ {\bm\epsilon}  |{\bm y}^{(n-1)} } \Bigl[
    (\tilde{{ z}}^{(N)}_{{\bm\epsilon},1,t} ({\bm x}^{(n)},\ldots, {\bm x}^{(N)}  |{\bm y}^{(n-1)})  + \beta^{1/2}_t \tilde{\sigma}^{(N)}_{{\bm\epsilon},1,t} ({\bm x}^{(n)},\ldots,{\bm x}^{(N)}  |{\bm y}^{(n-1)} ) ) \Bigr], \label{eq:UCB_y_EX}
\end{align}
where the max operator is ignored when $n=N$.
Similarly, we define  the pessimistic maximum value at  the final stage under given ${\bm y}^{(n-1)} $,
$\mr{ LCB}^{(G)}_t ({\bm x}^{(n) } | {\bm y}^{(n-1)} ) $, as
\begin{align}
     & \mr{ LCB}^{(G)}_t ({\bm y}^{(n-1)} )                  \\
     & = \max_{  ( {\bm x}^{(n)} ,\ldots , {\bm x}^{(N)} ) }
    \mathbb{E}_{{\bm\epsilon}    |{\bm y}^{(n-1)}    } \Bigl[
    (\tilde{z}^{(N)}_{{\bm\epsilon},1,t} ({\bm x}^{(n)},\ldots, {\bm x}^{(N)}  |{\bm y}^{(n-1)})      - \beta^{1/2}_t \tilde{\sigma}^{(N)}_{{\bm\epsilon},1,t} ({\bm x}^{(n)},\ldots,{\bm x}^{(N)}  |{\bm y}^{(n-1)} ) ) \Bigr] . \label{eq:LCB_y_EXP}
\end{align}
Moreover, for each $T \geq 1$,   we define  the pessimistic maximum value at  the final stage as
\begin{equation}
    Q_T =  \max_{  ( {\bm x}^{(1)} ,\ldots , {\bm x}^{(N)} ) }
    \mathbb{E}_{ {\bm\epsilon}} \Bigl[(\tilde{z}^{(N)}_{{\bm\epsilon},1,T} ({\bm x}^{(1)},\ldots, {\bm x}^{(N)} ) - \beta^{1/2}_T \tilde{\sigma}^{(N)}_{{\bm\epsilon},1,T} ({\bm x}^{(1)},\ldots,{\bm x}^{(N)} ) )\Bigr].
\end{equation}

Then, we define the  pessimistic improvement for the final stage with respect to ${\bm x}^{(n)}$ by
\begin{equation}
    \tilde{a}^{(n)}_t ({\bm x}^{(n)} | {\bm y}^{(n-1) } )=  \mr{ UCB}^{(G)}_t ({\bm x}^{(n) } | {\bm y}^{(n-1)} )     - \max \{ \mr{ LCB}^{(G)}_t ( {\bm y}^{(n-1)} ) , Q_{t+n-1} \}. \label{eq:improve_y_EX}
\end{equation}
We also define the maximum uncertainty for the final stage with respect to  ${\bm x}^{(n)}$ as
\begin{align}
    \tilde{b}^{(n)}_t ({\bm x}^{(n)} | {\bm y}^{(n-1) } )  =      \max_{  ( {\bm x}^{(n+1)} ,\ldots , {\bm x}^{(N)} ) } \mathbb{E}_{ {\bm\epsilon} |{\bm y}^{(n-1)} } [
    \tilde{\sigma}^{(N)}_{{\bm\epsilon},1,t} ({\bm x}^{(n)},\ldots,{\bm x}^{(N)}  |{\bm y}^{(n-1)} ) ]. \label{eq:US_y_EX}
\end{align}
Then, we give the   AF
$\tilde{c}^{(n)}_t ({\bm x}^{(n)} | {\bm y}^{(n-1) } )$ by
\begin{equation}
    \tilde{c}^{(n)}_t ({\bm x}^{(n)} | {\bm y}^{(n-1) } ) = \max \Bigl\{  \tilde{a}^{(n)}_t ({\bm x}^{(n)} | {\bm y}^{(n-1) } ),  \eta_t \tilde{b}^{(n)}_t ({\bm x}^{(n)} | {\bm y}^{(n-1) } ) \Bigr\},  \label{eq:af_se_EXq}
\end{equation}
where   $\eta_t $ is a given learning rate.
Furthermore, we select the next point ${\bm x}^{(n)}_{t+n} $ by
\begin{align}
    {\bm x}^{(n)}_{t+n}  & = \argmax _{  {\bm x}^{(n) } \in \mathcal{X}^{(n)} }  \tilde{c}^{(n)}_t ({\bm x}^{(n)} | {\bm y}^{(n-1) } _{t+n-1}) , \\
    {\bm y}^{(n) }_{t+n} & = {\bm f}^{(n)} ({\bm y}^{(n-1)} _{t+n-1},{\bm x}^{(n)}_{t+n }) +{\bm\epsilon}^{(n)}_{t+n},
    \label{eq:seq_rule_EX}
\end{align}
where ${\bm y}^{(0)} _t = {\bm 0}$.
Finally, we define the estimated solution $(  \hat{\bm x}^{(1)}_{G,T}, \ldots , \hat{\bm x}^{(N)}_{G,T} ) $
by using the pessimistic maximum value as follows:
\begin{align}
     & (  \hat{\bm x}^{(1)}_{G,T}, \ldots , \hat{\bm x}^{(N)}_{G,T} ) =                        \\
     & \argmax_{  ( {\bm x}^{(1)} ,\ldots , {\bm x}^{(N)} ) \in \mathcal{X}, 1 \leq t \leq T }
    \mathbb{E}_{ {\bm\epsilon}} \Bigl[
    (\tilde{z}^{(N)}_{{\bm\epsilon},1,t} ({\bm x}^{(1)},\ldots, {\bm x}^{(N)} )    - \beta^{1/2}_t \tilde{\sigma}^{(N)}_{{\bm\epsilon},1,t} ({\bm x}^{(1)},\ldots,{\bm x}^{(N)} ) )\Bigr].
\end{align}

Then, the following theorem holds.
\begin{theorem}\label{theorem:i_regret_EXP}
    Assume that~\cref{assumption:regu2,assumption:L1,assumption:L2} hold.
    Also assume that  $\tilde{\bm {z}}^{(n)}_{{\bm\epsilon},t} ({\bm x}^{(s+1)},\ldots,{\bm x}^{(n)} |{\bm y}) \in \tilde{\mathcal{Y}}^{(n)}$
    for any
    $s \in \{0,\ldots, N-1 \}$, $n \in \{s+1 ,\ldots, N \}$, iteration $t \geq 1$, realization ${\bm \epsilon}$, given $ {\bm y} \in \tilde{\mathcal{Y}}^{(s)} $ and input $ ({\bm x}^{(s+1)},\ldots,{\bm x}^{(n)}) $.
    Let $\delta \in (0,1)$ and $\xi >0$, and define $\beta_t $ by~\cref{eq:definition_beta_t} and $\eta_t = (1+ \log t)^{-1}$.
    Then, when the optimization is performed using~\cref{eq:seq_rule_EX}, the following inequality holds with probability at least
    $1- (N+1) \delta$:
    \begin{equation}
        G({\bm x}^{(1)}_{G,\ast}, \ldots ,{\bm x}^{(N)}_{G,\ast} ) -
        G(  \hat{\bm x}^{(1)}_{G,T}, \ldots , \hat{\bm x}^{(N)}_{G,T} )  < \xi,
    \end{equation}
    where $T $ is the smallest positive integer satisfying $T \in  N \mathbb{Z}_{\geq 0}$ and
    \begin{equation}
        \frac{N}{T} \sqrt{C^{2(N+2)}_{6,T} ( C_7 + C_8 T \tilde{\gamma}_T )} < \xi.
    \end{equation}
    Here, $C_{6,t}$, $C_7$ and $C_8 $ are given by
    $ C_{2,t}  = 4NM^2_{\text{prod}} M_{\text{sum}} C^{2N-2}_{0,t} C^N_1, \ C_{3,t} = NC_{2,t}^N, \ C_{6,t} = 2 C_{3,t}  \eta^{-1}_t   (2 \beta^{1/2}_t + 2 ) ,
        C_5     = (1+L_f )^N M_{ \text{prod}} M_{\text{sum}} , \ C_7 = 2 \left (          8 C_5 \log \frac{5C_5}{\delta} N   \right )^2 , \
        C_8 =  \frac{   4 M^2_{\text{sum}}   }{\log (1+\sigma^{-2} )}.$
\end{theorem}
In order to prove~\cref{theorem:i_regret_EXP}, we give two lemmas.
\begin{lemma}\label{lem:next_EX}
    Assume that the same condition as in~\cref{theorem:Nstage3} holds.
    Then, the following holds with probability at least $1-\delta$:
    \begin{align}
        \tilde{\sigma}^{(N)}_{{\bm\epsilon},1,t} ({\bm x} ^{(s)}, \ldots ,{\bm x}^{(N)} | {\bm z}^{(s-1)}_{\bm\epsilon} ) & \leq C_{3,t} \tilde{\sigma}^{(N)}_{{\bm\epsilon},1,t} ({\bm x} ^{(s+1)}, \ldots ,{\bm x}^{(N)} | {\bm z}^{(s)}_{\bm\epsilon} )  +C_{3,t} \sum_{i=1} ^{M^{(s)} } \sigma^{(s)} _{{\bm\epsilon},i,t}  ({\bm z}^{(s-1)}_{\bm\epsilon} ,{\bm x}^{(s)} ).
    \end{align}
\end{lemma}
\begin{proof}
    By using the same argument as in~\cref{lem:next}, we have~\cref{lem:next_EX}.
\end{proof}
\begin{lemma}\label{lem:c_bound_EX}
    Assume that the same condition as in~\cref{theorem:Nstage3} holds.
    Then, the following inequality holds:
    \begin{equation}
        \eta_t \tilde{b}^{(n)}_t ({\bm x}^{(n)} |{\bm y}^{(n-1)} ) \leq \tilde{c}^{(n)}_t ({\bm x}^{(n)} |{\bm y}^{(n-1)} )
        \leq (2 \beta^{1/2}_t  + \eta_t ) \tilde{b}^{(n)}_t ({\bm x}^{(n)} |{\bm y}^{(n-1)} ).
    \end{equation}
\end{lemma}
\begin{proof}
    By using the same argument as in~\cref{lem:c_bound}, we get~\cref{lem:c_bound_EX}.
\end{proof}
By using these lemmas, we prove~\cref{theorem:i_regret_EXP}.
\begin{proof}
    From~\cref{lem:c_bound_EX}, the following holds:
    \begin{align}
        \tilde{c}^{(1)}_t ({\bm x}^{(1)}_{t+1} |{\bm 0} ) & \leq (2 \beta^{1/2}_t + \eta_t ) \tilde{b}^{(1)}_t ({\bm x}^{(1)}_{t+1} ~{\bm 0} ) \\
                                                          & =  (2 \beta^{1/2}_t + \eta_t )
        \mathbb{E}_{\bm\epsilon} [ \tilde{\sigma}^{(N)} _{ {\bm \epsilon},1,t} ( {\bm x}^{(1)}_{t+1},\hat{\bm x}^{(2)}_{t+1},\ldots,\hat{\bm x}^{(N)} _{t+1} | {\bm 0} )].
    \end{align}
    In addition, for the positive integer $N K \equiv T \in N\mathbb{Z}_{\geq 0}$ satisfying the theorem's inequality,
    $\tilde{c}^{(1)}_t ({\bm x}^{(1)}_{t+1} |{\bm 0} ) $ satisfies that
    \begin{align}
        \sum_{ t \in N \mathbb{Z}_{\geq 0} }^T \tilde{c}^{(1)}_t ({\bm x}^{(1)}_{t+1} |{\bm 0} ) & \leq ( 2 \beta^{1/2}_T + 2 ) \sum_{ t \in N \mathbb{Z}_{\geq 0} }^T \mathbb{E}_{\bm\epsilon} [ \tilde{\sigma}^{(N)} _{ {\bm \epsilon},1,t} ( {\bm x}^{(1)}_{t+1},\hat{\bm x}^{(2)}_{t+1},\ldots,\hat{\bm x}^{(N)} _{t+1} | {\bm 0} )]                                                                                              \\
                                                                                                 & = ( 2 \beta^{1/2}_T + 2 ) \sum_{ t \in N \mathbb{Z}_{\geq 0} }^T \mathbb{E}_{{\bm\epsilon}^{(1)}} [   \mathbb{E}_{{\bm\epsilon}|{\bm\epsilon}^{(1)}} \Bigl[    \tilde{\sigma}^{(N)} _{ {\bm \epsilon},1,t} ( {\bm x}^{(1)}_{t+1},\hat{\bm x}^{(2)}_{t+1},\ldots,\hat{\bm x}^{(N)} _{t+1} | {\bm 0} )|{\bm\epsilon}^{(1)}] \Bigr] .
    \end{align}
    Here, the conditional expectation $ \mathbb{E}_{{\bm\epsilon}|{\bm\epsilon}^{(1)}} [     \tilde{\sigma}^{(N)} _{ {\bm \epsilon},1,t} ( {\bm x}^{(1)}_{t+1},\hat{\bm x}^{(2)}_{t+1},\ldots,\hat{\bm x}^{(N)} _{t+1} | {\bm 0} )|{\bm\epsilon}^{(1)}]$ is
    a non-negative random variable with respect to
    ${\bm \epsilon}^{(1)}$, and satisfies that
    \begin{equation}
        \mathbb{E}_{{\bm\epsilon}|{\bm\epsilon}^{(1)}} [     \tilde{\sigma}^{(N)} _{ {\bm \epsilon},1,t} ( {\bm x}^{(1)}_{t+1},\hat{\bm x}^{(2)}_{t+1},\ldots,\hat{\bm x}^{(N)} _{t+1} | {\bm 0} )|{\bm\epsilon}^{(1)}]
        \leq (1+L_f )^N M_{ \text{prod}} M_{\text{sum}} = C_5 ,
    \end{equation}
    where the inequality is given by $k^{(n)} (\cdot,\cdot) \leq 1$.
    Hence, from Lemma~3 of~\cite{kirschner2018information}, the following holds with probability at least $1-\delta$:
    \begin{align}
         & \sum_{ t \in N \mathbb{Z}_{\geq 0} }^T \mathbb{E}_{{\bm\epsilon}^{(1)}} [   \mathbb{E}_{{\bm\epsilon}|{\bm\epsilon}^{(1)}} [     \tilde{\sigma}^{(N)} _{ {\bm \epsilon},1,t} ( {\bm x}^{(1)}_{t+1},\hat{\bm x}^{(2)}_{t+1},\ldots,\hat{\bm x}^{(N)} _{t+1} | {\bm 0} )|{\bm\epsilon}^{(1)}]]                             \\
         & \leq
        4 C_5 \log \frac{1}{\delta} + 8 C_5 \log (4 C_5) +1        +2  \sum_{ t \in N \mathbb{Z}_{\geq 0} }^T   \mathbb{E}_{{\bm\epsilon}|{\bm\epsilon}^{(1)}_t} [     \tilde{\sigma}^{(N)} _{ {\bm \epsilon},1,t} ( {\bm x}^{(1)}_{t+1},\hat{\bm x}^{(2)}_{t+1},\ldots,\hat{\bm x}^{(N)} _{t+1} | {\bm 0} )|{\bm\epsilon}^{(1)}_t] \\
         & \leq
        8 C_5 \log \frac{5C_5}{\delta}  +2  \sum_{ t \in N \mathbb{Z}_{\geq 0} }^T   \mathbb{E}_{{\bm\epsilon}|{\bm\epsilon}^{(1)}_t} [     \tilde{\sigma}^{(N)} _{ {\bm \epsilon},1,t} ( {\bm x}^{(1)}_{t+1},\hat{\bm x}^{(2)}_{t+1},\ldots,\hat{\bm x}^{(N)} _{t+1} | {\bm 0} )|{\bm\epsilon}^{(1)}_t].
    \end{align}
    Moreover, from~\cref{lem:next_EX}, with probability at least $1-\delta$ the following inequality holds uniformly:
    \begin{align}
         & \mathbb{E}_{{\bm\epsilon}|{\bm\epsilon}^{(1)}_t} [     \tilde{\sigma}^{(N)} _{ {\bm \epsilon},1,t} ( {\bm x}^{(1)}_{t+1},\hat{\bm x}^{(2)}_{t+1},\ldots,\hat{\bm x}^{(N)} _{t+1} | {\bm 0} )|{\bm\epsilon}^{(1)}_t] \\
         & \leq
        \mathbb{E}_{{\bm\epsilon}|{\bm\epsilon}^{(1)}_t} [
        C_{3,T} \tilde{\sigma}^{(N)}_{ {\bm\epsilon},1,t} ( \hat{\bm x}^{(2)}_{t+1},\ldots,\hat{\bm x}^{(N)} _{t+1} | {\bm y}^{(1)}_{t+1} )
        ]         +
        C_{3,T} \sum_{i=1}^{M^{(1)}} \sigma^{(1)}_{ {\bm\epsilon},1,t} ({\bm 0}, {\bm x}^{(1)}_{t+1}).
    \end{align}
    Therefore, it follows that
    \begin{align}
         & \sum_{ t \in N \mathbb{Z}_{\geq 0} }^T \mathbb{E}_{{\bm\epsilon}^{(1)}} [   \mathbb{E}_{{\bm\epsilon}|{\bm\epsilon}^{(1)}} [     \tilde{\sigma}^{(N)} _{ {\bm \epsilon},1,t} ( {\bm x}^{(1)}_{t+1},\hat{\bm x}^{(2)}_{t+1},\ldots,\hat{\bm x}^{(N)} _{t+1} | {\bm 0} )|{\bm\epsilon}^{(1)}]]             \\
         & \leq
        8 C_5 \log \frac{5C_5}{\delta}     +2  \sum_{ t \in N \mathbb{Z}_{\geq 0} }^T   \mathbb{E}_{{\bm\epsilon}|{\bm\epsilon}^{(1)}_t} [     \tilde{\sigma}^{(N)} _{ {\bm \epsilon},1,t} ( {\bm x}^{(1)}_{t+1},\hat{\bm x}^{(2)}_{t+1},\ldots,\hat{\bm x}^{(N)} _{t+1} | {\bm 0} )|{\bm\epsilon}^{(1)}_t]         \\
         & \leq 8 C_5 \log \frac{5C_5}{\delta} +2 C_{3,T}  \sum_{ t \in N \mathbb{Z}_{\geq 0} }^T  \sum_{i=1}^{M^{(1)}} \sigma^{(1)}_{ {\bm\epsilon},1,t} ({\bm 0}, {\bm x}^{(1)}_{t+1})      +2 C_{3,T}  \sum_{ t \in N \mathbb{Z}_{\geq 0} }^T
        \mathbb{E}_{{\bm\epsilon}|{\bm\epsilon}^{(1)}_t} [
        \tilde{\sigma}^{(N)}_{ {\bm\epsilon},1,t} ( \hat{\bm x}^{(2)}_{t+1},\ldots,\hat{\bm x}^{(N)} _{t+1} | {\bm y}^{(1)}_{t+1} )
        ]                                                                                                                                                                                                                                                                                                           \\
         & \leq 8 C_5 \log \frac{5C_5}{\delta} +2 C_{3,T}  \sum_{ t \in N \mathbb{Z}_{\geq 0} }^T  \sum_{i=1}^{M^{(1)}} \sigma^{(1)}_{ {\bm\epsilon},1,t} ({\bm 0}, {\bm x}^{(1)}_{t+1})  +2 C_{3,T}  \eta^{-1}_T \sum_{ t \in N \mathbb{Z}_{\geq 0} }^T
        \eta _t
        \tilde{b}^{(2)}_{t} (\hat{\bm x}^{(2)}_{t+1} |  {\bm y}^{(1)}_{t+1} )                                                                                                                                                                                                                                       \\
         & \leq 8 C_5 \log \frac{5C_5}{\delta} +2 C_{3,T}  \sum_{ t \in N \mathbb{Z}_{\geq 0} }^T  \sum_{i=1}^{M^{(1)}} \sigma^{(1)}_{ {\bm\epsilon},1,t} ({\bm 0}, {\bm x}^{(1)}_{t+1}) +2 C_{3,T}  \eta^{-1}_T \sum_{ t \in N \mathbb{Z}_{\geq 0} }^T
        \tilde{c}^{(2)}_{t} (\hat{\bm x}^{(2)}_{t+1} |  {\bm y}^{(1)}_{t+1} )                                                                                                                                                                                                                                       \\
         & \leq 8 C_5 \log \frac{5C_5}{\delta} +2 C_{3,T}  \sum_{ t \in N \mathbb{Z}_{\geq 0} }^T  \sum_{i=1}^{M^{(1)}} \sigma^{(1)}_{ {\bm\epsilon},1,t} ({\bm 0}, {\bm x}^{(1)}_{t+1})    +2 C_{3,T}  \eta^{-1}_T \sum_{ t \in N \mathbb{Z}_{\geq 0} }^T
        \tilde{c}^{(2)}_{t} ({\bm x}^{(2)}_{t+2} |  {\bm y}^{(1)}_{t+1} )                                                                                                                                                                                                                                           \\
         & \leq 8 C_5 \log \frac{5C_5}{\delta} +2 C_{3,T}  \sum_{ t \in N \mathbb{Z}_{\geq 0} }^T  \sum_{i=1}^{M^{(1)}} \sigma^{(1)}_{ {\bm\epsilon},1,t} ({\bm 0}, {\bm x}^{(1)}_{t+1})    +2 C_{3,T}  \eta^{-1}_T   (2 \beta^{1/2}_t + 2 )  \sum_{ t \in N \mathbb{Z}_{\geq 0} }^T
        \tilde{b}^{(2)}_{t} ({\bm x}^{(2)}_{t+2} |  {\bm y}^{(1)}_{t+1} )                                                                                                                                                                                                                                           \\
         & \leq 8 C_5 \log \frac{5C_5}{\delta} +2 C_{3,T}  \sum_{ t \in N \mathbb{Z}_{\geq 0} }^T  \sum_{i=1}^{M^{(1)}} \sigma^{(1)}_{ {\bm\epsilon},1,t} ({\bm 0}, {\bm x}^{(1)}_{t+1})   +2 C_{3,T}  \eta^{-1}_T   (2 \beta^{1/2}_t + 2 )                                                                         \\
         & \qquad  \cdot \sum_{ t \in N \mathbb{Z}_{\geq 0} }^T
        \mathbb{E}_{{\bm\epsilon}^{(2)}} \Bigl[   \mathbb{E}_{{\bm\epsilon}|{\bm y}^{(1)}_{t+1},{\bm\epsilon}^{(2)}} \Bigl[       \tilde{\sigma}^{(N)} _{ {\bm \epsilon},1,t} ( {\bm x}^{(2)}_{t+2},\hat{\bm x}^{(3)}_{t+2},\ldots,\hat{\bm x}^{(N)} _{t+2} | {\bm y}^{(1)}_{t+1} )|{\bm\epsilon}^{(2)}\Bigr]\Bigr] \\
         & \leq 8 C_5 \log \frac{5C_5}{\delta} +C_{6,T}  \sum_{ t \in N \mathbb{Z}_{\geq 0} }^T  \sum_{i=1}^{M^{(1)}} \sigma^{(1)}_{ {\bm\epsilon},1,t} ({\bm 0}, {\bm x}^{(1)}_{t+1})                                                                                                                              \\
         & \quad +C_{6,T}  \sum_{ t \in N \mathbb{Z}_{\geq 0} }^T
        \mathbb{E}_{{\bm\epsilon}^{(2)}} \Bigl[   \mathbb{E}_{{\bm\epsilon}|{\bm y}^{(1)}_{t+1},{\bm\epsilon}^{(2)}} \Bigl[     \tilde{\sigma}^{(N)} _{ {\bm \epsilon},1,t} ( {\bm x}^{(2)}_{t+2},\hat{\bm x}^{(3)}_{t+2},\ldots,\hat{\bm x}^{(N)} _{t+2} | {\bm y}^{(1)}_{t+1} )|{\bm\epsilon}^{(2)}\Bigr]\Bigr] .
    \end{align}
    By repeating this process, with probability at least $1- (N+1) \delta $, the following holds:
    \begin{align}
         & \sum_{ t \in N \mathbb{Z}_{\geq 0} }^T \mathbb{E}_{{\bm\epsilon}^{(1)}} [   \mathbb{E}_{{\bm\epsilon}|{\bm\epsilon}^{(1)}} [     \tilde{\sigma}^{(N)} _{ {\bm \epsilon},1,t} ( {\bm x}^{(1)}_{t+1},\hat{\bm x}^{(2)}_{t+1},\ldots,\hat{\bm x}^{(N)} _{t+1} | {\bm 0} )|{\bm\epsilon}^{(1)}]] \\
         & \leq
        8 C_5 \log \frac{5C_5}{\delta} N C_{6,T}^N   +C_{6,T}^N    \sum_{ t \in N \mathbb{Z}_{\geq 0} }^T  \sum_{n=1}^N  \sum_{i=1}^{M^{(n)}} \sigma^{(n)}_{ {\bm\epsilon},i,t} ({\bm y}^{(n-1)}_{t+n-1}, {\bm x}^{(n)}_{t+n}).
    \end{align}
    By combining this and
    \begin{align}
        \sum_{ t \in N \mathbb{Z}_{\geq 0} }^T \tilde{c}^{(1)}_t ({\bm x}^{(1)}_{t+1} |{\bm 0} ) & \leq ( 2 \beta^{1/2}_T + 2 ) \sum_{ t \in N \mathbb{Z}_{\geq 0} }^T \mathbb{E}_{{\bm\epsilon}^{(1)}} [   \mathbb{E}_{{\bm\epsilon}|{\bm\epsilon}^{(1)}} \Bigl[      \tilde{\sigma}^{(N)} _{ {\bm \epsilon},1,t} ( {\bm x}^{(1)}_{t+1},\hat{\bm x}^{(2)}_{t+1},\ldots,\hat{\bm x}^{(N)} _{t+1} | {\bm 0} )|{\bm\epsilon}^{(1)}]\Bigr] \\
                                                                                                 & \leq C_{6,T} \sum_{ t \in N \mathbb{Z}_{\geq 0} }^T \mathbb{E}_{{\bm\epsilon}^{(1)}} [   \mathbb{E}_{{\bm\epsilon}|{\bm\epsilon}^{(1)}} \Bigl[     \tilde{\sigma}^{(N)} _{ {\bm \epsilon},1,t} ( {\bm x}^{(1)}_{t+1},\hat{\bm x}^{(2)}_{t+1},\ldots,\hat{\bm x}^{(N)} _{t+1} | {\bm 0} )|{\bm\epsilon}^{(1)}]\Bigr] ,
    \end{align}
    we get
    \begin{equation}
        \sum_{ t \in N \mathbb{Z}_{\geq 0} }^T \tilde{c}^{(1)}_t ({\bm x}^{(1)}_{t+1} |{\bm 0} )       \leq
        8 C_5 \log \frac{5C_5}{\delta} N C_{6,T}^{N+1}     +C_{6,T}^{N+1}    \sum_{ t \in N \mathbb{Z}_{\geq 0} }^T  \sum_{n=1}^N  \sum_{i=1}^{M^{(n)}} \sigma^{(n)}_{ {\bm\epsilon},i,t} ({\bm y}^{(n-1)}_{t+n-1}, {\bm x}^{(n)}_{t+n}).
    \end{equation}
    Thus, noting that $(a+b)^2 \leq 2 a^2 + 2 b^2 $, using the Cauchy--Schwarz inequality and $\tilde{\gamma}_T$ we have
    \begin{align}
        \left (
        \sum_{ t \in N \mathbb{Z}_{\geq 0} }^T \tilde{c}^{(1)}_t ({\bm x}^{(1)}_{t+1} |{\bm 0} )
        \right )^2 & \leq
        2 \left (          8 C_5 \log \frac{5C_5}{\delta} N C_{6,T}^{N+1}   \right )^2   +
        2 C^{2(N+1)}_{6,T} T M_{\text{sum}} \sum_{n=1}^N  \sum_{i=1}^{M^{(n)}} \sigma^{(n)2}_{ {\bm\epsilon},i,t} ({\bm y}^{(n-1)}_{t+n-1}, {\bm x}^{(n)}_{t+n})
        \\
                   & \leq
        2 \left (          8 C_5 \log \frac{5C_5}{\delta} N C_{6,T}^{N+1}   \right )^2         +
        2 C^{2(N+1)}_{6,T} T M_{\text{sum}} \frac{   2 M_{\text{sum}} \tilde{\gamma}_T  }{\log (1+\sigma^{-2} )} \\
                   & =
        C^{2(N+1)}_{6,T} ( C_7 + C_8 T \tilde{\gamma}_T ).
    \end{align}
    This implies that
    \begin{align}
        \sum_{ t \in N \mathbb{Z}_{\geq 0} }^T \tilde{c}^{(1)}_t ({\bm x}^{(1)}_{t+1} |{\bm 0} )
        \leq
        \sqrt{C^{2(N+1)}_{6,T} ( C_7 + C_8 T \tilde{\gamma}_T )}.
    \end{align}
    Furthermore, letting $\tilde{t} = \argmax_{ t \in N \mathbb{Z}_{\geq 0} , t \leq T} \tilde{c}^{(1)}_t ({\bm x}^{(1)}_{t+1} |{\bm 0} )$
    we get
    \begin{align}
        K
        \tilde{c}^{(1)}_{\tilde{t}} ({\bm x}^{(1)}_{\tilde{t}+1} |{\bm 0} )
         & \leq
        \sum_{ t \in N \mathbb{Z}_{\geq 0} }^T \tilde{c}^{(1)}_t ({\bm x}^{(1)}_{t+1} |{\bm 0} ) \\
         & \leq
        \sqrt{C^{2(N+1)}_{6,T} ( C_7 + C_8 T \tilde{\gamma}_T )}.
    \end{align}
    By dividing both sides by $K$, we obtain
    \begin{align}
        \tilde{c}^{(1)}_{\tilde{t}} ({\bm x}^{(1)}_{\tilde{t}+1} |{\bm 0} )
         & \leq K^{-1} \sqrt{C^{2(N+1)}_{6,T} ( C_7 + C_8 T \tilde{\gamma}_T )} \\
         & =
        \frac{N}{T} \sqrt{C^{2(N+1)}_{6,T} ( C_7 + C_8 T \tilde{\gamma}_T )} . \label{eq:N_T}
    \end{align}
    Finally,  from the definition of the estimated solution and CIs, we get
    \begin{alignat}{2}
        G({\bm x}^{(1)}_{G,\ast}, \ldots ,{\bm x}^{(N)}_{G,\ast} )     & \leq \min _{ t \in N \mathbb{Z}_{\geq 0} , t \leq T }  \mr{ UCB}^{(G)} _t ({\bm x}^{(1)}_{G,\ast} ,\ldots ,{\bm x}^{(N)}_{G,\ast} )
                                                                       &                                                                                                                                     & \leq  \mr{ UCB}^{(G)} _{ \tilde{t} }   ({\bm x}^{(1)}_{G,\ast} ,\ldots ,{\bm x}^{(N)}_{G,\ast} )  , \\
        G(\hat{\bm x}^{(1)}_{G,T}, \ldots , \hat{\bm x}^{(N)}_{G,T}  ) & \geq
        \max _{ t \in N \mathbb{Z}_{\geq 0} , t \leq T }  \mr{ LCB}^{(G)} _t ({\bm x}^{(1)} ,\ldots ,{\bm x}^{(N)} )
                                                                       &                                                                                                                                     & \geq
        \mr{ LCB}^{(G)} _{ \tilde{t} }  ({\bm x}^{(1)}_{G,\ast} ,\ldots ,{\bm x}^{(N)}_{G,\ast} ) .
    \end{alignat}
    Thus, it follows that
    \begin{align}
        G({\bm x}^{(1)}_{G,\ast}, \ldots ,{\bm x}^{(N)}_{G,\ast} )  -
        G(\hat{\bm x}^{(1)}_{G,T}, \ldots , \hat{\bm x}^{(N)}_{G,T}  ) & \leq 2 \beta^{1/2}_T \mathbb{E}_{\bm\epsilon} [ \tilde{\sigma}^{(N)} _{{\bm\epsilon},1,\tilde{t} } (   {\bm x}^{(1)}_{G,\ast}, \ldots ,{\bm x}^{(N)}_{G,\ast}  )] \\
                                                                       & \leq 2\beta^{1/2} _T  \tilde{b}^{(1)}_{ \tilde{t} }  ({\bm x}^{(1)}_{G,\ast} | {\bm 0} )                                                                          \\
                                                                       & \leq 2 {\beta}^{1/2}_T   \eta ^{-1}  _{\tilde{t}}   \tilde{c}^{(1)}_{ \tilde{t} }    ({\bm x}^{(1)}_{G,\ast} | {\bm 0} )  \leq
        2 {\beta}^{1/2} _T  \eta ^{-1}  _{T}   \tilde{c}^{(1)}_{ \tilde{t} }    ({\bm x}^{(1)}_{\tilde{t} +1}  | {\bm 0} ) .
    \end{align}
    Hence,  by combining this and~\cref{eq:N_T}, we have
    \begin{align}
        G({\bm x}^{(1)}_{G,\ast}, \ldots ,{\bm x}^{(N)}_{G,\ast} )  -
        G(\hat{\bm x}^{(1)}_{G,T}, \ldots , \hat{\bm x}^{(N)}_{G,T}  ) & \leq
        2 {\beta}^{1/2} _T  \eta ^{-1}  _{T} \frac{N}{T} \sqrt{C^{2(N+1)}_{6,T} ( C_7 + C_8 T \tilde{\gamma}_T )}                                          \\
                                                                       & \leq C_{6,T} \frac{N}{T} \sqrt{C^{2(N+1)}_{6,T} ( C_7 + C_8 T \tilde{\gamma}_T )} \\
                                                                       & =
        \frac{N}{T} \sqrt{C^{2(N+2)}_{6,T} ( C_7 + C_8 T \tilde{\gamma}_T )} < \xi.
    \end{align}
\end{proof}

%-------------------------------------------------------------------------------------------
\subsection{Optimistic Improvement-based AF  for the  Final Stage Output}
We give an   optimistic improvement-based AF for $F$ under the noisy setting.
First, we define the sum of the squares of the observation noise $\epsilon_{\text{sum}}$ as
\begin{equation}
    \epsilon_{\text{sum}} = \sum_{n=1}^N \sum_{m=1}^{M^{(n)}} \epsilon^{(n)2}_m .
\end{equation}
Note that $ \epsilon_{\text{sum}} $ is bounded by $ M_{\text{sum}} A^2$ under~\cref{assumption:regu2}.
Moreover, we assume the following assumption for $ \epsilon_{\text{sum}} $.
\begin{assumption}\label{assumption:esum}
    Under~\cref{assumption:regu2}, there exists a positive constant $C$ such that
    $\mathbb{P} ( \epsilon_{\text{sum}} < V ) > CV
    $ for any $V$ with $0<V \leq M_{\text{sum}} A^2$.
\end{assumption}
For example, if  $\epsilon_{\text{sum}}$ is a discrete random variable with $\mathbb{P} (\epsilon_{\text{sum}}  =0) >0$,
then~\cref{assumption:esum} holds.
Similarly, if $\epsilon_{\text{sum}}$ is a continuous random variable whose probability density function $p_{\epsilon_{\text{sum}}} (x)$
satisfies $p_{\epsilon_{\text{sum}}} (x) >K>0$, where $x$ is an arbitrary element of some interval $[0,U]$.
Then,~\cref{assumption:esum} also holds.
Thus,~\cref{assumption:esum} guarantees that $\epsilon_{\text{sum}}$ can take values within an arbitrary neighborhood of zero.
Next, we define the variable $C_{9,t} $ as
\begin{equation}
    C_{9,t}
    =
    2 \beta^{1/2}_t \eta^{-1}_t
    (2 \beta^{1/2} _t+ 2) ^N C^N_{3,t} \eta^{-N}_t ,
\end{equation}
where $\eta_t = (1 + \log t )^{-1} $.
Then, we assume the following assumption.
\begin{assumption}\label{assumption:C9}
    For any $ T \geq 1$, $C_{9,t }$ satisfies that
    \begin{equation}
        \sum_{t=T}^ {t^\prime} C^{-2}_{9,t} \to \infty \quad (\text{as} \   t^\prime \to \infty  ).
    \end{equation}
\end{assumption}
Note that $C_{9,t} $ is a polynomial function on $\beta_t $.
Furthermore, by considering  the definition of $\beta_t $, the closed form of the mutual information, and $\tilde{\gamma}_t$, we can show that  the order of $C_{9,t} $ is expressed as the polynomial function of $\tilde{\gamma}_t$.
Here, under certain conditions,  it is known that the order of $\tilde{\gamma}_t$ for commonly used kernels such as Gaussian kernels and linear kernels is a logarithmic order~\cite{srinivas2010gaussian}.
Then,~\cref{assumption:C9} holds if we use such kernels.
Under this setting, we propose an algorithm to the regret  $r^{(\text{S})}_{F,T}$.

First, for each $t\ge 1$, we define the estimated solution $\hat{\bx}\stg{1}_{F, t}, \dots, \hat{\bx}\stg{N}_{F, t} $ as follows:
\begin{equation}
    \hat{\bx}\stg{1}_{F, t}, \dots, \hat{\bx}\stg{N}_{F, t}   =\argmax _{ \substack{1 \leq t^\prime \leq t\\ ({\bm x}^{(1)},\ldots,{\bm x}^{(N)} ) \in \mathcal{X} }}
    (  \tilde{z}^{(N)}_{ {\bm 0},1,t^\prime} ({\bm x}^{(1)},\ldots,{\bm x}^{(N)} )   -
    \beta^{1/2}_t \tilde{\sigma}^{(N)}_{ {\bm 0},1,t^\prime}  ({\bm x}^{(1)},\ldots,{\bm x}^{(N)} ) ).
\end{equation}
Then, we give the optimistic improvement-based AF.
For any $n\in [N]$, given an observation $\by\stg{n-1}$ of stage $n-1$,
optimistic maximum estimator $\widehat {\mr{ UCB}}^{(F)}_t ({\bm x}^{(n) } | {\bm y}^{(n-1)} ) $ w.r.t. $\bx\stg{n}$ is defined as:
\begin{align}
     & \widehat {\mr{ UCB}}^{(F)}_t ({\bm x}^{(n) } | {\bm y}^{(n-1)} ) \\
     & = \max_{  ( {\bm x}^{(n+1)} ,\ldots , {\bm x}^{(N)} ) }
    \Bigl(\tilde{{ z}}^{(N)}_{{\bm 0},1,t} ({\bm x}^{(n)},\ldots, {\bm x}^{(N)}  |{\bm y}^{(n-1)})  + \beta^{1/2}_t \tilde{\sigma}^{(N)}_{{\bm 0},1,t} ({\bm x}^{(n)},\ldots,{\bm x}^{(N)}  |{\bm y}^{(n-1)} ) \Bigr), \label{eq:UCB_y_zero}
\end{align}
where the max operator is not needed when $n=N$.
Similarly, pessimistic maximum estimator $\widehat{\mr{ LCB}}^{(F)}_t ( {\bm y}^{(n-1)} ) $ under given an observation $\by\stg{n-1}$ is defined as follows:
\begin{align}
     & \widehat{\mr{ LCB}}^{(F)}_t ({\bm y}^{(n-1)} )        \\
     & = \max_{  ( {\bm x}^{(n)} ,\ldots , {\bm x}^{(N)} ) }
    \Bigl(\tilde{z}^{(N)}_{{\bm 0},1,t} ({\bm x}^{(n)},\ldots, {\bm x}^{(N)}  |{\bm y}^{(n-1)})   - \beta^{1/2}_t \tilde{\sigma}^{(N)}_{{\bm 0},1,t} ({\bm x}^{(n)},\ldots,{\bm x}^{(N)}  |{\bm y}^{(n-1)} ) \Bigr). \label{eq:LCB_y_zero}
\end{align}
Moreover, pessimistic maximum estimator of $F$ is given by:
\begin{equation}
    \hat{Q}_T =  \max_{  ( {\bm x}^{(1)} ,\ldots , {\bm x}^{(N)} ) }  \Bigl (\tilde{z}^{(N)}_{{\bm 0},1,T} ({\bm x}^{(1)},\ldots, {\bm x}^{(N)} ) - \beta^{1/2}_t \tilde{\sigma}^{(N)}_{{\bm 0},1,T} ({\bm x}^{(1)},\ldots,{\bm x}^{(N)} ) \Bigr).
\end{equation}
Then, we define the optimistic improvement with w.r.t. ${\bm x}^{(n)}$ as:
\begin{equation}
    \hat{a}^{(n)}_t ({\bm x}^{(n)} | {\bm y}^{(n-1) } )=  \widehat{\mr{ UCB}}^{(F)}_t ({\bm x}^{(n) } | {\bm y}^{(n-1)} )    - \max \{ \widehat{\mr{ LCB}}^{(F)}_t ( {\bm y}^{(n-1)} ) , \hat{Q}_{t+n-1} \}. \label{eq:improve_y_zero}
\end{equation}
Furthermore, we define the maximum uncertainty w.r.t. $( {\bm y}^{(n-1) }, {\bm x}^{(n)})$ as:
\begin{equation}
    \hat{b}^{(n)}_t ({\bm x}^{(n)} | {\bm y}^{(n-1) } )    =  \max_{  ( {\bm x}^{(n+1)} ,\ldots , {\bm x}^{(N)} ) }
    \tilde{\sigma}^{(N)}_{{\bm 0},1,t} ({\bm x}^{(n)},\ldots,{\bm x}^{(N)}  |{\bm y}^{(n-1)} ) . \label{eq:US_y_zero}
\end{equation}
From~\cref{eq:improve_y_zero,eq:US_y_zero}, the AF $\hat{c}^{(n)}_t ({\bm x}^{(n)} | {\bm y}^{(n-1) } )$ for this setting is given by:
\begin{equation}
    \hat{c}^{(n)}_t ({\bm x}^{(n)} | {\bm y}^{(n-1) } ) = \max \{ \hat{a}^{(n)}_t ({\bm x}^{(n)} | {\bm y}^{(n-1) } ), \eta_t \hat{b}^{(n)}_t ({\bm x}^{(n)} | {\bm y}^{(n-1) } ) \},  \label{eq:af_seq_zero}
\end{equation}
where $\eta_t $ is some learning rate tends to zero.
Using this AF $\hat{c}^{(n)}_t$, we propose the following selection rule:
\begin{align}
    {\bm x}^{(n)}_{t+n}  & = \argmax _{  {\bm x}^{(n) } \in \mathcal{X}^{(n)} }  \hat{c}^{(n)}_t ({\bm x}^{(n)} | {\bm y}^{(n-1) } _{t+n-1}) ,  \\
    {\bm y}^{(n) }_{t+n} & = {\bm f}^{(n)} ({\bm y}^{(n-1)} _{t+n-1},{\bm x}^{(n)}_{t+n }) +{\bm\epsilon}^{(n)}_{t+n}, \label{eq:seq_rule_zero}
\end{align}
where ${\bm y}^{(0)} _t = {\bm 0}$.
Then, the following theorem holds.
\begin{theorem}\label{theorem:i_regret_zero}
    Assume that~\cref{assumption:regu2,assumption:L1,assumption:L2,assumption:esum,assumption:C9}  hold.
    Also assume that  $\tilde{\bm {z}}^{(n)}_{{\bm\epsilon},t} ({\bm x}^{(s+1)},\ldots,{\bm x}^{(n)} |{\bm y}) \in \tilde{\mathcal{Y}}^{(n)}$
    for any
    $s \in \{0,\ldots, N-1 \}$, $n \in \{s+1 ,\ldots, N \}$, iteration $t \geq 1$, realization ${\bm \epsilon}$, given $ {\bm y} \in \tilde{\mathcal{Y}}^{(s)} $ and input $ ({\bm x}^{(s+1)},\ldots,{\bm x}^{(n)}) $.
    Let $\delta \in (0,1)$ and $\xi >0$, and define $\beta_t $ by~\cref{eq:definition_beta_t} and $\eta_t = (1+ \log t)^{-1}$.
    Then, there exists a sequence $0= T_0 < T_1 < T_2 < \cdots $ such that $T_k \in N \mathbb{Z}_{ \geq 0} $ and
    \begin{equation}
        \mathbb{P} ( \exists t \in N \mathbb{Z}_{\geq 0} \ \text{s.t.} \ T_{k-1} \leq t \leq T_k , \ 2 C^2_{9,t} M_{\text{sum}} \epsilon_{\text{sum},t} < \xi^2 /2 ) > 1- \frac{ 6 \delta  }{\pi^2 k^2}. \label{eq:ep_sum_small_prob}
    \end{equation}
    Moreover, when the optimization is performed using~\cref{eq:seq_rule_zero}, the following inequality holds with probability at least $1- 2 \delta$:
    \begin{equation}
        F({\bm x}^{(1)}_{F,\ast}, \ldots ,{\bm x}^{(N)}_{F,\ast} ) -
        F(  \hat{\bm x}^{(1)}_{F,T_K}, \ldots , \hat{\bm x}^{(N)}_{F,T_K} )  < \xi,
    \end{equation}
    where $T_K$ is an element of the sequence $\{ T_k \}_{k=0}^\infty $  satisfying
    \begin{equation}
        \frac{4 C^2_{9,T_K} M^2_{\text{sum}}}{ \log (1+\sigma^{-2} ) } \tilde{\gamma} _{T_K} K^{-1} < \xi^2/2.
    \end{equation}
\end{theorem}
In order to prove~\cref{theorem:i_regret_zero}, we first give four lemmas.
\begin{lemma}\label{lem:fms_zero}
    Assume that the same condition as in~\cref{theorem:i_regret_zero} holds.
    Then, for any $s \in \{1,\ldots, N-1 \}$, $n \in \{ s+1, \ldots , N\}$, $m \in [M^{(n)}]$, iteration $t \geq 1$, realization $\bm\epsilon$ and input ${\bm x}^{(1)},\ldots,{\bm x}^{(N)}$, the following holds with probability at least $1-\delta$:
    \begin{align}
         & |\sigma^{(n)} _{m,t} ( \tilde{\bm  z}^{(n-1)}_{{\bm 0},t} ( {\bm x}^{(s) },\ldots, {\bm x}^{(n-1)} | {\bm z}^{(s-1)}_{\bm \epsilon}),{\bm x}^{(n)} )  -
        \sigma^{(n)} _{m,t} ( \tilde{\bm  z}^{(n-1)}_{{\bm 0},t} ( {\bm x}^{(s+1) },\ldots, {\bm x}^{(n-1)} | {\bm z}^{(s)}_{\bm\epsilon}),{\bm x}^{(n)} )
        |                                                                                                                                                          \\
         &
        \leq
        2    M_{\text{prod}} C^{N-1}_{0,t}  \sum_{p=0} ^{n-s-1}
        \sum_{i=1}^{   M^{(n-1-p)} } \sigma^{(n-1-p)} _{i,t} \Bigl(       \tilde{ \bm z}^{(n-2-p)}_{{\bm 0},t} ( {\bm x}^{(s) },\ldots, {\bm x}^{(n-2-p)} | {\bm z}_{\bm\epsilon}^{(s-1)}),{\bm x}^{(n-1-p)}\Bigr)
        \\
         & \quad + 2    M_{\text{prod}} C^{N-1}_{0,t} \sum_{q=1}^{M^{(s)}}  | \epsilon^{(s)}_q |.
    \end{align}
\end{lemma}
\begin{proof}
    By using the same argument as in the proof of~\cref{lem:fms}, the following holds with probability at least $1-\delta$:
    \begin{align}
         & |\sigma^{(n)} _{m,t} ( \tilde{\bm  z}^{(n-1)}_{{\bm 0},t} ( {\bm x}^{(s) },\ldots, {\bm x}^{(n-1)} | {\bm z}_{\bm\epsilon}^{(s-1)}),{\bm x}^{(n)} )     -   \sigma^{(n)} _{m,t} ( \tilde{\bm  z}^{(n-1)}_{{\bm 0},t} ( {\bm x}^{(s+1) },\ldots, {\bm x}^{(n-1)} | {\bm z}_{\bm\epsilon}^{(s)}),{\bm x}^{(n)} )|                 \\
         & \leq 2 \beta^{1/2}_t  L_\sigma \sum_{j=1}^{M^{(n-1)} } \sigma^{(n-1)} _{j,t} \Bigl(    \tilde{ \bm z}^{(n-2)}_{{\bm 0},t} ( {\bm x}^{(s+1) },\ldots, {\bm x}^{(n-2)} | {\bm z}_{\bm\epsilon}^{(s)}),{\bm x}^{(n-1)}\Bigr)    + L_\sigma M^{(n-1)} (L_f + \beta^{1/2}_t L_\sigma )                                               \\
         & \quad  \cdot  \sum_{i=1} ^{M^{(n-2)}} \Bigl[ | \tilde{  z}^{(n-2)}_{{\bm 0},i,t} ( {\bm x}^{(s) },\ldots, {\bm x}^{(n-2)} | {\bm z}_{\bm\epsilon}^{(s-1)})   - \tilde{  z}^{(n-2)}_{{\bm 0},i,t} ( {\bm x}^{(s+1) },\ldots, {\bm x}^{(n-2)} | {\bm z}^{(s)}_{\bm\epsilon}) | \Bigr]                                             \\
         & \leq 2 \beta^{1/2}_t L_\sigma  \sum_{j=1}^{M^{(n-1)} } \sigma^{(n-1)} _{j,t} (   \tilde{ \bm z}^{(n-2)}_{{\bm 0},t} ( {\bm x}^{(s) },\ldots, {\bm x}^{(n-2)} | {\bm z}^{(s-1)}_{\bm\epsilon}),{\bm x}^{(n-1)})                                                                                                                  \\
         & \quad + 2 \beta^{1/2} _t L_\sigma M^{(n-1)} (L_f + \beta^{1/2}_t L_\sigma )    \sum_{i=1} ^{M^{(n-2)}}  \sigma^{(n-2)} _{i,t} (   \tilde{ \bm z}^{(n-3)}_{{\bm 0},t} ( {\bm x}^{(s) },\ldots, {\bm x}^{(n-3)} | {\bm z}_{\bm\epsilon}^{(s-1)}),{\bm x}^{(n-2)})                                                                 \\
         & \quad + L_\sigma M^{(n-1)} M^{(n-2)} (L_f +\beta^{1/2}_t L_\sigma )^2                                                                                                                                                                                                                                                           \\
         & \quad \cdot \sum_{q=1} ^{M^{(n-3)}}\Bigl[   | \tilde{  z}^{(n-3)}_{{\bm 0},q,t} ( {\bm x}^{(s) },\ldots, {\bm x}^{(n-3)} | {\bm z}_{\bm\epsilon}^{(s-1)})    - \tilde{  z}^{(n-3)}_{{\bm 0},q,t} ( {\bm x}^{(s+1) },\ldots, {\bm x}^{(n-3)} | {\bm z}_{\bm\epsilon}^{(s)}) |  \Bigr]                                            \\
         & \leq                                                                                                                                                                                                                                                                                                                            \\
         & \vdots                                                                                                                                                                                                                                                                                                                          \\
         & \leq 2 \beta^{1/2} _t L_\sigma   M_{\text{prod}}  (L_f + \beta^{1/2}_t L_\sigma +1)^{N-2}                                                                                                                                                                                                                                       \\
         & \quad  \cdot \sum_{p=0} ^{n-s-2}\sum_{i=1}^{   M^{(n-1-p)} } \Bigl[ \sigma^{(n-1-p)} _{i,t} (   \tilde{ \bm z}^{(n-2-p)}_{{\bm 0},t} ( {\bm x}^{(s) },\ldots, {\bm x}^{(n-2-p)} | {\bm z}^{(s-1)}_{\bm\epsilon}),{\bm x}^{(n-1-p)}) \Bigr]                                                                                      \\
         & \quad +    2 M_{\text{prod}}   L_\sigma (L_f +\beta^{1/2} _t L_\sigma +1)^{N-2}    \sum_{q=1} ^{M^{(s)}}\Bigl[     |    \mu^{(s)}_{q,t} (          {\bm z}^{(s-1)}_{\bm\epsilon} ,{\bm x}^{(s)} )  -     f^{(s)}_{q} (          {\bm z}^{(s-1)}_{\bm\epsilon} ,{\bm x}^{(s)} )                    -\epsilon^{(s)}_q |    \Bigr] \\
         & \leq     2    M_{\text{prod}}  (L_f + \beta^{1/2}_t L_\sigma +1)^{N-1}                                                                                                                                                                                                                                                          \\
         & \quad \cdot \sum_{p=0} ^{n-s-1} \sum_{i=1}^{   M^{(n-1-p)} }  \Bigl[    \sigma^{(n-1-p)} _{i,t} (   \tilde{ \bm z}^{(n-2-p)}_{{\bm 0},t} ( {\bm x}^{(s) },\ldots, {\bm x}^{(n-2-p)} | {\bm z}^{(s-1)}_{\bm\epsilon}),{\bm x}^{(n-1-p)})    \Bigr]                                                                               \\
         & \quad + 2    M_{\text{prod}}  (L_f + \beta^{1/2}_t L_\sigma +1)^{N-1} \sum_{q=1}^{M^{(s)}}  | \epsilon^{(s)}_q |                                                                                                                                                                                                                \\
         & =     2    M_{\text{prod}}  C_{0,t}^{N-1}  \sum_{p=0} ^{n-s-1}  \sum_{i=1}^{   M^{(n-1-p)} }  \Bigl[   \sigma^{(n-1-p)} _{i,t} (   \tilde{ \bm z}^{(n-2-p)}_{{\bm 0},t} ( {\bm x}^{(s) },\ldots, {\bm x}^{(n-2-p)} | {\bm z}^{(s-1)}_{\bm\epsilon}),{\bm x}^{(n-1-p)})      \Bigr]                                              \\
         & \quad + 2    M_{\text{prod}} C_{0,t}^{N-1} \sum_{q=1}^{M^{(s)}}  | \epsilon^{(s)}_q | .
    \end{align}
\end{proof}
\begin{lemma}\label{lem:nextAAA_zero}
    Assume that the same condition as in~\cref{theorem:i_regret_zero} holds.
    Then, for any $s \in \{1,\ldots, N-1 \}$, $j \geq 0$ with $s +j \leq N$,  iteration $t \geq 1$, realization $\bm\epsilon$ and input ${\bm x}^{(1)},\ldots,{\bm x}^{(N)}$, the following holds with probability at least $1-\delta$:
    \begin{align}
         & \tilde{\sigma}^{(N-j)}_{{\bm 0},t} ({\bm x} ^{(s)}, \ldots ,{\bm x}^{(N-j)} | {\bm z}^{(s-1)}_{\bm\epsilon} )           \\
         & \leq
        C_{2,t} \tilde{\sigma}^{(N-j)}_{{\bm 0},t} ({\bm x} ^{(s+1)}, \ldots ,{\bm x}^{(N-j)} | {\bm z}^{(s)}_{\bm\epsilon} )  +
        C_{2,t} \tilde{\sigma}^{(N-j-1)}_{{\bm 0},t} ({\bm x} ^{(s)}, \ldots ,{\bm x}^{(N-j-1)} | {\bm z}^{(s-1)}_{\bm \epsilon} ) \\
         & \quad  +
        C_{2,t} \sum _{i=1}^{M^{(s)} }
        \sigma^{(s)}_{i,t} (  {\bm z}^{(s-1)}_{\bm\epsilon}  ,{\bm x}^{(s)} )    + C_{2,t} \sum _{i=1}^{M^{(s)} }  |\epsilon ^{(s)}_i |,
    \end{align}
    where
    \begin{align}
         & \tilde{\sigma}^{(N-j)}_{{\bm 0},t} ({\bm x} ^{(s)}, \ldots ,{\bm x}^{(N-j)} | {\bm z}^{(s-1)} _{\bm\epsilon})                                                                                                                                                         \\
         & = \sum_{p=j}^{N-s}  \prod_{q=1}^p  M^{(N-q+1)}  L_f ^p  \sum _{i=1}^{M^{(N-p)} } \Bigl[          \sigma^{(N-p)}_{i,t} (   \tilde{\bm z}^{(N-p-1)}_{{\bm 0},t} ({\bm x}^{(s)} , \ldots , {\bm x}^{(N-p-1)} |{\bm z}^{(s-1)}_{\bm\epsilon} ) ,{\bm x}^{(N-p)} ) \Bigr].
    \end{align}
\end{lemma}
\begin{proof}
    From the definition of $\tilde{\sigma}^{(N-j)}_{{\bm 0},t} ({\bm x} ^{(s)}, \ldots ,{\bm x}^{(N-j)} | {\bm z}^{(s-1)}_{\bm\epsilon} ) $,
    the following inequality holds with probability at least $1-\delta$:
    \begin{align}
         & \tilde{\sigma}^{(N-j)}_{{\bm 0},t} ({\bm x} ^{(s)}, \ldots ,{\bm x}^{(N-j)} | {\bm z}^{(s-1)}_{\bm\epsilon} )                                                                                    \\
         & = \sum_{p=j}^{N-s}  \prod_{q=1}^p  M^{(N-q+1)}  L_f ^p
        \sum _{i=1}^{M^{(N-p)} } \Bigl[   \sigma^{(N-p)}_{i,t} (   \tilde{\bm z}^{(N-p-1)}_{{\bm 0},t} ({\bm x}^{(s)} , \ldots , {\bm x}^{(N-p-1)} |{\bm z}^{(s-1)}_{\bm\epsilon} ) ,{\bm x}^{(N-p)} ) \Bigr]
        \\
         & \leq
        M_{\text{prod}} C^{N-1}_0 \sum_{p=j}^{N-s}
        \sum _{i=1}^{M^{(N-p)} }\Bigl[  \sigma^{(N-p)}_{i,t} (   \tilde{\bm z}^{(N-p-1)}_{{\bm 0},t} ({\bm x}^{(s)} , \ldots , {\bm x}^{(N-p-1)} |{\bm z}^{(s-1)} _{\bm\epsilon}) ,{\bm x}^{(N-p)} )\Bigr]  \\
         & =
        M_{\text{prod}} C^{N-1}_0 \sum_{p=j}^{N-s-1}
        \sum _{i=1}^{M^{(N-p)} } \Bigl[  \sigma^{(N-p)}_{i,t} (   \tilde{\bm z}^{(N-p-1)}_{{\bm 0},t} ({\bm x}^{(s+1)} , \ldots , {\bm x}^{(N-p-1)} |{\bm z}^{(s)}_{\bm\epsilon} ) ,{\bm x}^{(N-p)} )\Bigr] \\
         & \quad  + M_{\text{prod}} C^{N-1}_0 \sum_{p=j}^{N-s-1}
        \sum _{i=1}^{M^{(N-p)} }\Bigl[   \left ( \sigma^{(N-p)}_{i,t} \Bigl( \tilde{\bm z}^{(N-p-1)}_{{\bm 0},t} ({\bm x}^{(s)} , \ldots , {\bm x}^{(N-p-1)} |{\bm z}^{(s-1)}_{\bm\epsilon} ) ,{\bm x}^{(N-p)} \Bigr)\Bigr]
        \right .                                                                                                                                                                                            \\
         & \quad  \left.   -
        \sigma^{(N-p)}_{i,t} (   \tilde{\bm z}^{(N-p-1)}_{{\bm 0},t} ({\bm x}^{(s+1)} , \ldots , {\bm x}^{(N-p-1)} |{\bm z}^{(s)}_{\bm\epsilon} ) ,{\bm x}^{(N-p)} )
        \right )     +
        M_{\text{prod}} C^{N-1}_0
        \sum _{i=1}^{M^{(s)} }
        \sigma^{(s)}_{i,t} (  {\bm z}^{(s-1)}_{\bm\epsilon}  ,{\bm x}^{(s)} ) .
    \end{align}
    Hence, from~\cref{lem:fms_zero}, it follows that
    \begin{align}
         & |\sigma^{(N-p)} _{i,t} ( \tilde{\bm  z}^{(N-p-1)}_{{\bm 0},t} ( {\bm x}^{(s:N-p-1) } | {\bm z}_{\bm\epsilon}^{(s-1)}),{\bm x}^{(N-p)} )    -         \sigma^{(N-p)} _{i,t} ( \tilde{\bm  z}^{(N-p-1)}_{{\bm 0},t} ( {\bm x}^{(s+1:N-p-1) }| {\bm z}_{\bm\epsilon}^{(s)}),{\bm x}^{(N-p)} ) | \\
         &
        \leq 2 M_{\text{prod}} C^{N-1}_0  \sum_{r=0} ^{N-p-s-1}
        \sum_{j=1}^{   M^{(N-p-1-r)} } \Bigl[  \sigma^{(N-p-1-r)} _{j,t} \Bigl(       \tilde{ \bm z}^{(N-p-2-r)}_{{\bm 0},t} ( {\bm x}^{(s:N-p-2-r) } | {\bm z}_{\bm\epsilon}^{(s-1)}),{\bm x}^{(N-p-1-r)}\Bigr) \Bigr]                                                                                 \\
         & \quad +2 M_{\text{prod}} C^{N-1}_0  \sum_{q=1}^{M^{(s)} }  |\epsilon^{(s)}_q |.
    \end{align}
    Hence, using the same argument as in the proof of~\cref{lem:nextAAA}, we have the desired result.
\end{proof}
\begin{lemma}\label{lem:next_zero}
    Assume that the same condition as in~\cref{theorem:i_regret_zero} holds.
    Then, for any $s \in \{1,\ldots, N-1 \}$,  iteration $t \geq 1$, realization $\bm\epsilon$ and input ${\bm x}^{(1)},\ldots,{\bm x}^{(N)}$, the following holds with probability at least $1-\delta$:
    \begin{align}
         & \tilde{\sigma}^{(N)}_{{\bm 0},1,t} ({\bm x} ^{(s)}, \ldots ,{\bm x}^{(N)} | {\bm z}_{\bm\epsilon}^{(s-1)} )                      \\
         & \quad   \leq C_{3,t} \tilde{\sigma}^{(N)}_{{\bm 0},1,t} ({\bm x} ^{(s+1)}, \ldots ,{\bm x}^{(N)} | {\bm z}^{(s)}_{\bm\epsilon} )
        +C_{3,t} \sum_{i=1} ^{M^{(s)} } \sigma^{(s)} _{i,t}  ({\bm z}^{(s-1)}_{\bm\epsilon} ,{\bm x}^{(s)} )
        +C_{3,t} \sum_{i=1} ^{M^{(s)} } | \epsilon^{(s)}_i |
        .
    \end{align}
\end{lemma}
\begin{proof}
    By repeating~\cref{lem:nextAAA_zero}, the following holds with probability at least $1-\delta$:
    \begin{align}
         & \tilde{\sigma}^{(N)}_{{\bm 0},1,t} ({\bm x} ^{(s)}, \ldots ,{\bm x}^{(N)} | {\bm z}^{(s-1)}_{\bm\epsilon} )                                             \\
         & =
        \tilde{\sigma}^{(N-0)}_{{\bm 0},t} ({\bm x} ^{(s)}, \ldots ,{\bm x}^{(N-0)} | {\bm z}^{(s-1)}_{\bm \epsilon} )                                             \\
         & \leq
        C_{2,t} \tilde{\sigma}^{(N-0)}_{{\bm 0},t} ({\bm x} ^{(s+1)}, \ldots ,{\bm x}^{(N-0)} | {\bm z}^{(s)}_{\bm\epsilon} )   +
        C_{2,t} \sum _{i=1}^{M^{(s)} }
        \sigma^{(s)}_{i,t} (  {\bm z}^{(s-1)}_{\bm\epsilon}  ,{\bm x}^{(s)} ) +C_{2,t}  \sum _{i=1}^{M^{(s)}} |\epsilon^{(s)}_i |                                  \\
         & \quad +C_{2,t} \tilde{\sigma}^{(N-1)}_{{\bm 0},t} ({\bm x} ^{(s)}, \ldots ,{\bm x}^{(N-1)} | {\bm z}^{(s-1)}_{\bm\epsilon} )                            \\
         & \leq
        C_{2,t} \tilde{\sigma}^{(N-0)}_{{\bm 0},t} ({\bm x} ^{(s+1)}, \ldots ,{\bm x}^{(N-0)} | {\bm z}^{(s)}_{\bm \epsilon} ) +
        C_{2,t} \sum _{i=1}^{M^{(s)} }
        \sigma^{(s)}_{i,t} (  {\bm z}^{(s-1)}_{\bm\epsilon}  ,{\bm x}^{(s)} )    +C_{2,t}  \sum _{i=1}^{M^{(s)}} |\epsilon^{(s)}_i |                               \\
         & \quad +
        C^2_{2,t} \tilde{\sigma}^{(N-1)}_{{\bm 0},t} ({\bm x} ^{(s+1)}, \ldots ,{\bm x}^{(N-1)} | {\bm z}^{(s)} _{\bm\epsilon})     +
        C^2_{2,t} \sum _{i=1}^{M^{(s)} }
        \sigma^{(s)}_{i,t} (  {\bm z}^{(s-1)}_{\bm\epsilon}  ,{\bm x}^{(s)} )      +C^2_{2,t}  \sum _{i=1}^{M^{(s)}} |\epsilon^{(s)}_i |                           \\
         & \quad  +C^2_{2,t} \tilde{\sigma}^{(N-2)}_{{\bm 0},t} ({\bm x} ^{(s)}, \ldots ,{\bm x}^{(N-2)} | {\bm z}^{(s-1)}_{\bm\epsilon} )                         \\
         & \leq                                                                                                                                                    \\
         & \vdots                                                                                                                                                  \\
         & \leq
        (C_{2,t} + C^2_{2,t} + \cdots + C^N_{2,t} ) \tilde{\sigma}^{(N-0)}_{{\bm 0},t} ({\bm x} ^{(s+1)}, \ldots ,{\bm x}^{(N-0)} | {\bm z}^{(s)}_{\bm \epsilon} ) \\
         & \quad +(C_{2,t} + C^2_{2,t} + \cdots + C^N_{2,t} )  \sum _{i=1}^{M^{(s)} }
        \sigma^{(s)}_{i,t} (  {\bm z}^{(s-1)}_{\bm \epsilon}   ,{\bm x}^{(s)} )  +(C_{2,t} + C^2_{2,t} + \cdots + C^N_{2,t} )  \sum _{i=1}^{M^{(s)} }
        | \epsilon^{(s)}_i |
        \\
         & \leq
        N C^N_{2,t}  \tilde{\sigma}^{(N)}_{{\bm 0},1,t} ({\bm x} ^{(s+1)}, \ldots ,{\bm x}^{(N)} | {\bm z}^{(s)}_{\bm\epsilon} )     + N C^N_{2,t}  \sum _{i=1}^{M^{(s)} }
        \sigma^{(s)}_{i,t} (  {\bm z}^{(s-1)}_{\bm\epsilon}  ,{\bm x}^{(s)} )  + N C^N_{2,t}  \sum _{i=1}^{M^{(s)} }
        |\epsilon^{(s)}_i |.
    \end{align}
\end{proof}
\begin{lemma}\label{lem:c_bound_zero}
    Assume that the same condition as in~\cref{theorem:i_regret_zero} holds.
    Then, for any $n \in [ N]$,  iteration $t \geq 1$,  ${\bm y}^{(n-1)} \in \tilde{\mathcal{Y} }^{(n-1)} $ and input ${\bm x}^{(n)}  \in \mathcal{X}^{(n)}  $, the following holds:
    \begin{equation}
        \eta_t \hat{b}^{(n)}_t ({\bm x}^{(n)} |{\bm y}^{(n-1)} )  \leq \hat{c}^{(n)}_t ({\bm x}^{(n)} |{\bm y}^{(n-1)} )      \leq (2 \beta^{1/2} _t + \eta_t ) \hat{b}^{(n)}_t ({\bm x}^{(n)} |{\bm y}^{(n-1)} ).
    \end{equation}
\end{lemma}
\begin{proof}
    By using the same argument as in the proof of~\cref{lem:c_bound},  we get~\cref{lem:c_bound_zero}.
\end{proof}
Using these lemmas we prove~\cref{theorem:i_regret_zero}.
\begin{proof}
    Let $t \in N \mathbb{Z}_{\geq 0} $.
    Then, from~\cref{lem:c_bound_zero}, ${\bm x}^{(1)}_{t+1} $ satisfies that
    \begin{equation}
        \hat{c}^{(1)}_t ({\bm x}^{(1)}_{t+1} | {\bm 0} ) \leq (2 \beta ^{1/2}_t + \eta _t ) \hat{b}^{(1)}_t ({\bm x}^{(1)}_{t+1} | {\bm 0} )
        =
        (2 \beta ^{1/2} _t + \eta _t ) \tilde{\sigma}^{(N)}_{{\bm 0},1,t} ({\bm x}^{(1)}_{t+1}, \tilde{\bm x}^{(2)} ,\ldots , \tilde{\bm x}^{(N)} | {\bm 0} ).
        \label{eq:ineq001_zero}
    \end{equation}
    In addition, from~\cref{eq:ineq001_zero} and~\cref{lem:next_zero}, using the same argument as in the proof of~\cref{theorem:i_regret}, with probability at least $1-\delta$, $\hat{c}^{(1)}_t ({\bm x}^{(1)}_{t+1} | {\bm 0} )$ can be bounded as follows:
    \begin{align}
        \hat{c}^{(1)}_t ({\bm x}^{(1)}_{t+1} | {\bm 0} ) & \leq
        (2 \beta ^{1/2}_t + \eta _t ) C_{3,t} \sum_{i=1} ^{M^{(1)} } \sigma^{(1)}_{{\bm 0},i,t} ({\bm 0},{\bm x}^{(1)}_{t+1} )   +
        (2 \beta ^{1/2}_t + \eta _t ) C_{3,t} \sum_{i=1} ^{M^{(1)} } |\epsilon^{(1)}_i |
        \\
                                                         & \qquad  +
        (2 \beta ^{1/2} _t+ \eta _t )C_{3,t}  \tilde{\sigma}^{(N)}_{{\bm 0},1,t} (\tilde{\bm x}^{(2)} ,\ldots , \tilde{\bm x}^{(N)} |{\bm y}^{(1)}_{t+1} ) \\
                                                         & \leq
        (2 \beta ^{1/2}_t + \eta _t ) C_{3,t} \sum_{i=1} ^{M^{(1)} } \sigma^{(1)}_{{\bm 0},i,t} ({\bm 0},{\bm x}^{(1)}_{t+1} )  +
        (2 \beta ^{1/2}_t + \eta _t ) C_{3,t} \sum_{i=1} ^{M^{(1)} } |\epsilon^{(1)}_i |
        \\
                                                         & \qquad+
        (2 \beta ^{1/2} _t+ \eta _t )^2 C_{3,t} \eta^{-1}_t  \tilde{\sigma}^{(N)}_{{\bm 0},1,t} ({\bm x}^{(2)}_{t+2} ,\ldots , \tilde{\bm x}^{(N)} |{\bm y}^{(1)}_{t+1} ).
    \end{align}
    By using~\cref{lem:next_zero} again, it follows that
    \begin{align}
        \hat{c}^{(1)}_t ({\bm x}^{(1)}_{t+1} | {\bm 0} ) & \leq
        (2 \beta^{1/2} _t+ \eta_t +1 ) ^N C^N_{3,t} \eta^{-N}_t  \sum_{n=1}^{N} \sum_{i=1} ^{M^{(n)} } \sigma^{(n)}_{{\bm 0},i,t} ({\bm y}^{(n-1)},{\bm x}^{(n)}_{t+n} ) \\
                                                         & \qquad +
        (2 \beta^{1/2} _t+ \eta_t +1 ) ^N C^N_{3,t} \eta^{-N}_t  \sum_{n=1}^{N} \sum_{i=1} ^{M^{(n)} }|\epsilon^{(n)}_i |                                                \\
                                                         & \leq
        (2 \beta^{1/2} _t+ 1 +1 ) ^N C^N_{3,t} \eta^{-N}_t  \sum_{n=1}^{N} \sum_{i=1} ^{M^{(n)} } \sigma^{(n)}_{{\bm 0},i,t} ({\bm y}^{(n-1)},{\bm x}^{(n)}_{t+n} )      \\
                                                         & \qquad+
        (2 \beta^{1/2} _t+ 1 +1 ) ^N C^N_{3,t} \eta^{-N}_t  \sum_{n=1}^{N} \sum_{i=1} ^{M^{(n)} }|\epsilon^{(n)}_i | .
    \end{align}
    Thus, multiplying both sides by $2 \beta^{1/2}_t \eta^{-1}_t $, we get
    \begin{align}
        2 \beta^{1/2}_t \eta^{-1}_t \hat{c}^{(1)}_t ({\bm x}^{(1)}_{t+1} | {\bm 0} ) & \leq
        2 \beta^{1/2}_t \eta^{-1}_t
        (2 \beta^{1/2} _t+ 2 ) ^N C^N_{3,t} \eta^{-N}_t  \sum_{n=1}^{N} \sum_{i=1} ^{M^{(n)} } \sigma^{(n)}_{{\bm 0},i,t} ({\bm y}^{(n-1)},{\bm x}^{(n)}_{t+n} ) \\
                                                                                     & \qquad +
        2 \beta^{1/2}_t \eta^{-1}_t
        (2 \beta^{1/2} _t+ 2) ^N C^N_{3,t} \eta^{-N}_t  \sum_{n=1}^{N} \sum_{i=1} ^{M^{(n)} }|\epsilon^{(n)}_i |                                                 \\
                                                                                     &
        =
        C_{9,t}  \sum_{n=1}^{N} \sum_{i=1} ^{M^{(n)} } \sigma^{(n)}_{{\bm 0},i,t} ({\bm y}^{(n-1)},{\bm x}^{(n)}_{t+n} ) +
        C_{9,t}  \sum_{n=1}^{N} \sum_{i=1} ^{M^{(n)} }|\epsilon^{(n)}_i |.
    \end{align}
    Here, using  $(a+b)^2 \leq 2 a^2+2b^2$ and the Cauchy--Schwarz inequality, we obtain
    \begin{equation}
        (2 \beta^{1/2}_t \eta^{-1}_t \hat{c}^{(1)}_t ({\bm x}^{(1)}_{t+1} | {\bm 0} ))^2       \leq
        2 C^2_{9,t} M_{\text{sum}} \sum_{n=1}^{N} \sum_{i=1} ^{M^{(n)} } \sigma^{(n)2}_{{\bm 0},i,t} ({\bm y}^{(n-1)},{\bm x}^{(n)}_{t+n} ) +
        2 C^2_{9,t} M_{\text{sum}} \epsilon_{ \text{sum}} . \label{eq:ep_sum_bound}
    \end{equation}

    Next, we show the existence of the sequence $T_0 < T_1 < \cdots $ satisfying
    \begin{equation}
        \mathbb{P} ( \exists t \in N \mathbb{Z}_{\geq 0} \ \text{s.t.} \ T_{k-1} \leq t \leq T_k , \ 2 C^2_{9,t} M_{\text{sum}} \epsilon_{\text{sum},t} < \xi^2 /2 )
        > 1- \frac{ 6 \delta  }{\pi^2 k^2}.
    \end{equation}
    From~\cref{assumption:esum}, we have
    \begin{equation}
        \mathbb{P} ( 2 C^2_{9,t} M_{\text{sum}} \epsilon_{\text{sum},t} < \xi^2 /2 )
        =
        \mathbb{P} (   \epsilon_{\text{sum},t} < M^{-1}_{\text{sum}} C^{-2} _{9,t} \xi^2 /4 )
        >
        \frac{C \xi^2}{ 4 M_{\text{sum}} C^2_{9,t}}.
    \end{equation}
    This implies that
    \begin{equation}
        1- \mathbb{P} ( 2 C^2_{9,t} M_{\text{sum}} \epsilon_{\text{sum},t} < \xi^2 /2 ) \leq 1- \frac{C \xi^2}{ 4 M_{\text{sum}} C^2_{9,t}}.
    \end{equation}
    Therefore, using $1+x \leq e^x$ we get
    \begin{align}
        \prod_{q=t}^{t^\prime}  ( 1- \mathbb{P} ( 2 C^2_{9,q} M_{\text{sum}} \epsilon_{\text{sum},q} < \xi^2 /2 )  )
         & \leq
        \prod_{q=t}^{t^\prime} \left (  1- \frac{C \xi^2}{ 4 M_{\text{sum}} C^2_{9,q}} \right )     \\
         & \leq
        \prod_{q=t}^{t^\prime}  \exp \left ( - \frac{C \xi^2}{ 4 M_{\text{sum}} C^2_{9,q}} \right ) \\
         & =
        \exp \left ( - \frac{C \xi^2}{ 4 M_{\text{sum}} } \sum_{q=t}^{t^\prime}   C^{-2}_{9,q} \right ). \label{eq:p_bound_c9}
    \end{align}
    Moreover, from~\cref{assumption:C9},
    the right hand side of~\cref{eq:p_bound_c9} tends to zero when
    $t^\prime \to \infty $.
    Thus, we can construct the sequence $T_1,T_2 , \ldots $ satisfying~\cref{eq:ep_sum_small_prob}.
    Then, with probability at least  $1-\delta$, the following holds:
    \begin{equation}
        \forall k \in \mathbb{N},\ \exists \tilde{T}_k \quad \mr{s.t.} \quad T_{k-1} \leq \tilde{T}_k \leq T_k , \ 2 C^2_{9,\tilde{T}_k} M_{\text{sum}} \epsilon_{\text{sum},\tilde{T}_k} < \xi^2 /2.
    \end{equation}

    On the other hand, for the positive number  $K$ satisfying the theorem's inequality, we define
    \begin{equation}
        \hat{T} = \argmin _{ 1 \leq k \leq K }   (2 \beta^{1/2}_{\tilde{T}_k}  \eta^{-1}_{\tilde{T}_k} \hat{c}^{(1)}_{\tilde{T}_k} ({\bm x}^{(1)}_{{\tilde{T}_k}+1} | {\bm 0} ))^2.
    \end{equation}
    Then, it follows that
    \begin{align}
        K  (2 \beta^{1/2}_{\hat{T}}  \eta^{-1}_{\hat{T}} \hat{c}^{(1)}_{\hat{T}} ({\bm x}^{(1)}_{{\hat{T}}+1} | {\bm 0} ))^2
         & \leq
        \sum_{k=1}^K (2 \beta^{1/2}_{\tilde{T}_k}  \eta^{-1}_{\tilde{T}_k} \hat{c}^{(1)}_{\tilde{T}_k} ({\bm x}^{(1)}_{{\tilde{T}_k}+1} | {\bm 0} ))^2
        \\
         & \leq K \xi^2/2 + 2 C^2_{9,T_K} M_{\text{sum}} \sum _{t=1} ^{T_K} \sum_{n=1}^{N} \sum_{i=1} ^{M^{(n)} } \sigma^{(n)2}_{{\bm 0},i,t} ({\bm y}^{(n-1)},{\bm x}^{(n)}_{t+n} ) \\
         & \leq K \xi^2/2
        + \frac{4 C^2_{9,T_K} M^2_{\text{sum}}}{ \log (1+\sigma^{-2} ) } \tilde{\gamma} _{T_K}.
    \end{align}
    By dividing both sides by $K$, we obtain
    \begin{align}
        (2 \beta^{1/2}_{\hat{T}}  \eta^{-1}_{\hat{T}} \hat{c}^{(1)}_{\hat{T}} ({\bm x}^{(1)}_{{\hat{T}}+1} | {\bm 0} ))^2 & \leq
        \xi^2/2
        + \frac{4 C^2_{9,T_K} M^2_{\text{sum}}}{ \log (1+\sigma^{-2} ) } \tilde{\gamma} _{T_K} K^{-1}                                                    \\
                                                                                                                          & < \xi^2/2 + \xi^2/2 = \xi^2.
    \end{align}
    This implies that
    \begin{equation}
        2 \beta^{1/2}_{\hat{T}}  \eta^{-1}_{\hat{T}} \hat{c}^{(1)}_{\hat{T}} ({\bm x}^{(1)}_{{\hat{T}}+1} | {\bm 0} ) < \xi. \label{eq:xi_bound}
    \end{equation}
    Finally,  from the definition of the estimated solution and CIs, $F(\cdot )$ can be bounded as follows:
    \begin{align}
        F({\bm x}^{(1)}_{F,\ast}, \ldots ,{\bm x}^{(N)}_{F,\ast} )
         & \leq \min _{ t \in N \mathbb{Z}_{\geq 0} , t \leq T_K }  \widehat{\mr{ UCB}}^{(F)} _t ({\bm x}^{(1)}_{F,\ast} ,\ldots ,{\bm x}^{(N)}_{F,\ast} ) \\
         & \leq  \widehat{\mr{ UCB}}^{(F)} _{\hat{T} }  ({\bm x}^{(1)}_{F,\ast} ,\ldots ,{\bm x}^{(N)}_{F,\ast} ) ,                                        \\
        F(\hat{\bm x}^{(1)}_{F,T_K}, \ldots , \hat{\bm x}^{(N)}_{F,T_K}  )
         & \geq
        \max _{ t \in N \mathbb{Z}_{\geq 0} , t \leq T_K }  \widehat{\mr{ LCB}}^{(F)} _t ({\bm x}^{(1)} ,\ldots ,{\bm x}^{(N)} )                           \\
         & \geq
        \widehat{\mr{ LCB}}^{(F)} _{\hat{T} }  ({\bm x}^{(1)}_{F,\ast} ,\ldots ,{\bm x}^{(N)}_{F,\ast} ) .
    \end{align}
    Therefore, the following holds with probability at least $1-2 \delta$:
    \begin{align}
        F({\bm x}^{(1)}_{F,\ast}, \ldots ,{\bm x}^{(N)}_{F,\ast} )  -
        F(\hat{\bm x}^{(1)}_{F,T_K}, \ldots , \hat{\bm x}^{(N)}_{F,T_K}  )
         & \leq 2 \beta^{1/2}_{\hat{T}} \tilde{\sigma}^{(N)} _{{\bm 0},1,\hat{T} } (   {\bm x}^{(1)}_{F,\ast}, \ldots ,{\bm x}^{(N)}_{F,\ast}  ) \\
         & \leq 2\beta^{1/2} _{\hat{T}}  \hat{b}^{(1)}_{ \hat{T} }  ({\bm x}^{(1)}_{F,\ast} | {\bm 0} )                                          \\
         & \leq 2 {\beta}^{1/2} _{\hat{T}}  \eta ^{-1}  _{\hat{T}}   \hat{c}^{(1)}_{ \hat{T} }    ({\bm x}^{(1)}_{F,\ast} | {\bm 0} )  \leq
        2 {\beta}^{1/2} _{\hat{T}}  \eta ^{-1}  _{\hat{T}}   \hat{c}^{(1)}_{ \hat{T} }    ({\bm x}^{(1)}_{\hat{T} +1}  | {\bm 0} ) . \label{eq:F_ast_noise}
    \end{align}
    Hence, by substituting~\cref{eq:xi_bound} into~\cref{eq:F_ast_noise}, we have~\cref{theorem:i_regret_zero}.
\end{proof}

%-------------------------------------------------------------------------------------------
\section{Sufficient Conditions and Modifications for the Proposed Method}\label{app:modify}
In this section, we consider theorem's conditions and its modifications.
First, in the noiseless setting, we assume that $\tilde{\bm \mu}_t ^{(m)} ({\bm x}^{(n)},\ldots,{\bm x}^{(m) } | {\bm y}^{(n-1)} ) \in \mathcal{Y}^{(m)}$ to construct the valid CI
.
For this assumption, the following sufficient condition exists.
\begin{theorem}\label{theorem:sufficient_condition_mutilde}
    Assume that  each $\mathcal{X}^{(n)}$ is a compact set, and each observation is noiseless.
    Also assume that each $f^{(n)}_m$ is a function defined on $[-2B_{n-1},2B_{n-1}]^{M^{(n-1)}} \times \mathcal{X}^{(n)} $ and satisfies $f^{(n)}_m \in \mathcal{H}_{k^{(n)} } $, where $B_n $ is some positive constant satisfying $\| f^{(n)}_m \| _{\mathcal{H}_{k^{(n)}}} \leq B_n$ and $B_0 =0$.
    Then, $\tilde{\bm \mu}^{(n)} _t ({\bm x}^{(1)},\ldots,{\bm x}^{(n)} )  \in [-2B_{n},2B_{n}]^{M^{(n)}} $
    for any $n \in [N]$, $t \geq 1$ and ${\bm x}^{(1)},\ldots,{\bm x}^{(n)}$.
\end{theorem}
\begin{proof}
    From the reproducing property of $k^{(n)} $, noting that $k^{(n)} ({\bm a},{\bm a} ) \leq 1$ we have
    \begin{align}
        | f^{(1)}_m ( {\bm a} ) | & =| \langle  f^{(1)}_m ( \cdot ), k^{(1)} (\cdot,  {\bm a}  ) \rangle _{\mathcal{H}_{k^{(1)}} } | \\
                                  & \leq \| f^{(1)}_m \|  _{\mathcal{H}_{k^{(1)}}  }k^{(1)} ( {\bm a} , {\bm a}  )^{1/2} \leq B_1  .
    \end{align}
    In addition, since $\mathcal{X}^{(1)}$ is the compact set, $[-2B_{0},2B_{0}]^{M^{(0)}} \times \mathcal{X}^{(1)}$ is also the compact set.
    Hence, from~\cref{lem:ci} the following holds for any $m \in [M^{(1)}]$, $t \geq 1$ and $ ({\bm w},{\bm x} ) \in [-2B_{0},2B_{0}]^{M^{(0)}} \times \mathcal{X}^{(1)}$:
    \begin{align}
        | \mu^{(1)}_{m,t} ({\bm w},{\bm x} ) |
         & =
        |
        \mu^{(1)}_{m,t} ({\bm w},{\bm x} ) -f^{(n)}_m  ({\bm w},{\bm x} ) +f^{(n)}_m  ({\bm w},{\bm x} )
        |                       \\
         & \leq
        |
        \mu^{(1)}_{m,t} ({\bm w},{\bm x} ) -f^{(n)}_m  ({\bm w},{\bm x} ) |+|f^{(n)}_m  ({\bm w},{\bm x} )
        |                       \\
         & \leq B_1 +B_1 =2B_1.
    \end{align}
    This implies that ${\bm\mu}^{(1)}_t ({\bm 0},{\bm x}^{(1)})=\tilde{\bm \mu}^{(1)}_t ({\bm x} ^{(1)}) \in  [-2B_{1},2B_{1}]^{M^{(1)}} $.
    By repeating this process, we get
    \begin{align}
        {\bm \mu}^{(n)}_t   (  \tilde{\bm \mu}^{(n-1)}_t ({\bm x} ^{(1)},\ldots ,{\bm x}^{(n-1)}),{\bm x}^{(n)} )
        =
        \tilde{\bm \mu}^{(n)}     ({\bm x}^{(1)},\ldots,{\bm x}^{(n)} )
        \in  [-2B_{n},2B_{n}]^{M^{(n)}}.
    \end{align}
\end{proof}
\Cref{theorem:sufficient_condition_mutilde} implies that by defining   $\mathcal{Y}^{(n)}$  as $ [-2B_{n},2B_{n}]^{M^{(n)}} $ and
$B = \max _{1 \leq n \leq N} B_n$, we obtain~\cref{assumption:regu1} and $\tilde{\bm \mu}_t ^{(n)} ({\bm x}^{(1)},\ldots,{\bm x}^{(n)} ) \in \mathcal{Y}^{(n)}$. Similarly, under the same assumption we have $\tilde{\bm \mu}_t ^{(m)} ({\bm x}^{(n)},\ldots,{\bm x}^{(m) } | {\bm y}^{(n-1)} ) \in \mathcal{Y}^{(m)}$.

Next, we consider the condition $ \tilde{\bm z}_{{\bm\epsilon},t} ^{(n)} ({\bm x}^{(1)},\ldots,{\bm x}^{(n)} )  \in \tilde{\mathcal{Y}}^{(n)}$
for the noisy observation setting.
In the noisy setting, it is not easy to give a sufficient condition for this condition to be satisfied.
Nevertheless, we can avoid this condition by modifying the definition of $\tilde{\bm z}_{{\bm\epsilon},t} ^{(n)} ({\bm x}^{(1)},\ldots,{\bm x}^{(n)} )$.
Let $\mathcal{L} = [-L,L] ^d $ be a $d$-dimensional hypercube.
For each ${\bm a} =(a_1,\ldots, a_d) \in \mathbb{R}^d $, suppose that  $\mathcal{P} (\mathcal{M},{\bm a} )$ is a
projection of ${\bm a} $ onto $ \mathcal{M} $, where the $i$-th element of  $\mathcal{P} (\mathcal{M},{\bm a} )$, $\mathcal{P}_i (\mathcal{M},{\bm a} )$, is given by
\begin{equation}
    \mathcal{P}_i (\mathcal{M},{\bm a} )=\argmin_{l \in [-L,L] }   |a_i - l|.
\end{equation}
Then, the following theorem holds.
\begin{theorem}\label{theorem:def_muhat}
    Assume that  each $\mathcal{X}^{(n)}$ is a compact set, and each observation noise $\epsilon^{(n)}_m $ is a zero mean random variable with $-A \leq \epsilon^{(n)}_m \leq A$.
    Also assume that each $f^{(n)}_m$ is a function defined on $[-A_{n-1}-B_{n-1},A_{n-1}+B_{n-1}]^{M^{(n-1)}} \times \mathcal{X}^{(n)} $ and satisfies $f^{(n)}_m \in \mathcal{H}_{k^{(n)} } $, where $A_0=B_0=0$, $A_n =A$ and $B_n $ is some positive constant satisfying $\| f^{(n)}_m \| _{\mathcal{H}_{k^{(n)}}} \leq B_n$.
    For each $t \geq 1$ and $({\bm x}^{(1)},\ldots,{\bm x}^{(n)} )$, define
    \begin{equation}
        \hat{\bm z}^{(n)}_{{\bm \epsilon},t}  ({\bm x}^{(1)},\ldots,{\bm x}^{(n)} ) =    \begin{cases}
            {\bm \epsilon}^{(1)}     +
            \mathcal{P} ( \mathcal{Y}^{(1)},   {\bm \mu}^{(1)}_t ({\bm 0},{\bm x}^{(1)} )  )
                                                                                                                                                                           & (n=1),     \\
            {\bm \epsilon}^{(n)} +  \mathcal{P} (  \mathcal{Y}^{(n)},     {\bm \mu}^{(n)}_t (\hat{\bm z}^{(n-1)} _{{\bm \epsilon},t} ({\bm x}^{(1:n-1)}),{\bm x}^{(n)} ) ) & (n \ge 2),
        \end{cases}
    \end{equation}
    where $ \mathcal{Y}^{(n)} = [-B_n, B_n ]^{M^{(n)}}$.
    Then, $\hat{\bm z}^{(n)}_{{\bm \epsilon},t}  ({\bm x}^{(1)},\ldots,{\bm x}^{(n)} )  \in [-A_{n}-B_{n},A_{n}+B_{n}]^{M^{(n)}} = \tilde{\mathcal{Y}}^{(n)}$ for all $n \in [N]$, $t \geq 1$, ${\bm \epsilon}$ and  $({\bm x}^{(1)},\ldots,{\bm x}^{(n)} ) $.
    Moreover, ${\bm f}^{(n)} $ satisfies  ${\bm f}^{(n)} ({\bm w},{\bm x} ) + {\bm \epsilon}^{(n)} \in \tilde{\mathcal{Y}}^{(n)} $
    for all ${\bm w} \in \tilde{\mathcal{Y}}^{(n-1)} $, ${\bm x} \in \mathcal{X}^{(n)}$ and ${\bm \epsilon}^{(n)}$.
\end{theorem}
\begin{proof}
    From the reproducing property of $k^{(n)} (\cdot,\cdot)$, and the assumptions $k^{(n)} (\cdot,\cdot) \leq 1$ and $\| f^{(n)}_m \|_{\mathcal{H}_{k^{(n)}}} \leq B_n$,
    we have $f ^{(n)} _m ( {\bm w},{\bm x} ) \in [-B_n,B_n]$.
    Therefore, noting that $-A \leq \epsilon^{(n)}_m \leq A$, we get ${\bm f}^{(n)} $ satisfies  ${\bm f}^{(n)} ({\bm w},{\bm x} ) + {\bm \epsilon}^{(n)} \in \tilde{\mathcal{Y}}^{(n)} $.
    Similarly, from the definition of $\mathcal{P} ( \mathcal{Y}^{(n)} ,{\bm a} ) $, it follows that
    \begin{equation}
        \mathcal{P}  (  \mathcal{Y}^{(n)},     {\bm \mu}^{(n)}_t (\hat{\bm z}^{(n-1)} _{{\bm \epsilon},t} ({\bm x}^{(1)},\ldots,{\bm x}^{(n-1)} ),{\bm x}^{(n)} ) )
        \in \mathcal{Y}^{(n)} = [-B_n,B_n ] ^{M^{(n)}}.
    \end{equation}
    Thus, we obtain $\hat{\bm z}^{(n)}_{{\bm \epsilon},t}  ({\bm x}^{(1)},\ldots,{\bm x}^{(n)} )  \in \tilde{\mathcal{Y}^{(n)}}$.
\end{proof}
In this modified  $\hat{\bm z}^{(n)}_{{\bm \epsilon},t}  ({\bm x}^{(1)},\ldots,{\bm x}^{(n)} )$,
similar results given in~\cref{theorem:Nstage3} hold.

\begin{theorem}\label{theorem:Nstage3_modified}
    Assume that the same condition as in~\cref{theorem:def_muhat} holds.
    Given $\delta \in (0,1)$, define $B = \max_{1 \leq n \leq N} B_n $ and $\beta_t $ as in~\cref{eq:definition_beta_t}.
    Moreover, assume that~\cref{assumption:L1,assumption:L2} hold.
    Then, with probability at least $1-\delta$, the following holds for any realization of ${\bm\epsilon}$:
    \begin{align}
        |{z}^{(n)}_{{\bm \epsilon},m}({\bm x}^{(1)}, \ldots, {\bm x}^{(n)} )  -\hat{z}^{(n)}_{{\bm \epsilon},m,t} ({\bm x}^{(1)}, \ldots, {\bm x}^{(n)} )    |  \leq &
        {\beta}^{1/2}_t  \hat{\sigma}^{(n)}_{{\bm \epsilon},m,t}  ({\bm x}^{(1)}, \ldots, {\bm x}^{(n)} )                                                                                                                                              \\
                                                                                                                                                                     & \quad  \forall n \in [N],  m \in [M^{(n)}], t \geq 1       , \label{eq:n_s_hat}
    \end{align}
    where $\hat{z}^{(n)}_{{\bm \epsilon},m,t} $ is the $m$-th element of
    $\hat{\bm z}^{(n)}_{{\bm \epsilon},t} $, and  $\hat{\sigma}^{(n)}_{{\bm \epsilon},m,t} ({\bm x}^{(1)}, \ldots, {\bm x}^{(n)} )$ is given by
    \begin{align}
         & \hat{\sigma}^{(n)}_{{\bm \epsilon},m,t } ( {\bm x}^{(1)},\ldots,{\bm x}^{(n)})    =                                                \\
         & \quad    {\sigma}^{(n)}_{m,t } (\hat{\bm z}^{(n-1)}_{{\bm \epsilon},t}  ({\bm x}^{(1)},\ldots,{\bm x}^{(n-1)} ) ,{\bm x}^{(n)})  +
        L_f \sum_{s=1}^ {M^{(n-1)}} \hat{\sigma}^{(n-1)}_{{\bm \epsilon},s,t } ( {\bm x}^{(1)},\ldots,{\bm x}^{(n-1)}),
    \end{align}
    and $\hat{\sigma}^{(1)}_{{\bm \epsilon},s,t } ( {\bm x}^{(1)}) = \sigma^{(1)}_{s,t} ({\bm 0},{\bm x}^{(1)} )$.
\end{theorem}
\begin{proof}
    For any $t \geq 1$, $n \in [N]$, $m \in [M^{(n)}]$, ${\bm x}^{(1)},\ldots,{\bm x}^{(n)} $ and realization of ${\bm \epsilon}$, it follows that
    \begin{align}
         & | f^{(n)}_m (\hat{\bm z}_{{\bm\epsilon},t}^{(n-1)} ({\bm x}^{(1)},\ldots,{\bm x}^{(n-1)}),{\bm x}^{(n)} )     -  { \mu}^{(n)}_{m,t} (\hat{\bm z}^{(n-1)} _{{\bm \epsilon},t} ({\bm x}^{(1)},\ldots,{\bm x}^{(n-1)}),{\bm x}^{(n)} )             |                                                                      \\
         & =| f^{(n)}_m (\hat{\bm z}_{{\bm\epsilon},t}^{(n-1)} ({\bm x}^{(1)},\ldots,{\bm x}^{(n-1)}),{\bm x}^{(n)} )     -\mathcal{P}_m (  \mathcal{Y}^{(n)},     {\bm \mu}^{(n)}_t (\hat{\bm z}^{(n-1)} _{{\bm \epsilon},t} ({\bm x}^{(1)},\ldots,{\bm x}^{(n-1)}),{\bm x}^{(n)} ) ) |                                          \\
         & \quad +|\mathcal{P} _m(  \mathcal{Y}^{(n)},     {\bm \mu}^{(n)}_t (\hat{\bm z}^{(n-1)} _{{\bm \epsilon},t} ({\bm x}^{(1)},\ldots,{\bm x}^{(n-1)}),{\bm x}^{(n)} ) )-                  { \mu}^{(n)}_{m,t} (\hat{\bm z}^{(n-1)} _{{\bm \epsilon},t} ({\bm x}^{(1)},\ldots,{\bm x}^{(n-1)}),{\bm x}^{(n)} )             | \\
         & \geq  | f^{(n)}_m (\hat{\bm z}_{{\bm\epsilon},t}^{(n-1)} ({\bm x}^{(1)},\ldots,{\bm x}^{(n-1)}),{\bm x}^{(n)} )   -\mathcal{P} _m (  \mathcal{Y}^{(n)},     {\bm \mu}^{(n)}_t (\hat{\bm z}^{(n-1)} _{{\bm \epsilon},t} ({\bm x}^{(1)},\ldots,{\bm x}^{(n-1)}),{\bm x}^{(n)} ) ) |,
    \end{align}
    where the first equality is derived by $  f^{(n)}_m (\hat{\bm z}_{{\bm\epsilon},t}^{(n-1)} ({\bm x}^{(1)},\ldots,{\bm x}^{(n-1)}),{\bm x}^{(n)} )   \in \mathcal{Y}^{(n)}$ and the definition of $\mathcal{P} ( \mathcal{Y}^{(n)},{\bm a} )$.
    Thus, for
    $i \in [M^{(2)}]$ and $ ({\bm x}^{(1)},{\bm x}^{(2)} )$, the following inequality holds with probability at least $1-\delta$:
    \begin{align}
         & | z^{(2)}_{{\bm\epsilon},i}  ({\bm x}^{(1)},{\bm x}^{(2)} ) -\hat{z}^{(2)}_{{\bm\epsilon},i,t}  ({\bm x}^{(1)},{\bm x}^{(2)} )|                                                                                                       \\
         & \leq  | f^{(2)}_i ({\bm z}^{(1)}_{\bm\epsilon} ({\bm x}^{(1)}) ,{\bm x}^{(2)})   - f^{(2)}_i (\hat{\bm z}_{{\bm\epsilon},t}^{(1)} ({\bm x}^{(1)}),{\bm x}^{(2)} ) |                                                                   \\
         & \qquad +| f^{(2)}_i (\hat{\bm z}_{{\bm\epsilon},t}^{(1)} ({\bm x}^{(1)} )   ,{\bm x}^{(2)} )   -\mathcal{P}_i (  \mathcal{Y}^{(2)},     {\bm \mu}^{(2)}_t (\hat{\bm z}^{(1)} _{{\bm \epsilon},t} ({\bm x}^{(1)}),{\bm x}^{(2)} ) ) |  \\
         & \leq L_f \| {\bm f}^{(1)} ({\bm 0},{\bm x}^{(1)}) -\mathcal{P} (  \mathcal{Y}^{(1)},     {\bm \mu}^{(1)}_t ({\bm 0} ,{\bm x}^{(1)}) )  \| _1                                                                                          \\
         & \qquad+ | f^{(2)}_i (\hat{\bm z}_{{\bm\epsilon},t}^{(1)} ({\bm x}^{(1)} )   ,{\bm x}^{(2)} )   -\mathcal{P} _i (  \mathcal{Y}^{(2)},     {\bm \mu}^{(2)}_t (\hat{\bm z}^{(1)} _{{\bm \epsilon},t} ({\bm x}^{(1)}),{\bm x}^{(2)} ) ) | \\
         & \leq L_f \sum_{j=1} ^{M^{(1)}}    | { f}_j^{(1)} ({\bm 0},{\bm x}^{(1)}) -\mathcal{P}_j (  \mathcal{Y}^{(1)},     {\bm \mu}^{(1)}_t ({\bm 0} ,{\bm x}^{(1)}) )  |                                                                     \\
         & \qquad+ | f^{(2)}_i (\hat{\bm z}_{{\bm\epsilon},t}^{(1)} ({\bm x}^{(1)} )   ,{\bm x}^{(2)} )   -\mathcal{P} _i (  \mathcal{Y}^{(2)},     {\bm \mu}^{(2)}_t (\hat{\bm z}^{(1)} _{{\bm \epsilon},t} ({\bm x}^{(1)}),{\bm x}^{(2)} ) ) | \\
         & \leq L_f \sum_{j=1} ^{M^{(1)}}    | { f}_j^{(1)} ({\bm 0},{\bm x}^{(1)}) -   { \mu}^{(1)}_{j,t} ({\bm 0} ,{\bm x}^{(1)})   |                                                                                                          \\
         & \qquad+ | f^{(2)}_i (\hat{\bm z}_{{\bm\epsilon},t}^{(1)} ({\bm x}^{(1)} )   ,{\bm x}^{(2)} )   -  { \mu}^{(2)}_{i,t} (\hat{\bm z}^{(1)} _{{\bm \epsilon},t} ({\bm x}^{(1)}),{\bm x}^{(2)} )  |                                        \\
         & \leq \beta^{1/2} _t \sigma ^{(2)}_{i,t} (\hat{\bm z}^{(1)} _{{\bm \epsilon},t} ({\bm x}^{(1)}),{\bm x}^{(2)} )+ L_f \beta^{1/2}_t  \sum_{j=1}^{M^{(1)}} \sigma^{(1)}_{j,t} ({\bm 0},{\bm x}^{(1)} )                                   \\
         & = \beta^{1/2}_t \hat{\sigma}^{(2)} _{{\bm\epsilon},i,t} ({\bm x}^{(1)},{\bm x}^{(2) } ).                                                                                                                                              %\label{eq:2bound_hat}
    \end{align}
    Therefore, by repeating this process up to $n$, we get the desired inequality.
\end{proof}
We emphasize that by using the same technique as used in this proof, it can also be shown that~\cref{theorem:EX_CandS,theorem:i_regret_EXP,theorem:i_regret_zero} hold when using  $ \hat{\bm z}_{{\bm\epsilon},t} ^{(n)} ({\bm x}^{(1)},\ldots,{\bm x}^{(n)} )$ instead of  $ \tilde{\bm z}_{{\bm\epsilon},t} ^{(n)} ({\bm x}^{(1)},\ldots,{\bm x}^{(n)} )$.

Finally, we provide the sufficient condition for the Lipschitz continuity assumption (L2).
\begin{theorem}\label{thm:3kernels}
Let $k({\bm x},{\bm y}): \mathbb{R}^d \times \mathbb{R}^d \to \mathbb{R}$ be one of the following kernel functions:
\begin{description}
\item [Linear kernel:] $k({\bm x},{\bm y}) =a^2 {\bm x}^\top {\bm y}$, where $a$ is a positive parameter.
\item [Gaussian kernel:] $k({\bm x},{\bm y}) =a^2 \exp (-\| {\bm x} - {\bm y} \|^2/(2 \rho^2 ) )$, where $a$ and $ \rho$ are positive parameters.
\item [Mat\'{e}rn kernel:]
$$
k({\bm x},{\bm y}) =a^2 \frac{ 2^{1-\nu} }{\Gamma (\nu)}  \left ( \sqrt{2 \nu} \frac{\| {\bm x} - {\bm y} \|}{\rho} \right )^\nu K_\nu \left (   \sqrt{2 \nu}   \frac{ \| {\bm x} - {\bm y} \| }{\rho} \right ) ,              $$
where $a$ and $\rho $ are positive parameters, $\nu$ is a degree of freedom with $\nu >1$, $\Gamma$ is the gamma function, and $K_\nu$ is the modified Bessel function of the second kind.
\end{description}
Moreover, assume that a user-specified variance parameter $\sigma^2 $ is positive.
Then, for any $t \geq 1$ and observed points ${\bm x}_1,\ldots, {\bm x}_t $, the posterior standard deviation $\sigma_t ({\bm x} )$ satisfies that
\begin{align}
^\forall {\bm x},{\bm y} \in \mathbb{R}^d, \ | \sigma_t ({\bm x} ) - \sigma_t ({\bm y} ) | \leq C \| {\bm x} - {\bm y} \| _1,
\label{eq:ineq_sigma_Lip}
\end{align}
where $C$ is a positive constant given by
\begin{align}
C= \left \{
\begin{array}{ll}
a & \text{if} \ k({\bm x},{\bm y}) \ \text{is {\ the \ linear kernel}}, \\
\frac{\sqrt{2} a}{\rho} & \text{if} \ k({\bm x},{\bm y}) \ \text{is {\ the \ Gaussian kernel}}, \\
\frac{\sqrt{2} a}{\rho}
\sqrt{\frac{\nu}{\nu-1} } & \text{if} \ k({\bm x},{\bm y}) \ \text{is {\ the \  Mat\'{e}rn kernel}}.
\end{array}
\right . \nonumber
\end{align}
\end{theorem}

\begin{proof}

First, we show the case of the linear kernel.
Let the matrix ${\bm X}_t$ be ${\bm X}_t =( {\bm x}_1, \ldots, {\bm x}_t )^\top $.
Then, $\sigma^2 _t ({\bm x} )$ is given by
\begin{align}
\sigma^2 _t ({\bm x} ) &=a^2  {\bm x}^\top {\bm x} -a^4{\bm x}^\top {\bm X}^\top_t ( a^2 {\bm X}_t {\bm X}^\top_t +\sigma^2 {\bm I}_t )^{-1} {\bm X}_t {\bm x} \nonumber \\
&= a^2  {\bm x}^\top {\bm x} -a^2{\bm x}^\top {\bm X}^\top_t (  {\bm X}_t {\bm X}^\top_t +a^{-2} \sigma^2 {\bm I}_t )^{-1} {\bm X}_t {\bm x} \\
&= a^2 {\bm x}^\top ( {\bm I}_d - {\bm X}^\top_t (  {\bm X}_t {\bm X}^\top_t +a^{-2} \sigma^2 {\bm I}_t )^{-1} {\bm X}_t ) {\bm x}. \nonumber
\end{align}
The matrix ${\bm X}_t$ can be decomposed as
$$
{\bm X}_t = {\bm H}^\prime {\bm\Lambda} {\bm H}^\top ,
$$
where ${\bm H}^\prime = ({\bm h}^\prime_1,\ldots,{\bm h}^\prime_t)^\top $ and
${\bm H} = ({\bm h}_1,\ldots,{\bm h}_d)^\top $ are orthogonal matrices, and ${\bm\Lambda}$ is the $t \times d$ rectangular diagonal matrix whose $(j,j)$ element is the $j$th singular value $s_j \geq 0$ of ${\bm X}_t$.
Thus,
$ {\bm I}_d - {\bm X}^\top_t (  {\bm X}_t {\bm X}^\top_t +a^{-2} \sigma^2 {\bm I}_t )^{-1} {\bm X}_t$ can be rewritten as follows:
$$
 {\bm I}_d - {\bm X}^\top_t (  {\bm X}_t {\bm X}^\top_t +a^{-2} \sigma^2 {\bm I}_t )^{-1} {\bm X}_t={\bm H} {\bm \Theta} {\bm H}^\top,
$$
where $\Theta$ is the diagonal matrix whose $(j,j)$ element is $1-s^2_j/(s^2_j+a^{-2}\sigma^2 )$.
Thus, the posterior standard deviation $\sigma _t ({\bm x} ) $ can be expressed as
$$
\sigma _t ({\bm x} ) = \sqrt{ a^2 {\bm x}^\top  {\bm H} {\bm \Theta} {\bm H}^\top {\bm x}  }
=a \| {\bm\Theta}^{1/2} {\bm H}^\top {\bm x} \|.
$$
Hence, using the triangle inequality we have
\begin{align}
| \sigma_t ({\bm x} ) -  \sigma_t ({\bm y} ) | &= a |    \| {\bm\Theta}^{1/2} {\bm H}^\top {\bm x} \| - \| {\bm\Theta}^{1/2} {\bm H}^\top {\bm y} \|           | \nonumber \\
&\leq a \| {\bm\Theta}^{1/2} {\bm H}^\top {\bm x} -  {\bm\Theta}^{1/2} {\bm H}^\top {\bm y}    \| \nonumber \\
& =   a \| {\bm\Theta}^{1/2} {\bm H}^\top ( {\bm x} -  {\bm y} )   \|. \label{eq:ThetaHx}
\end{align}
Noting that
 the diagonal element $\theta_j$ of
 ${\bm \Theta}$ satisfies $0 \leq \theta_j \leq 1$, from $\| {\bm x} - {\bm y} \| \leq \| {\bm x} -{\bm y} \|_1$ we get
\begin{align}
\| {\bm\Theta}^{1/2} {\bm H}^\top ( {\bm x} -  {\bm y} )   \|
&= \sqrt{  ({\bm x}-{\bm y} )^\top {\bm H} {\bm \Theta} {\bm H}^\top  ({\bm x}-{\bm y} ) } \nonumber \\
&\leq \sqrt{  ({\bm x}-{\bm y} )^\top {\bm H} {\bm I}_d {\bm H}^\top  ({\bm x}-{\bm y} ) }
= \| {\bm x} -{\bm y} \| \leq \| {\bm x} -{\bm y} \|_1. \label{eq:xminusyL1}
\end{align}
Therefore, by substituting  \eqref{eq:xminusyL1} into \eqref{eq:ThetaHx}, we have the desired result.

Next, we show the case of the Gaussian kernel.
From Bochner's theorem, the Gaussian kernel can be rewritten as follows (see, e.g., section 4.2.1 in \cite{rasmussen2005gaussian}):
\begin{align}
k({\bm x},{\bm y} ) = a^2 \int _{ \mathbb{R}^d } e^{2 \pi {\rm i} ({\bm x}-{\bm y} )^\top {\bm \lambda} } (2 \pi \rho^2)^{d/2} e^{-2 \pi^2 \rho^2 \| {\bm \lambda} \|^2} \text{d} {\bm \lambda} , \nonumber
\end{align}
where ${\rm i}$ is the imaginary unit.
Furthermore, for each natural number $s \in \mathbb{N}$, let $\mathcal{I}_s$ and  $\mathcal{C}_s$
 be
 families  of sets given by
\begin{align}
\mathcal{I}_s &= \left \{   \left [-s+\frac{j-1}{2^s} , -s + \frac{j}{2^s} \right ) \mid j=1 , \ldots, 2s 2^s \right \}, \nonumber \\
\mathcal{C}_s &= \{   I_1  \times \cdots \times I_d \mid I_1 , \ldots , I_d \in \mathcal{I}_s \}. \nonumber
\end{align}
In addition, for each element $C_{s,k} = [a^{(1)}_{s,k},b^{(1)}_{s,k} ) \times \cdots [a^{(d)}_{s,k},b^{(d)}_{s,k} )$ of
 $\mathcal{C}_s $, $(k=1,\ldots, (2s2^s)^d)$, we define the representative point ${\bm \lambda}_{s,k}$ of
$C_{s,k}$ as
$$
{\bm \lambda}_{s,k} = \left (   \frac{  a^{(1)}_{s,k}+b^{(1)}_{s,k}      }{2} , \ldots ,   \frac{  a^{(d)}_{s,k}+b^{(d)}_{s,k}      }{2}     \right )^\top
= (\lambda_{s,k}^{(1)} , \ldots , \lambda_{s,k} ^{(d) } )^\top.
$$
Moreover, let ${\bm \phi }_s ({\bm x} ) $ be the $(2s2^s)^d$-dimensional vector whose $k$th element $\phi_{s,k} ({\bm x} )$ is given by
\begin{align}
\phi_{s,k} ({\bm x} )
= a e^{ 2 \pi {\rm i} {\bm x} ^\top {\bm\lambda}_{s,k} } (2 \pi \rho ^2 )^{d/4}  e^{-\pi^2 \rho^2 \| {\bm\lambda}_{s,k} \| ^2 }    \left ( \frac{1}{2^s} \right ) ^{d/2} .
\end{align}
Then, the inner product $ \langle {\bm \phi }_s ({\bm x} )  , {\bm \phi }_s ({\bm y} )  \rangle  \equiv  \overline{{\bm \phi }_s ({\bm x} ) }^\top
 {\bm \phi }_s ({\bm y} )$
satisfies
\begin{align}
\lim_{s \to \infty}
 \langle {\bm \phi }_s ({\bm x} )  , {\bm \phi }_s ({\bm y} )  \rangle
&=
\lim_{s \to \infty}
\sum_{k=1}^{    (2s2^s)^d     }
a^2  e^{2 \pi {\rm i} ( - {\bm x} + {\bm y} )^\top {\bm \lambda}_{s,k} } (2 \pi \rho^2)^{d/2} e^{-2 \pi^2 \rho^2 \| {\bm \lambda}_{s,k} \|^2}
 \left ( \frac{1}{2^s} \right ) ^{d}  \nonumber \\
 &= a^2 \int _{ \mathbb{R}^d } e^{2 \pi {\rm i} ( - {\bm x} + {\bm y} )^\top {\bm \lambda} } (2 \pi \rho^2)^{d/2} e^{-2 \pi^2 \rho^2 \| {\bm \lambda} \|^2} \text{d} {\bm \lambda} \nonumber \\
 &= a^2 \int _{ \mathbb{R}^d } e^{2 \pi {\rm i} ({\bm x}-{\bm y} )^\top {\bm \lambda} } (2 \pi \rho^2)^{d/2} e^{-2 \pi^2 \rho^2 \| {\bm \lambda} \|^2} \text{d} {\bm \lambda}=k({\bm x},{\bm y} ). \nonumber
\end{align}
Furthermore, we define $\check{\sigma}^2_{t,s} ({\bm x} )$ and $\epsilon_{t,s} ({\bm x} )$ as
\begin{align}
\check{\sigma}^2_{t,s} ({\bm x} ) &= \langle {\bm \phi }_s ({\bm x} ) ,  {\bm \phi }_s ({\bm x} ) \rangle -
( \langle {\bm \phi }_s ({\bm x} ) ,  {\bm \phi }_s ({\bm x}_1 ) \rangle , \ldots ,  \langle {\bm \phi }_s ({\bm x} ) ,  {\bm \phi }_s ({\bm x}_t ) \rangle ) \\
&
( {\bm K}_{t,s} + \sigma^2 {\bm I}_t )^{-1}
( \langle {\bm \phi }_s ({\bm x} ) ,  {\bm \phi }_s ({\bm x}_1 ) \rangle , \ldots ,  \langle {\bm \phi }_s ({\bm x} ) ,  {\bm \phi }_s ({\bm x}_t ) \rangle )  ^\top, \nonumber \\
\epsilon_{t,s} ({\bm x} ) &= k({\bm x},{\bm x} ) -( k({\bm x} , {\bm x}_1 ) , \ldots ,   k({\bm x} , {\bm x}_t ) )
({\bm K}_t + \sigma^2 {\bm I}_t )^{-1}( k({\bm x} , {\bm x}_1 ) , \ldots ,   k({\bm x} , {\bm x}_t ) )^\top \\
&\quad - \check{\sigma}^2_{t,s} ({\bm x} ) \\
&= \sigma^2_t ({\bm x} ) -  \check{\sigma}^2_{t,s} ({\bm x} ), \nonumber
\end{align}
where ${\bm K}_{t,s} $ and ${\bm K}_t$ are $t \times t$ matrices whose $(i,j)$ elements are given by $\langle {\bm \phi }_s ({\bm x}_i ) ,  {\bm \phi }_s ({\bm x}_j ) \rangle$ and $\langle {\bm \phi }_s ({\bm x}_i ) ,  {\bm \phi }_s ({\bm x}_j ) \rangle$, respectively.
Then, noting that $\lim_{s \to \infty}
 \langle {\bm \phi }_s ({\bm x} )  , {\bm \phi }_s ({\bm y} )  \rangle
=k({\bm x},{\bm y} )$ we get
$$
\lim _{s \to \infty } \check{\sigma}^2_{t,s} ({\bm x} ) = \sigma^2_t ({\bm x} ), \quad \lim_{s \to \infty} \epsilon _{t,s} ({\bm x} ) =0.
$$
We now consider $\sigma_t ({\bm x} )$ and $\sigma_t ({\bm y} )$.
Without loss of generality, we can assume that $\sigma_t ({\bm x} ) \geq \sigma_t ({\bm y} )$.
Then, we have
\begin{equation}
| \sigma_t ({\bm x} ) - \sigma_t ({\bm y} ) | = \sigma_t ({\bm x} ) - \sigma_t ({\bm y} ). \label{eq:sigmax_y_diff}
\end{equation}
In addition, the following inequality holds:
\begin{align}
\sigma_t ({\bm x} ) = \sqrt{\sigma^2_t ({\bm x} ) } = \sqrt{  \check{\sigma}^2_{t,s} ({\bm x} ) + \epsilon_{t,s} ({\bm x} ) }
 &\leq \sqrt{  \check{\sigma}^2_{t,s} ({\bm x} ) + |\epsilon_{t,s} ({\bm x} ) |} \\
 &\leq \sqrt{  \check{\sigma}^2_{t,s} ({\bm x} )} + \sqrt{|\epsilon_{t,s} ({\bm x} ) |}. \label{eq:sigma_x_bound}
\end{align}
Similarly, if $\epsilon_{t,s} ({\bm y} ) > 0$, then
$\sigma_t ({\bm y} )$ satisfies
$$
\sigma_t ({\bm y} )  = \sqrt{  \check{\sigma}^2_{t,s} ({\bm y} ) + \epsilon_{t,s} ({\bm y} ) } \geq \sqrt{  \check{\sigma}^2_{t,s} ({\bm y} ) }
\geq \sqrt{  \check{\sigma}^2_{t,s} ({\bm y} ) } - \sqrt{ | \epsilon_{t,s} ({\bm y} )|}.
$$
On the other hand, if $\epsilon_{t,s} ({\bm y} ) \leq 0$,  then
$\sigma_t ({\bm y} )$ satisfies
$$
\sigma_t ({\bm y} )  = \sqrt{  \check{\sigma}^2_{t,s} ({\bm y} ) + \epsilon_{t,s} ({\bm y} ) }
=
\sqrt{  \check{\sigma}^2_{t,s} ({\bm y} ) -| \epsilon_{t,s} ({\bm y} ) |}
\geq \sqrt{  \check{\sigma}^2_{t,s} ({\bm y} ) } - \sqrt{ | \epsilon_{t,s} ({\bm y} )|},
$$
where the last inequality is given by $\sqrt{u} - \sqrt{v} \leq \sqrt{u-v} $, $(u \geq  v \geq 0)$.
Hence, for both cases, the following holds:
\begin{equation}
\sigma_t ({\bm y} ) \geq \sqrt{  \check{\sigma}^2_{t,s} ({\bm y} ) } - \sqrt{ | \epsilon_{t,s} ({\bm y} )|} . \label{eq:sigma_y_bound}
\end{equation}
Thus, by substituting \eqref{eq:sigma_x_bound} and \eqref{eq:sigma_y_bound} into \eqref{eq:sigmax_y_diff}, we obtain
\begin{equation}
| \sigma_t ({\bm x} ) - \sigma_t ({\bm y} ) | \leq
\sqrt{  \check{\sigma}^2_{t,s} ({\bm x} )}
-\sqrt{  \check{\sigma}^2_{t,s} ({\bm y} )}
+ \sqrt{|\epsilon_{t,s} ({\bm x} ) |} + \sqrt{|\epsilon_{t,s} ({\bm y} ) |}. \label{eq:sigmax_y_diff2}
\end{equation}
Furthermore, we define the matrix ${\bm X}_{t,s} $ as ${\bm X}_{t,s} = ( {\bm \phi }_s ({\bm x}_1 ), \ldots , {\bm \phi }_s ({\bm x}_t ) )^\ast$,
where ${\bm A}^\ast$ is the conjugate transpose of ${\bm A}$.
Then, $\check{\sigma}^2_{t,s} ({\bm x} )$ can be rewritten as follows:
$$
\check{\sigma}^2_{t,s} ({\bm x} ) = \overline{ {\bm \phi}_{s} ({\bm x} ) } ^\top
(
{\bm I}_{   (2s2^s)^d  } -{\bm X}^\ast_{t,s} ( {\bm X}_{t,s} {\bm X}_{t,s}^\ast + \sigma^2{\bm I}_t )^{-1} {\bm X}_{t,s}
)
{\bm \phi}_{s} ({\bm x} ).
$$
Therefore, by using the singular decomposition of ${\bm X}_{t,s} $, we have
$$
\check{\sigma}^2_{t,s} ({\bm x} ) = \overline{ {\bm \phi}_{s} ({\bm x} ) } ^\top  {\bm U}
{\bm \Theta}
{\bm U}^\ast
{\bm \phi}_{s} ({\bm x} ),
$$
where ${\bm U}$ and ${\bm \Theta}$ are unitary and diagonal matrices, respectively.
By using the same argument as in the case of the linear kernel, it can be shown that the $(k,k)$ element $\theta_k$
 of ${\bm \Theta}$ satisfies $0 \leq \theta_k \leq 1$.
Hence, noting that
$$
\check{\sigma}_{t,s} ({\bm x} ) = \sqrt{   \check{\sigma}^2_{t,s} ({\bm x} )     }     =    \|
{\bm \Theta} ^{1/2}
{\bm U}^\ast
{\bm \phi}_{s} ({\bm x} ) \|
$$
and \eqref{eq:sigmax_y_diff2}, from the triangle inequality we get
\begin{align}
| \sigma_t ({\bm x} ) - \sigma_t ({\bm y} ) | &\leq
\|
{\bm \Theta} ^{1/2}
{\bm U}^\ast
{\bm \phi}_{s} ({\bm x} ) \| -
\|
{\bm \Theta} ^{1/2}
{\bm U}^\ast
{\bm \phi}_{s} ({\bm y} ) \|
+ \sqrt{|\epsilon_{t,s} ({\bm x} ) |} + \sqrt{|\epsilon_{t,s} ({\bm y} ) |} \nonumber \\
&
\leq
\|
{\bm \Theta} ^{1/2}
{\bm U}^\ast
{\bm \phi}_{s} ({\bm x} )  -
{\bm \Theta} ^{1/2}
{\bm U}^\ast
{\bm \phi}_{s} ({\bm y} ) \|
+ \sqrt{|\epsilon_{t,s} ({\bm x} ) |} + \sqrt{|\epsilon_{t,s} ({\bm y} ) |} \nonumber \\
&=
\|
{\bm \Theta} ^{1/2}
{\bm U}^\ast (
{\bm \phi}_{s} ({\bm x} )  -
{\bm \phi}_{s} ({\bm y} ) ) \|
+ \sqrt{|\epsilon_{t,s} ({\bm x} ) |} + \sqrt{|\epsilon_{t,s} ({\bm y} ) |} \nonumber \\
&\leq
\|
{\bm \phi}_{s} ({\bm x} )  -
{\bm \phi}_{s} ({\bm y} ) \|
+ \sqrt{|\epsilon_{t,s} ({\bm x} ) |} + \sqrt{|\epsilon_{t,s} ({\bm y} ) |} , \label{eq:phi_linear}
\end{align}
where the last inequality is given by $0 \leq \theta_k \leq 1$.
Moreover, for each $j$ with $1 \leq j \leq d-1$, let ${\bm x}[j] = (y_1,\ldots, y_j, x_{j+1}, \ldots, x_d )^\top$,
and let ${\bm x} [0] \equiv {\bm x} $ and  ${\bm x} [d] \equiv {\bm y} $.
Then, the following inequality holds:
\begin{align}
\|
{\bm \phi}_{s} ({\bm x} )  -
{\bm \phi}_{s} ({\bm y} ) \|
&=
\left
\|
\sum_{j=1}^d
\{
{\bm \phi}_{s} ({\bm x} [j-1])  -
{\bm \phi}_{s} ({\bm x}[j]) \} \right \| \\
&\leq
\sum_{j=1}^d
\|
{\bm \phi}_{s} ({\bm x} [j-1])  -
{\bm \phi}_{s} ({\bm x}[j])
\|. \label{eq:dec_phi}
\end{align}
Thus, by substituting \eqref{eq:dec_phi} into \eqref{eq:phi_linear}, we obtain
\begin{align}
| \sigma_t ({\bm x} ) - \sigma_t ({\bm y} ) | \leq
\sum_{j=1}^d
\|
{\bm \phi}_{s} ({\bm x} [j-1])  -
{\bm \phi}_{s} ({\bm x}[j])
\|
+ \sqrt{|\epsilon_{t,s} ({\bm x} ) |} + \sqrt{|\epsilon_{t,s} ({\bm y} ) |}. \label{eq:phi_linear2}
\end{align}
In addition, for any $j$ and $k$ with
 $1 \leq j \leq d$ and $1 \leq k \leq (2s 2^s)^d $, from the definition of  $\phi_{s,k} ({\bm x} )$ we have
\begin{align}
\phi_{s,k} ({\bm x}[j-1] ) -
\phi_{s,k} ({\bm x}[j] )
 &= a D (2 \pi \rho ^2 )^{d/4}  e^{-\pi^2 \rho^2 \| {\bm\lambda}_{s,k} \| ^2 }    \left ( \frac{1}{2^s} \right ) ^{d/2}
(
e^{ 2 \pi {\rm i} x_j {\lambda}^{(j)}_{s,k} }
-
e^{ 2 \pi {\rm i} y_j {\lambda}^{(j)}_{s,k} }
), \nonumber \\
D&=
e^{ 2 \pi {\rm i} (y_1,\ldots,y_{j-1}, x_{j+1} ,\ldots ,x_d) (    {\lambda}^{(1)}_{s,k}  , \ldots,     {\lambda}^{(j-1)}_{s,k}    ,  {\lambda}^{(j+1)}_{s,k}                       , \ldots ,   {\lambda}^{(d)}_{s,k}   )^\top
 }. \nonumber
\end{align}
Hence, it follows that
\begin{align}
&\overline{      (\phi_{s,k} ({\bm x}[j-1] ) -
\phi_{s,k} ({\bm x}[j] )  )                  }   (\phi_{s,k} ({\bm x}[j-1] ) -
\phi_{s,k} ({\bm x}[j] ) ) \nonumber \\
&=
|
\phi_{s,k} ({\bm x}[j-1] ) -
\phi_{s,k} ({\bm x}[j] )
|  ^2 \nonumber \\
&\leq
 a^2  (2 \pi \rho ^2 )^{d/2}  e^{-2\pi^2 \rho^2 \| {\bm\lambda}_{s,k} \| ^2 }    \left ( \frac{1}{2^s} \right ) ^{d}
|
e^{ 2 \pi {\rm i} x_j {\lambda}^{(j)}_{s,k} }
-
e^{ 2 \pi {\rm i} y_j {\lambda}^{(j)}_{s,k} }
|^2 . \label{eq:exp_f}
\end{align}
Thus, noting that $| \cos (u) - \cos (v) | \leq |u-v|$ and
$| \sin (u) - \sin (v) | \leq |u-v|$ for any
$u,v \in \mathbb{R}$, we get
\begin{align}
&|
e^{ 2 \pi {\rm i} x_j {\lambda}^{(j)}_{s,k} }
-
e^{ 2 \pi {\rm i} y_j {\lambda}^{(j)}_{s,k} }
|\\
&=
|
\cos (2 \pi x_j  {\lambda}^{(j)}_{s,k} ) + {\rm i} \sin  (2 \pi x_j  {\lambda}^{(j)}_{s,k} )
-
\cos (2 \pi y_j  {\lambda}^{(j)}_{s,k} ) - {\rm i} \sin  (2 \pi y_j  {\lambda}^{(j)}_{s,k} )
| \nonumber \\
&=
\sqrt{
|
\cos (2 \pi x_j  {\lambda}^{(j)}_{s,k} ) - \cos (2 \pi y_j  {\lambda}^{(j)}_{s,k} )
|^2
+
|
\sin (2 \pi x_j  {\lambda}^{(j)}_{s,k} ) - \sin (2 \pi y_j  {\lambda}^{(j)}_{s,k} )
|^2
} \nonumber \\
&\leq
\sqrt{
8 \pi^2 \lambda^{(j) 2}_{s,k} (x_j-y_j )^2
} . \label{eq:diffxjyj}
\end{align}
Therefore, by substituting
\eqref{eq:diffxjyj} into \eqref{eq:exp_f}, we obtain
\begin{align}
&\overline{      (\phi_{s,k} ({\bm x}[j-1] ) -
\phi_{s,k} ({\bm x}[j] )  )                  }   (\phi_{s,k} ({\bm x}[j-1] ) -
\phi_{s,k} ({\bm x}[j] ) ) \nonumber \\
&\leq
 a^2  (2 \pi \rho ^2 )^{d/2}  e^{-2\pi^2 \rho^2 \| {\bm\lambda}_{s,k} \| ^2 }    \left ( \frac{1}{2^s} \right ) ^{d}
8 \pi^2 \lambda^{(j) 2}_{s,k} (x_j-y_j )^2 . \nonumber
\end{align}
This implies that
\begin{align}
\|
{\bm \phi}_{s} ({\bm x} [j-1])  -
{\bm \phi}_{s} ({\bm x}[j])
\|^2
\leq
\sum_{k=1} ^{(2s 2^s)^d} a^2  (2 \pi \rho ^2 )^{d/2}  e^{-2\pi^2 \rho^2 \| {\bm\lambda}_{s,k} \| ^2 }    \left ( \frac{1}{2^s} \right ) ^{d}
8 \pi^2 \lambda^{(j) 2}_{s,k} (x_j-y_j )^2. \label{eq:lambda2int}
\end{align}
Moreover, the following holds when  $s \to \infty$:
\begin{align}
&\lim_{s \to \infty}  \sum_{k=1} ^{(2s 2^s)^d} a^2  (2 \pi \rho ^2 )^{d/2}  e^{-2\pi^2 \rho^2 \| {\bm\lambda}_{s,k} \| ^2 }    \left ( \frac{1}{2^s} \right ) ^{d}
8 \pi^2 \lambda^{(j) 2}_{s,k} (x_j-y_j )^2 \nonumber \\
&=
 a^2  (2 \pi \rho ^2 )^{d/2}
8 \pi^2  (x_j-y_j )^2
\int _{ \mathbb{R}^d}
e^{-2\pi^2 \rho^2 {\bm \lambda}^\top {\bm \lambda} } \lambda^2_j \text{d} {\bm \lambda} \nonumber \\
&= a^2 8 \pi^2 (x_j-y_j)^2 \left ( \int_\mathbb{R} (2 \pi \rho^2)^{1/2} \lambda_j^2 e^{-2\pi^2 \rho^2 \lambda^2_j} \text{d} \lambda_j \right )
\prod_{i \neq j}^d \left ( \int_\mathbb{R} (2 \pi \rho^2)^{1/2}  e^{-2\pi^2 \rho^2 \lambda^2_i} \text{d} \lambda_i \right ) . \nonumber
\end{align}
By putting $2 \pi \rho \lambda_i = u_i $ for each $i$ with $1\leq i \leq d$, we have
\begin{align}
&a^2 8 \pi^2 (x_j-y_j)^2 \left ( \int_\mathbb{R} (2 \pi \rho^2)^{1/2} \lambda_j^2 e^{-2\pi^2 \rho^2 \lambda^2_j} \text{d} \lambda_j \right )
\prod_{i \neq j}^d \left ( \int_\mathbb{R} (2 \pi \rho^2)^{1/2}  e^{-2\pi^2 \rho^2 \lambda^2_i} \text{d} \lambda_i \right ) \nonumber \\
&=a^2 2 \rho^{-2}  (x_j -y_j )^2 \int _\mathbb{R}  (2 \pi )^{-1/2} u^2_j e^{- u^2_j /2}  \text{d} u_j
\prod_{i \neq j}^d \left ( \int _\mathbb{R}  (2 \pi )^{-1/2}  e^{- u^2_i /2}  \text{d} u_i  \right ) \nonumber \\
&= \frac{2 a^2}{\rho ^2} (x_j -y_j )^2. \label{eq:2arho}
\end{align}
Thus, \eqref{eq:lambda2int} can be rewritten as follows:
\begin{align}
&\|
{\bm \phi}_{s} ({\bm x} [j-1])  -
{\bm \phi}_{s} ({\bm x}[j])
\|^2
 \nonumber \\
&\leq
\frac{2 a^2}{\rho ^2} (x_j -y_j )^2 \\
&\quad +
\sum_{k=1} ^{(2s 2^s)^d} a^2  (2 \pi \rho ^2 )^{d/2}  e^{-2\pi^2 \rho^2 \| {\bm\lambda}_{s,k} \| ^2 }    \left ( \frac{1}{2^s} \right ) ^{d}
8 \pi^2 \lambda^{(j) 2}_{s,k} (x_j-y_j )^2
-\frac{2 a^2}{\rho ^2} (x_j -y_j )^2 \nonumber \\
&\leq
\frac{2 a^2}{\rho ^2} (x_j -y_j )^2 \\
&\quad +
\left |
\sum_{k=1} ^{(2s 2^s)^d} a^2  (2 \pi \rho ^2 )^{d/2}  e^{-2\pi^2 \rho^2 \| {\bm\lambda}_{s,k} \| ^2 }    \left ( \frac{1}{2^s} \right ) ^{d}
8 \pi^2 \lambda^{(j) 2}_{s,k} (x_j-y_j )^2
-\frac{2 a^2}{\rho ^2} (x_j -y_j )^2
\right | \nonumber \\
&\equiv \frac{2 a^2}{\rho ^2} (x_j -y_j )^2 + \tilde{\epsilon}_{s,j}, \nonumber
\end{align}
where $\tilde{\epsilon}_{s,j}$ satisfies that $\lim_{s \to \infty} |\tilde{\epsilon}_{s,j}| =0$ from \eqref{eq:2arho}.
Hence, we get
\begin{align}
\|
{\bm \phi}_{s} ({\bm x} [j-1])  -
{\bm \phi}_{s} ({\bm x}[j])
\|
\leq
\sqrt{
 \frac{2 a^2}{\rho ^2} (x_j -y_j )^2 + \tilde{\epsilon}_{s,j}
}
\leq
\frac{\sqrt{2} a}{\rho} |x_j -y_j| + \sqrt{\tilde{\epsilon}_{s,j} } . \label{eq:sqrt2arho}
\end{align}
By substituting \eqref{eq:sqrt2arho} into \eqref{eq:phi_linear2}, we obtain
$$
| \sigma_t ({\bm x} ) - \sigma_t ({\bm y} ) | \leq
\frac{\sqrt{2} a}{\rho} \| {\bm x} -{\bm y} \|_1 +
\sum_{j=1}^d
\sqrt{
\tilde{\epsilon}_{s,j}
}+ \sqrt{|\epsilon_{t,s} ({\bm x} ) |} + \sqrt{|\epsilon_{t,s} ({\bm y} ) |} .
$$
Furthermore, because the number $s$ is an arbitrary natural number, and
$$
\lim_{s \to \infty} \left (
\sum_{j=1}^d
\sqrt{
\tilde{\epsilon}_{s,j} }
+ \sqrt{|\epsilon_{t,s} ({\bm x} ) |} + \sqrt{|\epsilon_{t,s} ({\bm y} ) |}
\right ) =0,
$$
we have
$$
| \sigma_t ({\bm x} ) - \sigma_t ({\bm y} ) | \leq
\frac{\sqrt{2} a}{\rho} \| {\bm x} -{\bm y} \|_1.
$$

Finally, we show the case of the Mat\'{e}rn kernel.
From Bochner's theorem, the Mat\'{e}rn kernel can be rewritten as follows (see, section 4.2.1 in \cite{rasmussen2005gaussian}):
\begin{align}
k({\bm x},{\bm y} ) = a^2 \int _{ \mathbb{R}^d } e^{2 \pi {\rm i} ({\bm x}-{\bm y} )^\top {\bm \lambda} }
\frac{ 2^d \pi^{d/2} \Gamma (\nu +d/2) (2 \nu)^\nu   }{ \Gamma (\nu ) \rho^{2 \nu}  }
\left (
\frac{ 2 \nu }{\rho^2} + 4 \pi^2 \|{\bm \lambda} \|^2
\right )^{-(\nu+d/2)}
\text{d} {\bm \lambda} . \nonumber
\end{align}
For each $s \in \mathbb{N}$ and $k$ with $k=1,\ldots, (2s2^s)^d$,
we define
$\mathcal{I}_s$, $\mathcal{C}_s$, the element $C_{s,k} = [a^{(1)}_{s,k},b^{(1)}_{s,k} ) \times \cdots [a^{(d)}_{s,k},b^{(d)}_{s,k} )$ of
$\mathcal{C}_s $, and the representative point ${\bm \lambda}_{s,k}$ of
$C_{s,k}$ as in the case of the Gaussian kernel.
Similarly, let ${\bm \phi }_s ({\bm x} ) $ be the $(2s2^s)^d$-dimensional vector whose $k$th element $\phi_{s,k} ({\bm x} )$ is given by
\begin{align}
&\phi_{s,k} ({\bm x} ) \\
&= a e^{ 2 \pi {\rm i} {\bm x} ^\top {\bm\lambda}_{s,k} }
\left (
\frac{ 2^d \pi^{d/2} \Gamma (\nu +d/2) (2 \nu)^\nu   }{ \Gamma (\nu ) \rho^{2 \nu}  }
\left (
\frac{ 2 \nu }{\rho^2} + 4 \pi^2 \| {\bm\lambda}_{s,k} \|^2
\right )^{-(\nu+d/2)}
\right )^{1/2}
\left ( \frac{1}{2^s} \right ) ^{d/2} .
\end{align}
Then, by using the same argument as in the case of the Gaussian kernel, we obtain the following inequality similar to \eqref{eq:phi_linear2}:
\begin{align}
| \sigma_t ({\bm x} ) - \sigma_t ({\bm y} ) | \leq
\sum_{j=1}^d
\|
{\bm \phi}_{s} ({\bm x} [j-1])  -
{\bm \phi}_{s} ({\bm x}[j])
\|
+ | \epsilon_{t,s} ({\bm x},{\bm y} ) |,  \label{eq:phi_linear3}
\end{align}
where  $\lim_{s \to \infty}  | \epsilon_{t,s} ({\bm x},{\bm y} ) |  =0$.
Moreover, for any $j$ and $k$ with  $1 \leq j \leq d$ and $1 \leq k \leq (2s 2^s)^d $,
from the definition of $\phi_{s,k} ({\bm x} )$ we get
\begin{align}
&\phi_{s,k} ({\bm x}[j-1] ) -
\phi_{s,k} ({\bm x}[j] ) \nonumber \\
& = a D
\left (
\frac{ 2^d \pi^{d/2} \Gamma (\nu +d/2) (2 \nu)^\nu   }{ \Gamma (\nu ) \rho^{2 \nu}  }
\left (
\frac{ 2 \nu }{\rho^2} + 4 \pi^2 \| {\bm\lambda}_{s,k} \|^2
\right )^{-(\nu+d/2)}
\right )^{1/2}
 \left ( \frac{1}{2^s} \right ) ^{d/2} \\
&\quad
(
e^{ 2 \pi {\rm i} x_j {\lambda}^{(j)}_{s,k} }
-
e^{ 2 \pi {\rm i} y_j {\lambda}^{(j)}_{s,k} }
), \nonumber \\
&D=
e^{ 2 \pi {\rm i} (y_1,\ldots,y_{j-1}, x_{j+1} ,\ldots ,x_d) (    {\lambda}^{(1)}_{s,k}  , \ldots,     {\lambda}^{(j-1)}_{s,k}    ,  {\lambda}^{(j+1)}_{s,k}                       , \ldots ,   {\lambda}^{(d)}_{s,k}   )^\top
 }. \nonumber
\end{align}
It follows that
\begin{align}
&\overline{      (\phi_{s,k} ({\bm x}[j-1] ) -
\phi_{s,k} ({\bm x}[j] )  )                  }   (\phi_{s,k} ({\bm x}[j-1] ) -
\phi_{s,k} ({\bm x}[j] ) ) \nonumber \\
&=
|
\phi_{s,k} ({\bm x}[j-1] ) -
\phi_{s,k} ({\bm x}[j] )
|  ^2 \nonumber \\
&\leq
 a^2 \frac{ 2^d \pi^{d/2} \Gamma (\nu +d/2) (2 \nu)^\nu   }{ \Gamma (\nu ) \rho^{2 \nu}  }
\left (
\frac{ 2 \nu }{\rho^2} + 4 \pi^2 \| {\bm\lambda}_{s,k} \|^2
\right )^{-(\nu+d/2)}   \left ( \frac{1}{2^s} \right ) ^{d}
|
e^{ 2 \pi {\rm i} x_j {\lambda}^{(j)}_{s,k} }
-
e^{ 2 \pi {\rm i} y_j {\lambda}^{(j)}_{s,k} }
|^2. \label{eq:exp_f3}
\end{align}
By substituting
\eqref{eq:diffxjyj} into \eqref{eq:exp_f3}, we have
\begin{align}
&\overline{      (\phi_{s,k} ({\bm x}[j-1] ) -
\phi_{s,k} ({\bm x}[j] )  )                  }   (\phi_{s,k} ({\bm x}[j-1] ) -
\phi_{s,k} ({\bm x}[j] ) ) \nonumber \\
&\leq
 a^2 \frac{ 2^d \pi^{d/2} \Gamma (\nu +d/2) (2 \nu)^\nu   }{ \Gamma (\nu ) \rho^{2 \nu}  }
\left (
\frac{ 2 \nu }{\rho^2} + 4 \pi^2 \| {\bm\lambda}_{s,k} \|^2
\right )^{-(\nu+d/2)}   \left ( \frac{1}{2^s} \right ) ^{d}
8 \pi^2 \lambda^{(j) 2}_{s,k} (x_j-y_j )^2 . \nonumber
\end{align}
This implies that
\begin{align}
&\|
{\bm \phi}_{s} ({\bm x} [j-1])  -
{\bm \phi}_{s} ({\bm x}[j])
\|^2 \nonumber \\
&\leq
\sum_{k=1} ^{(2s 2^s)^d} a^2
\frac{ 2^d \pi^{d/2} \Gamma (\nu +d/2) (2 \nu)^\nu   }{ \Gamma (\nu ) \rho^{2 \nu}  }
\left (
\frac{ 2 \nu }{\rho^2} + 4 \pi^2 \| {\bm\lambda}_{s,k} \|^2
\right )^{-(\nu+d/2)}
 \left ( \frac{1}{2^s} \right ) ^{d}
8 \pi^2 \lambda^{(j) 2}_{s,k} (x_j-y_j )^2. \nonumber
\end{align}
Furthermore, the following holds when  $s \to \infty$:
\begin{align}
&\lim_{s \to \infty}  \sum_{k=1} ^{(2s 2^s)^d} a^2
\frac{ 2^d \pi^{d/2} \Gamma (\nu +d/2) (2 \nu)^\nu   }{ \Gamma (\nu ) \rho^{2 \nu}  }
\left (
\frac{ 2 \nu }{\rho^2} + 4 \pi^2 \| {\bm\lambda}_{s,k} \|^2
\right )^{-(\nu+d/2)}
 \left ( \frac{1}{2^s} \right ) ^{d}
8 \pi^2 \lambda^{(j) 2}_{s,k} (x_j-y_j )^2 \nonumber \\
&= a^2 8 \pi^2 (x_j-y_j)^2
\int_{ \mathbb{R}^d}
\frac{ 2^d \pi^{d/2} \Gamma (\nu +d/2) (2 \nu)^\nu   }{ \Gamma (\nu ) \rho^{2 \nu}  }
\left (
\frac{ 2 \nu }{\rho^2} + 4 \pi^2 \| {\bm\lambda} \|^2
\right )^{-(\nu+d/2)} \lambda_j^2
\text{d} {\bm \lambda} .
 \nonumber
\end{align}
In addition, by putting
 $2 \nu = \tilde{\nu}$ and ${\bm \Sigma } = (4 \pi^2 \rho^2 ) ^{-1}  {\bm I}_d $, we obtain
\begin{align}
&\frac{ 2^d \pi^{d/2} \Gamma (\nu +d/2) (2 \nu)^\nu   }{ \Gamma (\nu ) \rho^{2 \nu}  }
\left (
\frac{ 2 \nu }{\rho^2} + 4 \pi^2 \| {\bm\lambda} \|^2
\right )^{-(\nu+d/2)} \nonumber \\
&= \frac{  \Gamma ( (\tilde{\nu} +d)/2 )   }  { \Gamma (\tilde{\nu} /2 ) \tilde{\nu}^{d/2} \pi ^{d/2}  | {\bm \Sigma}|^{1/2}       }
\left (
1+ \frac{1}{\tilde{\nu} }  {\bm \lambda}^\top {\bm\Sigma}^{-1} {\bm \lambda}
\right )^{ -(\tilde{\nu}+d)/2   } \equiv f({\bm \lambda}; \tilde{\nu},{\bm \Sigma }). \nonumber
\end{align}
Note that
 $f({\bm \lambda}; \tilde{\nu},{\bm \Sigma })$ is the probability density function of ${\bm T}_{\tilde{\nu}} ({\bm 0},{\bm\Sigma})$,
where ${\bm T}_{\tilde{\nu}} ({\bm 0},{\bm\Sigma})$ is the multivariate $t$-distribution with
location parameter ${\bm 0}$, scale matrix ${\bm\Sigma}$ and
$\tilde{\nu}$ degrees of freedom.
It is known that the mean vector and covariance matrix of ${\bm T}_{\tilde{\nu}} ({\bm 0},{\bm\Sigma})$ are respectively given by  ${\bm 0}$ and
$\frac{\tilde{\nu}}{\tilde{\nu}-2} {\bm\Sigma}$ when
$\tilde{\nu} >2$ (see, e.g., \cite{kotz2004multivariate}).
From the assumption $\nu >1$, noting that $\tilde{\nu} =2 \nu >2$ we have
\begin{align}
&\int_{ \mathbb{R}^d}
\frac{ 2^d \pi^{d/2} \Gamma (\nu +d/2) (2 \nu)^\nu   }{ \Gamma (\nu ) \rho^{2 \nu}  }
\left (
\frac{ 2 \nu }{\rho^2} + 4 \pi^2 \| {\bm\lambda} \|^2
\right )^{-(\nu+d/2)} \lambda_j^2
\text{d} {\bm \lambda} \\
&=
\int_{ \mathbb{R}^d}
f({\bm \lambda}; \tilde{\nu},{\bm \Sigma })
 \lambda_j^2
\text{d} {\bm \lambda} \nonumber \\
&= \frac{\tilde{\nu}}{\tilde{\nu}-2} \frac{1}{4 \pi^2 \rho^2}
=
\frac{{\nu}}{{\nu}-1} \frac{1}{4 \pi^2 \rho^2}. \nonumber
\end{align}
This implies that
\begin{align}
&\lim_{s \to \infty}  \sum_{k=1} ^{(2s 2^s)^d} a^2
\frac{ 2^d \pi^{d/2} \Gamma (\nu +d/2) (2 \nu)^\nu   }{ \Gamma (\nu ) \rho^{2 \nu}  }
\left (
\frac{ 2 \nu }{\rho^2} + 4 \pi^2 \| {\bm\lambda}_{s,k} \|^2
\right )^{-(\nu+d/2)}
 \left ( \frac{1}{2^s} \right ) ^{d}
2  \lambda^{(j) 2}_{s,k} (x_j-y_j )^2 \nonumber \\
&= \frac{2a^2} { \rho^2}  (x_j-y_j)^2
\frac{{\nu}}{{\nu}-1} .
 \nonumber
\end{align}
Therefore, by using the same argument as in the case of the Gaussian kernel, we obtain
\begin{align}
\|
{\bm \phi}_{s} ({\bm x} [j-1])  -
{\bm \phi}_{s} ({\bm x}[j])
\|
\leq
\frac{\sqrt{2} a}{\rho}
\sqrt{\frac{\nu}{\nu-1} }
|x_j -y_j| +|\hat{\epsilon}_{s,j}  |,
 \label{eq:matern_phi}
\end{align}
where $\lim_{s \to \infty} |\hat{\epsilon}_{s,j}  | =0$.
Hence, by substituting \eqref{eq:matern_phi} into \eqref{eq:phi_linear3}, and taking $s \to \infty$ we get
$$
| \sigma_t ({\bm x} ) - \sigma_t ({\bm y} ) | \leq
\frac{\sqrt{2} a}{\rho}
\sqrt{\frac{\nu}{\nu-1} } \|{\bm x} -{\bm y} \|_1.
$$
\end{proof}
The condition that $\sigma^2$ in Theorem \ref{thm:3kernels} is positive is necessary only for the inverse matrix calculation.
Note that $\sigma^2$ is a user-specified variance parameter of a formal GP model, and is different from the true noise variance.
That is, Theorem \ref{thm:3kernels}  holds even when the variance of the true noise is zero, i.e., in the noiseless setting.
Also note that $C$ in Theorem \ref{thm:3kernels}  is a constant independent of $\sigma^2$.
The result for the Mat\'{e}rn kernel is for the case of $\nu >1$ degrees of freedom, and it is a future work to clarify whether the same result holds for $\nu \leq 1$ as well.
On the other hand, unfortunately, it can be shown that \eqref{eq:ineq_sigma_Lip} does not hold for $\nu =1/2$, which is often used in practice for Mat\'{e}rn kernels.
\begin{theorem}
In the setting of Theorem \ref{thm:3kernels}, the Mat\'{e}rn kernel with $\nu =1/2$ does not satisfy  \eqref{eq:ineq_sigma_Lip}.
\end{theorem}
\begin{proof}
Let $C$ be an arbitrary positive number.
The Mat\'{e}rn kernel with  $\nu =1/2$ is given by
$$
k({\bm x},{\bm y} ) = a^2 \exp ( - \| {\bm x}-{\bm y} \| /\rho ).
$$
In addition, suppose that  ${\bm x}_1 =\cdots ={\bm x}_t ={\bm 0}$.
Moreover, we define  ${\bm K}_t $ as
$$
{\bm K}_t = a^2 {\bm 1}_t {\bm 1}^\top_t + \sigma^2 {\bm I}_t.
$$
Then, the inverse matrix ${\bm K}^{-1}_t$ can be expressed as
$$
{\bm K}^{-1}_t = \sigma^{-2} {\bm I}_t - \frac{ \frac{a^2}{\sigma^4} {\bm 1}_t {\bm 1}^\top_t }{  1+ \frac{a^2}{\sigma^2} t  }.
$$
Therefore, the posterior variance at point ${\bm 0}$ is given by
$$
\sigma^2_t ({\bm 0} ) = a^2 - a^4 {\bm 1}^\top_t {\bm K}^{-1}_t  {\bm 1}_t
= a^2 -a^4  \frac{t}{\sigma^2 + a^2 t} = \frac{  a^2 \sigma^2 }{   \sigma^2 + a^2 t }.
$$
Next, let $s$ be a number with
 $0 < s < \rho/2$, and let   ${\bm x} = (s, 0, \ldots , 0 )^\top$.
Then, we have
\begin{align}
\sigma^2_t ({\bm x} ) &=
a^2 - a^4 \exp (-2s/\rho)  {\bm 1}^\top_t {\bm K}^{-1}_t  {\bm 1}_t \nonumber \\
&=
a^2 - a^4 \exp (-2s/\rho) \frac{t}{\sigma^2 + a^2 t} \nonumber \\
&=
a^2 - a^4 (1+ \exp (-2s/\rho) -1) \frac{t}{\sigma^2 + a^2 t} \nonumber \\
&= \sigma^2_t ({\bm 0} ) + \frac{a^4 t }{\sigma^2+a^2 t} (1- \exp (-2s/\rho ) ) \\
&= \sigma^2_t ({\bm 0} )   \left \{ 1 +  \frac{a^4 t }{\sigma^2_t ({\bm 0}) (\sigma^2+a^2 t)} (1- \exp (-2s/\rho ) ) \right \} \equiv \sigma^2 _t ({\bm 0} )   (  1 + u). \nonumber
\end{align}
Thus, from
 $u \geq 0$ we get
$$
| \sigma_t ({\bm x} ) -\sigma_t ({\bm 0} ) |
=
\sigma_t ({\bm x} ) -\sigma_t ({\bm 0} )
= \sigma_t ({\bm 0} ) \sqrt{1 +u } - \sigma_t ({\bm 0} ) .
$$
Furthermore, by using Taylor's expansion of
 $f(u) = \sqrt{1+u} $ at point $u=0$, we obtain
$$
\sqrt{1+u}  \geq  1 + \frac{1}{2} u - \frac{1}{8} u^2.
$$
Moreover, for each $t$, there exists a number $s $ such that  $0 < s < \rho/2$ and $u \leq 1 $.
Therefore, it follows that
$$
\sqrt{1+u}  \geq  1 + \frac{1}{2} u - \frac{1}{8} u^2 \geq 1 + \frac{1}{2} u - \frac{1}{8} u = 1+ \frac{3}{8} u.
$$
By using this, we have
$$
| \sigma_t ({\bm x} ) -\sigma_t ({\bm 0} ) | \geq \frac{3}{8} u \sigma_t ({\bm 0} )
= \frac{3}{8} \frac{a^4 t }{\sigma_t ({\bm 0}) (\sigma^2+a^2 t)} (1- \exp (-2s/\rho ) ).
$$
In addition, noting that
$
\exp ( -2s /\rho )  \leq  1 -2s /\rho + (2s /\rho)^2 /2
$ and $1- s/\rho \geq 1/2$, we get
$$
1- \exp ( -2s /\rho )  \geq 2s /\rho - (2s /\rho)^2 /2 = 2s/ \rho ( 1- s / \rho ) \geq s/ \rho .
$$
Therefore, the following inequality holds:
$$
| \sigma_t ({\bm x} ) -\sigma_t ({\bm 0} ) | \geq
\frac{3}{8} \frac{a^4 t / \rho }{\sigma_t ({\bm 0}) (\sigma^2+a^2 t)} s =
\frac{3}{8} \frac{a^4  / \rho }{\sigma_t ({\bm 0}) (\sigma^2/t+a^2 )}  \| {\bm x} - {\bm 0} \|_1.
$$
Hence, since $\lim_{t \to \infty} \sigma_t ({\bm 0} ) =0$, the following inequality holds for sufficiently large $t$:
$$
| \sigma_t ({\bm x} ) -\sigma_t ({\bm 0} ) | \geq
\frac{3}{8} \frac{a^4  / \rho }{\sigma_t ({\bm 0}) (\sigma^2/t+a^2 )}  \| {\bm x} - {\bm 0} \|_1
>
C
\| {\bm x} - {\bm 0} \|_1.
$$
\end{proof}

%-------------------------------------------------------------------------------------------
\section{Details of the Experimental Settings and Pseudo-codes}
\label{app:exp_setting}
In this section, we describe the experimental settings.

\subsection{Common Settings}
We used a multi-start L-BFGS-B method~\cite{byrd1995limited} (SciPy~\cite{2020SciPy-NMeth} implementation) to perform various optimization such as optimizing AFs, finding the optimal value of synthetic functions.
First, we sample $1000$ initial points using Latin hypercube sampling (LHS)~\cite{mckay2000comparison}.
Then, we run L-BFGS-B with parameter $\mathtt{ftol}=10^{-3}, \mathtt{gtol}=10^{-3}$ for each initial point and pick the top 5 results.
Finally, we run L-BFGS-B with default parameters for these five results and return the best result.
We implemented GP models and all the comparison methods mainly using PyTorch~\cite{paszke2019pytorch} and GPyTorch~\cite{gardner2018gpytorch}.
By utilizing the automatic differentiation of PyTorch, we can easily apply gradient methods to optimize AFs.

\subsection{Comparison Methods}
\paragraph{CBO}
In CBO, a scalar output is assumed for each stage.
For each iteration $t$, it first chooses the controllable parameter of the final stage $\bx\stg{N}_t$ and desired output of previous stage $y\stg{N-1}_{\mr{desire}}$ by maximizing EI:
\begin{equation}
    (y\stg{N-1}_{\mr{desire}}, \bx\stg{N}_t) =\argmax_{y\stg{N-1}, \bx\stg{N}}  \sigma\stg{N}_{t-1}(y\stg{N-1}, \bx\stg{N}) (Z\Phi(Z) + \phi(Z)),
\end{equation}
where $\Phi, \phi$ are the cumulative distribution function and probability density function of the standard normal distribution, respectively,
$Z=0$ if $\sigma\stg{N}_{t-1}(y\stg{N-1}, \bx\stg{N})=0$ and be  \sloppy$Z= (\mu_{t-1}\stg{N}(y\stg{N-1}, \bx\stg{N}) -F_{\mr{best}}  )/ \sigma\stg{N}_{t-1}(y\stg{N-1}, \bx\stg{N}) $ otherwise.
We could not find any description about the range of optimization parameters in~\cite{dai2016cascade}.
We used $\dX\stg{N}$ for the range of $\bx\stg{N}$, and we used the range twice as wide as the actual range for the range of $y\stg{N-1}$, which is supposed to be unknown.
Then, CBO chooses $(y\stg{N-2}_{\mr{desire}}, \bx_t\stg{N-1})$ of stage $N-1$ as follows:
\begin{equation}
    (y\stg{N-2}_{\mr{desire}}, \bx\stg{N-1}_t)
   =\argmin_{y\stg{N-2}, \bx\stg{N-1}}  \left( \kappa_1 v^{-1} + \kappa_2 v \right)\| m -y_{\mr{desire}}\stg{N-1}  \|_2^2  +\mr{cost}(y\stg{N-2}, \bx\stg{N-1}) ,
    \label{eq:cbo}
\end{equation}
where $m=\mu_{t-1}\stg{N-1}(y\stg{N-2}, \bx\stg{N-1}),\; v=\sigma^{(N-1)\, 2}_{t-1}(y\stg{N-2}, \bx\stg{N-1})$, $\mr{cost}(\cdot)$ is the cost function, and $\kappa_1,\kappa_2$ are hyperparameters.
By repeating this operation, a controllable parameter of stage 1 $\bx_t\stg{1}$ is determined finally.
We used $\mr{cost}(\cdot) =0$ for simplicity and set $\kappa_1=1, \kappa_2=1$.

In the solar cell simulator experiments, output of stage $1$ and stage $2$ are vectors.
To deal with vector output, we replace a predictive mean and variance in~\cref{eq:cbo} with a mean vector and covariance matrix.
Therefore, for the vector output setting, the following AF was used instead of~\cref{eq:cbo}:
\begin{equation}
    (\by\stg{n-1}_{\mr{desire}}, \bx\stg{n}_t) =\argmin_{\by\stg{n-1}, \bx\stg{x}} \Bigl[  \left(\bs{m} -\by_{\mr{desire}}\stg{n} \right)^\top  \left( \kappa_1{\bs{\Sigma}^{-1}} + \kappa_2\bs{\Sigma} \right)   \left(\bs{m} -\by_{\mr{desire}}\stg{n} \right) \Bigr],
\end{equation}
where $\bs{m}=\bs{\mu}_{t-1}\stg{n}(\by\stg{n-1}, \bx\stg{n})$ and $\bs{\Sigma}$ is a diagonal matrix whose $(i,i)$-th element is defined as $\sigma^{(n)\, 2}_{m,t-1}(\by\stg{n-1}, \bx\stg{n})$.

\paragraph{FB-EI, FB-UCB}
In FB-EI and FB-UCB, the next sampling point is determined by using a fully black-box GP model.
To construct this model, we employed an ARD Gaussian kernel and set the noise variance of GP to $10^{-4}$.
The kernel parameters were estimated by maximizing the marginal likelihood.
In particular, FB-UCB used GP-UCB method~\cite{srinivas2010gaussian}, and we set its exploration parameter $\beta_{\mr{GP-UCB}}^{1/2} =2$.

\paragraph{EI-based}
Since EI-based AF is computed through sampling, we have to use stochastic gradient methods to optimize it in a naive implementation.
However, L-BFGS-B can also be applied by utilizing reparameterization-trick~\cite{kingma2014auto}.
At the beginning of the optimization, we draw \textit{base-samples} $\bs{\omega}\stg{n} \in \R^{S}$ from standard multivariate Gaussian distribution for each middle stage.
Then, instead of sampling each $\{y_s\stg{n}\}_{s=1}^{S}$ directly from Gaussian distribution, we sample it as follows:
\begin{equation}
    y\stg{n}_s=\mu_{y\stg{n}_s}  +  \sigma_{y\stg{n}_s}   \omega\stg{n}_s.
\end{equation}
Here, $\mu_{y\stg{n}_s}, \sigma_{y\stg{n}_s}$ are the mean and standard deviation of the Gaussian distribution that follows $y\stg{n}_s$, respectively.
The EI-based AF becomes a deterministic and differentiable function with the above modifications, and the L-BFGS-B method can be applied.
Moreover, EI-based AF can also be applied to the vector output setting.
We only need to change it to sample $\by\stg{n}_s$ instead of $y\stg{n}_s$ in the middle stage.

In EI-SUS-R for the suspension setting experiments, we applied the stock reduction rule (13) except for the stock obtained in the last iteration.

\subsection{Synthetic functions and Solar Cell Simulator}
\paragraph{Sample Paths:}
In the sample path experiments, we used random Fourier features (RFFs) to draw continuous functions from GP priors.
We first sampled 1000 RFFs and built Bayesian linear regression (BLR) model.
From the BLR model, we sampled weight parameters and constructed functions.

\paragraph{Rosenbrock Function:}
For any $d \ge 2$,  $d$-dimensional Rosenbrock function is defined as follows:
\begin{equation}
    f(\bx) =\sum_{i=1}^{d} \left( 100(x_{i+1} -x_i^2)^2 +(x_i -1)^2 \right).
\end{equation}
In our experiments, we used negative Rosenbrock functions, which are multiplied by $-1$.

\paragraph{Sphere Function:}
For any $d \ge 2$,  $d$-dimensional Sphere function is defined as follows:
\begin{equation}
    f(\bx) =\sum_{i=1}^{d} x_i^2.
\end{equation}
In our experiments, we used negative Sphere functions, which are multiplied by $-1$.

\paragraph{Matyas Function:}
Matyas function ($d = 2$) is defined as follows:
\begin{equation}
    f(\bx) = 0.26 ( x_1^2 + x_2^2) - 0.48 x_1 x_2.
\end{equation}
In our experiments, we used negative Matyas functions, which are multiplied by $-1$.

\paragraph{Solar Cell Simulator:}
The simulators for stages one and two are Python implementations of the physical models described in Section~4 of~\cite{bentzen2006phosphorus}.
The simulator of stage three is based on PC1Dmod~6.2~\cite{haug2016pc1dmod}, which is the software for simulating solar cells.
We confirmed that PC1D sometimes caused errors due to convergence failure of the internal calculations.
However, standard BO frameworks cannot handle the situation where the observation fails.
Thus, we used a kernel ridge regression model constructed using the data collected from PC1D as the simulator of stage 3.
To create this simulator, we ran PC1D on each of the 2000 input points sampled using the LHS, and used the 1935 data points among them that could be run without error.

\paragraph{Hydrogen plasma treatment process:}
The real-world datasets for the first and second stages are from \citep{miyagawa2021application-a} and the simulator in \url{https://www.pvlighthouse.com.au/equivalent-circuit}, respectively.
For both stages, we fitted the GPs with a Gaussian kernel, in which hyperparameters are selected by the marginal likelihood maximization.
Then, as with sample paths, we sampled 1000 RFFs, built BLR models, and generate continuous sample paths once.
We used these sample paths as the surrogate objectives.

\subsection{Pseudo-codes of the proposed methods}
We describe the proposed method of~\cref{sec:method} in~\cref{alg:proposed_seq}.
Additionally, we also describe the proposed method of extension setting in~\cref{alg:proposed_sus}.

\begin{algorithm}[t]
    \caption{Cascade process optimization in sequential observation}
    \label{alg:proposed_seq}
    \KwIn{Initial data $\{\mcal{D}_0\stg{n}\}_{n=1}^N$, $\beta$, $\eta_t$}
    \For{$t=0,N,2N,\dots, T$}{
    Fit GP models using $\{\mcal{D}_{t-1}\stg{n}\}_{n=1}^N$\\
    \For{$n=1,\ldots,N$}{
        Select $ \bx\stg{n}_{t+n}$ by maximizing~\cref{eq:EI_lower_mc} or~\cref{eq:af_CI}\\
        Observe output $\by_{t+n}\stg{n}$ corresponding input $(\by\stg{n-1}_{t+n-1}, \bx\stg{n}_{t+n})$\\
        $\mcal{D}_{t+n}\stg{n}\leftarrow  \mcal{D}_{t+n-1}\stg{n}\cup \{((\by\stg{n-1}_{t+n-1}, \bx_{t+n}\stg{n}),\by_{t+n}\stg{n})\}$\\
    }
    $t \leftarrow t+N$\\
    }
    \KwOut{Estimated solution $\hat{\bx}_t\stg{1},\dots,\hat{\bx}_t\stg{N}$}
\end{algorithm}

\begin{algorithm}[t]
    \caption{Cascade process optimization in suspension setting}
    \label{alg:proposed_sus}
    \KwIn{$\{\mcal{D}_0\stg{n}\}_{n=1}^N$, $\beta$, $\eta_t$, stage cost $\{\lambda\stg{n}\}_{n=1}^{N}$, budget $\lambda_{\max}$}
    $t \leftarrow 1$, $\mcal{S}_t\stg{0}\leftarrow \{\bs{0}\}, \{ \mcal{S}_t\stg{n}\leftarrow \emptyset \}_{n=1}^{N-1}$, spend cost $\lambda \leftarrow 0$\\
    \While{$\lambda\le \lambda_{\max}$}{
    Fit GP models using $\{\mcal{D}_{t-1}\stg{n}\}_{n=1}^N$\\
    Select $n_t,\by_t\stg{n_t-1}, \bx_t\stg{n}$ by~\cref{eq:next_point}\\
    Observe $\by_t\stg{n_t}$\\
    $\mcal{D}_t\stg{n_t}\leftarrow  \mcal{D}_{t-1}\stg{n_t}\cup \{((\by_t\stg{n_t-1}, \bx_t\stg{n_t}),\by_t\stg{n_t})\}$\\
    Remove stock $\mcal{S}_t\stg{n_t-1}\leftarrow  \mcal{S}_t\stg{n_t-1}\setminus \by_t\stg{n_t-1}$\\
    \uIf{$n_t\neq N$}{
        Add observed stock $\mcal{S}_t\stg{n_t} \leftarrow  \mcal{S}_t\stg{n_t} \cup \{\by_t\stg{n_t}\}$\\
    }
    $\lambda \leftarrow \lambda+\lambda\stg{n_t},\; t \leftarrow t+1$\\
    }
    \KwOut{Estimated solution $\hat{\bx}_t\stg{1},\dots,\hat{\bx}_t\stg{N}$}
\end{algorithm}
%-------------------------------------------------------------------------------------------
\begin{figure}[t]
  \centering
  \includegraphics[width=0.49\linewidth]{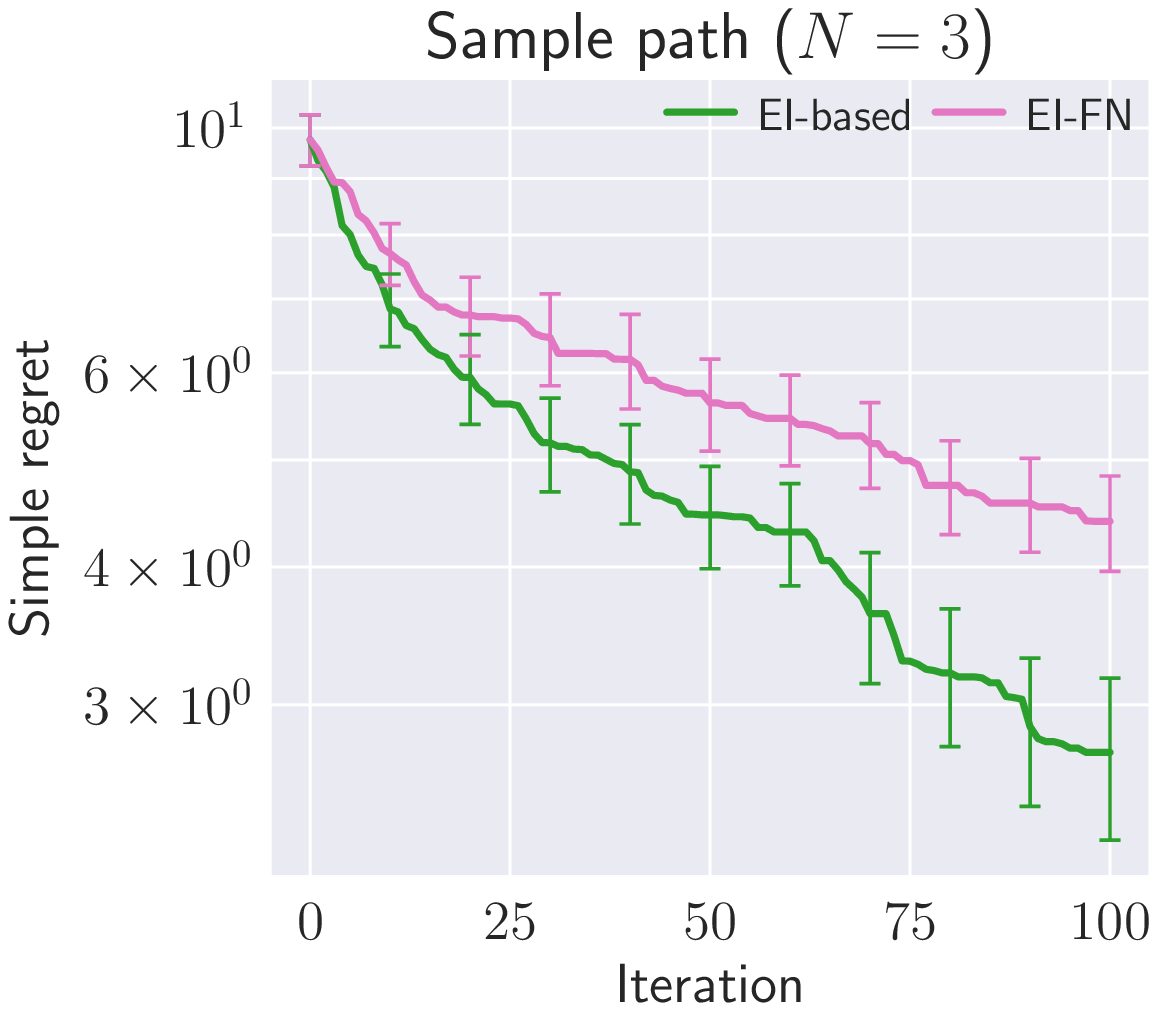}
  \includegraphics[width=0.49\linewidth]{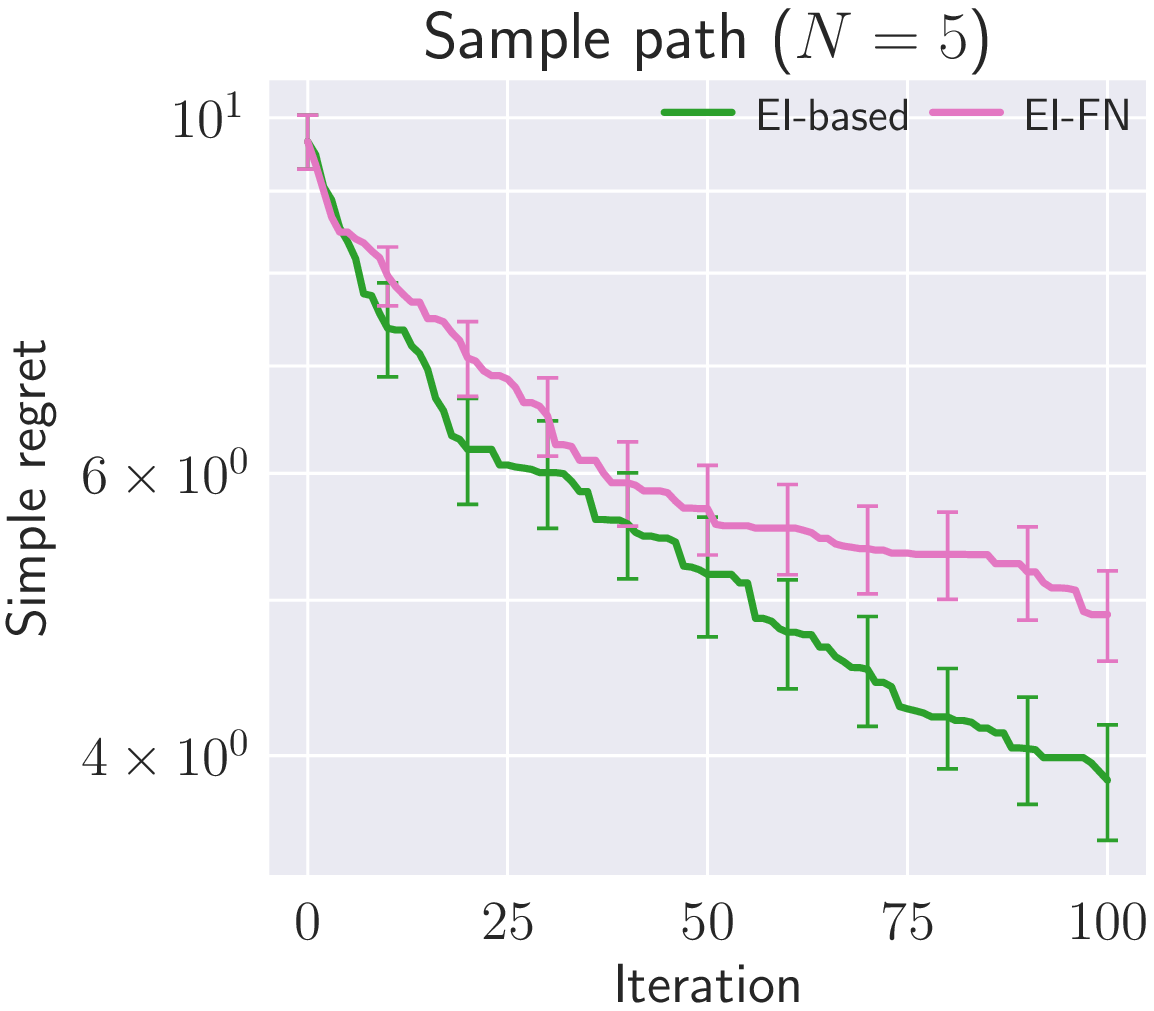}
  % \begin{subfigure}{.33\linewidth}
  %     \centering
  %     \includegraphics[width=\linewidth]{figure/EIFN_solarcell3s_noise_m2_cycle50_result_same_2.pdf}
  %     \subcaption{Solar cell simulator}
  % \end{subfigure}
  \caption{Results of comparison between \textsc{EI-based} and EI-FN}.
  \label{fig:compEIFN}
\end{figure}

\section{Additional Experimental Results}
\label{app:additional_exp}
%
% The proposed method is an adaptive method that makes decisions at each stage depending on the observed values of the previous stages.
% %
% We demonstrated the effectiveness of this adaptive method by comparing EI-based AF with EI-FN, a nonadaptive method.
% %
% EI-FN is proposed in the context of function-network optimization represented by a DAG and is based on EI strategies~\cite{astudillo2021bayesian}.
% %
% In the case of using EI-FN for a cascade process, EI-FN can be viewed as a nonadaptive version of EI that does not change its decision-making in response to previous observations.
% %
% In this experiment, we used a solar cell simulator, a three- and five-stage cascade consisting of GP pre-distributed sample paths.
% %
% We set $\ell_d^{(n)}=1, \ell_w^{(n)}=1$, and the other experimental settings are the same as those described in~\cref{sec:experiments}.
% %
% \Cref{fig:compEIFN} shows the results of 20 runs with different random seeds.
% %
% In this setting, EI-based-AF consistently outperforms EI-FN.
% %
% This result indicates that adaptive decision-making in response to the output of each stage is effective in a cascade process.

We additionally show the comparison between \textsc{EI-based} AF and EI-FN \cite{astudillo2021bayesian}.
In this experiment, we used a three- and five-stage cascade consisting of GP pre-distributed sample paths.
We set $\ell_d^{(n)}=1, \ell_w^{(n)}=1$, and the other experimental settings are the same as those described in~\Cref{sec:experiments}.
\Cref{fig:compEIFN} shows the results of 20 runs with different random seeds.
Since the parameters $\ell_d^{(n)}$ and $\ell_w^{(n)}$ are relatively small, the sample paths can be sensitive to the input uncertainty.
Therefore, in this experiment, \textsc{EI-based} AF that performs adaptive decision-making using intermediate observations clearly outperforms EI-FN, which is nonadaptive.

% We also provide the experimental results when the estimated solution is reported as the LCB maximization point.
% %
% Concretely, estimated solution $\hat{\bx}_t\stg{1},\dots, \hat{\bx}_t\stg{N}$ is given by:
% %
% \begin{equation}
%     \hat{\bx}_t\stg{1},\dots, \hat{\bx}_t\stg{N} = \argmax_{\bx\stg{1},\dots,\bx\stg{N} \in \dX } \mr{LCB}\stg{F}_t (\bx\stg{1},\dots,\bx\stg{N}).
% \end{equation}
% %
% In this setting, performance was evaluated by regret $F({\bx}_*\stg{1},\dots, {\bx}_*\stg{N}) - F(\hat{\bx}_t\stg{1},\dots, \hat{\bx}_t\stg{N})$.
% %
% However, the optimal value of the solar cell simulator is unknown.
% %
% Thus, we use \textit{inference value} $F(\hat{\bx}_t\stg{1},\dots, \hat{\bx}_t\stg{N})$ for the simulator experiments.
% %
% Note that in~\cref{sec:experiments}, we show the results when the estimated solution is defined as the parameters that give the best-observed value.
% %

% We experimented with the same settings as described in~\cref{sec:experiments}, changing only the definition of the estimated solution as described above.
% %
% \Cref{fig:synth_result_lcb,fig:solarcell_result_lcb} show the results of synthetic functions and the solar cell simulator, respectively.
% %
% In addition, we also provide the result of the extension setting in~\cref{fig:extension_result_lcb}.
% %
% In the experiments except for sample paths setting, performance measures became unstable because the model parameters were tuned in each iteration.

%-------------------------------------------------------------------------------------------

\clearpage

% \bibliographystyle{myapalike}
% \bibliography{ref}

\end{document}